%% file: main.tex
\documentclass{article}
\usepackage{graphicx}
\usepackage[top=20truemm,bottom=20truemm,left=20truemm,right=20truemm]{geometry}
\usepackage{authblk}
\usepackage{natbib}
\usepackage{hyperref}

\usepackage{amsmath}
\usepackage{amsfonts}
\usepackage{amsthm}
\usepackage{mathtools}
\usepackage{thm-restate}
\usepackage{algorithm}
\usepackage{algorithmic}
\usepackage{bm}
\usepackage{bbm}
\usepackage{url}
\usepackage{here}

\theoremstyle{plain}

\newtheorem{lemma}{Lemma}[section]

\theoremstyle{definition}
\newtheorem{definition}{Definition}[section]
\newtheorem{assumption}{Assumption}[section]
\theoremstyle{remark}

\title{Randomized Kriging Believer for Parallel Bayesian Optimization\\ with Regret Bounds}
\author[1]{Shuhei Sugiura}
\author[1,2]{Ichiro Takeuchi}
\author[1]{Shion Takeno}
\affil[1]{Department of Engineering, Nagoya University, Aichi, Japan}
\affil[2]{RIKEN AIP, Tokyo, Japan}
\begin{document}
\maketitle
\begin{abstract}
We consider the optimization problem of an expensive-to-evaluate black-box function, in which we can obtain noisy function values in parallel.
For this problem, parallel Bayesian optimization (PBO) is a promising approach, which aims to optimize with fewer function evaluations by selecting a diverse input set for parallel evaluation.
However, existing PBO methods suffer from poor practical performance or lack theoretical guarantees.
In this study, we propose a PBO method, called randomized kriging believer (KB), based on a well-known KB heuristic and inheriting the advantages of the original KB: low computational complexity, a simple implementation, versatility across various BO methods, and applicability to asynchronous parallelization.
Furthermore, we show that our randomized KB achieves Bayesian expected regret guarantees.
We demonstrate the effectiveness of the proposed method through experiments, including those on real-data emulators.
\end{abstract}

\input{manuscripts/1_intro}
\input{manuscripts/2_preliminary}
\input{manuscripts/3_method}
\input{manuscripts/4_analysis}
\input{manuscripts/5_experiment}
\input{manuscripts/6_conclusion}
\section*{Acknowkedgements}
This work was supported by JSPS KAKENHI Grant Number JP24K20847, JST PRESTO Grant Number JPMJPR24J6, JST CREST Grant Numbers JPMJCR21D3, JPMJCR22N2, JST Moonshot R\&D Grant Number JPMJMS2033-05, and RIKEN Center for Advanced Intelligence Project.

\bibliography{reference}
\bibliographystyle{plainnat}

\clearpage
\appendix
\input{manuscripts/9_appendix}
\input{manuscripts/9_appendix_fig}
\end{document}

%% file: manuscripts/1_intro.tex
\section{Introduction}
\label{sec:intro}

Bayesian optimization (BO) \citep{Kushner1964-new,Mockus1978-Application} is a promising approach for optimization of expensive-to-evaluate black-box functions with a smaller number of function evaluations.
For this purpose, BO sequentially queries the input that maximizes the acquisition function (AF) based on a Bayesian model.
BO has been leveraged to a wide range of applications, such as AutoML \citep{Snoek2012-Practical}, robotics \citep{berkenkamp2023-bayesian}, and materials informatics \citep{ueno2016combo}.
However, in many real-world applications, observations can be obtained in parallel.
For example, if the objective function involves computer simulation and multiple computational resources are available, parallel execution of simulations is important to minimize wall-clock time for optimization.
Naively applying vanilla BO methods to such problems can waste query budget, since the input points often cluster and provide redundant information.
%

Parallel BO (PBO) \citep{Shahriari2016-Taking} aims to improve optimization efficiency by selecting diverse input points that effectively leverage parallel evaluations.
Simple heuristics to extend BO methods to PBO by promoting diversity \citep{Ginsbourger2010-kriging,Azimi2010-batch,Gonzalez2016-batch} are widely used for their advantages, such as simple implementation, low computational complexity, and applicability to asynchronous parallelization.
A representative heuristic is kriging believer (KB) \citep{Ginsbourger2010-kriging}, which promotes diversity by sequentially imputing predictive values as fictitious observations at input points currently under evaluation.
%
%
While the effectiveness of these heuristics has been shown empirically, they generally lack theoretical guarantees.

In contrast to heuristic approaches, several PBO methods with theoretical guarantees have been proposed, including parallel Thompson sampling (PTS) \citep{Kandasamy2018-Parallelised,nava2022diversified} and batched upper confidence bound (BUCB) \citep{Desautels2014-Parallelizing}.
Under standard regularity assumptions, these methods have guarantees regarding regret \citep{Srinivas2010-Gaussian}.
However, such methods typically suffer from poor practical performance, lack theoretical guarantees for practical tuning of the confidence width parameter, or involve complex implementations.
This gap between strong theoretical guarantees and practical effectiveness motivates the development of PBO methods that are both theoretically principled and empirically competitive.

\begin{figure}[t]
\centering
\includegraphics[width=\linewidth]{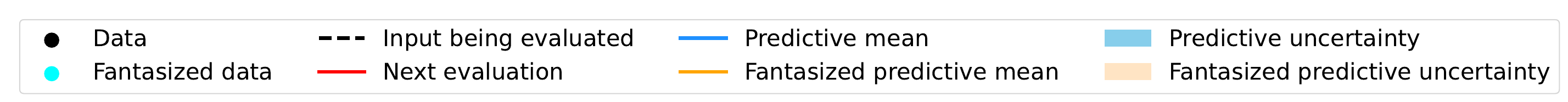}
\includegraphics[height=.2\linewidth]{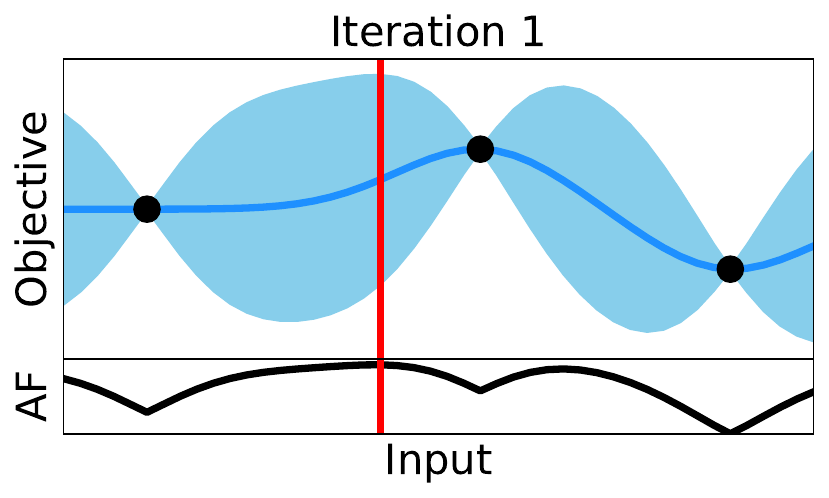}
\includegraphics[height=.2\linewidth]{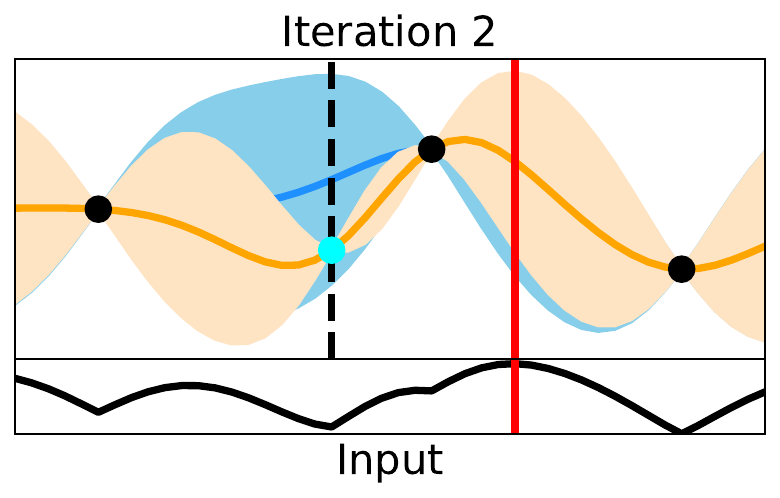}
\includegraphics[height=.2\linewidth]{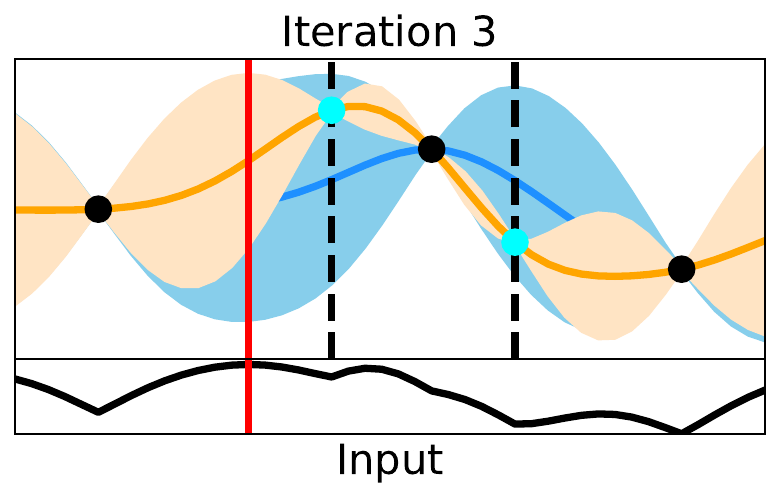}
\caption{
Schematic illustration of the proposed method for three consecutive iterations.
%
%
For efficient parallel optimization, it is necessary to evaluate diverse input points while avoiding redundant evaluation.
%
%
The left figure shows iteration 1, where no inputs are under evaluation, and thus, the next input is chosen as in standard BO.
%
%
The middle and right figures show iterations 2 and 3, respectively.
There are inputs currently being evaluated, indicated by dashed lines. 
In these cases, RKB generates fantasized data from the predictive distribution at the inputs under evaluation.
Then, the next input is selected by maximizing an AF based on a model trained on both observed (real) and fantasized data.
As a result, BO avoids redundant evaluation and ensures diversity among the evaluated inputs.
}
\label{fig:proposed}
\end{figure}

This paper proposes and analyzes a randomized variant of the KB heuristic, referred to as randomized KB (RKB), shown in Fig.~\ref{fig:proposed}.
Our contributions are summarized as follows:
\begin{enumerate}
    \item We propose a PBO method, RKB, that selects a diverse input set by conditioning on one random posterior realization for the ongoing evaluation points, inheriting the practical advantages of the original KB: ease of implementation, low computational complexity, applicability to asynchronous parallelization, and versatility across a wide range of BO methods.
    \item We establish upper bounds on both Bayesian cumulative regret (BCR) and Bayesian simple regret (BSR) for RKB when combined with several BO methods that admit regret guarantees \citep{Srinivas2010-Gaussian,Takeno2023-randomized,takeno2024-posterior,takeno2025-regretEI,takeno2025-regret} in Theorems \ref{thm1}, \ref{thm2}, and \ref{thm3}, which especially provide the BSR upper bound independent of the number of parallel workers.
    \item We demonstrate the effectiveness of the proposed method via a wide range of experiments, including emulators of real-world data.
\end{enumerate}

\subsection{Related work}
\label{sec:related}



\paragraph{Joint selection methods.}
Several studies have proposed PBO methods that select a set of inputs either by jointly optimizing a utility-based AF that explicitly quantifies a batch's utility, or by sampling from a joint distribution designed to encourage diversity.
Representative examples include PBO methods based on expected improvement (EI) \citep{Ginsbourger2010-kriging,Chevalier2013-fast,Marmin2015-differentiating,wang2020-parallel}, predictive entropy search (PES) \citep{Amar2015-Parallel}, and knowledge gradient (KG) \citep{wu2016-parallel}, all of which jointly optimize AFs over a batch of $O(Q)$ inputs, where $Q$ is the number of parallel workers.
%
%
Alternatively, determinantal point process (DPP)-based approaches, such as DPP sampling \citep{kathuria2016-batched} and DPP-TS \citep{nava2022diversified}, generate input sets by sampling from a DPP distribution that explicitly incorporates diversity.
These DPP-based methods admit regret guarantees \citep{kathuria2016-batched,nava2022diversified}.
However, both approaches incur substantial computational costs that grow rapidly with $Q$, due to high-dimensional AF optimization or sampling from high-dimensional distributions using Markov chain Monte Carlo (MCMC) methods.

\paragraph{Greedy selection methods.}
Greedy PBO methods, including our RKB, have also been actively studied.
Importantly, these greedy approaches have low computational complexity in $Q$, comparable to sequential BO, and naturally support asynchronous parallelization.
BUCB \citep{Desautels2014-Parallelizing}, UCB with pure exploration (UCB-PE) \citep{Contal2013-Parallel}, and DPP-based greedy approximations \citep{kathuria2016-batched} have regret guarantees but often suffer from over-exploration.
In contrast, KB \citep{Ginsbourger2010-kriging}, constant liar \citep{Ginsbourger2010-kriging}, simulation matching \citep{Azimi2010-batch}, and local penalization (LP) \citep{Gonzalez2016-batch} often show superior performance. 
Furthermore, Monte Carlo (MC)-estimation-based PBO for utility-based AFs, such as EI \citep{Snoek2012-Practical,Wilson2018-Maximizing,balandat2020-botorch,wang2020-parallel} and max-value entropy search (MES) \citep{Takeno2020-Multifidelity,takeno2022-sequential,Takeno2022-generalized} has also been widely used.
The MC estimation-based methods greedily select inputs by maximizing the AF conditioned on ongoing evaluations and averaging over multiple posterior samples.
Thus, our RKB can be interpreted as a special case that uses a single MC sample, although our regret analysis applies only to this single-sample setting.
The empirical effectiveness of these heuristics has been demonstrated in multiple studies above, but lacks theoretical guarantees.

\paragraph{Distributed selection methods.}
A third line of work is fully distributed PBO, especially for the case that even linear computational dependence on $Q$ is prohibitive.
Distributed PBO methods, including PTS \citep{Kandasamy2018-Parallelised,hernandez-lobato2017-parallel,vakili2021-scalable,nava2022diversified} and Boltzmann policies derived from vanilla BOs  \citep{Garcia-Barcos2019-fully}, select input sets in a fully distributed manner, such as independent posterior sampling, by which they avoid computational dependence on $Q$ in terms of wall-clock time.
However, by design, these methods do not incorporate other selected inputs, which can result in batches with insufficient diversity.
%

%

\paragraph{Regret analysis for PBO.}
Regarding PBO, several works have established regret guarantees.
High-probability regret bounds have been derived for parallel UCB variants \citep{Desautels2014-Parallelizing,Contal2013-Parallel,kathuria2016-batched}.
\citet{Desautels2014-Parallelizing} showed that the deterioration of regret with respect to the batch size $Q$ can be avoided by performing uncertainty sampling $O(Q^c)$ times at an initial phase, where $c \geq 1$ is a constant.
%
%
Regret bounds for PTS and DPP-TS have been established \citep{Kandasamy2018-Parallelised,chowdhury2019-on,nava2022diversified}.
Notably, \citet{nava2022diversified} derived a BSR bound that is independent of $Q$, for which we obtain the comparable upper bound for RKB.
However, the above algorithms with regret guarantees suffer from practical disadvantages discussed already.
Finally, the analysis of regret lower bounds in our problem setup is limited to vanilla BO for one-dimensional objectives \citep{scarlett2018-tight}.



%% file: manuscripts/2_preliminary.tex
\section{Preliminaries}
\label{sec:preliminary}

This section provides background knowledge.

\subsection{Problem statement}
We consider the optimization of a black-box and expensive-to-evaluate function $f: \mathcal{X} \rightarrow \mathbb{R}$:
\begin{equation}
    \bm x^* = \mathop{\rm argmax}\limits_{\bm x \in \mathcal X} f(\bm x).
\end{equation}
For this problem, we sequentially query $\bm x_t$ for all iterations $t \in \mathbb N$ and obtain an observation $y_t$, which can be contaminated, aiming for sample-efficient optimization.
We consider the general case that up to $Q-1$ queries can be unobserved.
This setting recovers vanilla BO with $Q = 1$ and includes synchronous and asynchronous parallelization with $Q$ workers.

%

\subsection{Gaussian process regression}
We assume the following regularity assumption for the GP regression model \citep{Rasmussen2005-Gaussian}:
\begin{assumption}\label{asm_basic}
    \it For $d\in\mathbb N$ and $r>0$, let $k$ be a kernel on $\mathcal X\subset\left[0,r\right]^d$ that satisfies $\max_{\bm x\in\mathcal X}k{\left(\bm x,\bm x\right)}\le1$.
    Then, function $f$ follows $\mathcal{GP}{\left(0,k\right)}$, and its input-output data is defined as
    \begin{equation}
    \begin{split}
    \left(\bm x_t,y_t\right)\in\mathcal X\times\mathbb R,\quad y_t=f{\left(\bm x_t\right)}+\varepsilon_t\qquad\left(t\in\mathbb N\right),
    \end{split}
    \end{equation}
    where $\varepsilon_t\sim\mathcal N{\left(0,\sigma_{\rm noise}^2\right)}$ with $\sigma_{\rm noise}>0$.
\end{assumption}
For any index set $\mathcal S\subset\mathbb N$, we define data set $\mathcal D_{\mathcal S}$ as $\left\{\left(\bm x_i, y_i\right)\right\}_{i\in\mathcal S}$.
For any $t\in\mathbb N$, we denote the set $\left\{1,\ldots,t\right\}$ as $\left[t\right]$ and the data set $\mathcal D_{\left[t-1\right]}$ as $\mathcal D_{t-1}$. 
Let $\mathcal N_{t-1}\subset\left[t-1\right]$ be the index set of data available at the beginning of the $t$-th iteration. 
Then, the set $\mathcal N_{t-1}$ satisfies $\left|\mathcal N_{t-1}\right|\ge t-Q$ and $\mathcal N_t\supset \mathcal N_{t-1}$.
Under Assumption~\ref{asm_basic}, the posterior distribution conditioned on $\mathcal D_{\mathcal N_t}$ is also a GP.
The posterior covariance between $f{\left(\bm x\right)}$ and $f{\left(\bm x'\right)}$ is written as
\begin{equation}
k{\left(\bm x,\bm x';\mathcal D_{\mathcal N_t}\right)}=k{\left(\bm x,\bm x'\right)}-\bm k_{\mathcal N_t}{\left(\bm x\right)}^\top\!\!\left(\bm K_{\mathcal N_t}+\sigma_{\rm noise}^2\bm I_{n_t}\right)^{-1}\!\!\bm k_{\mathcal N_t}{\left(\bm x'\right)}\qquad\left(\bm x,\bm x'\in\mathcal X\right),
\end{equation}
where $I_{n_t}\in\mathbb R^{n_t \times n_t}$ with $n_t=\left|\mathcal N_t\right|$ denotes the identity matrix, $\bm K_{\mathcal N_t}\in\mathbb R^{n_t \times n_t}$ is a kernel matrix whose $\left(i,j\right)$-th element is $k{\left(\bm x_{[\mathcal N_t]_i},\bm x_{[\mathcal N_t]_j}\right)}$, where $[\mathcal N_t]_i$ is the $i$-th smallest element of $\mathcal N_t$.
Let $k{\left(\bm x,\bm x';\emptyset\right)}=k{\left(\bm x,\bm x'\right)}$.
Then, the posterior mean and variance of $f{\left(\bm x\right)}$ are given by
\begin{equation}
\mu{\left(\bm x;{\mathcal D}_{\mathcal N_t}\right)}=\bm k_{\mathcal N_t}{\left(\bm x\right)}^\top\!\!\left(\bm K_{\mathcal N_t}+\sigma_{\rm noise}^2\bm I_{n_t}\right)^{-1}\!\!\bm y_{\mathcal N_t},\quad\sigma^2{\left(\bm x;{\mathcal D}_{\mathcal N_t}\right)}=k{\left(\bm x,\bm x;\mathcal D_{\mathcal N_t}\right)}\quad\left(\bm x\in\mathcal X\right),
\end{equation}
where $\bm y_{\mathcal N_t}=\left[y_i\right]_{i \in \mathcal N_t}^\top\in\mathbb R^{n_t}$, and $\bm k_{\mathcal N_t}{\left(\bm x\right)}=\left[k{\left(\bm x,\bm x_i\right)}\right]_{i \in \mathcal N_t}^\top\in\mathbb R^{n_t}$.

For continuous input domains, we assume the following regularity condition as in \citep{Srinivas2010-Gaussian,Kandasamy2018-Parallelised,nava2022diversified}, which is satisfied by Gaussian kernels and Mat\'ern-$\nu$ kernels with $\nu > 2$ \citep{Srinivas2010-Gaussian}:
\begin{assumption}\label{asm_derivative}
\it Assume that $\mathcal X$ is compact and convex, and $f\sim\mathcal{GP}{\left(0,k\right)}$ satisfies
\begin{equation}
\exists a,b>0,\forall L\ge0,\forall j\in\left[d\right],{\rm Pr}\left(\sup_{\bm x\in\mathcal X}\left|\frac{\partial f{\left(\bm x\right)}}{\partial x_j}\right|>L\right)\le a\exp{\left(-\frac{L^2}{b^2}\right)},
\end{equation}
where $\left[x_1\;\cdots\;x_d\right]^\top=\bm x$.
\end{assumption}

\subsection{Acquisition functions for BO}
BO algorithm $\mathcal A$ selects an input to evaluate by maximizing the AF $\alpha:\mathcal X\to\mathbb R$.
That is, $\mathcal A{\left(\mathcal D_{t-1}\right)}={\rm argmax}_{\bm x\in\mathcal X}\alpha{\left(\bm x;\mathcal D_{t-1}\right)}$.
Here, we describe three AFs used in our numerical experiment.

The first is UCB \cite{Srinivas2010-Gaussian} defined as
\begin{equation}
\alpha_{\rm UCB}{\left(\bm x;\mathcal D_{t-1}\right)}=\mu{\left(\bm x;\mathcal D_{t-1}\right)}+\beta_t^{\frac{1}{2}}\sigma{\left(\bm x;\mathcal D_{t-1}\right)}\quad\left(\bm x\in\mathcal X\right),
\end{equation}
where $\beta_t > 0$ is the confidence width parameter.
The second is EI \cite{Mockus1978-Application} defined as
\begin{equation}\label{EI}
\alpha_{\rm EI}{\left(\bm x;\mathcal D_{t-1}\right)}
= \sigma{\left(\bm x;{\mathcal D}_{t-1}\right)}
\left(s \left(\bm x;{\mathcal D}_{t-1}\right)\Phi{\left(s \left(\bm x;{\mathcal D}_{t-1}\right)\right)}+\phi{\left(s \left(\bm x;{\mathcal D}_{t-1}\right)\right)}\right)\quad\left(\bm x\in\mathcal X\right),
\end{equation}
where $\Phi$ and $\phi$ denote the cumulative distribution function and probability density function of the standard normal distribution, respectively, and $s \left(\bm x;{\mathcal D}_{t-1}\right)$ is defined as
\begin{equation}
s\left(\bm x;{\mathcal D}_{t-1}\right)
=\frac{\mu{\left(\bm x;{\mathcal D}_{t-1}\right)}-\max_{\bm x'\in\mathcal X}\mu{\left(\bm x';{\mathcal D}_{t-1}\right)}}{\sigma{\left(\bm x;{\mathcal D}_{t-1}\right)}}.
\end{equation}
The third is PI from the maximum of a sample path (PIMS) \cite{takeno2024-posterior} defined as
\begin{equation}
\alpha_{\rm PIMS}{\left(\bm x;\mathcal D_{t-1}\right)}=1-\Phi{\left(\frac{g_t^*-\mu{\left(\bm x;{\mathcal D}_{t-1}\right)}}{\sigma{\left(\bm x;{\mathcal D}_{t-1}\right)}}\right)}\quad\left(\bm x\in\mathcal X\right),
\end{equation}
where $g_t^* = \max_{\bm x \in \mathcal{X}} g_t(\bm x)$ and $g_t\sim p{{\left(f\mid\mathcal D_{t-1}\right)}}$. Hence, PIMS is random even if the data is fixed.
\subsection{Kriging believer for PBO}

Here, we describe KB \cite{Ginsbourger2010-kriging}, which serves as the basis of our approach.
KB selects $\bm x_t$, incorporating the diversity of selected inputs, as
\begin{equation}\label{eq:KB}
\bm x_t=\mathcal A{\left(\mathcal D^{\rm KB}_{t-1}\right)},\quad\mathcal D^{\rm KB}_{t-1}=\left\{\left(\bm x_i,y^{\left(t\right)}_i\right)\right\}_{i=1}^{t-1},\quad y^{\left(t\right)}_i=
\begin{cases}
y_i&\left(i\in \mathcal N_{t-1}\right)\\
\mu{\left(\bm x_i;{\mathcal D}_{\mathcal N_{t-1}}\right)}&\left(i\in \left[t-1\right]\backslash \mathcal N_{t-1}\right)
\end{cases},
\end{equation}
where $\mathcal A$ is an arbitrary vanilla BO algorithm. 
As shown in Eq. \eqref{eq:KB}, KB uses the posterior mean of $f$ in place of real data that have not yet been obtained. 
Such conditioning on fantasized data prevents redundant evaluation.
\subsection{Bayesian regret and maximum information gain}
As a criterion of the performance of BO methods, we employ Bayesian regret \cite{Russo2014-learning,Kandasamy2018-Parallelised}.
The BSR and BCR are defined as
\begin{equation}\label{bayesian_regret}
{\rm BCR}_T=\mathbb E{\left[\sum_{t=1}^Tf^*-f{\left(\bm x_t\right)}\right]},\quad{\rm BSR}_T=\mathbb E{\left[f^*-f{\left(\hat{\bm x}_T\right)}\right]}\qquad\left(T\in\mathbb N\right),
\end{equation}
where $f^* = f(\bm x^*)$ and $\hat{\bm x}_T={\rm argmax}_{\bm x\in\mathcal X}\mu{\left(\bm x;\mathcal D_T\right)}$.
In Eq. \eqref{bayesian_regret}, the expectation is taken with respect to all randomness, including $f$, $\varepsilon_t$, and the algorithm. The bounds of BSR and BCR are represented using maximum information gain (MIG) \cite{Srinivas2010-Gaussian}, defined as follows:
\begin{definition}\label{MIG}
\it Let Assumption \ref{asm_basic} hold. Then, for any $T\in\mathbb N$, MIG $\gamma_T$ is defined as
\begin{equation}
\gamma_T=\max_{\bm x_1,\ldots,\bm x_T\in\mathcal X} I{\left(\bm y_T;\bm f_T\right)},
\end{equation}
where $I$ denotes Shannon mutual information, $\bm y_T=\left[y_1\;\cdots\;y_T\right]^\top$, and $\bm f_T=\left[f{\left(\bm x_1\right)}\;\cdots\;f{\left(\bm x_T\right)}\right]^\top$.
\end{definition}
MIG $\gamma_T$ has a sublinear and concave upper bound $\bar\gamma_T$ for commonly used kernels \cite{Srinivas2010-Gaussian,vakili2021-information,iwazaki2025improved}: 
$\bar\gamma_T=O{\left(d\log T\right)}$ for linear kernels $k_{\rm Lin}{\left(\bm x,\bm x'\right)}=\bm x^\top\bm x'$; 
$\bar\gamma_T=O{\left({\left(\log T\right)}^{d+1}\right)}$ for Gaussian kernels $k_{\rm Gauss}{\left(\bm x,\bm x'\right)}=\exp{\left(-\frac{1}{2l^2}\left\|\bm x-\bm x'\right\|^2\right)}$; 
and $\bar\gamma_T=O{\left(T^{\frac{d}{2\nu+d}}{\left(\log T\right)}^{\frac{4\nu + d}{2\nu+d}}\right)}$ for Mat\'ern-$\nu$ kernels $k{\left(\bm x,\bm x'\right)}=\frac{2^{1-\nu}}{\Gamma{\left(\nu\right)}}\left(\frac{\sqrt{2\nu}}{l}\left\|\bm x-\bm x'\right\|\right)^\nu J_\nu{\left(\frac{\sqrt{2\nu}}{l}\left\|\bm x-\bm x'\right\|\right)}$,
where $l>0$ and $\nu>0$ are lengthscale and smoothness parameters, respectively, and $\Gamma$ and $J_\nu$ are Gamma and modified Bessel functions of the second kind, respectively.

%% file: manuscripts/3_method.tex
\section{Randomized kriging believer}
\label{sec:proposed}


Algorithm~\ref{alg:RKB} presents the pseudocode of RKB. 
The key difference from the original KB is that the fantasized data $\mathcal D^{\rm RKB}_{t-1}$ is defined as
\begin{equation}\label{data_RKB}
\mathcal D^{\rm RKB}_{t-1}=\left\{\left(\bm x_i,y^{\left(t\right)}_i\right)\right\}_{i=1}^{t-1},\quad y^{\left(t\right)}_i=
\begin{cases}
y_i&\left(i\in \mathcal N_{t-1}\right)\\
g_t{\left(\bm x_i\right)}+\varepsilon^{\left(t\right)}_i&\left(i\in\bar{\mathcal N}_{t-1}\right)
\end{cases},
\end{equation}
where $\bar{\mathcal N}_{t-1}=\left[t-1\right]\backslash \mathcal N_{t-1}$, $g_t\sim p{\left(f\mid{\mathcal D}_{\mathcal N_{t-1}}\right)}$ and $\varepsilon^{\left(t\right)}_i\sim\mathcal N{\left(0,\sigma_{\rm noise}^2\right)}$.
Thus, data set $\mathcal D^{\rm RKB}_{t-1}$ contains values of a posterior sample path $g_t$ of $f$ in place of real data that have not yet been obtained. 
Note that we do not need to generate the whole sample path $g_t$ to generate $g_t{\left(\bm x_i\right)}$ because $g_t{\left(\bm x_i\right)}$ follows the normal distribution as
\begin{equation}\label{sampling_g}
\left[g_t{\left(\bm x_i\right)}\right]_{i\in\bar{\mathcal N}_{t-1}}\sim\mathcal N{\left(\bm\mu_g,\bm K_g\right)},
\end{equation}
where $\bm\mu_g=\left[\mu{\left(\bm x_i;\mathcal D_{\mathcal N_{t-1}}\right)}\right]_{i\in\bar{\mathcal N}_{t-1}}$ and $\bm K_g\in\mathbb R^{\left|\bar{\mathcal N}_{t-1}\right|\times\left|\bar{\mathcal N}_{t-1}\right|}$ is a covariance matrix whose $\left(i,j\right)$-th element is $k{\left(\bm x_{[\bar{\mathcal N}_{t-1}]_i},\bm x_{[\bar{\mathcal N}_{t-1}]_j};\mathcal D_{\mathcal N_{t-1}}\right)}$, where $[\bar{\mathcal N}_{t-1}]_i$ is the $i$-th smallest element of $\bar{\mathcal N}_{t-1}$.

As a result of using a posterior realization as fantasized data, we obtain the following theoretical and practical benefits.
First, from the theoretical perspective, $\mathcal D^{\rm RKB}_{t-1}$ has the same distribution as $\mathcal D_{t-1}$ conditioned on $\mathcal D_{\mathcal N_{t-1}}$ by construction. 
This property is essential for the regret analysis in the next section.
Second, KB can be overconfident in the sense that the algorithm believes the point estimate, the posterior mean.
In contrast, RKB incorporates posterior uncertainty in a randomized manner, as in TS, thereby controlling the exploration-exploitation trade-off based on both posterior mean and uncertainty in practice.

Our RKB algorithm is similar to the hallucination believer (HB) \citep{takeno2023-towards} for preferential BO \citep{brochu2010-tutorial}.
HB leverages a posterior sample defined over the training inputs, which remains random because only preference data are observed, by which HB avoids a time-consuming MCMC method for the GP preference model.
The algorithms of HB and RKB are closely related, as both rely on posterior sampling and conditioning on hallucinated observations.
However, the motivation for using posterior samples differs fundamentally between the two approaches.

\begin{algorithm}[t]
\caption{Randomized Kriging Believer}\label{alg:RKB}
\begin{algorithmic}[1]
\STATE {\bfseries Input:} input space $\mathcal X$, objective function $f:\mathcal X\to\mathbb R$, kernel $k:\mathcal X\times\mathcal X\to\mathbb R$, noise variance $\sigma_{\rm noise}^2>0$, sequential BO algorithm $\mathcal A$
\STATE ${\mathcal D}_{\mathcal N_0}\gets\emptyset$
\FOR{$t = 1, \dots$}
\STATE Set $\mathcal D^{\rm RKB}_{t-1}$ by Eqs. \eqref{data_RKB} and \eqref{sampling_g}
\STATE $\bm x_t\gets\mathcal A{\left(\mathcal D^{\rm RKB}_{t-1}\right)}$
\STATE Assign a worker to evaluate $\bm x_t$
\STATE Wait for a free worker and set $\mathcal N_t$ and ${\mathcal D}_{\mathcal N_t}$ as the obtained indices and the dataset
\ENDFOR
\end{algorithmic}
\end{algorithm}

%% file: manuscripts/4_analysis.tex
\section{Regret analysis}
\label{sec:analysis}

This section provides the regret analysis of RKB and the theoretical conditions for base BO algorithms.
We study regret in the Bayesian setting \citep{Srinivas2010-Gaussian,Russo2014-learning,Kandasamy2018-Parallelised,Desautels2014-Parallelizing,iwazaki2025improved} defined as in Section~\ref{sec:preliminary}, 
although the frequentist setting \citep{Srinivas2010-Gaussian,Chowdhury2017-on,iwazaki2025-improvedGPbandit,iwazaki2025gaussian} has also been extensively investigated.

\subsection{Theoretical conditions for base BO algorithms}

Our RKB can parallelize any BO algorithm $\mathcal A$. 
To obtain regret bounds, we need the following condition of $\mathcal A$:
\begin{restatable}{condition}{cndalg}\label{cnd_alg}
Let $t\in\mathbb N$ be an arbitrary number. 
If either (i) Assumption~\ref{asm_basic} holds and $|\mathcal{X}| < \infty$, or (ii) Assumptions~\ref{asm_basic} and \ref{asm_derivative} hold, then the following propositions hold.
There exist (possibly random) variables $u_t\left(\mathcal D_{t-1}\right)$ and $v_t\left(\mathcal D_{t-1}\right)\in\mathbb R$ that satisfy the following for all $\mathcal D_{t-1}\in2^{\left(\mathcal X\times\mathbb R\right)}$:
\begin{equation}\label{eq_cnd_alg1}
\mathbb E{\left[f^*-f{\left(\bm x_t\right)}\mid\mathcal D_{t-1}\right]}\le\mathbb E{\left[u_t{\left(\mathcal D_{t-1}\right)}\sigma{\left(\bm x_t;\mathcal D_{t-1}\right)}+v_t{\left(\mathcal D_{t-1}\right)}\mid\mathcal D_{t-1}\right]},
\end{equation}
where $\bm x_t=\mathcal A{\left(\mathcal D_{t-1}\right)}$.
Moreover, there exist $\zeta_t$ and $\xi_t\in\mathbb R$ that satisfy the following for all $\mathcal D_{t-1}\in2^{\left(\mathcal X\times\mathbb R\right)}$:
\begin{equation}\label{eq_cnd_alg2}
\mathbb E{\left[u_t^2{\left(\mathcal D_{t-1}\right)}\right]}\le\zeta_t,\quad\mathbb E{\left[v_t{\left(\mathcal D_{t-1}\right)}\right]}\le\xi_t,\quad \sum_{t=1}^T\zeta_t = \tilde{O}(T),\quad\ \sum_{t=1}^T\xi_t = O(1),
\end{equation}
where $\tilde{O}$ hides polylogarithmic factors.
\end{restatable}
Condition~\ref{cnd_alg} is satisfied at least by UCB \citep{Srinivas2010-Gaussian}, randomized UCB \citep{Takeno2023-randomized,takeno2025-regret}, PIMS \citep{takeno2024-posterior}, and EI from the maximum of a sample path (EIMS) \citep{takeno2025-regretEI} for both finite and continuous input domains.
Note that, although Condition~\ref{cnd_alg} holds for TS \citep{Russo2014-learning,takeno2024-posterior}, RKB combined with TS results in PTS.
For more details, see Appendix \ref{apn:cnd_alg}.
\subsection{Regret bounds}

First, we present the BCR bounds for finite and continuous input domains:
\begin{restatable}[BCR bound for finite input domains]{theorem}{thmOne}\label{thm1}
\it Suppose that Assumption~\ref{asm_basic} and Condition~\ref{cnd_alg} hold and that $\gamma_T$ has an upper bound $\bar\gamma_T$ concave for $T$.
Let $\left|\mathcal X\right|<\infty$ and $\bm x_t=\mathcal A{\left(\mathcal D^{\rm RKB}_{t-1}\right)}$. Then, the following holds:
\begin{equation}
{\rm BCR}_T\le B_T+\sqrt{C_1C_2QT\bar\gamma_{T/Q}}\qquad\left(T\in\mathbb N\right),
\end{equation}
where $C_1=2/\log{\left(1+\sigma_{\rm noise}^{-2}\right)}$, $C_2=2+2\log{\left(\left|\mathcal X\right|/2\right)}$, $Q=\max_{t\in\mathbb N}t-\left|\mathcal N_{t-1}\right|$, and $B_T$ is defined as
\begin{equation}\label{B_T}
B_T=\sqrt{C_1\gamma_T{\sum}_{t=1}^T\zeta_t}+\sum_{t=1}^T\xi_t = \tilde{O}(\sqrt{T \gamma_T})\quad\left(T\in\mathbb N\right).
\end{equation}
\end{restatable}
\begin{restatable}[BCR bound for continuous input domains]{theorem}{thmTwo}\label{thm2}
\it Suppose that Assumptions~\ref{asm_basic} and \ref{asm_derivative} and Condition~\ref{cnd_alg} hold.
Assume that the kernel is a linear kernel, a Gaussian kernel, or a Mat\'ern-$\nu$ kernel with $\nu>1$.
Let $\bar\gamma_T$ be a concave upper bound of $\gamma_T$.
Define $L_\sigma = O(1)$ as in Lemma~\ref{kusakawa22_e_4}, 
$L=\max{\left\{L_\sigma,b\left(\sqrt{\log{\left(ad\right)}}+\sqrt{\pi}/2\right)\right\}}$,
and $s_t=2+2d\log{\lceil drLt^2\rceil}-2\log2$.
Let $\bm x_t=\mathcal A{\left(\mathcal D^{\rm RKB}_{t-1}\right)}$. Then, the following holds:
\begin{equation}
{\rm BCR}_T\le B_T+\frac{\pi^2}{3}+\frac{\pi^2}{6}\sqrt{s_T}+\sqrt{C_1Qs_TT\bar\gamma_{T/Q}}\quad\left(T\in\mathbb N\right),
\end{equation}
where $C_1$ and $Q$ are the same as in Theorem \ref{thm1} and $B_T$ is the same as that of Theorem~\ref{thm1}.
\end{restatable}
See Appendices~\ref{apn:thm1} and \ref{apn:thm2} for the proof of Theorems~\ref{thm1} and \ref{thm2}, respectively.

The upper bounds in Theorems~\ref{thm1} and \ref{thm2} $\tilde{O}(\sqrt{QT \bar{\gamma}_{T/Q}})$ are tighter with respect to the dependence on $Q$ than the known results on BCR $\tilde{O}(\sqrt{QT \bar{\gamma}_{T}})$ \citep{Desautels2014-Parallelizing,Kandasamy2018-Parallelised,nava2022diversified} by leveraging the proof technique modified from \citep{vakili2021-scalable}, shown in Lemma~\ref{variance_bound2}.\footnote{We conjecture that the similar upper bound immediately follows from Lemma~\ref{variance_bound2} for TS-based methods \citep{Kandasamy2018-Parallelised,nava2022diversified}}
Let us consider the batch setting in which $T = Q B$ for the number of batches $B \in \mathbb{N}$ to interpret the upper bound.
Then, $\sqrt{QT \bar{\gamma}_{T/Q}} = Q\sqrt{ B \bar{\gamma}_{B}}$.
Hence, our BCR bounds are $o(B)$ and $O(Q)$ if $\bar{\gamma}_T = o(T / \log T)$ in contrast to that the existing bound $\tilde{O}(\sqrt{QT \bar{\gamma}_{T}})$ implies $\omega(Q)$.
%
%
%

The key step in the proof is decomposing the regret into the regret incurred in the fantasized sample path $g_t$ and the error between the fantasized sample path $g_t$ and the objective function $f$ as follows:
\begin{equation}
    \begin{split}
        {\rm BCR}_T
        =&\underbrace{\mathbb E{\left[\sum_{t=1}^Tg_t^*-g_t{\left(\bm x_t\right)}\right]}}_{A_1}+\underbrace{\mathbb E{\left[\sum_{t=1}^Tg_t{\left(\bm x_t\right)}-f{\left(\bm x_t\right)}\right]}}_{A_2},
    \end{split}
\end{equation}
where we leverage the fact that $\mathcal D_{t-1}$ and $\mathcal D^{\rm RKB}_{t-1}$ are identically distributed.
By this decomposition, we can obtain the upper bound $A_1 \leq B_T$, which can be seen as regret incurred by sequential optimization.
Actually, vanilla BO methods satisfying Condition~\ref{cnd_alg} achieve the regret bound $B_T$ as shown in Lemma~\ref{general_RA}.
Furthermore, $A_2$ can be bounded from above by $\tilde{O}\bigl(\sqrt{QT\gamma_{T/Q}}\bigr)$, which contains the penalized term $Q$ incurred by imputing the posterior samples.
For the proof of Theorem~\ref{thm2}, we further consider the discretization error that is bounded above by $\frac{\pi^2}{3}+\frac{\pi^2}{6}s_T$.

Next, we show the BSR bound in a consistent way for finite and continuous input domains:
\begin{restatable}[BSR bound]{theorem}{thmThree} \label{thm3}
\it 
Suppose that Condition~\ref{cnd_alg} holds and that either (i) Assumption~\ref{asm_basic} holds and $|\mathcal{X}| < \infty$, or (ii) Assumptions~\ref{asm_basic} and \ref{asm_derivative} hold.
Let $\bm x_t=\mathcal A{\left(\mathcal D^{\rm RKB}_{t-1}\right)}$.
Then, the following holds:
\begin{equation}
{\rm BSR}_T\le \frac{B_T}{T}\qquad\left(T\in\mathbb N\right),
\end{equation}
where $B_T$ is the same as in Theorem~\ref{thm1}.
\end{restatable}
See Appendix~\ref{apn:thm3} for the proof.

Importantly, Theorem~\ref{thm3} shows the vanishing BSR upper bound independent of $Q$.
Although \citet{iwazaki2025improved,takeno2026regret} have shown a tighter high-probability cumulative regret bound for the sequential BO methods, BSR bounds for sequential BO methods \citep{Russo2014-learning,Takeno2023-randomized,takeno2024-posterior,takeno2025-regret}, PTS \citep{nava2022diversified}, DPP-TS \citep{nava2022diversified}, and UCB-PE \citep{Contal2013-Parallel} have the same rate as ours.
Thus, although we conjecture that our BSR bound remains loose when $Q$ is small, this looseness is a common limitation for the prior studies.
In addition, the existing PBO methods listed above are fully distributed or joint selection methods, as discussed in Section \ref{sec:related}.
Thus, RKB is the first greedy PBO method to explicitly encourage diversity and to achieve the BSR guarantee without dependence on $Q$.

The key observation for the proof is the monotonically decreasing property of the BSR, that is,
\begin{equation}
    {\rm BSR}_t \geq {\rm BSR}_T,
\end{equation}
for all $t \leq T$.
Thus, we have the BSR upper bound by the average of fantasized regret as follows:
\begin{equation}
    {\rm BSR}_T\le\frac{1}{T} \mathbb E{\left[\sum_{t=1}^Tg_t^*-g_t{\left(\bm x_t\right)}\right]}.
\end{equation}
Then, this term can be bounded from above by $B_T$ without a penalty term with respect to $Q$.

%% file: manuscripts/5_experiment.tex
\begin{figure}[t]
\centering
\includegraphics[width=\linewidth]{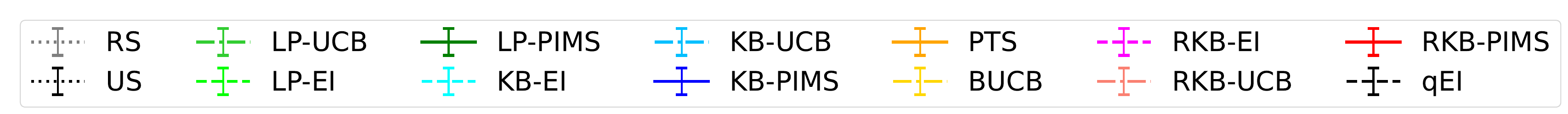}\\
\includegraphics[height=.27\linewidth]{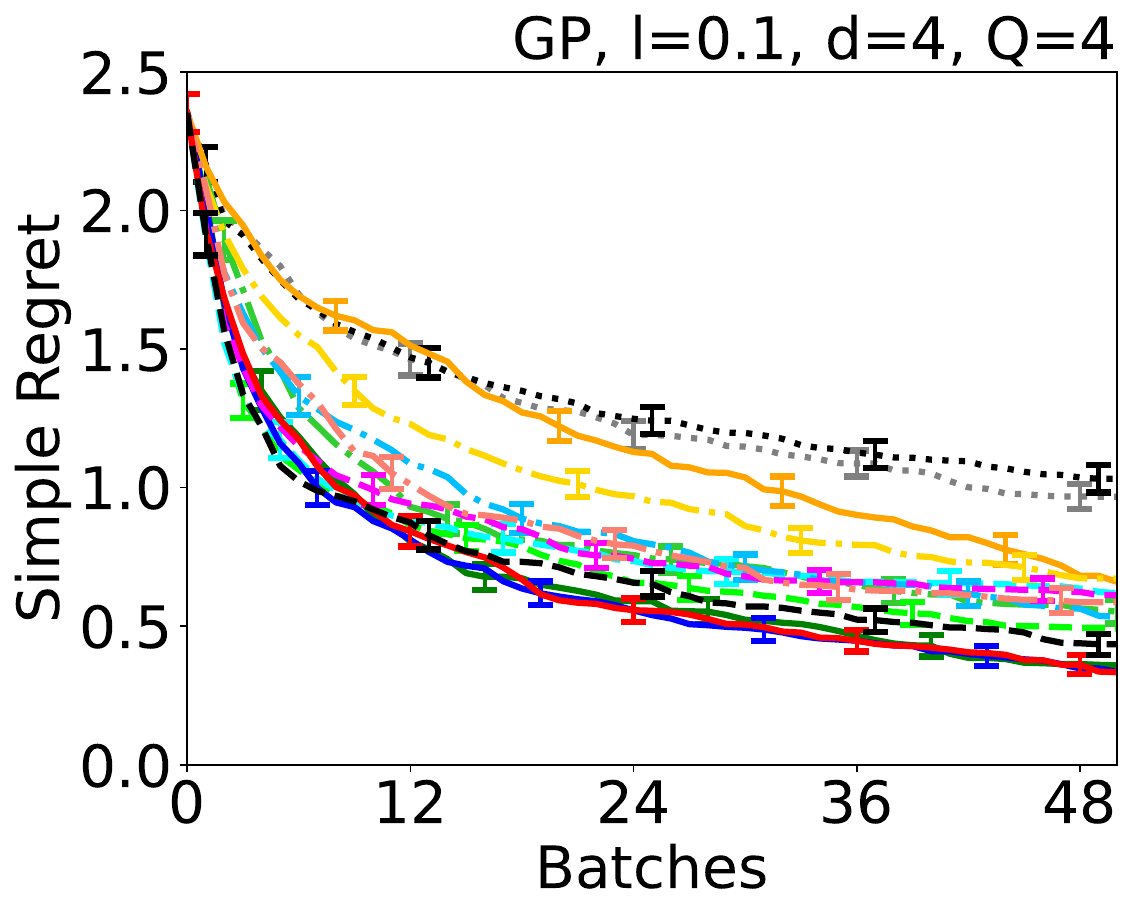}
\includegraphics[height=.27\linewidth]{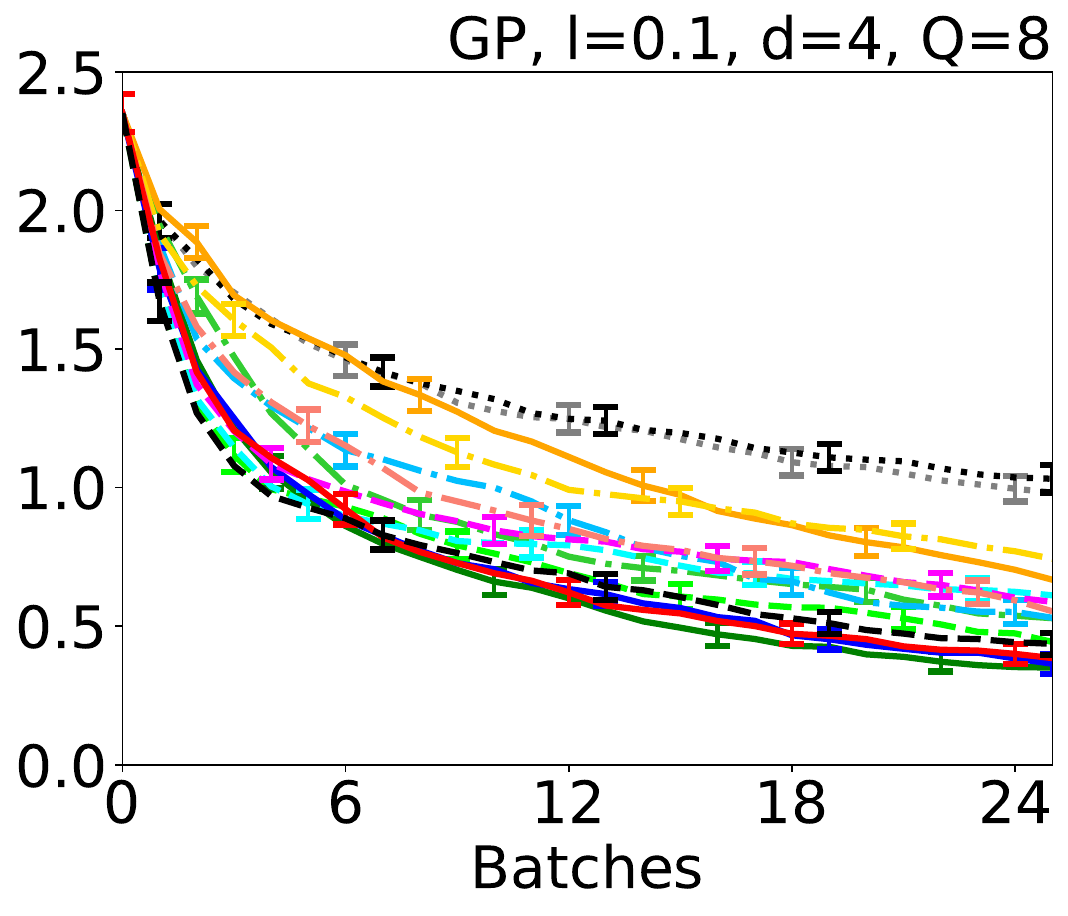}
\includegraphics[height=.27\linewidth]{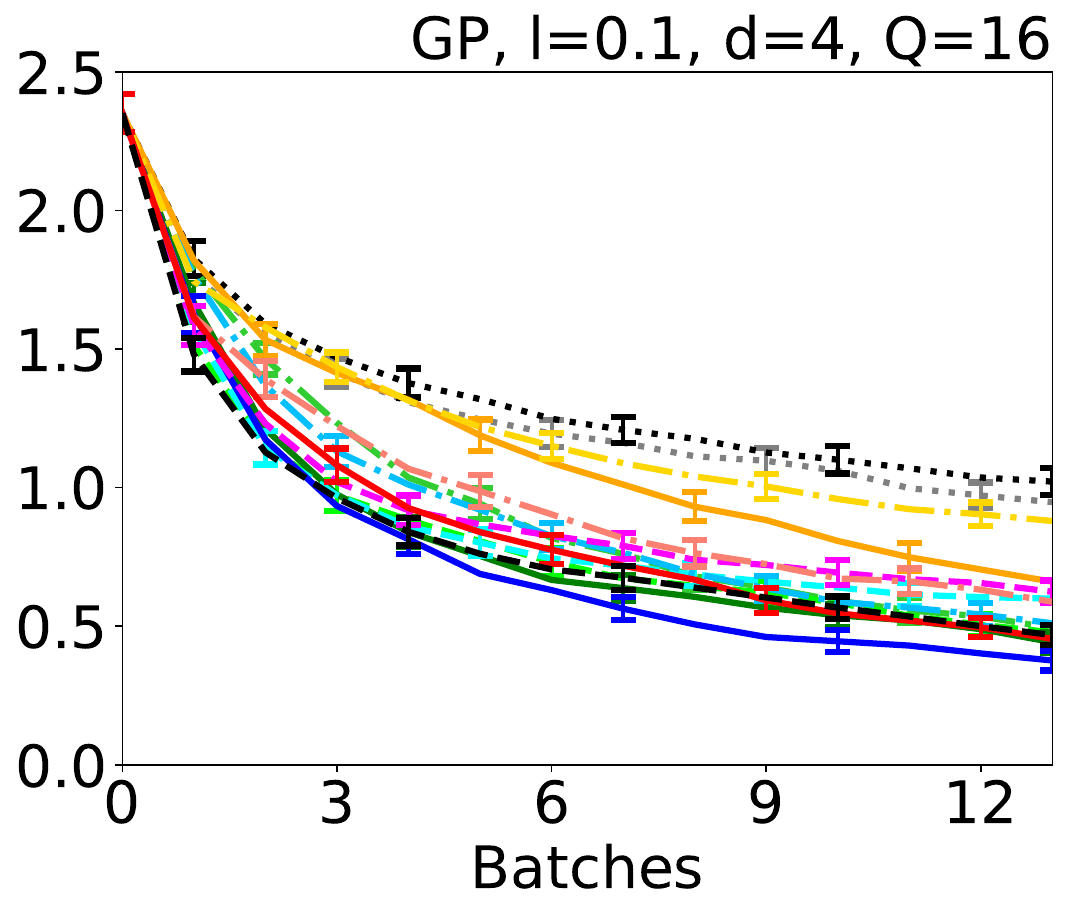}

\includegraphics[height=.26\linewidth]{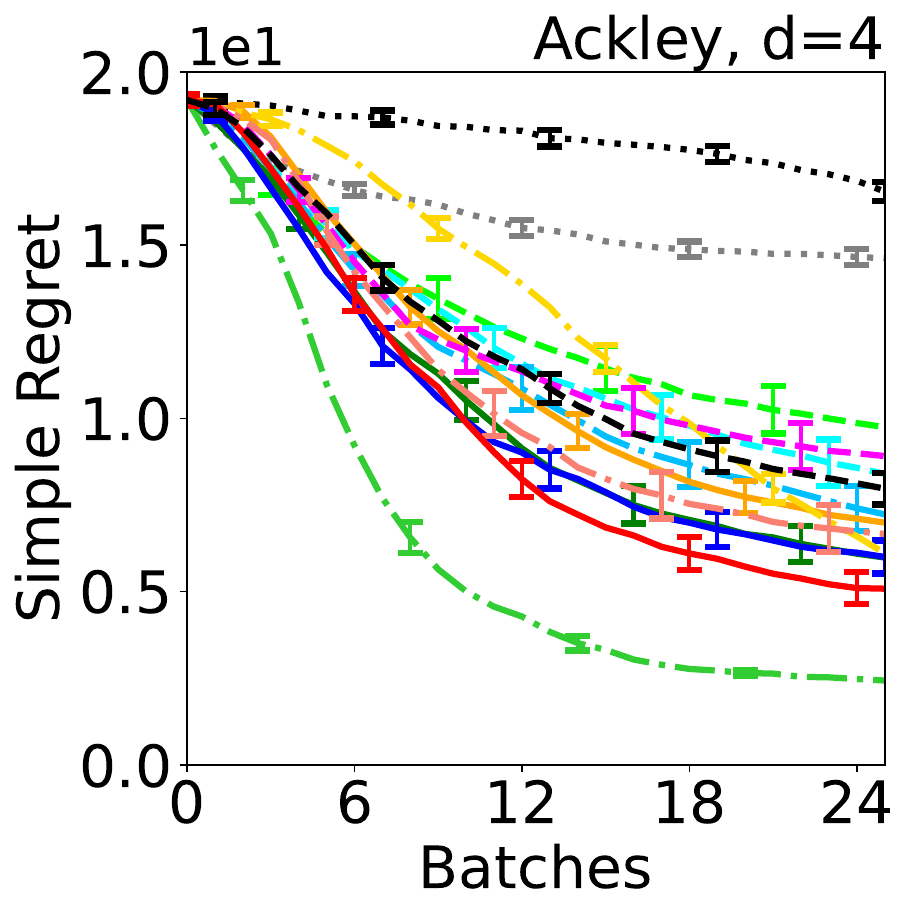}
\includegraphics[height=.26\linewidth]{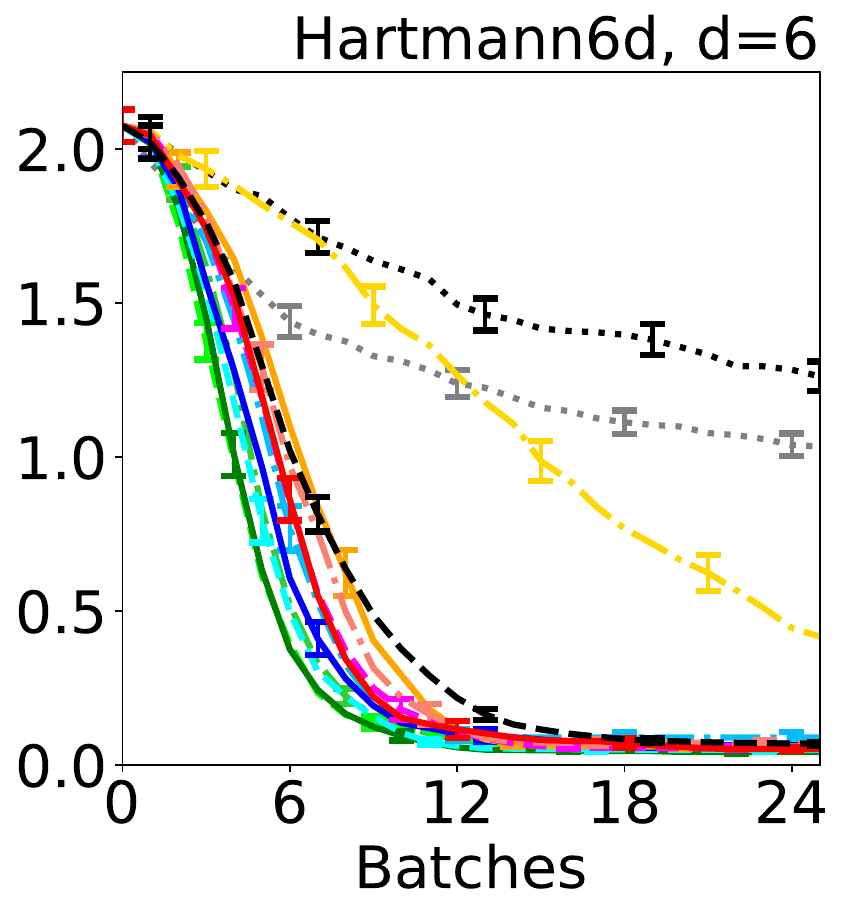}
\includegraphics[height=.26\linewidth]{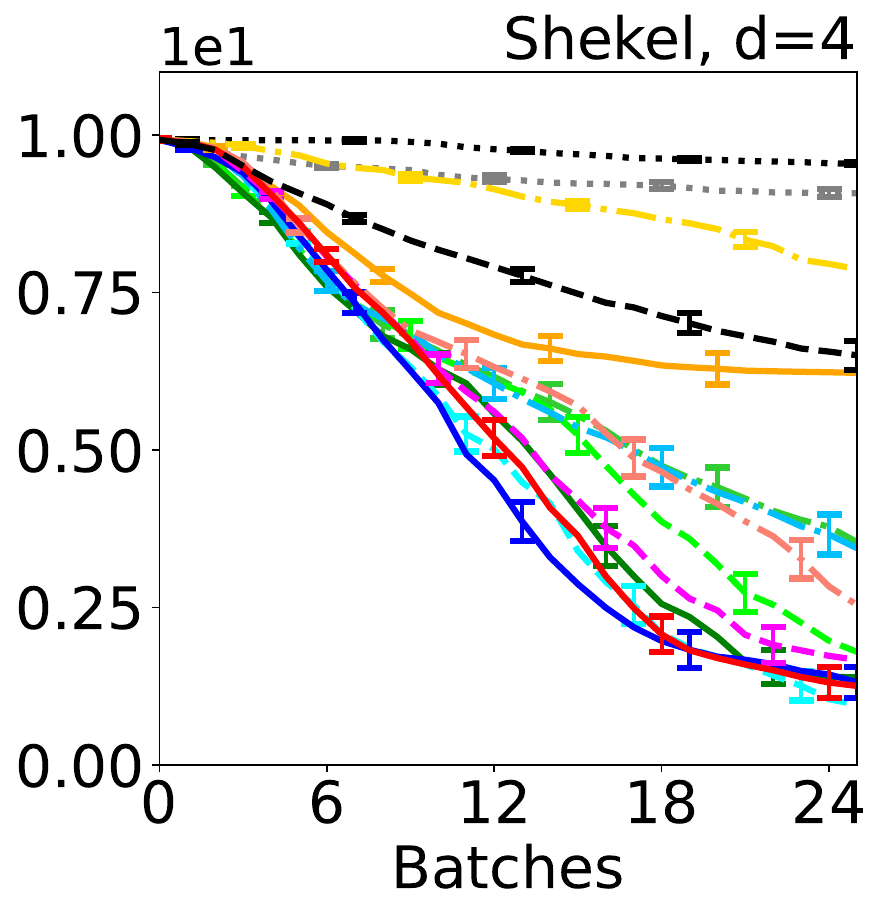}
\includegraphics[height=.26\linewidth]{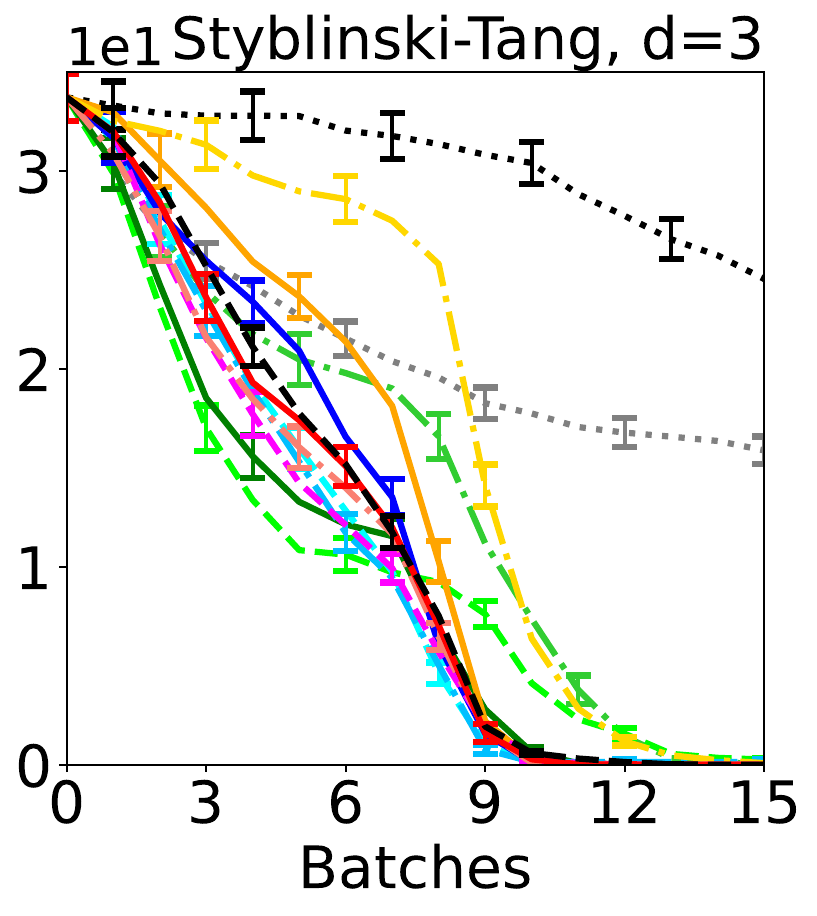}

\includegraphics[height=.24\linewidth]{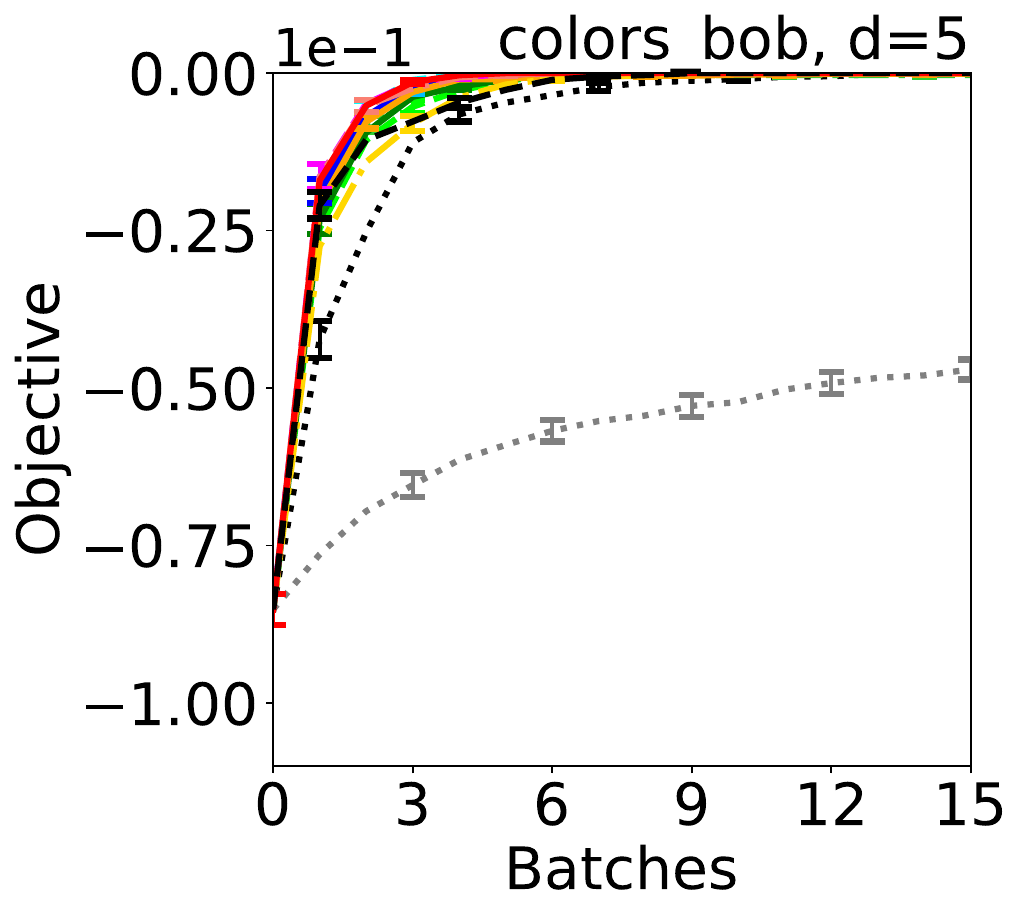}
\includegraphics[height=.24\linewidth]{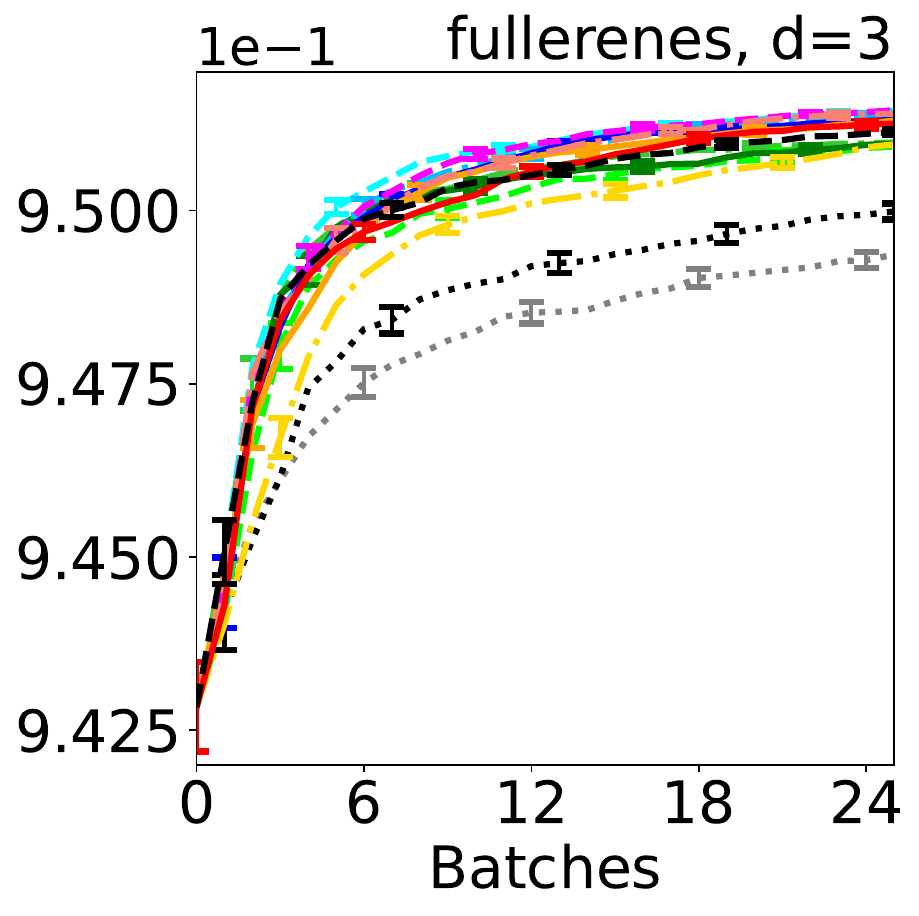}
\includegraphics[height=.24\linewidth]{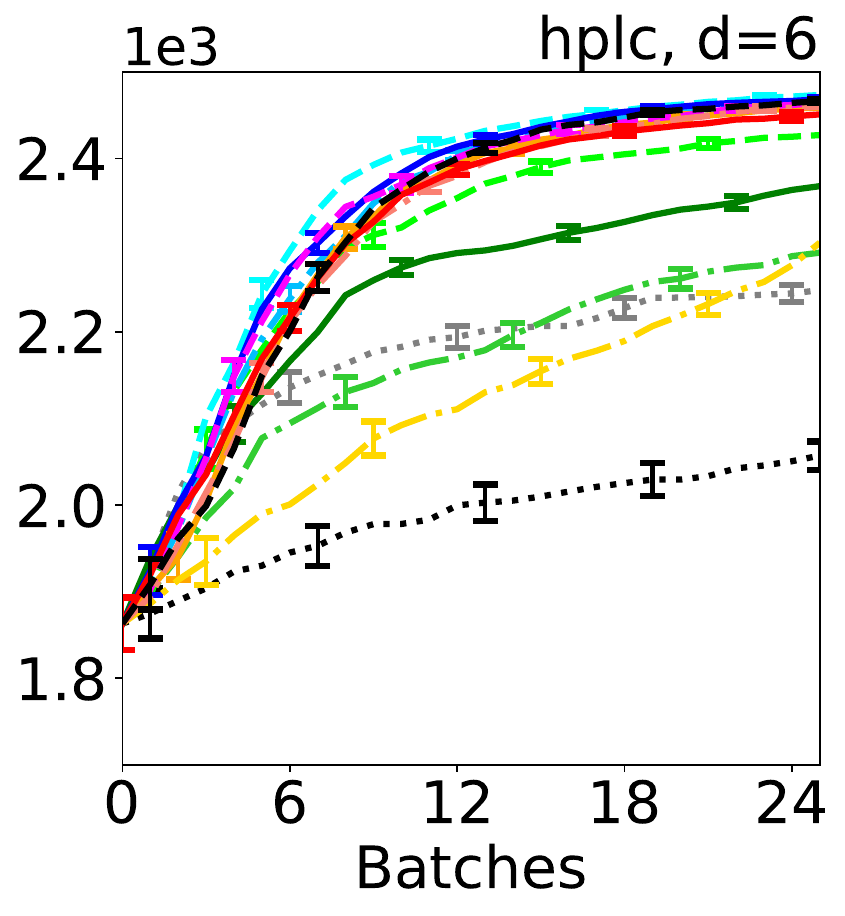}
\includegraphics[height=.24\linewidth]{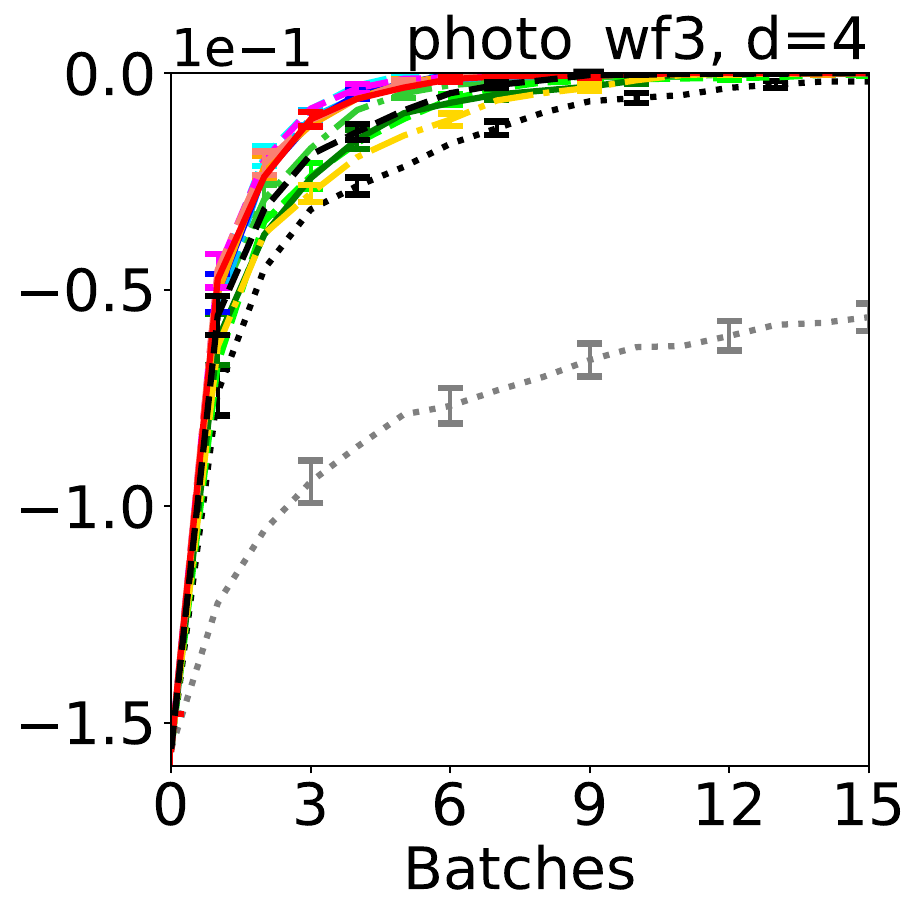}
\caption{
Mean and standard error of simple regret or best objective value across the 100 experiments on each condition. Rows correspond to the following objective functions: first, synthetic; second, benchmarks; third, emulators. One batch corresponds to $Q$ iterations, and $Q=8$ for benchmarks and emulators. Methods qLEI and qLNEI are denoted as qEI in the legend.
}
\label{fig:result}
\end{figure}

\begin{figure}[t]
\centering
\begin{minipage}[t]{.86\linewidth}
\vspace{0pt}
\includegraphics[width=\linewidth]{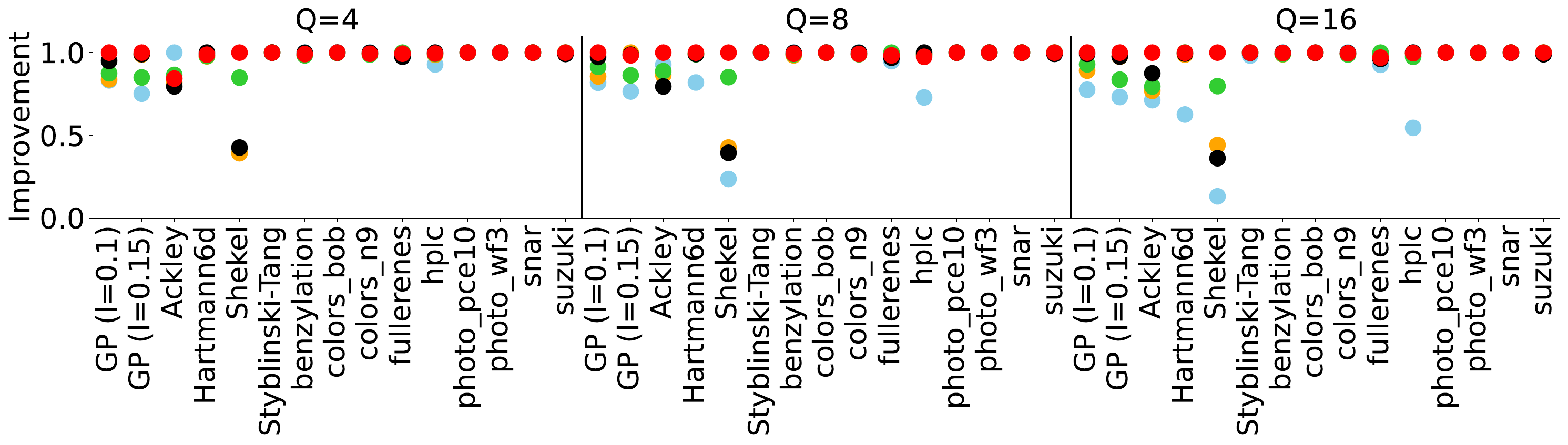}
\end{minipage}
\begin{minipage}[t]{.13\linewidth}
\vspace{0pt}
\includegraphics[width=\linewidth]{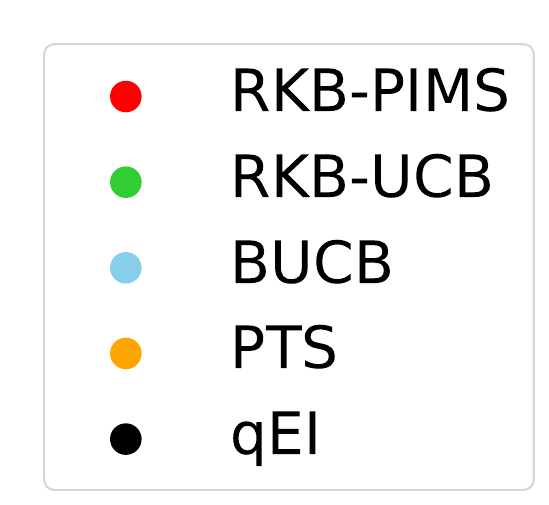}
\end{minipage}
\caption{
Mean improvement of the best objective value $\max_{i\in\left[t\right]}f{\left(\bm x_i\right)}$ across the 100 experiments on each condition.
The improvement amount has been normalized so that its maximum value among the five compared methods is one.
Methods qLEI and qLNEI are denoted as qEI in the legend.
}
\label{fig:result_comparison}
\end{figure}

\clearpage
\section{Experiments}
\label{sec:experiment}

We conducted experiments to demonstrate the efficiency of our RKB.
As objective functions, we used synthetic and benchmark functions, as well as emulators trained on real-world data.
As a performance measure, we report the simple regret $f(\bm x^*)-\max_{i\in\left[t\right]}f{\left(\bm x_i\right)}$ for synthetic and benchmark functions and the best objective value $\max_{i\in\left[t\right]}f{\left(\bm x_i\right)}$ for the emulators of real-world data since the optimum $f(\bm x^*)$ is unknown.

We compared RKB, KB, and LP in combination with UCB, EI, and PIMS.
These combinations are denoted by connecting the names of the parallelization method and the AF with a hyphen, for example, KB-EI.
Other methods included for comparison are BUCB, PTS, uncertainty sampling (US), and random search (RS).
Variants of qEI implemented in BoTorch \citep{balandat2020-botorch} are also included for comparison as follows: synthetic, q log noisy expected improvement (qLNEI); benchmark, q log expected improvement (qLEI); emulator, qEI.
We chose the qEI variants based on the observation noise level (synthetic function experiments are noisy) and the package's compatibility with the emulator, which uses TensorFlow 1 \citep{tensorflow2015-whitepaper}.
Among the compared methods, RKB-PIMS, RKB-UCB, BUCB, PTS, and US have theoretical backing. 
For PIMS and PTS, which involve the posterior sampling, we consistently employ random Fourier features \citep{Rahimi2008-Random} based approximation \citep{Wilson2020-efficiently,wilson2021-pathwise}.

\subsection{Setting}
\paragraph{Common setting.}
We consider simple synchronous parallelization with $Q\in\left\{4,8,16\right\}$ workers to facilitate the interpretation of the methods' effectiveness.
In this setting, all ongoing evaluations are completed at the end of each iteration whose index is a multiple of $Q$.
We also conducted experiments in the setting of asynchronous parallel optimization with $Q\in\left\{4,8,16\right\}$, and the results are presented in Appendix \ref{apn:result}.
We report the mean and standard error of the performance measure across 100 random trials for initial data generation, PBO algorithm's randomness, and synthetic objective function generation (in synthetic experiments only).

\paragraph{For synthetic function experiment.}
We generated the objective function by sampling from $\mathcal{GP}{\left(0,k_{\rm Gauss}\right)}$, where lengthscale parameter $l \in \left\{0.1,0.15\right\}$ and input dimension $d=4$.
We show the results for $l=0.1$ in this section, and the results for $l=0.15$ are given in Appendix \ref{apn:result}.
The search space is defined as $\mathcal X=\left\{0.1,0.2,\ldots,1.0\right\}^4$, which consists of $10^4$ grid points.
The variance of the Gaussian noise added to the observations is $\sigma_{\rm noise}^2=10^{-3}$.
As the prediction model, all algorithms used a GP with the true noise variance and the same kernel as that used to generate the objective function.
Each experiment was initialized with $8$ data points in $\mathcal X$ that were closest to $8$ points in $\left[0,1\right]^4$ chosen using Latin hypercube sampling (LHS) \cite{Loh1996-Latin}.
We set the confidence width parameter to $\beta_t=2\log{\left(\left|\mathcal X\right|t^2/\sqrt{2\pi}\right)}$ for KB-UCB and RKB-UCB, following the theoretical value in \cite{Takeno2023-randomized}.
For BUCB, we multiply $\beta_t$ by $1+\left(\left(t-1\right)\bmod Q\right)/\sigma_{\rm noise}^2$, where $\bmod$ denotes the remainder operator, following the theoretical value in \cite{Desautels2014-Parallelizing}.

\paragraph{For benchmark and emulator experiment.}
We employed four benchmark objective functions, Ackley, Hartmann6d, Shekel, and Styblinski--Tang on 4, 6, 4, and 3-dimensional search spaces, respectively, defined in \url{https://www.sfu.ca/~ssurjano/optimization.html}.
As the objective function, we also used emulators provided by Olympus \cite{hase2021olympus}, a benchmarking framework for optimization.
The experiments were conducted on all emulators, except the alkox emulator, which returns only 0 and seems to have bugs.
%
%
The prediction model is a GP using a Gaussian kernel with automatic relevance determination \cite{Rasmussen2005-Gaussian}.
We selected the lengthscales and prior variance $\sigma_{\rm pri}^2$ of the kernel by marginal likelihood maximization \cite{Rasmussen2005-Gaussian} every $Q$ iterations.
For computational stability, the noise variance in the model was set to $10^{-8}$.
Each experiment was initialized with $16$ data points chosen by LHS.
We set the confidence width parameter to $\beta_t=0.2d\log{\left(2t\right)}$, following the heuristics used in \cite{kandasamy2015-high,Kandasamy2017-Multi}.
%
%
As with the synthetic function experiments, we multiplied $\beta_t$ by $1+\left(\left(t-1\right)\bmod Q\right)\sigma_{\rm pri}^2/\sigma_{\rm noise}^2$ for BUCB.
We show the results for $Q = 8$, and other results with $Q\in\left\{4,16\right\}$ are presented in Appendix \ref{apn:result}.

\subsection{Result}
Figure \ref{fig:result} shows the average and standard error of the performance measure across the 100 experiments on each condition.
%
%
As shown in Fig. \ref{fig:result}, RKB is comparable to KB and LP in each AF combination, except for LP-UCB in the Ackley function experiment.
In all cases, RKB-PIMS and RKB-UCB perform at least as well as other theoretically guaranteed methods, PTS, BUCB, and US.
Moreover, PTS, BUCB, and US often underperform the proposed methods, likely due to over-exploration.
Furthermore, in all cases, RKB-PIMS performs comparably to or better than qEI, a state-of-the-art method.
While qEI underperformed RKB-PIMS significantly in the Shekel function experiment, we found no contrary example in our experiments.
Figure \ref{fig:result_comparison} summarizes the results across all conditions, focusing on the improvement in the value of the objective function throughout the experiments.
As shown in Fig. \ref{fig:result_comparison}, RKB exhibits consistently high performance under various conditions.
%

%% file: manuscripts/6_conclusion.tex
\section{Conclusion}

We proposed the general-purpose PBO method, called randomized kriging believer (RKB), based on the widely used KB heuristic, inheriting the practical advantages of KB.
Furthermore, we showed the BCR and BSR upper bounds for RKB comparable to TS-based PBO methods \citep{nava2022diversified}.
%
%
Finally, we demonstrated the effectiveness of the proposed method via extensive experiments.

There are several directions for future work.
First, since our RKB is a general-purpose algorithm for parallelization, extensions to various problem settings, such as multi-fidelity \citep{Takeno2020-Multifidelity,Takeno2022-generalized}, multi-objective \citep{paria2020-flexible,inatsu2024-bounding}, and constrained BO \citep{takeno2022-sequential}, are promising.
Second, extending our regret analysis is of interest; for example, this includes analyses in the frequentist setting and its application to recent tighter regret bounds in \citep{iwazaki2025improved,takeno2026regret}.
Finally, as in prior work, our BCR analysis requires uncertainty sampling in the initial phase to avoid an additional $O(\sqrt{Q})$ term.
Addressing this common limitation in PBO analysis remains an important open problem.

%% file: manuscripts/9_appendix.tex
\section{Regret bound derived from Condition~\ref{cnd_alg}}
\label{apn:general_RA}
In this section, we explain the meaning of the term $B_T$ in Theorems~\ref{thm1}, \ref{thm2}, and \ref{thm3}.
For a sequencial BO method satisfying Condition~\ref{cnd_alg}, we have the following lemma:
\begin{lemma}[BCR bound for sequential optimization]\label{general_RA}
\it 
Suppose that Condition~\ref{cnd_alg} holds and that either (i) Assumption~\ref{asm_basic} holds and $|\mathcal{X}| < \infty$, or (ii) Assumptions~\ref{asm_basic} and \ref{asm_derivative} hold.
Let $\bm x_t=\mathcal A{\left(\mathcal D_{t-1}\right)}$. Then, the following holds:
\begin{equation}
{\rm BCR}_T\le B_T\qquad\left(T\in\mathbb N\right),
\end{equation}
where $B_T$ is defined using $C_1=2/\log{\left(1+\sigma_{\rm noise}^{-2}\right)}$ as
\begin{equation}
B_T=\sqrt{C_1\gamma_T{\sum}_{t=1}^T\zeta_t}+\sum_{t=1}^T\xi_t = \tilde{O}(\sqrt{T \gamma_T})\quad\left(T\in\mathbb N\right).
\end{equation}
\end{lemma}
Lemma~\ref{general_RA} is a generalization of existing regret analyses.
The term $B_T$ is the derived regret bound defined using coefficients $\zeta_t$ and $\xi_t$ depending on the BO method.
For examples of $\zeta_t$ and $\xi_t$ for a specific BO method, see Appendix \ref{apn:cnd_alg}.

We use the following lemma, shown in Lemma 5.4 of \citep{Srinivas2010-Gaussian}, to prove Lemma~\ref{general_RA}:
\begin{lemma}\label{variance_bound}
\it Let Assumption~\ref{asm_basic} hold.
Then, for any $T\in\mathbb N$ and $\mathcal D_T=\left\{\left(\bm x_i,y_i\right)\right\}_{i=1}^T$, the following holds:
\begin{equation}
\sum_{t=1}^T\sigma^2{\left(\bm x_t;\mathcal D_{t-1}\right)}\le C_1\gamma_T,
\end{equation}
where $C_1=2/\log{\left(1+\sigma_{\rm noise}^{-2}\right)}$.
\end{lemma}
This lemma also plays an important role in the proof of our main theorem.

\begin{proof}[\bf Proof of Lemma~\ref{variance_bound}]
The lemma is proved as
\begin{subequations}\label{260116112741}
\begin{align}
\sum_{t=1}^T\sigma^2{\left(\bm x_t;\mathcal D_{t-1}\right)}=&C_1\sum_{t=1}^T\frac{1}{2}\log{\left(1+\sigma_{\rm noise}^{-2}\right)}\sigma^2{\left(\bm x_t;{\mathcal D}_{t-1}\right)}\\
\le&C_1\sum_{t=1}^T\frac{1}{2}\log{\left(1+\sigma_{\rm noise}^{-2}\sigma^2{\left(\bm x_t;{\mathcal D}_{t-1}\right)}\right)}\\
=&C_1I{\left(\bm y_T;\bm f_T\right)}\\
\le&C_1\gamma_T,
\end{align}
\end{subequations}
where the grounds of the equations are the following: Eq. (\ref{260116112741}b), Lemma~\ref{260111120105}; Eq. (\ref{260116112741}c), Lemma~\ref{srinivas2010_5_3}; Eq. (\ref{260116112741}d), the Definition~\ref{MIG} of $\gamma_T$.
\end{proof}
Using Lemma~\ref{variance_bound}, we prove Lemma~\ref{general_RA}.
\begin{proof}[\bf Proof of Lemma~\ref{general_RA}]
The proof is the following:
\begin{subequations}\label{260116113933}
\begin{align}
{\rm BCR}_T=\mathbb E{\left[\sum_{t=1}^Tf^*-f{\left(\bm x_t\right)}\right]}
=&\sum_{t=1}^T\mathbb E_{\mathcal D_{t-1}}{\left[\mathbb E{\left[f^*-f{\left(\bm x_t\right)}\mid\mathcal D_{t-1}\right]}\right]}\\
\le&\sum_{t=1}^T\mathbb E_{\mathcal D_{t-1}}{\left[\mathbb E{\left[u_t\sigma{\left(\bm x_t;\mathcal D_{t-1}\right)}+v_t\mid\mathcal D_{t-1}\right]}\right]}\\
=&\mathbb E{\left[{\sum}_{t=1}^Tu_t\sigma{\left(\bm x_t;\mathcal D_{t-1}\right)}+v_t\right]}\\
\le&\mathbb E{\left[\sqrt{{\sum}_{t=1}^Tu_t^2}\sqrt{{\sum}_{t=1}^T\sigma^2{\left(\bm x_t;\mathcal D_{t-1}\right)}}+{\sum}_{t=1}^Tv_t\right]}\\
\le&\mathbb E{\left[\sqrt{{\sum}_{t=1}^Tu_t^2}\right]}\sqrt{C_1\gamma_T}+{\sum}_{t=1}^T\mathbb E{\left[v_t\right]}\\
\le&\sqrt{\mathbb E{\left[{\sum}_{t=1}^Tu_t^2\right]}}\sqrt{C_1\gamma_T}+{\sum}_{t=1}^T\mathbb E{\left[v_t\right]}\\
\le&B_T,
\end{align}
\end{subequations}
where the grounds of the equations are the following: Eq. (\ref{260116113933}b), Eq. (\ref{eq_cnd_alg1}) of Condition~\ref{cnd_alg}; Eq. (\ref{260116113933}d), Cauchy--Schwarz inequality; Eq. (\ref{260116113933}e), Lemma~\ref{variance_bound}; Eq. (\ref{260116113933}f), Jensen's inequality; Eq. (\ref{260116113933}g), Eq. (\ref{eq_cnd_alg2}) of Condition~\ref{cnd_alg} and the definition of $B_T$.
\end{proof}

\section{Proof of Theorem~\ref{thm1}}
\label{apn:thm1}
We use the following three lemmas to prove Theorem 1:
\begin{lemma}\label{virtual_regret}
\it 
Suppose that Condition~\ref{cnd_alg} holds and that either (i) Assumption~\ref{asm_basic} holds and $|\mathcal{X}| < \infty$, or (ii) Assumptions~\ref{asm_basic} and \ref{asm_derivative} hold.
Let $\bm x_t=\mathcal A{\left(\mathcal D^{\rm RKB}_{t-1}\right)}$. 
Then, the following holds:
\begin{equation}
\mathbb E{\left[\sum_{t=1}^Tg_t^*-g_t{\left(\bm x_t\right)}\right]}\le B_T,
\end{equation}
where $g_t\sim p{\left(f\mid\mathcal D_{\mathcal N_{t-1}}\right)}$ and $g_t^*=\max_{\bm x\in\mathcal X}g_t{\left(\bm x\right)}$. Bound $B_T$ is defined using $C_1=2/\log{\left(1+\sigma_{\rm noise}^{-2}\right)}$ as
\begin{equation}
B_T=\sqrt{C_1\gamma_T{\sum}_{t=1}^T\zeta_t}+\sum_{t=1}^T\xi_t\qquad\left(T\in\mathbb N\right).
\end{equation}
\end{lemma}
\begin{proof}[\bf Proof of Lemma~\ref{virtual_regret}]
The proof is the following:
\begin{subequations}\label{260116123114}
\begin{align}
\mathbb E{\left[\sum_{t=1}^Tg_t^*-g_t{\left(\bm x_t\right)}\right]}=&\sum_{t=1}^T\mathbb E_{\mathcal D^{\rm RKB}_{t-1}}{\left[\mathbb E{\left[g_t^*-g_t{\left(\bm x_t\right)}\mid\mathcal D^{\rm RKB}_{t-1}\right]}\right]}\\
=&\sum_{t=1}^T\mathbb E_{\mathcal D^{\rm RKB}_{t-1}}{\left[\mathbb E{\left[f^*-f{\left(\bm x_t\right)}\mid\mathcal D_{t-1}=\mathcal D^{\rm RKB}_{t-1}\right]}\right]}\\
\le&\sum_{t=1}^T\mathbb E_{\mathcal D^{\rm RKB}_{t-1}}{\left[\mathbb E{\left[u_t{\left(\mathcal D^{\rm RKB}_{t-1}\right)}\sigma{\left(\bm x_t;\mathcal D^{\rm RKB}_{t-1}\right)}+v_t{\left(\mathcal D^{\rm RKB}_{t-1}\right)}\mid\mathcal D^{\rm RKB}_{t-1}\right]}\right]}\\
=&\mathbb E{\left[{\sum}_{t=1}^Tu_t{\left(\mathcal D^{\rm RKB}_{t-1}\right)}\sigma{\left(\bm x_t;\mathcal D^{\rm RKB}_{t-1}\right)}+v_t{\left(\mathcal D^{\rm RKB}_{t-1}\right)}\right]}\\
\le&\mathbb E{\left[\sqrt{{\sum}_{t=1}^Tu_t^2{\left(\mathcal D^{\rm RKB}_{t-1}\right)}}\sqrt{{\sum}_{t=1}^T\sigma^2{\left(\bm x_t;\mathcal D^{\rm RKB}_{t-1}\right)}}+{\sum}_{t=1}^Tv_t{\left(\mathcal D^{\rm RKB}_{t-1}\right)}\right]}\\
\le&\mathbb E{\left[\sqrt{{\sum}_{t=1}^Tu_t^2{\left(\mathcal D^{\rm RKB}_{t-1}\right)}}\sqrt{{\sum}_{t=1}^T\sigma^2{\left(\bm x_t;\mathcal D_{t-1}\right)}}+{\sum}_{t=1}^Tv_t{\left(\mathcal D^{\rm RKB}_{t-1}\right)}\right]}\\
\le&\mathbb E{\left[\sqrt{{\sum}_{t=1}^Tu_t^2{\left(\mathcal D^{\rm RKB}_{t-1}\right)}}\right]}\sqrt{C_1\gamma_T}+{\sum}_{t=1}^T\mathbb E{\left[v_t{\left(\mathcal D^{\rm RKB}_{t-1}\right)}\right]}\\
\le&\sqrt{\mathbb E{\left[{\sum}_{t=1}^Tu_t^2{\left(\mathcal D^{\rm RKB}_{t-1}\right)}\right]}}\sqrt{C_1\gamma_T}+{\sum}_{t=1}^T\mathbb E{\left[v_t{\left(\mathcal D^{\rm RKB}_{t-1}\right)}\right]}\\
\le&B_T,
\end{align}
\end{subequations}
where the grounds of the equations are the following: Eq. (\ref{260116123114}b), $p{\left(g_t\mid\mathcal D^{\rm RKB}_{t-1}\right)}=p{\left(f\mid\mathcal D_{t-1}=\mathcal D^{\rm RKB}_{t-1}\right)}$; Eq. (\ref{260116123114}c), Eq. (\ref{eq_cnd_alg1}) of Condition~\ref{cnd_alg}; Eq. (\ref{260116123114}e), Cauchy--Schwarz inequality; Eq. (\ref{260116123114}f), $\sigma{\left(\bm x_t;\mathcal D^{\rm RKB}_{t-1}\right)}=\sigma{\left(\bm x_t;\mathcal D_{t-1}\right)}$; Eq. (\ref{260116123114}g), Lemma~\ref{variance_bound}; Eq. (\ref{260116123114}h), Jensen's inequality; Eq. (\ref{260116123114}i), Eq. (\ref{eq_cnd_alg2}) of Condition~\ref{cnd_alg} and the definition of $B_T$.

\end{proof}
\begin{lemma}[Modified from Lemma 3 of \citep{vakili2021-scalable}]\label{variance_bound2}
\it Let Assumption~\ref{asm_basic} hold and MIG $\gamma_T$ has an upper bound $\bar\gamma_T$ concave for $T$.
Then, for any $t\in\mathbb N$, $\mathcal D_{t-1}=\left\{\left(\bm x_i,y_i\right)\right\}_{i=1}^{t-1}\in 2^{\mathcal X\times\mathbb R}$, and $ \mathcal N_{t-1}\subset\left[t-1\right]$ satisfying $t-\left| \mathcal N_{t-1}\right|\le Q$, the following holds:
\begin{equation}\label{260109180029}
\sum_{t=1}^T\sigma^2{\left(\bm x_t;\mathcal D_{ \mathcal N_{t-1}}\right)}\le C_1Q\bar\gamma_{T/Q}.
\end{equation}
\end{lemma}
\begin{proof}[\bf Proof of Lemma~\ref{variance_bound2}]
Let $\mathcal W_{t-1}\subset\left[Q\right]$ be the set of indices of available workers for the $t$-th evaluation, and let the worker $w_t=\min\mathcal W_{t-1}$ be actually assigned to evaluation.
Mathematically, the set $\mathcal W_{t-1}$ is defined as
\begin{equation}
\mathcal W_0=\left[Q\right],\qquad \mathcal W_t=\left(\mathcal W_{t-1}\backslash\left\{w_t\right\}\right)\cup\left\{w_{t'}\in\left[Q\right]\mid t'\in \mathcal N_t\backslash \mathcal N_{t-1}\right\}\qquad\left(t\in\mathbb N\right),
\end{equation}
where $w_{t'}$ with $t'\in\mathcal N_t\backslash \mathcal N_{t-1}$ is a worker that has completed evaluation and become available again at iteration $t+1$.
Let $\mathcal T_q\subset\left[T\right]$ be the set of evaluations that worker $q\in\left[Q\right]$ is assigned to, i.e.,
\begin{equation}
\mathcal T_q=\left\{t\le T\mid w_t=q\right\}.
\end{equation}
Obviously, the set $\mathcal T_q$ satisfies
\begin{equation}
\bigcup_{q=1}^Q\mathcal T_q=\left[T\right],\qquad\forall q,q'\in\left[Q\right],q\ne q'\Rightarrow\mathcal T_q\cap\mathcal T_{q'}=\emptyset.
\end{equation}
Therefore, we have
\begin{equation}\label{eq:proof_lem_variance_bound2}
\sum_{t=1}^T\sigma^2{\left(\bm x_t;\mathcal D_{ \mathcal N_{t-1}}\right)}=\sum_{q=1}^{Q}\sum_{t\in\mathcal T_q}\sigma^2{\left(\bm x_t;\mathcal D_{ \mathcal N_{t-1}}\right)}\le\sum_{q=1}^{Q}\sum_{t\in\mathcal T_q}\sigma^2{\left(\bm x_t;\mathcal D_{ \mathcal N_{t-1}\cap\mathcal T_q}\right)}.
\end{equation}
The inequality is obtained from $ \mathcal N_{t-1}\supset \mathcal N_{t-1}\cap\mathcal T_q$.
When worker $q$ is assigned to evaluation $t$, worker $q$ has already completed the past evaluations assigned to them, i.e.,
\begin{equation}
\forall t,t'\in\mathcal T_q,t>t'\Rightarrow t'\in \mathcal N_{t-1}.
\end{equation}
Hence, the inner summation of the right-hand side of Eq. (\ref{eq:proof_lem_variance_bound2}) is considered as the following operation: adding and learning each data in $\mathcal D_{\mathcal T_q}$ one by one and summing the prediction variances.
Then, Lemma \ref{variance_bound2} is proved as
\begin{equation}
\sum_{q=1}^{Q}\sum_{t\in\mathcal T_q}\sigma^2{\left(\bm x_t;\mathcal D_{ \mathcal N_{t-1}\cap\mathcal T_q}\right)}\le \sum_{q=1}^{Q}C_1\gamma_{\left|\mathcal T_q\right|}\le C_1Q\sum_{q=1}^{Q}\frac{1}{Q}\bar\gamma_{\left|\mathcal T_q\right|}\le C_1Q\bar\gamma_{T/Q},
\end{equation}
where the grounds are the following: the first inequality, Lemma \ref{variance_bound}; the third inequality, Jensen's inequality and concavity of $\bar\gamma_T$.
\end{proof}

\begin{lemma}\label{takeno24_3_2}
\it Let Assumption~\ref{asm_basic} hold.
For any $t\in\mathbb N$, $\mathcal D_{t-1}$, and finite $\mathcal X'\subset\mathcal X$, let $g_t\sim p{\left(f\mid\mathcal D_{t-1}\right)}$ and define $\eta_t$, $\theta_t\in\mathbb R$ as
\begin{equation}
\eta_t=\max_{\bm x\in\mathcal X'}\frac{g_t{\left(\bm x\right)}-\mu{\left(\bm x;\mathcal D_{t-1}\right)}}{\sigma{\left(\bm x;\mathcal D_{t-1}\right)}},\qquad\theta_t=\eta_t^2\mathbbm 1{\left\{\eta_t\ge0\right\}},
\end{equation}
where $\mathbbm 1{\left\{\eta_t\ge0\right\}}$ is $1$ if $\eta_t\ge0$. Otherwise $0$. Then, the following holds:
\begin{equation}
\mathbb E{\left[\theta_t\mid\mathcal D_{t-1}\right]}\le2+2\log{\left(\left|\mathcal X'\right|/2\right)}.
\end{equation}
\end{lemma}
\begin{proof}[\bf Proof of Lemma~\ref{takeno24_3_2}]
The proof mainly follows the proof of Lemma 3.2 of \citep{takeno2024-posterior}.
If $\left|\mathcal X'\right|=1$, $\eta_t$ and $\eta^2_t$ follow the standard normal and Chi-squared distributions, respectively.
Hence, we have 
$\mathbb E{\left[\theta_t\mid\mathcal D_{t-1}\right]}=1/2<0.61\simeq2+2\log{\left(\left|\mathcal X'\right|/2\right)}$.
Next, we consider the case $\left|\mathcal X'\right|>1$.
Let $\beta_\delta=2\log{\left(\left|\mathcal X'\right|/\left(2\delta\right)\right)}$ for any $\delta\in\left(0,1\right)$.
We first show
\begin{equation}\label{260108175104}
{\rm Pr}{\left(\theta_t\le\beta_\delta\mid\mathcal D_{t-1}\right)}\ge1-\delta\qquad\left(\delta\in\left(0,1\right)\right).
\end{equation}
The probability is transformed as
\begin{equation}
\begin{split}
{\rm Pr}{\left(\theta_t\le\beta_{\delta}\mid\mathcal D_{t-1}\right)}=&{\rm Pr}{\left(\eta_t^2\mathbbm 1{\left\{\eta_t\ge0\right\}}\le\beta_{\delta}\mid\mathcal D_{t-1}\right)}\\
=&{\rm Pr}{\left(\eta_t\mathbbm 1{\left\{\eta_t\ge0\right\}}\le\beta_\delta^{1/2}\mid\mathcal D_{t-1}\right)}\\
=&{\rm Pr}{\left(\eta_t\le\beta_\delta^{1/2}\mid\mathcal D_{t-1}\right)}.
\end{split}
\end{equation}
We define $\bm x_{f,\eta}^*$ and $\bm x_{g,\eta}^*$ as 
\begin{equation}
\bm x_{f,\eta}^*=\mathop{\rm argmax}\limits_{\bm x\in\mathcal X'}\frac{f{\left(\bm x\right)}-\mu{\left(\bm x;\mathcal D_{t-1}\right)}}{\sigma{\left(\bm x;\mathcal D_{t-1}\right)}},\qquad\bm x_{g,\eta}^*=\mathop{\rm argmax}\limits_{\bm x\in\mathcal X'}\frac{g_t{\left(\bm x\right)}-\mu{\left(\bm x;\mathcal D_{t-1}\right)}}{\sigma{\left(\bm x;\mathcal D_{t-1}\right)}}.
\end{equation}
Because of $g_t\sim p{\left(f\mid\mathcal D_{t-1}\right)}$, pairs $\left(f,\bm x_{f,\eta}^*\right)$ and $\left(g_t,\bm x_{g,\eta}^*\right)$ follow the same distribution under the condition $\mathcal D_{t-1}$.
Hence, we have
\begin{equation}
\begin{split}
{\rm Pr}{\left(\eta_t\le\beta_\delta^{1/2}\mid\mathcal D_{t-1}\right)}=&{\rm Pr}{\left(\frac{g_t{\left(\bm x_{g,\eta}^*\right)}-\mu{\left(\bm x_{g,\eta}^*;\mathcal D_{t-1}\right)}}{\sigma{\left(\bm x_{g,\eta}^*;\mathcal D_{t-1}\right)}}\le\beta_\delta^{1/2}\mid\mathcal D_{t-1}\right)}\\
=&{\rm Pr}{\left(\frac{f{\left(\bm x_{f,\eta}^*\right)}-\mu{\left(\bm x_{f,\eta}^*;\mathcal D_{t-1}\right)}}{\sigma{\left(\bm x_{f,\eta}^*;\mathcal D_{t-1}\right)}}\le\beta_\delta^{1/2}\mid\mathcal D_{t-1}\right)}\\
=&{\rm Pr}{\left(f{\left(\bm x_{f,\eta}^*\right)}\le\mu{\left(\bm x_{f,\eta}^*;\mathcal D_{t-1}\right)}+\beta_\delta^{1/2}\sigma{\left(\bm x_{f,\eta}^*;\mathcal D_{t-1}\right)}\mid\mathcal D_{t-1}\right)}.
\end{split}
\end{equation}
Then, from Lemma~\ref{takeno23_4_1}, we have
\begin{equation}
\begin{split}
&{\rm Pr}{\left(f{\left(\bm x_\eta^*\right)}\le\mu{\left(\bm x_\eta^*;\mathcal D_{t-1}\right)}+\beta_\delta^{1/2}\sigma{\left(\bm x_\eta^*;\mathcal D_{t-1}\right)}\mid\mathcal D_{t-1}\right)}\\
\ge&{\rm Pr}{\left(\forall\bm x\in\mathcal X',f{\left(\bm x\right)}\le\mu{\left(\bm x;\mathcal D_{t-1}\right)}+\beta_\delta^{1/2}\sigma{\left(\bm x;\mathcal D_{t-1}\right)}\mid\mathcal D_{t-1}\right)}\\
\ge&1-\delta,
\end{split}
\end{equation}
which proves Eq. (\ref{260108175104}).

We define the cumulative distribution function $F$ of $\theta_t$ conditioned by $\mathcal D_{t-1}$ and its generalized inverse function $F^{-1}$ as
\begin{equation}
F{\left(\beta\right)}={\rm Pr}{\left(\theta_t\le\beta\mid\mathcal D_{t-1}\right)},\quad F^{-1}{\left(p\right)}=\min{\left\{\beta'\in\mathbb R\cup\left\{\infty\right\}\mid F{\left(\beta'\right)}\ge p\right\}}\quad\left(\beta\ge0,p\in\left[0,1\right]\right).
\end{equation}
Then, Eq. (\ref{260108175104}) is equivalent to
\begin{equation}
\beta_\delta\ge F^{-1}{\left(1-\delta\right)}.
\end{equation}
Hence, we have
\begin{equation}
2+2\log{\left(\left|\mathcal X'\right|/2\right)}=\int_0^1\beta_\delta d\delta\ge\int_0^1F^{-1}{\left(1-\delta\right)}d\delta=\int_0^1F^{-1}{\left(\delta\right)}d\delta.
\end{equation}
From the same principle of the inverse transform sampling, the following holds:
\begin{equation}
\int_0^1F^{-1}{\left(\delta\right)}d\delta=\mathbb E_{\delta\sim\mathcal U{\left(0,1\right)}}{\left[F^{-1}{\left(\delta\right)}\right]}=\mathbb E{\left[\theta_t\mid\mathcal D_{t-1}\right]},
\end{equation}
where $\mathcal U$ denotes the uniform distribution. This concludes the proof.
\end{proof}
\thmOne*
\begin{proof}[\bf Proof of Theorem~\ref{thm1}]
Because of $g_t\sim p{\left(f\mid {\mathcal D}_{\mathcal N_{t-1}}\right)}$, we can transform ${\rm BCR}_T$ as
\begin{equation}\label{260111125418}
\begin{split}
{\rm BCR}_T=\mathbb E{\left[\sum_{t=1}^Tf^*-f{\left(\bm x_t\right)}\right]}=&\sum_{t=1}^T\mathbb E_{{\mathcal D}_{\mathcal N_{t-1}}}{\left[\mathbb E{\left[g_t^*-f{\left(\bm x_t\right)}\mid{\mathcal D}_{\mathcal N_{t-1}}\right]}\right]}\\
=&\sum_{t=1}^T\mathbb E_{{\mathcal D}_{\mathcal N_{t-1}}}{\left[\mathbb E{\left[g_t^*-g_t{\left(\bm x_t\right)}+g_t{\left(\bm x_t\right)}-f{\left(\bm x_t\right)}\mid{\mathcal D}_{\mathcal N_{t-1}}\right]}\right]}\\
=&\underbrace{\mathbb E{\left[\sum_{t=1}^Tg_t^*-g_t{\left(\bm x_t\right)}\right]}}_{A_1}+\underbrace{\mathbb E{\left[\sum_{t=1}^Tg_t{\left(\bm x_t\right)}-f{\left(\bm x_t\right)}\right]}}_{A_2}.
\end{split}
\end{equation}
From Lemma~\ref{virtual_regret}, we have $A_1\le B_T$. We transform $A_2$ as
\begin{equation}
\begin{split}
A_2=\mathbb E{\left[\sum_{t=1}^Tg_t{\left(\bm x_t\right)}-f{\left(\bm x_t\right)}\right]}=&\sum_{t=1}^T\mathbb E_{{\mathcal D}_{\mathcal N_{t-1}}}{\left[\mathbb E{\left[g_t{\left(\bm x_t\right)}-f{\left(\bm x_t\right)}\mid{\mathcal D}_{\mathcal N_{t-1}}\right]}\right]}\\
=&\sum_{t=1}^T\mathbb E{\left[g_t{\left(\bm x_t\right)}-\mu{\left(\bm x_t;{\mathcal D}_{\mathcal N_{t-1}}\right)}\right]}\\
=&\sum_{t=1}^T\mathbb E{\left[\frac{g_t{\left(\bm x_t\right)}-\mu{\left(\bm x_t;{\mathcal D}_{\mathcal N_{t-1}}\right)}}{\sigma{\left(\bm x_t;{\mathcal D}_{\mathcal N_{t-1}}\right)}}\sigma{\left(\bm x_t;{\mathcal D}_{\mathcal N_{t-1}}\right)}\right]},
\end{split}
\end{equation}
Let $\eta_t=\mathop{\max}\limits_{\bm x\in\mathcal X}\frac{g_t{\left(\bm x\right)}-\mu{\left(\bm x;{\mathcal D}_{\mathcal N_{t-1}}\right)}}{\sigma{\left(\bm x;{\mathcal D}_{\mathcal N_{t-1}}\right)}}$ and $\theta_t=\eta_t^2\mathbbm 1{\left\{\eta_t\ge0\right\}}$.
Then, $A_2$ is bounded above as
\begin{subequations}\label{260111140007}
\begin{align}
A_2\le\sum_{t=1}^T\mathbb E{\left[\eta_t\sigma{\left(\bm x_t;{\mathcal D}_{\mathcal N_{t-1}}\right)}\right]}\le&\sum_{t=1}^T\mathbb E{\left[\eta_t\mathbbm 1{\left\{\eta_t\ge0\right\}}\sigma{\left(\bm x_t;{\mathcal D}_{\mathcal N_{t-1}}\right)}\right]}\\
\le&\mathbb E{\left[\sqrt{{\sum}_{t=1}^T\eta_t^2\mathbbm 1{\left\{\eta_t\ge0\right\}}}\sqrt{{\sum}_{t=1}^T\sigma^2{\left(\bm x_t;{\mathcal D}_{\mathcal N_{t-1}}\right)}}\right]}\\
\le&\mathbb E{\left[\sqrt{{\sum}_{t=1}^T\theta_t}\right]}\sqrt{C_1Q\bar\gamma_{T/Q}}\\
\le&\sqrt{{\sum}_{t=1}^T\mathbb E{\left[\theta_t\right]}}\sqrt{C_1Q\bar\gamma_{T/Q}}\\
=&\sqrt{{\sum}_{t=1}^T\mathbb E_{\mathcal D_{\mathcal N_{t-1}}}{\left[\mathbb E{\left[\theta_t\mid\mathcal D_{\mathcal N_{t-1}}\right]}\right]}}\sqrt{C_1Q\bar\gamma_{T/Q}}\\
\le&\sqrt{C_1C_2QT\bar\gamma_{T/Q}}.
\end{align}
\end{subequations}
where $C_1=2/\log{\left(1+\sigma_{\rm noise}^{-2}\right)}$ and $C_2=2+2\log{\left(\left|\mathcal X\right|/2\right)}$. The grounds of the equations are the following: Eq. (\ref{260111140007}b), Cauchy--Schwarz inequality; Eq. (\ref{260111140007}c), Lemma~\ref{variance_bound2}; Eq. (\ref{260111140007}d), Jensen's inequality; Eq. (\ref{260111140007}f), Lemma~\ref{takeno24_3_2} with $\mathcal X'=\mathcal X$. As shown above, we have ${\rm BCR}_T=A_1+A_2$, $A_1\le B_T$, and $A_2\le\sqrt{C_1C_2QT\bar\gamma_{T/Q}}$, which proves Theorem~\ref{thm1}.
\end{proof}

\section{Proof of Theorem~\ref{thm2}}
\label{apn:thm2}
We first introduce the notation and lemmas needed to prove Theorem~\ref{thm2}.
Let $\left(l_t\right)_{t\in\mathbb N}$ be an arbitrary sequence of positive numbers.
We equip $\left[0,r\right]^d$ with the metric induced by 1-norm $\left\|\cdot\;\right\|_1$.
Then, the largest axis-aligned cube included in a ball of radius $l_t/2$ has edge length $l_t/d$. Hence,  set $\left[0,r\right]^d$ can be covered by $\lceil dr/l_t\rceil^d$ balls of radius $l_t/2$. By this fact and Exercise 4.26 (b) of \cite{Vershynin2018-high}, set $\mathcal X\subset\left[0,r\right]^d$ can also be covered by the same number of balls with radius $l_t$.
That is, there is some set of centers $\mathcal X_t\subset\mathcal X$ that satisfies $\left|\mathcal X_t\right|\le\lceil dr/l_t\rceil^d$ and
\begin{equation}
\forall\bm x\in\mathcal X,\exists\bm x'\in\mathcal X_t,\left\|\bm x-\bm x'\right\|_1\le l_t.
\end{equation}
For $\bm x\in\mathcal X$, we define $\left[\bm x\right]_t$ as one of $\bm x'\in\mathcal X_t$ that satisfies $\left\|\bm x-\bm x'\right\|_1\le l_t$.

We use the following lemmas to prove Theorem 2:
\begin{lemma}[Lemma H.2 of \cite{Takeno2023-randomized}]\label{takeno23_h_2}
\it Suppose that Assumptions~\ref{asm_basic}, \ref{asm_derivative}, and $l_t\le\frac{1}{bt^2}/\left(\sqrt{\log{\left(ad\right)}}+\sqrt{\pi}/2\right)$ hold.
Then, the following holds:
\begin{equation}
\sum_{t=1}^T\mathbb E{\left[\sup_{\bm x\in\mathcal X}\left|f{\left(\bm x\right)}-f{\left(\left[\bm x\right]_t\right)}\right|\right]}\le\frac{\pi^2}{6}\qquad\left(T\in\mathbb N\right).
\end{equation}
\end{lemma}
\begin{lemma}[Lemma D.4 of \cite{takeno2024-posterior}]\label{takeno24_d_4}
\it Assume the same premise on $k$ as Theorem~\ref{kusakawa22_e_4}.
Define $L_\sigma$ by Eq. (\ref{L_sigma}).
Suppose that Assumption~\ref{asm_basic} and $l_t\le\frac{1}{L_\sigma t^2}$ hold.
Then, the following holds for any $t\in\mathbb N$ and $\mathcal D_{t-1}$:
\begin{equation}
\sup_{\bm x\in\mathcal X}\left|\sigma{\left(\bm x;\mathcal D_{t-1}\right)}-\sigma{\left(\left[\bm x\right]_t;\mathcal D_{t-1}\right)}\right|\le\frac{1}{t^2}.
\end{equation}
\end{lemma}
\thmTwo*
\begin{proof}[\bf Proof of Theorem~\ref{thm2}]
In the same way as Eq. (\ref{260111125418}), we transform ${\rm BCR}_T$ as
\begin{equation}
{\rm BCR}_T=\mathbb E{\left[\sum_{t=1}^Tf^*-f{\left(\bm x_t\right)}\right]}=\underbrace{\mathbb E{\left[\sum_{t=1}^Tg_t^*-g_t{\left(\bm x_t\right)}\right]}}_{A_1}+\underbrace{\mathbb E{\left[\sum_{t=1}^Tg_t{\left(\bm x_t\right)}-f{\left(\bm x_t\right)}\right]}}_{A_2}.
\end{equation}
From Lemma~\ref{virtual_regret}, we have $A_1\le B_T$.
Let $l_t=\frac{1}{Lt^2}$ for $L=\max{\left\{L_\sigma,b\left(\sqrt{\log{\left(ad\right)}}+\sqrt{\pi}/2\right)\right\}}$.
We transform $A_2$ as
\begin{equation}
\begin{split}
A_2=&\mathbb E{\left[\sum_{t=1}^Tg_t{\left(\bm x_t\right)}-f{\left(\bm x_t\right)}\right]}\\
=&\underbrace{\mathbb E{\left[\sum_{t=1}^Tg_t{\left(\bm x_t\right)}-g{\left(\left[\bm x_t\right]_t\right)}\right]}}_{A_3}
+\underbrace{\mathbb E{\left[\sum_{t=1}^Tg_t{\left(\left[\bm x_t\right]_t\right)}-f{\left(\left[\bm x_t\right]_t\right)}\right]}}_{A_4}
+\underbrace{\mathbb E{\left[\sum_{t=1}^Tf{\left(\left[\bm x_t\right]_t\right)}-f{\left(\bm x_t\right)}\right]}}_{A_5}\\
\end{split}
\end{equation}
Then, from Lemma~\ref{takeno23_h_2}, we have $A_3\le\pi^2/6$ and $A_5\le\pi^2/6$.
We transform $A_4$ as
\begin{equation}
\begin{split}
A_4=&\mathbb E{\left[\sum_{t=1}^Tg_t{\left(\left[\bm x_t\right]_t\right)}-f{\left(\left[\bm x_t\right]_t\right)}\right]}\\
=&\mathbb E_{\mathcal D_{\mathcal N_{t-1}}}{\left[\mathbb E{\left[\sum_{t=1}^Tg_t{\left(\left[\bm x_t\right]_t\right)}-f{\left(\left[\bm x_t\right]_t\right)}\mid\mathcal D_{\mathcal N_{t-1}}\right]}\right]}\\
=&\mathbb E{\left[\sum_{t=1}^Tg_t{\left(\left[\bm x_t\right]_t\right)}-\mu{\left(\left[\bm x_t\right]_t;\mathcal D_{\mathcal N_{t-1}}\right)}\right]}\\
=&\mathbb E{\left[\sum_{t=1}^T\frac{g_t{\left(\left[\bm x_t\right]_t\right)}-\mu{\left(\left[\bm x_t\right]_t;\mathcal D_{\mathcal N_{t-1}}\right)}}{\sigma{\left(\left[\bm x_t\right]_t;\mathcal D_{\mathcal N_{t-1}}\right)}}\sigma{\left(\left[\bm x_t\right]_t;\mathcal D_{\mathcal N_{t-1}}\right)}\right]}\\
\end{split}
\end{equation}
Let $\eta_t=\mathop{\max}\limits_{\bm x\in\mathcal X_t}\frac{g_t{\left(\bm x\right)}-\mu{\left(\bm x;{\mathcal D}_{\mathcal N_{t-1}}\right)}}{\sigma{\left(\bm x;{\mathcal D}_{\mathcal N_{t-1}}\right)}}$.
Then, $A_4$ is bounded above as
\begin{equation}
\begin{split}
A_4\le&\mathbb E{\left[\sum_{t=1}^T\eta_t\sigma{\left(\left[\bm x_t\right]_t;\mathcal D_{\mathcal N_{t-1}}\right)}\right]}\\
\le&\underbrace{\mathbb E{\left[\sum_{t=1}^T\eta_t\sigma{\left(\bm x_t;\mathcal D_{\mathcal N_{t-1}}\right)}\right]}}_{A_6}
+\underbrace{\mathbb E{\left[\sum_{t=1}^T\eta_t\left|\sigma{\left(\bm x_t;\mathcal D_{\mathcal N_{t-1}}\right)}-\sigma{\left(\left[\bm x_t\right]_t;\mathcal D_{\mathcal N_{t-1}}\right)}\right|\right]}}_{A_7}\\
\end{split}
\end{equation}
In the same way as Eq. (\ref{260111140007}), $A_6$ is bounded above as
\begin{equation}
A_6=\mathbb E{\left[\sum_{t=1}^T\eta_t\sigma{\left(\bm x_t;\mathcal D_{\mathcal N_{t-1}}\right)}\right]}\le\sqrt{C_1Qs_TT\bar\gamma_{T/Q}},
\end{equation}
where $s_t\ge2+2\log{\left(\left|\mathcal X_t\right|/2\right)}$ is defined as
\begin{equation}
s_t=2+2d\log{\lceil dr/l_t\rceil}-2\log2=2+2d\log{\lceil drLt^2\rceil}-2\log2\qquad\left(t\in\mathbb N\right).
\end{equation}
Also, $A_7$ is bounded above as
\begin{subequations}\label{260112191011}
\begin{align}
A_7=&\mathbb E{\left[{\sum}_{t=1}^T\eta_t\left|\sigma{\left(\bm x_t;\mathcal D_{\mathcal N_{t-1}}\right)}-\sigma{\left(\left[\bm x_t\right]_t;\mathcal D_{\mathcal N_{t-1}}\right)}\right|\right]}\\
\le&{\sum}_{t=1}^T\frac{1}{t^2}\mathbb E{\left[\eta_t\right]}\\
\le&{\sum}_{t=1}^T\frac{1}{t^2}\mathbb E{\left[\eta_t\mathbbm 1{\left\{\eta_t\ge0\right\}}\right]}\\
=&{\sum}_{t=1}^T\frac{1}{t^2}\mathbb E{\left[\sqrt{\eta_t^2\mathbbm 1{\left\{\eta_t\ge0\right\}}}\right]}\\
\le&{\sum}_{t=1}^T\frac{1}{t^2}\sqrt{\mathbb E{\left[\eta_t^2\mathbbm 1{\left\{\eta_t\ge0\right\}}\right]}}\\
\le&{\sum}_{t=1}^T\frac{1}{t^2}\sqrt{2+2\log{\left(\left|\mathcal X_t\right|/2\right)}}\\
\le&{\sum}_{t=1}^T\frac{1}{t^2}\sqrt{2+2\log{\left(\left|\mathcal X_T\right|/2\right)}}\\
=&\frac{\pi^2}{6}\sqrt{s_T}.
\end{align}
\end{subequations}
The grounds of the equations are the following: Eq. (\ref{260112191011}b), Lemma~\ref{takeno24_d_4}; Eq. (\ref{260112191011}e), Jensen's inequality; Eq. (\ref{260112191011}f), Lemma~\ref{takeno24_3_2} with $\mathcal X'=\mathcal X_t$; Eq. (\ref{260112191011}g), the monotonicity of $\left|\mathcal X_t\right|$. Finally, the proof is concluded as
\begin{equation}
\begin{split}
&{\rm BCR}_T=A_1+A_2\le A_1+A_3+A_5+A_6+A_7,\quad A_2=A_3+A_4+A_5,\quad A_4\le A_6+A_7,\\
&A_1\le B_T,\quad A_3\le\frac{\pi^2}{6},\quad A_5\le\frac{\pi^2}{6},\quad A_6\le\sqrt{C_1Qs_TT\bar\gamma_{T/Q}},\quad A_7\le\frac{\pi^2}{6}\sqrt{s_T}.
\end{split}
\end{equation}
\end{proof}

\section{Proof of Theorem 3}
\label{apn:thm3}
\thmThree*
\begin{proof}[\bf Proof of Theorem~\ref{thm3}]
The regret ${\rm BSR}_t$ is monotonically decreasing with respect to $t$ as shown in the following:
\begin{equation}\label{260112123405}
\begin{split}
{\rm BSR}_t=&\mathbb E{\left[f^*-f{\left(\hat{\bm x}_t\right)}\right]}\\
=&\mathbb E{\left[f^*-f{\left(\hat{\bm x}_T\right)}+f{\left(\hat{\bm x}_T\right)}-f{\left(\hat{\bm x}_t\right)}\right]}\\
=&\mathbb E{\left[f^*-f{\left(\hat{\bm x}_T\right)}\right]}+\mathbb E_{\mathcal D_T}{\left[\mathbb E{\left[f{\left(\hat{\bm x}_T\right)}-f{\left(\hat{\bm x}_t\right)}\mid\mathcal D_T\right]}\right]}\\
=&\mathbb E{\left[f^*-f{\left(\hat{\bm x}_T\right)}\right]}+\mathbb E_{\mathcal D_T}{\left[\mu{\left(\hat{\bm x}_T;\mathcal D_T\right)}-\mu{\left(\hat{\bm x}_t;\mathcal D_T\right)}\right]}\\
\ge&\mathbb E{\left[f^*-f{\left(\hat{\bm x}_T\right)}\right]}\\
=&{\rm BSR}_T\qquad\qquad\qquad\qquad\left(T,t\in\mathbb N,T>t\right),
\end{split}
\end{equation}
where the inequality is obtained because $\hat{\bm x}_T$ maximizes $\mu{\left(\;\cdot\;;\mathcal D_T\right)}$.
Moreover, ${\rm BSR}_t$ is bounded above as
\begin{equation}\label{260112123418}
\begin{split}
{\rm BSR}_t=&\mathbb E{\left[f^*-f{\left(\hat{\bm x}_t\right)}\right]}\\
=&\mathbb E{\left[f^*-f{\left(\mathcal A{\left(\mathcal D_{t-1}\right)}\right)}+f{\left(\mathcal A{\left(\mathcal D_{t-1}\right)}\right)}-f{\left(\hat{\bm x}_t\right)}\right]}\\
=&\mathbb E{\left[f^*-f{\left(\mathcal A{\left(\mathcal D_{t-1}\right)}\right)}\right]}+\mathbb E_{\mathcal D_t}{\left[\mathbb E{\left[f{\left(\mathcal A{\left(\mathcal D_{t-1}\right)}\right)}-f{\left(\hat{\bm x}_t\right)}\mid\mathcal D_t\right]}\right]}\\
=&\mathbb E{\left[f^*-f{\left(\mathcal A{\left(\mathcal D_{t-1}\right)}\right)}\right]}+\mathbb E_{\mathcal D_t}{\left[\mu{\left(\mathcal A{\left(\mathcal D_{t-1}\right)};\mathcal D_t\right)}-\mu{\left(\hat{\bm x}_t;\mathcal D_t\right)}\right]}\\
\le&\mathbb E{\left[f^*-f{\left(\mathcal A{\left(\mathcal D_{t-1}\right)}\right)}\right]}\\
=&\mathbb E_{\mathcal D_{\mathcal N_{t-1}},\left(\bm x_i\right)_{i=1}^{t-1}}{\left[\mathbb E{\left[f^*-f{\left(\mathcal A{\left(\mathcal D_{t-1}\right)}\right)}\mid\mathcal D_{\mathcal N_{t-1}},\left(\bm x_i\right)_{i=1}^{t-1}\right]}\right]}\\
=&\mathbb E_{\mathcal D_{\mathcal N_{t-1}},\left(\bm x_i\right)_{i=1}^{t-1}}{\left[\mathbb E{\left[g_t^*-g_t{\left(\mathcal A{\left(\mathcal D^{\rm RKB}_{t-1}\right)}\right)}\mid\mathcal D_{\mathcal N_{t-1}},\left(\bm x_i\right)_{i=1}^{t-1}\right]}\right]}\\
=&\mathbb E{\left[g_t^*-g_t{\left(\mathcal A{\left(\mathcal D^{\rm RKB}_{t-1}\right)}\right)}\right]}\\
=&\mathbb E{\left[g_t^*-g_t{\left(\bm x_t\right)}\right]}\qquad\qquad\qquad\qquad\qquad\qquad\qquad\left(t\in\mathbb N\right).
\end{split}
\end{equation}
Line 5 in Eq. (\ref{260112123418}) is obtained because $\hat{\bm x}_t$ maximizes $\mu{\left(\;\cdot\;;\mathcal D_t\right)}$.
Line 7 in Eq. (\ref{260112123418}) is obtained from $g_t\sim p{\left(f\mid \mathcal D_{\mathcal N_{t-1}}\right)}$.
From Eqs. (\ref{260112123405}) and (\ref{260112123418}), we have
\begin{equation}
{\rm BSR}_T\le\frac{1}{T}\sum_{t=1}^T{\rm BSR}_t\le\frac{1}{T}\underbrace{\mathbb E{\left[\sum_{t=1}^Tg_t^*-g_t{\left(\bm x_t\right)}\right]}}_{A_1}
\end{equation}
From Lemma~\ref{virtual_regret}, we have $A_1\le B_T$, which concludes the proof.
\end{proof}
\section{Auxiliary lemmas}
\begin{lemma}[Lemma 5.3 of \cite{Srinivas2010-Gaussian}]\label{srinivas2010_5_3}
\it Let Assumption~\ref{asm_basic} hold.
Then, for any $T\in\mathbb N$ and $\mathcal D_T=\left\{\left(\bm x_i,y_i\right)\right\}_{i=1}^T$, the following holds:
\begin{equation}
I{\left(\bm y_T;\bm f_T\right)}=\frac{1}{2}\sum_{t=1}^T\log{\left(1+\sigma_{\rm noise}^{-2}\sigma^2{\left(\bm x_t;\mathcal D_{t-1}\right)}\right)},
\end{equation}
where $\bm y_T=\left[y_1\;\cdots\;y_T\right]^\top$ and $\bm f_T=\left[f{\left(\bm x_1\right)}\;\cdots\;f{\left(\bm x_T\right)}\right]^\top$.
\end{lemma}
\begin{lemma}\label{takeno23_4_1}
\it Let Assumption~\ref{asm_basic} hold.
Then, for any $\delta\in\left(0,1\right)$, $t\in\mathbb N$, $\mathcal D_{t-1}$, and finite $\mathcal X'\subset\mathcal X$, the following holds:
\begin{equation}
{\rm Pr}{\left(\forall\bm x\in\mathcal X',f{\left(\bm x\right)}\le\mu{\left(\bm x;\mathcal D_{t-1}\right)}+\beta_\delta^{1/2}\sigma{\left(\bm x;\mathcal D_{t-1}\right)}\mid\mathcal D_{t-1}\right)}\ge 1-\delta,
\end{equation}
where $\beta_\delta=2\log{\left(\left|\mathcal X'\right|/\left(2\delta\right)\right)}$.
\end{lemma}
\begin{proof}[\bf Proof of Lemma~\ref{takeno23_4_1}]
We define event $E{\left(\bm x\right)}$ as
\begin{equation}
E{\left(\bm x\right)}\Leftrightarrow f{\left(\bm x\right)}\le\mu{\left(\bm x;\mathcal D_{t-1}\right)}+\beta_\delta^{1/2}\sigma{\left(\bm x;\mathcal D_{t-1}\right)}\qquad\left(\bm x\in\mathcal X'\right).
\end{equation}
Then, from Lemma~\ref{takeno23_H_3}, we have
\begin{equation}
{\rm Pr}{\left(\neg E{\left(\bm x\right)}\mid\mathcal D_{t-1}\right)}=1-\Phi{\left(\beta_\delta^{1/2}\right)}\le\frac{\delta}{\left|\mathcal X'\right|}\qquad\left(\bm x\in\mathcal X'\right).
\end{equation}
Therefore, Lemma~\ref{takeno23_4_1} is proved as
\begin{equation}
\begin{split}
&{\rm Pr}{\left(\forall\bm x\in\mathcal X',E{\left(\bm x\right)}\mid\mathcal D_{t-1}\right)}\\
=&1-{\rm Pr}{\left(\exists\bm x\in\mathcal X',\neg E{\left(\bm x\right)}\mid\mathcal D_{t-1}\right)}\\
\ge&1-\sum_{\bm x\in\mathcal X'}{\rm Pr}{\left(\neg E{\left(\bm x\right)}\mid\mathcal D_{t-1}\right)}\\
\ge&1-\delta.
\end{split}
\end{equation}
\end{proof}
\begin{lemma}\label{260111120105}
\it The following inequality holds:
\begin{equation}
\left(1+a\right)^b\le 1+ab\qquad\left(a>0,b\in\left[0,1\right]\right).
\end{equation}
\end{lemma}
\begin{proof}[\bf Proof of Lemma~\ref{260111120105}]
Let $a$ be an arbitrary positive number.
We define $f:\left[0,1\right]\to\mathbb R$ as
\begin{equation}
f{\left(x\right)}=\left(1+a\right)^x\qquad\left(x\in\left[0,1\right]\right).
\end{equation}
Then, $f$ is convex because the second derivative $f''$ of $f$ satisfies
\begin{equation}
f''{\left(x\right)}=\left(\log{\left(1+a\right)}\right)^2\left(1+a\right)^x>0\qquad\left(x\in\left[0,1\right]\right).
\end{equation}
Therefore, we have 
\begin{equation}
f{\left(b\right)}=\left(1+a\right)^b\le\left(1-b\right)f{\left(0\right)}+bf{\left(1\right)}=1+ab\qquad\left(b\in\left[0,1\right]\right),
\end{equation}
which proves Lemma~\ref{260111120105}.
\end{proof}
\begin{lemma}\label{takeno23_H_3}
\it The survival function of the standard normal distribution can be bounded above as
\begin{equation}
1-\Phi{\left(c\right)}\le\frac{1}{2}\exp{\left(-\frac{c^2}{2}\right)}\qquad\left(c > 0\right),
\end{equation}
where $\Phi$ is the cumulative distribution function of the standard normal distribution.
\end{lemma}
\begin{proof}[\bf Proof of Lemma~\ref{takeno23_H_3}]
The proof is the following:
\begin{equation}
\begin{split}
1-\Phi{\left(c\right)}=&\int_c^\infty\frac{1}{\sqrt{2\pi}}\exp{\left(-\frac{r^2}{2}\right)}dr\\
=&\exp{\left(-\frac{c^2}{2}\right)}\int_c^\infty\frac{1}{\sqrt{2\pi}}\exp{\left(-\frac{r^2-c^2}{2}\right)}dr\\
=&\exp{\left(-\frac{c^2}{2}\right)}\int_c^\infty\frac{1}{\sqrt{2\pi}}\exp{\left(-\frac{\left(r-c\right)^2}{2}-c\left(r-c\right)\right)}dr\\
\le&\exp{\left(-\frac{c^2}{2}\right)}\int_c^\infty\frac{1}{\sqrt{2\pi}}\exp{\left(-\frac{\left(r-c\right)^2}{2}\right)}dr\\
=&\frac{1}{2}\exp{\left(-\frac{c^2}{2}\right)}.
\end{split}
\end{equation}
\end{proof}

\begin{lemma}[Theorem E.4 of \cite{kusakawa2022-Bayesian}]\label{kusakawa22_e_4}
\it Let Assumption \ref{asm_basic} hold. Suppose that the kernel $k$ is one of the following: linear kernel, Gaussian kernel, or Mat\'ern-$\nu$ kernel with $\nu>1$. 
Define $L_\sigma$ as
\begin{equation}\label{L_sigma}
\begin{split}
&L_\sigma=\begin{cases}
1&\mbox{($k$ is a linear kernel)},\\
\frac{\sqrt{2}}{l}&\mbox{($k$ is a Gaussian kernel)},\\
\frac{\sqrt{2}}{l}\sqrt{\frac{\nu}{\nu-1}}&\mbox{($k$ is a Mat\'ern-$\nu$ kernel)}.
\end{cases}
\end{split}
\end{equation}
Then, the following holds:
\begin{equation}
\begin{split}
&\forall t\in\mathbb N,\forall\mathcal D_{t-1}\in2^{\left(\mathcal X\times\mathbb R\right)},\forall\bm x,\bm x'\in\mathcal X,\\
&\left|\sigma{\left(\bm x;\mathcal D_{t-1}\right)}-\sigma{\left(\bm x';\mathcal D_{t-1}\right)}\right|\le L_\sigma\left\|\bm x-\bm x'\right\|_1.
\end{split}
\end{equation}
\end{lemma}
\section{Examples where condition \ref{cnd_alg} holds}
\label{apn:cnd_alg}
This section provides the details of five BO algorithms, UCB~\citep{Srinivas2010-Gaussian}, improved randomized UCB (IR-UCB)~\citep{Takeno2023-randomized}, PIMS~\citep{takeno2024-posterior}, EIMS~\citep{takeno2025-regretEI}, and TS \citep{Russo2014-learning} that satisfy Condition \ref{cnd_alg}. We also provide $u_t$, $v_t$, $\zeta_t$, and $\xi_t$ in those cases. If $\mathcal X$ is infinite, we employ Assumption~\ref{asm_derivative} and consider a finite subset $\mathcal X_t\subset\mathcal X$ for some $l_t>0$ as the proof of Theorem~\ref{thm2}. For $\bm x\in\mathcal X$, let $\left[\bm x\right]_t$ be one of $\bm x'\in\mathcal X_t$ that satisfies $\left\|\bm x-\bm x'\right\|_1\le l_t$. For convenience, we restate Condition \ref{cnd_alg}.
\cndalg*
\subsection{UCB}
In GP-UCB, the algorithm $\mathcal A$ is defined as
\begin{equation}\label{A_GP-UCB}
\mathcal A{\left(\mathcal D_{t-1}\right)}=\mathop{\rm argmax}\limits_{\bm x\in\mathcal X}\mu{\left(\bm x;\mathcal D_{t-1}\right)}+\beta_t^{1/2}\sigma{\left(\bm x;\mathcal D_{t-1}\right)}.
\end{equation}
If $\mathcal X$ is finite, Condition \ref{cnd_alg} holds for the following bounds:
\begin{equation}
\begin{split}
&u_t\left(\mathcal D_{t-1}\right)=\beta_t^{1/2},\quad v_t\left(\mathcal D_{t-1}\right)=f{\left(\bm x^*\right)}-\mu{\left(\bm x^*;\mathcal D_{t-1}\right)}-\beta_t^{1/2}\sigma{\left(\bm x^*;\mathcal D_{t-1}\right)},\\
&\zeta_t=\beta_t,\quad\xi_t=1/t^2,
\end{split}
\end{equation}
where $\beta_t=2\log{\left(\left|\mathcal X\right|t^2/\sqrt{2\pi}\right)}$.

If $\mathcal X$ is infinite, Condition \ref{cnd_alg} holds for the following bounds:
\begin{equation}
\begin{split}
&u_t\left(\mathcal D_{t-1}\right)=\beta_t^{1/2},\quad v_t\left(\mathcal D_{t-1}\right)=f{\left(\bm x^*\right)}-\mu{\left(\left[\bm x^*\right]_t;\mathcal D_{t-1}\right)}-\beta_t^{1/2}\sigma{\left(\left[\bm x^*\right]_t;\mathcal D_{t-1}\right)},\\
&\zeta_t=\beta_t,\quad\xi_t=2/t^2,
\end{split}
\end{equation}
where $\beta_t=2d\log{\left\lceil dr/l_t\right\rceil}+2\log{\left(t^2/\sqrt{2\pi}\right)}$ and $l_t=\left(bt^2\left(\sqrt{\log{\left(ad\right)}}+\sqrt{\pi}/2\right)\right)^{-1}$.
These results are given in the proof of Theorem B.1 of \citep{Takeno2023-randomized}.
\subsection{IR-UCB}
In IRGP-UCB, the algorithm $\mathcal A$ is defined as Eq. \eqref{A_GP-UCB} with random $\beta_t$.
If $\mathcal X$ is finite, Condition \ref{cnd_alg} holds for the following bounds:
\begin{equation}
\begin{split}
&u_t\left(\mathcal D_{t-1}\right)=\beta_t^{1/2},\quad v_t\left(\mathcal D_{t-1}\right)=0,\\
&\zeta_t=2+2\log{\left(\left|\mathcal X\right|/2\right)},\quad\xi_t=0,
\end{split}
\end{equation}
where $\beta_t$ follows the following distribution:
\begin{equation}
p{\left(\beta_t\right)}=\begin{cases}
\frac{1}{2}\exp{\left(-\frac{\beta_t-s}{2}\right)}&\left(\beta_t\ge s\right),\\
0&\left(\beta_t<s\right),
\end{cases}\qquad s=2\log{\left(\left|\mathcal X\right|/2\right)}.
\end{equation}
This result is given in the proof of Theorem 4.2 of \citep{Takeno2023-randomized}.

If $\mathcal X$ is infinite, Condition \ref{cnd_alg} holds for the following bounds:
\begin{equation}
\begin{split}
&u_t\left(\mathcal D_{t-1}\right)=\beta_t^{1/2},\quad v_t\left(\mathcal D_{t-1}\right)=f{\left(\bm x^*\right)}-f{\left(\left[\bm x^*\right]_t\right)},\\
&\zeta_t=2+2d\log{\left\lceil dr/l_t\right\rceil}-2\log2,\quad\xi_t=1/t^2,
\end{split}
\end{equation}
where $\beta_t$ follows the following distribution:
\begin{equation}
p{\left(\beta_t\right)}=\begin{cases}
\frac{1}{2}\exp{\left(-\frac{\beta_t-s_t}{2}\right)}&\left(\beta_t\ge s_t\right),\\
0&\left(\beta_t<s_t\right),
\end{cases}\qquad s_t=2d\log{\left\lceil dr/l_t\right\rceil}-2\log2,
\end{equation}
where $l_t=\left(bt^2\left(\sqrt{\log{\left(ad\right)}}+\sqrt{\pi}/2\right)\right)^{-1}$.
This result is given in the proof of Theorem 4.3 of \citep{Takeno2023-randomized}.
\subsection{PIMS}
In PIMS, the algorithm $\mathcal A$ is defined as
\begin{equation}
\mathcal A{\left(\mathcal D_{t-1}\right)}=\mathop{\rm argmax}\limits_{\bm x\in\mathcal X}1-\Phi{\left(\frac{g_t^*-\mu{\left(\bm x;\mathcal D_{t-1}\right)}}{\sigma{\left(\bm x;\mathcal D_{t-1}\right)}}\right)},
\end{equation}
where $g_t\sim p{\left(f\mid\mathcal D_{t-1}\right)}$ and $g_t^*=\max_{\bm x\in\mathcal X}g_t{\left(\bm x\right)}$.
If $\mathcal X$ is finite, Condition \ref{cnd_alg} holds for the following bounds:
\begin{equation}
\begin{split}
&u_t\left(\mathcal D_{t-1}\right)=\frac{g_t^*-\mu{\left(\bm x_t;\mathcal D_{t-1}\right)}}{\sigma{\left(\bm x_t;\mathcal D_{t-1}\right)}},\quad v_t\left(\mathcal D_{t-1}\right)=0,\\
&\zeta_t=2+2\log{\left(\left|\mathcal X\right|/2\right)},\quad\xi_t=0.
\end{split}
\end{equation}
This result is given in the proof of Theorem 4.1 of \citep{takeno2024-posterior}.

If $\mathcal X$ is infinite, Condition \ref{cnd_alg} holds for the following bounds:
\begin{equation}
\begin{split}
&u_t\left(\mathcal D_{t-1}\right)=\frac{{\tilde g}_t^*-\mu{\left(\tilde {\bm x}_t;\mathcal D_{t-1}\right)}}{\sigma{\left(\tilde {\bm x}_t;\mathcal D_{t-1}\right)}},\quad v_t\left(\mathcal D_{t-1}\right)=\frac{g_t^*-g_t{\left(\left[\bm z_t^*\right]_t\right)}}{\sigma{\left(\tilde {\bm x}_t;\mathcal D_{t-1}\right)}}\sigma{\left(\bm x_t;\mathcal D_{t-1}\right)},\\
&\zeta_t=2+2d\log{\left\lceil dr/l_t\right\rceil}-2\log2,\quad\xi_t=1/t^2,
\end{split}
\end{equation}
where ${\tilde g}_t^*$, $\bm z_t^*$, $\tilde {\bm x}_t$, and $l_t$ are defined as
\begin{equation}
\begin{split}
&{\tilde g}_t^*=\max_{\bm x\in\mathcal X_t}g_t{\left(\bm x\right)},\quad\bm z_t^*=\mathop{\rm argmax}\limits_{\bm x\in\mathcal X}g_t{\left(\bm x\right)},\quad\tilde {\bm x}_t=\mathop{\rm argmin}\limits_{\bm x\in\mathcal X_t}\frac{{\tilde g}_t^*-\mu{\left(\bm x;\mathcal D_{t-1}\right)}}{\sigma{\left(\bm x;\mathcal D_{t-1}\right)}},\\
&l_t=\left(bt^2\left(\sqrt{\log{\left(ad\right)}}+\sqrt{\pi}/2\right)\sqrt{\left(\sigma_{\rm noise}^2+t-1\right)/\sigma_{\rm noise}^2}\right)^{-1}.
\end{split}
\end{equation}
This result is given in the proof of Theorem 4.2 of \citep{takeno2024-posterior}.
\subsection{EIMS}
In EIMS, the algorithm $\mathcal A$ is defined as
\begin{equation}
\mathcal A{\left(\mathcal D_{t-1}\right)}=\mathop{\rm argmax}\limits_{\bm x\in\mathcal X}\mathbb E\left[\max{\left\{f{\left(\bm x\right)}-g_t^*,0\right\}}\mid\mathcal D_{t-1}, g^*_t \right],
\end{equation}
where $g_t\sim p{\left(f\mid\mathcal D_{t-1}\right)}$ and $g_t^*=\max_{\bm x\in\mathcal X}g_t{\left(\bm x\right)}$.
If $\mathcal X$ is finite, Condition \ref{cnd_alg} holds for the following bounds:
\begin{equation}
\begin{split}
&u_t\left(\mathcal D_{t-1}\right)=\frac{g_t^*-\mu{\left(\bm x_t;\mathcal D_{t-1}\right)}}{\sigma{\left(\bm x_t;\mathcal D_{t-1}\right)}},\quad v_t\left(\mathcal D_{t-1}\right)=0,\\
&\zeta_t=\log{\left(\frac{\sigma_{\rm noise}^2+t-1}{\sigma_{\rm noise}^2}\right)}+C_2+\sqrt{2\pi C_2},\quad\xi_t=0.
\end{split}
\end{equation}
where $C_2=2+2\log{\left(\left|\mathcal X\right|/2\right)}$.
This result is given in the proof of Theorem 4.6 of \citep{takeno2025-regretEI}.

If $\mathcal X$ is infinite, Condition \ref{cnd_alg} holds for the following bounds:
\begin{equation}
\begin{split}
&u_t\left(\mathcal D_{t-1}\right)=\frac{g_t^*-\mu{\left(\bm x_t;\mathcal D_{t-1}\right)}}{\sigma{\left(\bm x_t;\mathcal D_{t-1}\right)}},\quad v_t\left(\mathcal D_{t-1}\right)=0,\\
&\zeta_t=\log{\left(\frac{\sigma_{\rm noise}^2+t-1}{\sigma_{\rm noise}^2}\right)}+C_t+\sqrt{2\pi C_t},\quad\xi_t=0,
\end{split}
\end{equation}
where $C_t$ is defined as
\begin{equation}
C_t=8\left(1+d\log m_t\right),\quad m_t=\max{\left\{2,bdr\sqrt{\left(\sigma_{\rm noise}^2+t-1\right)\log{\left(2ad\right)}/\sigma_{\rm noise}^2}\right\}}.
\end{equation}
This result is given in the proof of Theorem 4.8 of \citep{takeno2025-regretEI}.
\subsection{TS}
In TS, the algorithm $\mathcal A$ is defined as
\begin{equation}
\mathcal A{\left(\mathcal D_{t-1}\right)}=\mathop{\rm argmax}\limits_{\bm x\in\mathcal X}g_t{\left(\bm x\right)},\quad g_t\sim p{\left(f\mid\mathcal D_{t-1}\right)}.
\end{equation}
If $\mathcal X$ is finite, Condition \ref{cnd_alg} holds for the following bounds:
\begin{equation}
\begin{split}
&u_t\left(\mathcal D_{t-1}\right)=\frac{g_t{\left(\bm x_t\right)}-\mu{\left(\bm x_t;\mathcal D_{t-1}\right)}}{\sigma{\left(\bm x_t;\mathcal D_{t-1}\right)}},\quad v_t\left(\mathcal D_{t-1}\right)=0,\\
&\zeta_t=2+2\log{\left(\left|\mathcal X\right|/2\right)},\quad\xi_t=0.
\end{split}
\end{equation}
This result is given in the proof of Theorem 3.1 of \citep{takeno2024-posterior}.

If $\mathcal X$ is infinite, Condition \ref{cnd_alg} holds for the following bounds:
\begin{equation}
\begin{split}
&u_t\left(\mathcal D_{t-1}\right)=\frac{g_t{\left(\left[\bm x_t\right]_t\right)}-\mu{\left(\left[\bm x_t\right]_t;\mathcal D_{t-1}\right)}}{\sigma{\left(\left[\bm x_t\right]_t;\mathcal D_{t-1}\right)}},\\
&v_t\left(\mathcal D_{t-1}\right)=f{\left(\bm x^*\right)}-f{\left(\left[\bm x^*\right]_t\right)}+f{\left(\left[\bm x_t\right]_t\right)}-f{\left(\bm x_t\right)}+u_t\left(\mathcal D_{t-1}\right)\left|\sigma{\left(\left[\bm x_t\right]_t;\mathcal D_{t-1}\right)}-\sigma{\left(\bm x_t;\mathcal D_{t-1}\right)}\right|,\\
&\zeta_t=2+2d\log{\left(\lceil dr/l_t\rceil\right)}-2\log2,\quad\xi_t=\left(2+\sqrt{\zeta_t}\right)/t^2,
\end{split}
\end{equation}
where $l_t$ is defined using $L_\sigma$ in Lemma \ref{kusakawa22_e_4} as
\begin{equation}
l_t=\frac{1}{Lt^2},\quad L=\max{\left\{L_\sigma,b\left(\sqrt{\log{\left(ad\right)}}+\sqrt{\pi}/2\right)\right\}}.
\end{equation}
This result is given in the proof of Theorem 3.2 of \citep{takeno2025-regretEI}.

%% file: manuscripts/9_appendix_fig.tex
\section{Result of additional experiments}
\label{apn:result}
We present the results of experiments in synchronous and asynchronous settings with various numbers of workers, $Q=\left\{4,8,16\right\}$.
\paragraph{Setting of synchronous experiments.}
The setting is the same as that of Section \ref{sec:experiment}, except for $Q$.

\paragraph{Setting of asynchronous experiments.}
We assumed that each evaluation takes time that follows an independent half-normal distribution.
We report the change in the simple regret $f(\bm x^*)-\max_{i\in\left[t\right]}f{\left(\bm x_i\right)}$ or the best objective value $\max_{i\in\left[t\right]}f{\left(\bm x_i\right)}$ over time instead of iterations.
Each experiment was conducted 100 times under different conditions, and each experiment was terminated after a certain number of evaluations.
Therefore, we present the results up to the earliest end time across the 100 experimental trials.
Other settings are the same as those of the experiments in the synchronous setting.
\paragraph{Result.}
The experimental results are presented in the following figures: for synthetic objectives, Figs. \ref{fig:result_syn_syn} and \ref{fig:result_syn_asyn}; for benchmark objectives, Figs. \ref{fig:result_ben_syn} and \ref{fig:result_ben_asyn}; and for emulators, Figs. \ref{fig:result_eml_syn}, \ref{fig:result_eml_asyn}, \ref{fig:result_eml_syn2}, and \ref{fig:result_eml_asyn2}.
Moreover, Fig. \ref{fig:result_comparison_apn} summarizes the results across all conditions, focusing on the improvement in the value of the objective function throughout the experiments.
As shown in Fig. \ref{fig:result_comparison_apn}, RKB-PIMS and RKB-UCB perform at least as well as other theoretically guaranteed methods, such as PTS and BUCB.
In particular, RKB-PIMS consistently demonstrates strong performance, even compared to qEI, a state-of-the-art method.
These results indicate that RKB has not only theoretical support but also strong practical performance.
\begin{figure}[t]
\centering
\includegraphics[width=.9\linewidth]{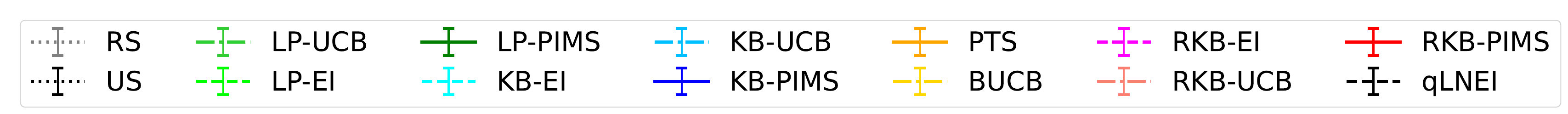}\\
\includegraphics[height=.25\linewidth]{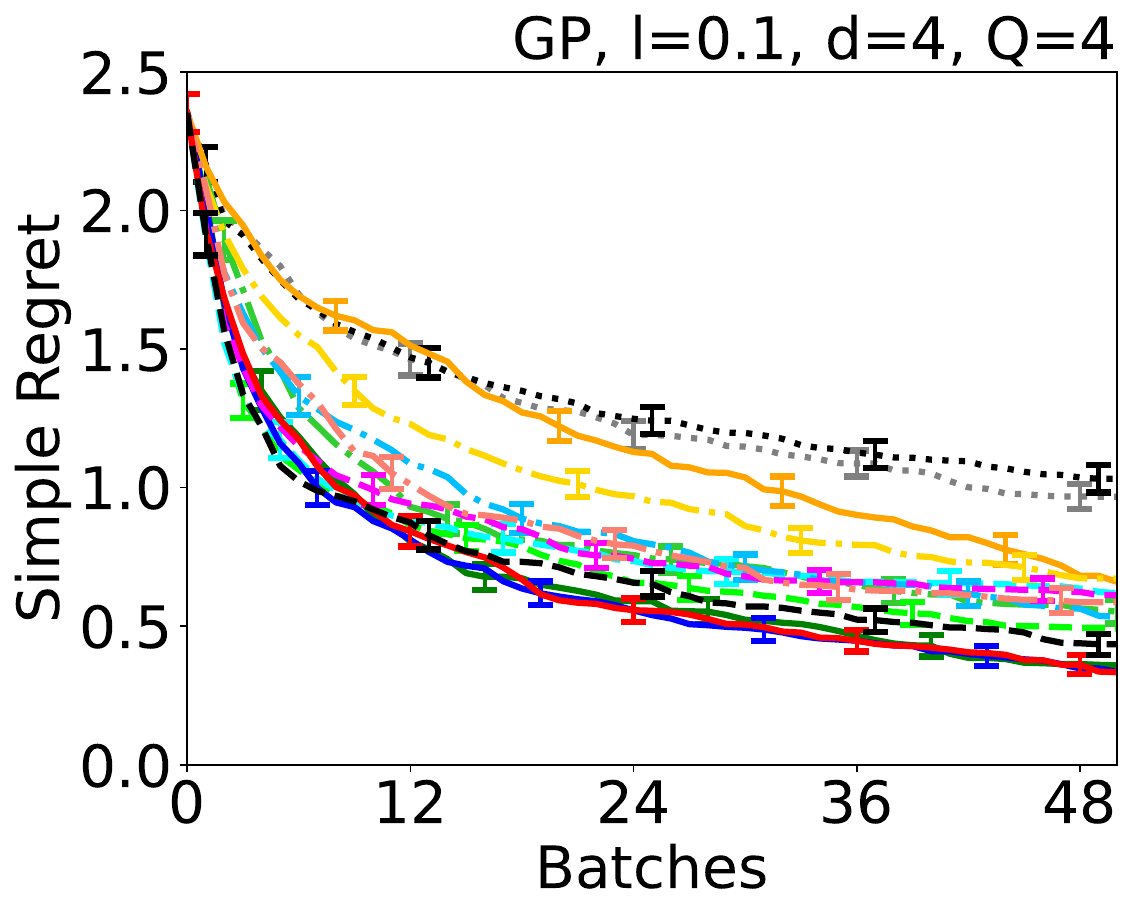}
\includegraphics[height=.25\linewidth]{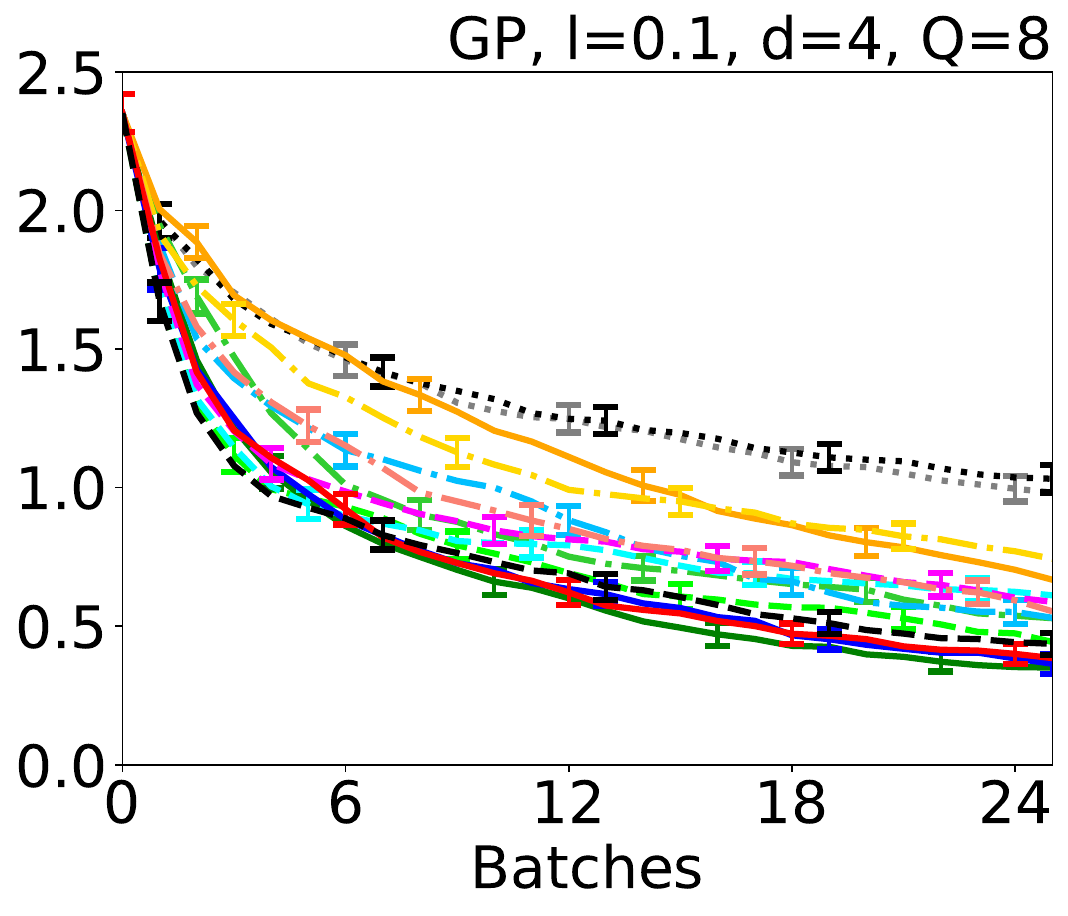}
\includegraphics[height=.25\linewidth]{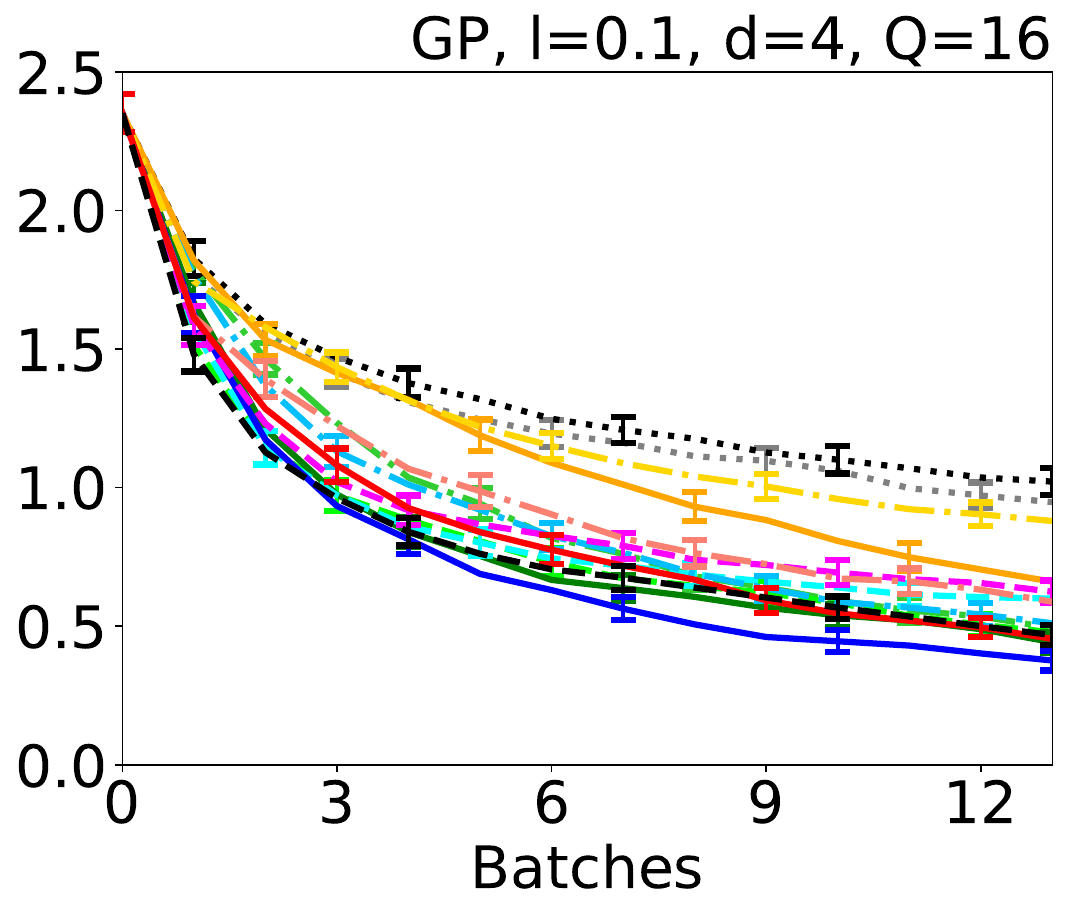}
\includegraphics[height=.25\linewidth]{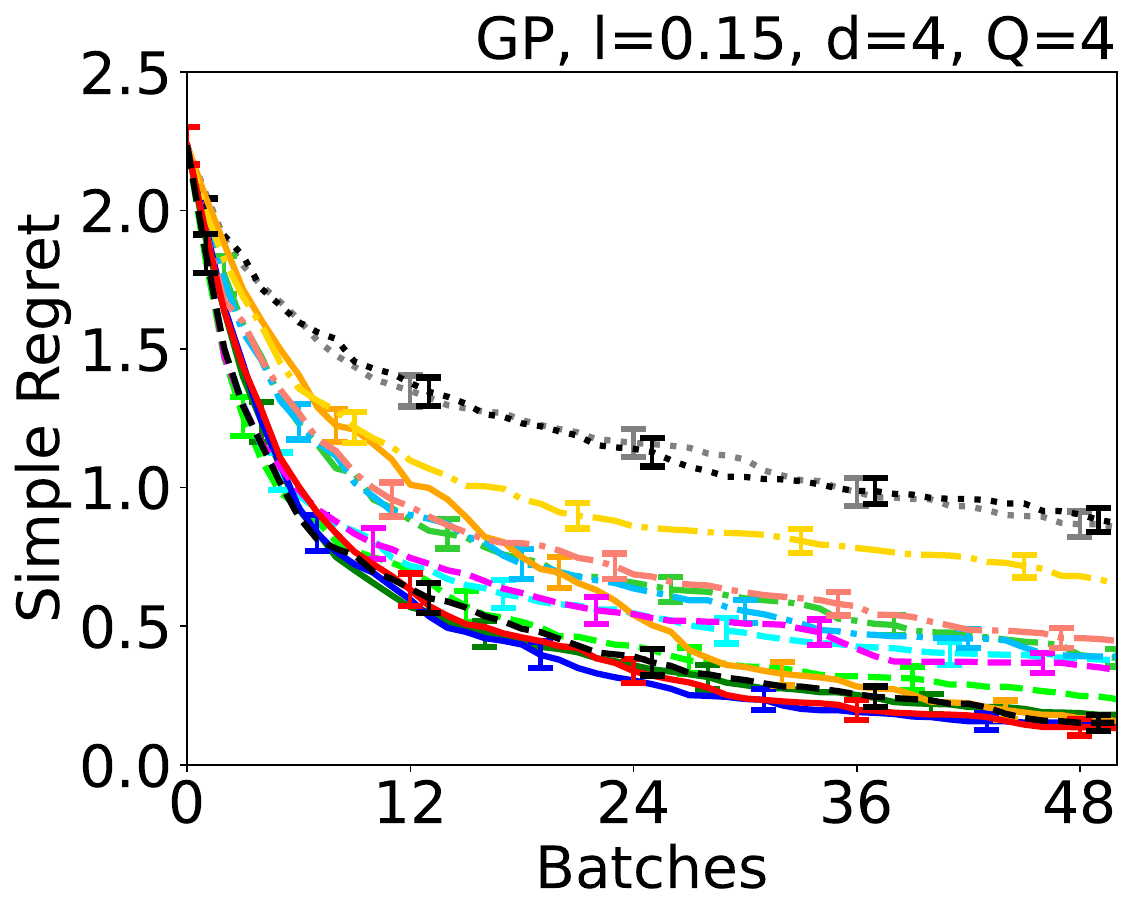}
\includegraphics[height=.25\linewidth]{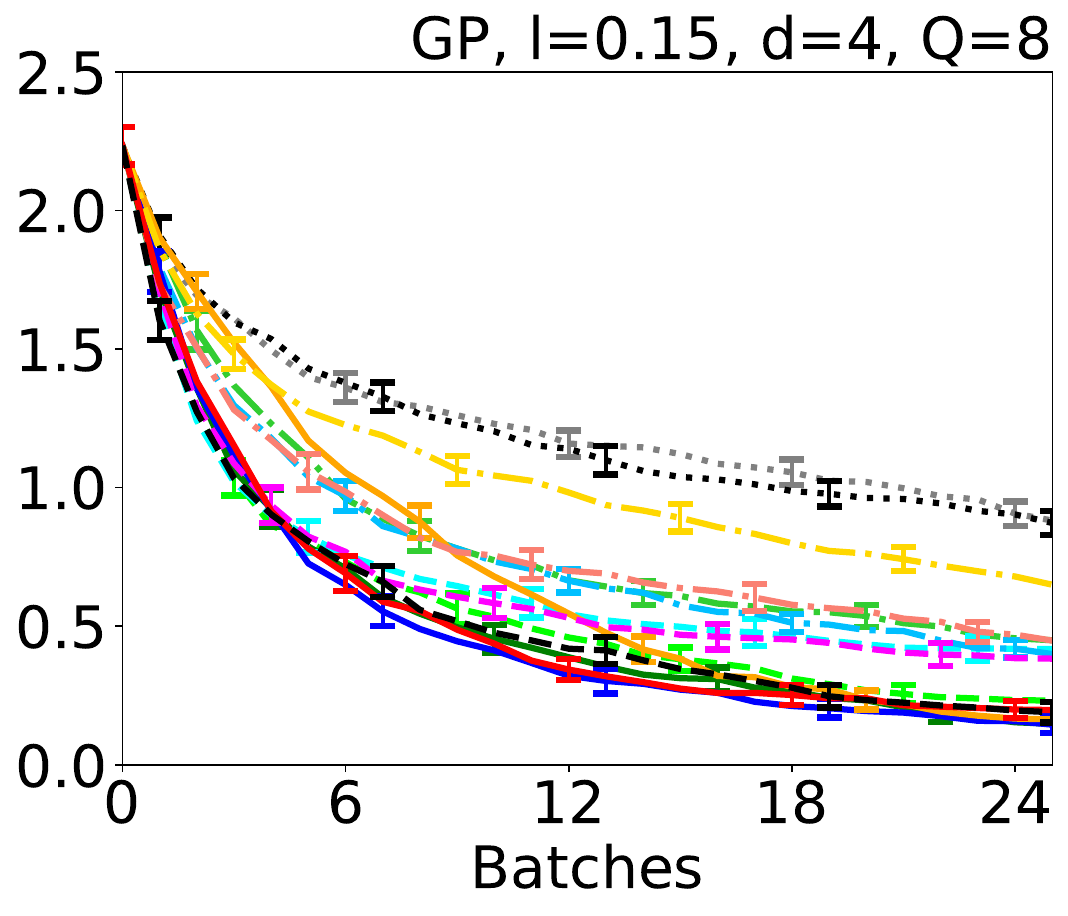}
\includegraphics[height=.25\linewidth]{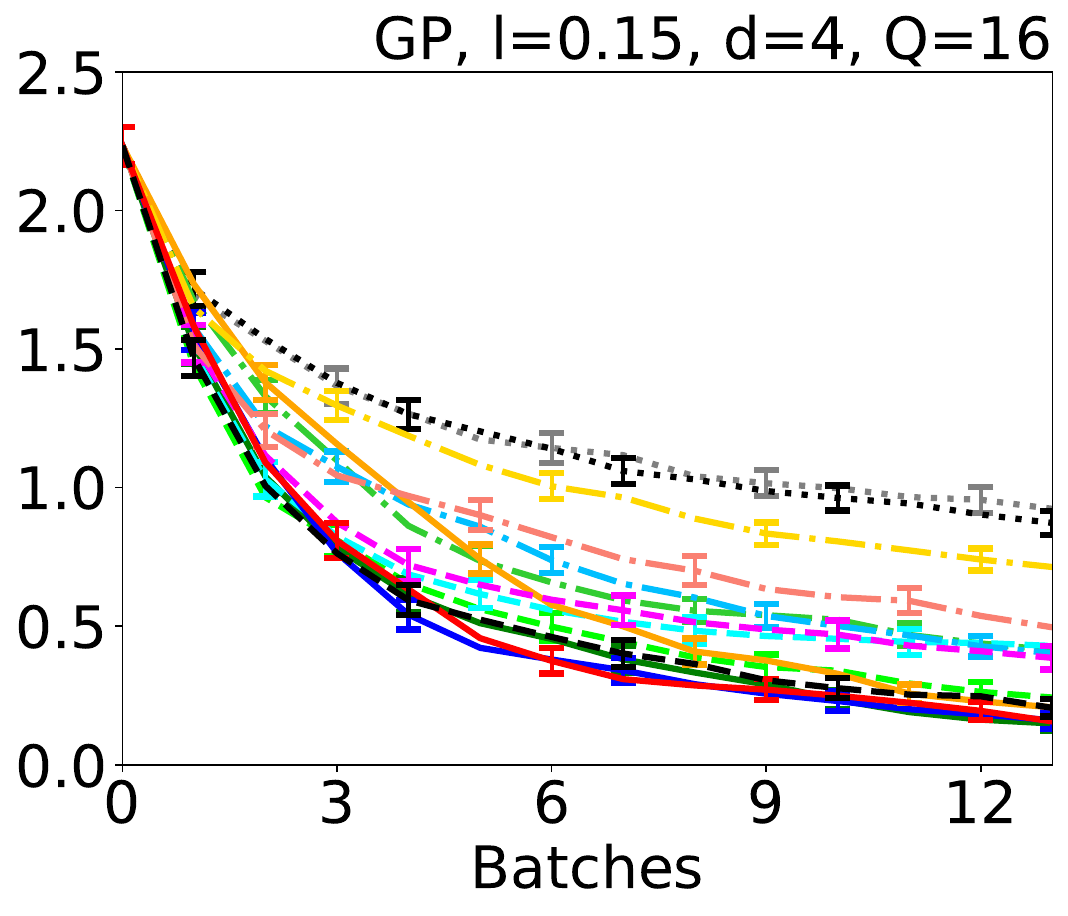}
\caption{
Result of experiments on synthetic objectives with synchronous setting. The lines and error bars mean average and standard error of the simple regret $f^*-\max_{i\in\left[t\right]}f{\left(\bm x_i\right)}$ across the 100 experiments on each condition. One batch corresponds to $Q$ iterations.
}
\label{fig:result_syn_syn}
\end{figure}

\begin{figure}[t]
\centering
\includegraphics[width=.9\linewidth]{figures/legend_syn.pdf}\\
\includegraphics[height=.25\linewidth]{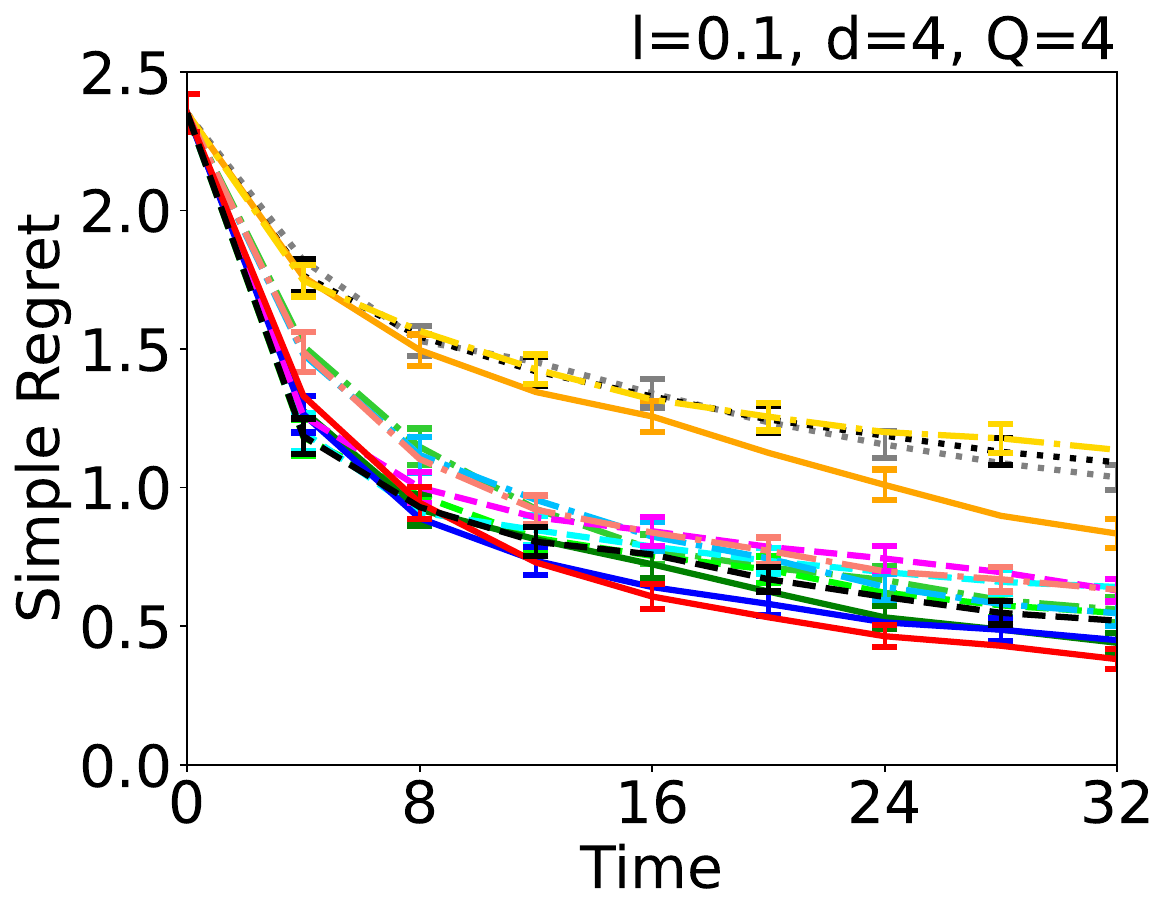}
\includegraphics[height=.25\linewidth]{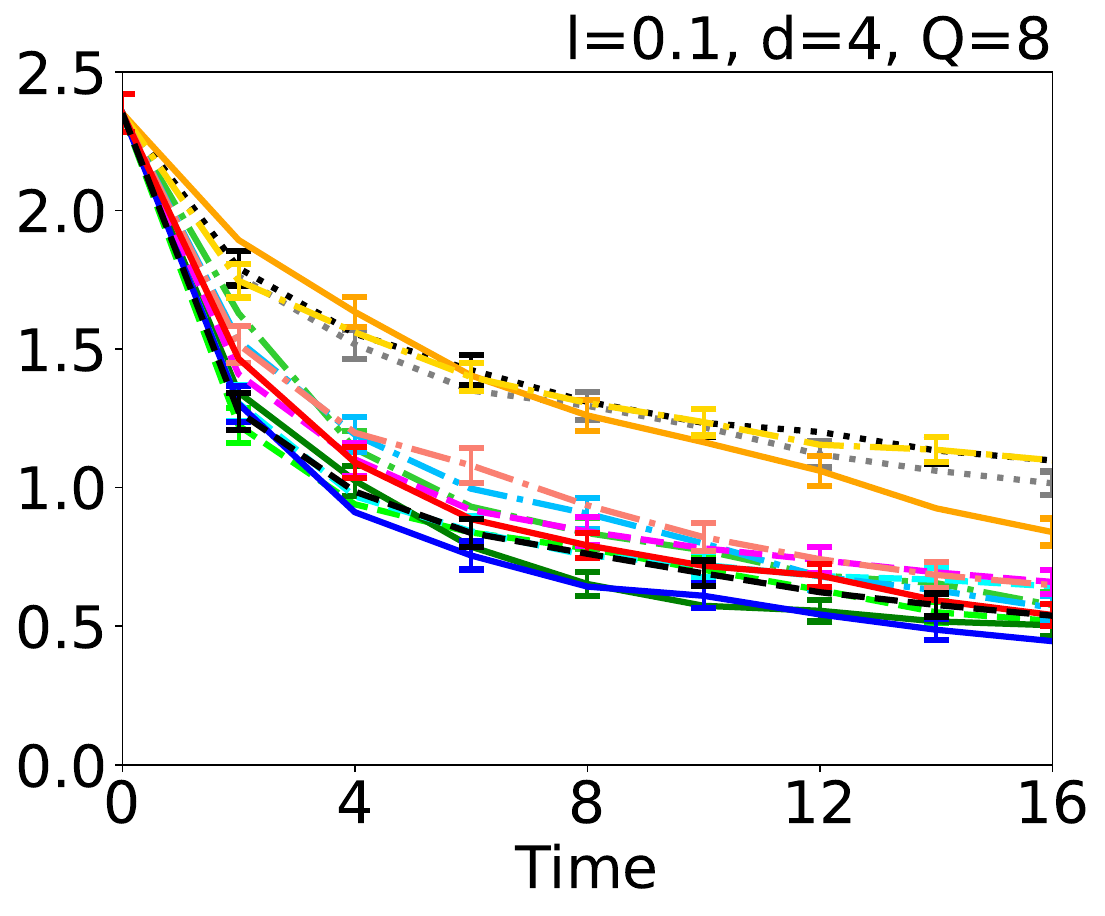}
\includegraphics[height=.25\linewidth]{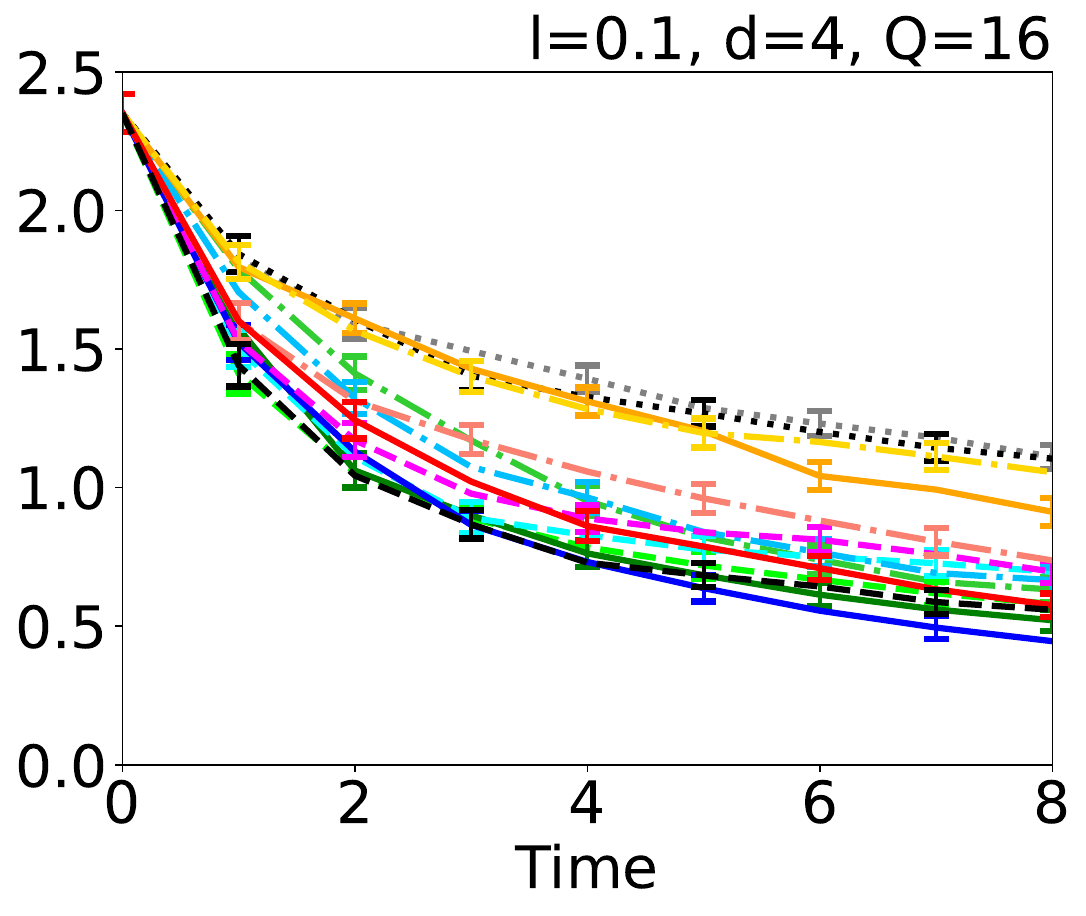}
\includegraphics[height=.25\linewidth]{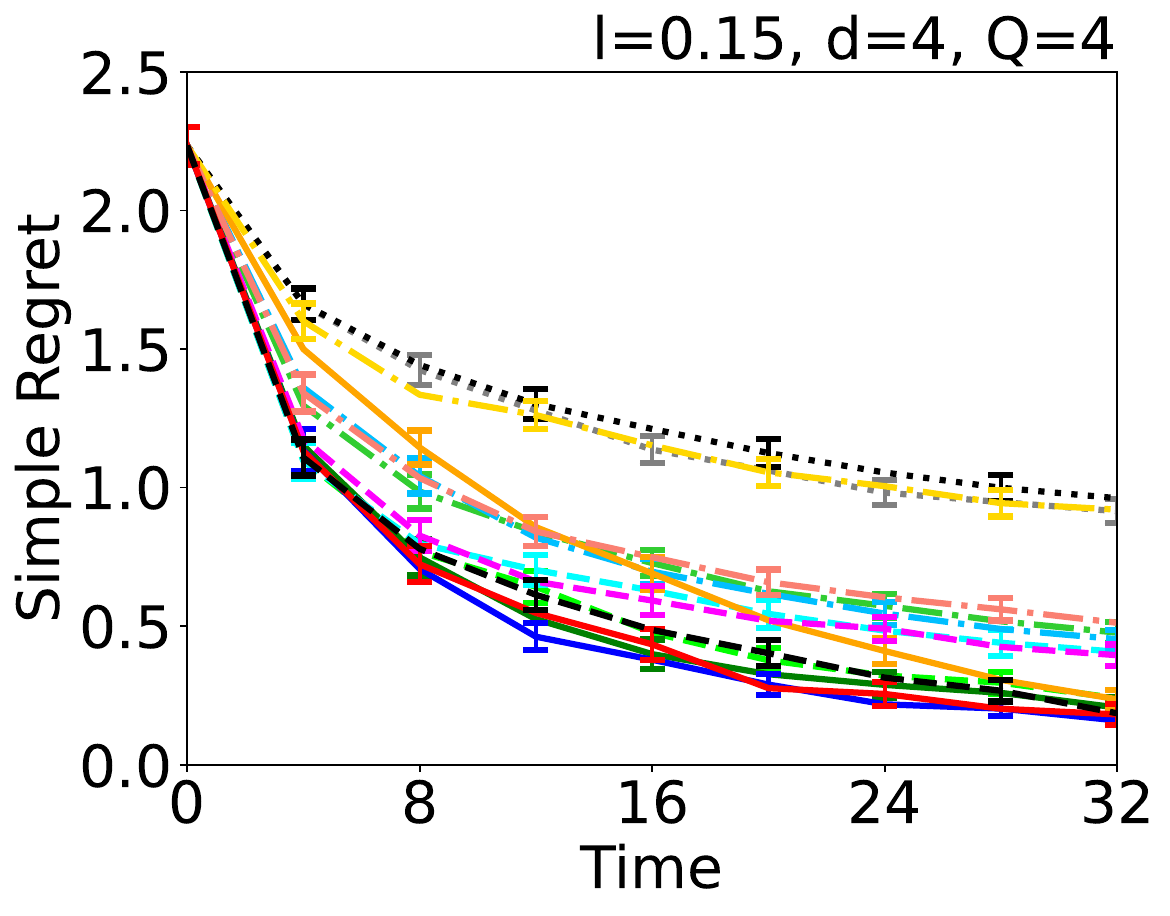}
\includegraphics[height=.25\linewidth]{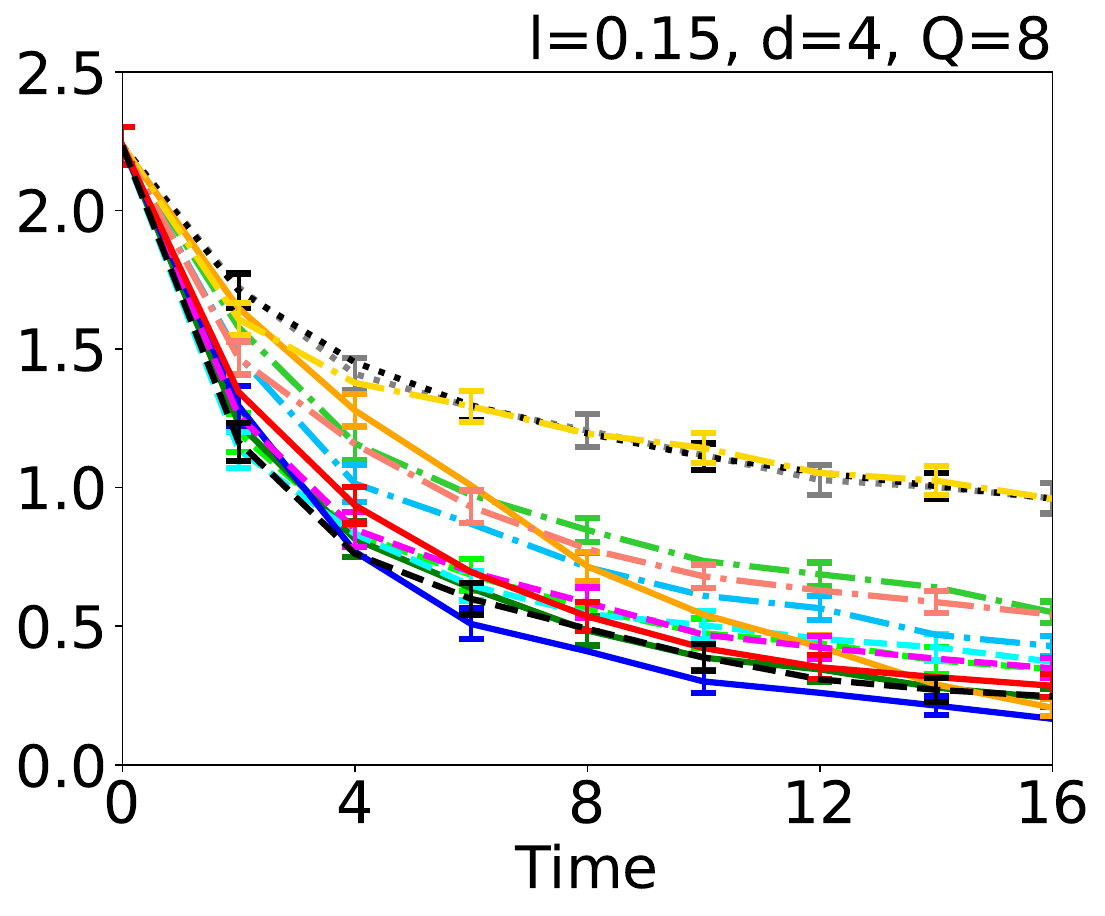}
\includegraphics[height=.25\linewidth]{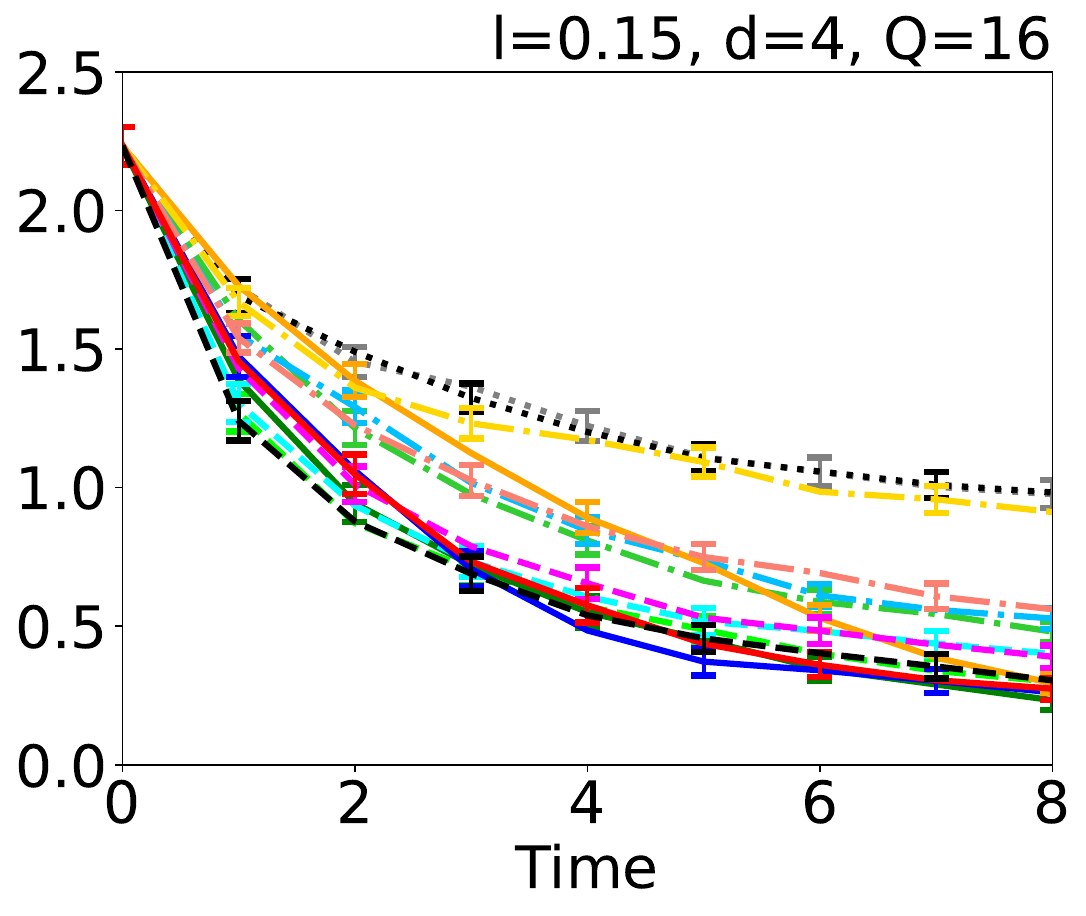}
\caption{
Result of experiments on synthetic objectives with asynchronous setting. The lines and error bars mean average and standard error of the simple regret $f^*-\max_{i\in\left[t\right]}f{\left(\bm x_i\right)}$ across the 100 experiments on each condition.
}
\label{fig:result_syn_asyn}
\end{figure}
\begin{figure}[t]
\centering
\includegraphics[width=.9\linewidth]{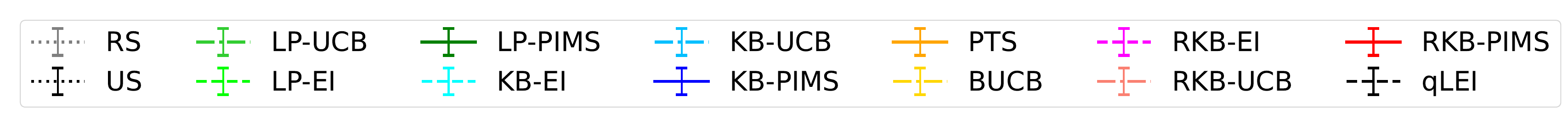}\\
\includegraphics[height=.25\linewidth]{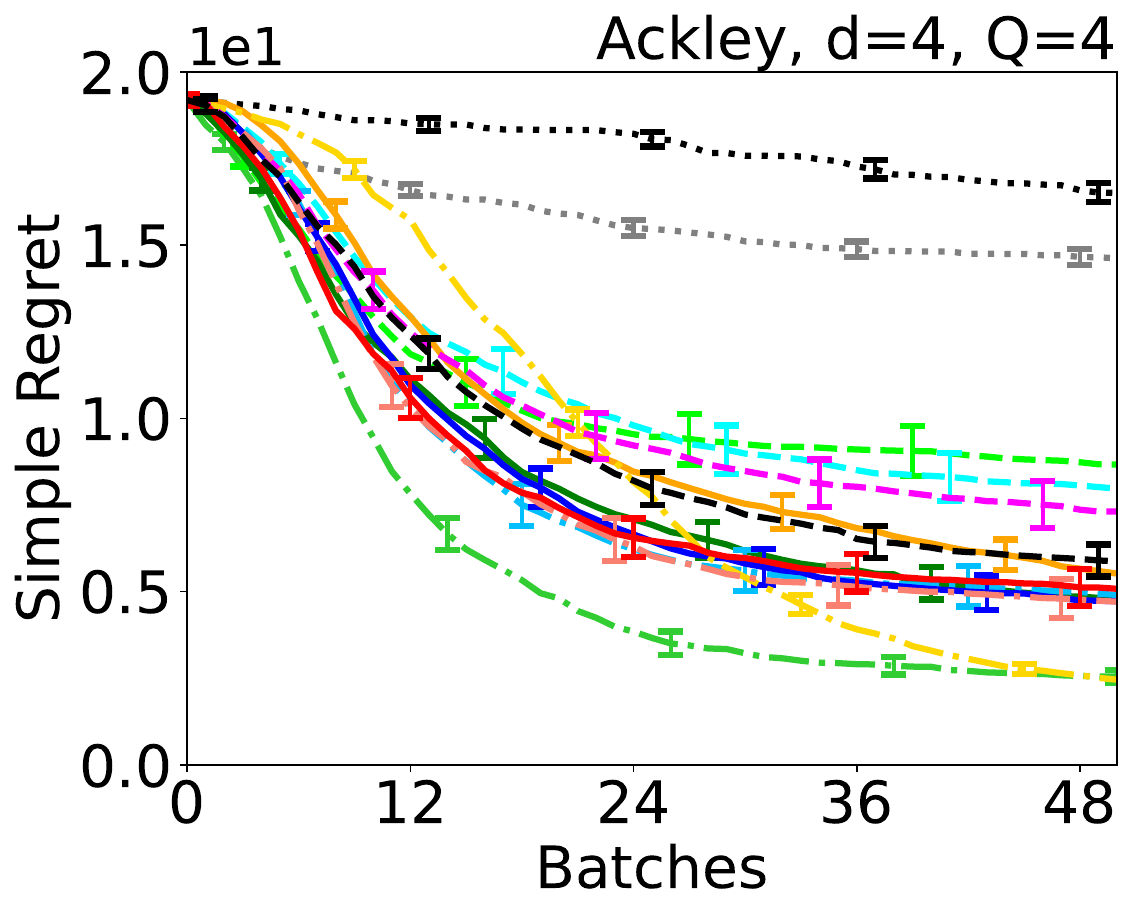}
\includegraphics[height=.25\linewidth]{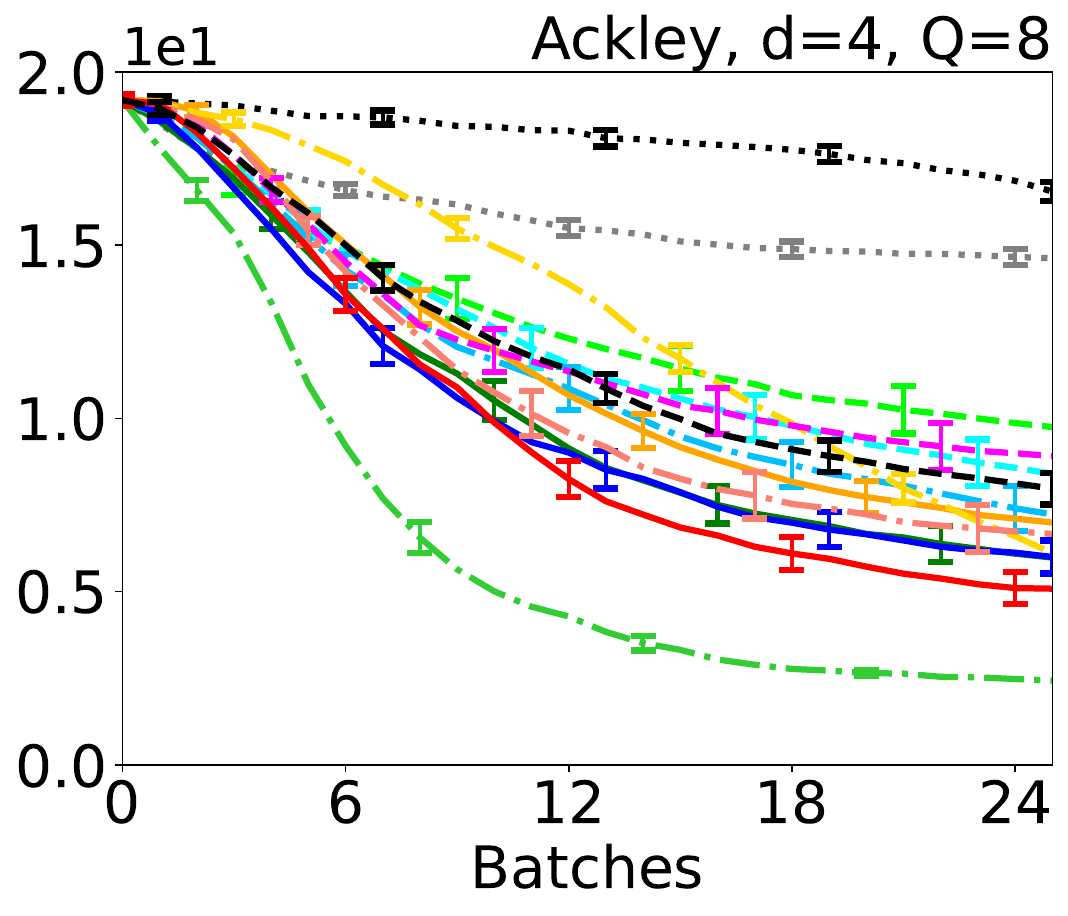}
\includegraphics[height=.25\linewidth]{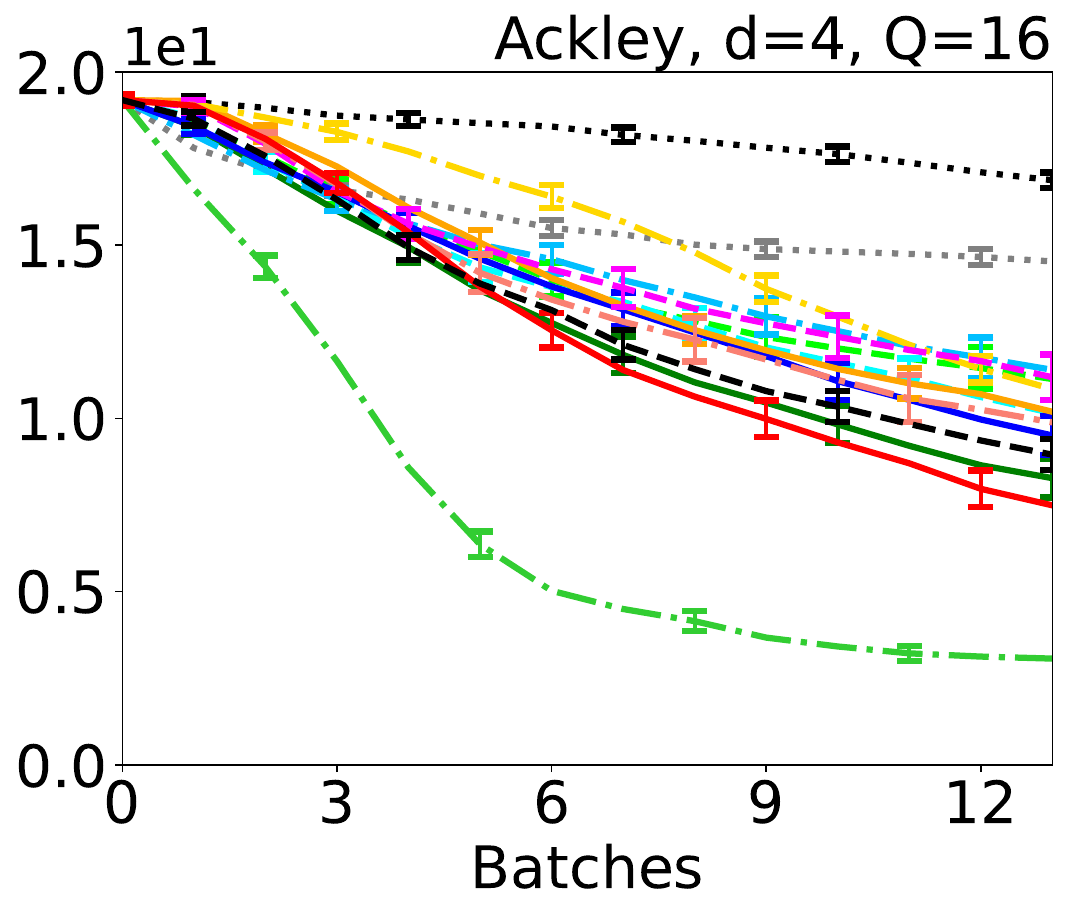}
\includegraphics[height=.25\linewidth]{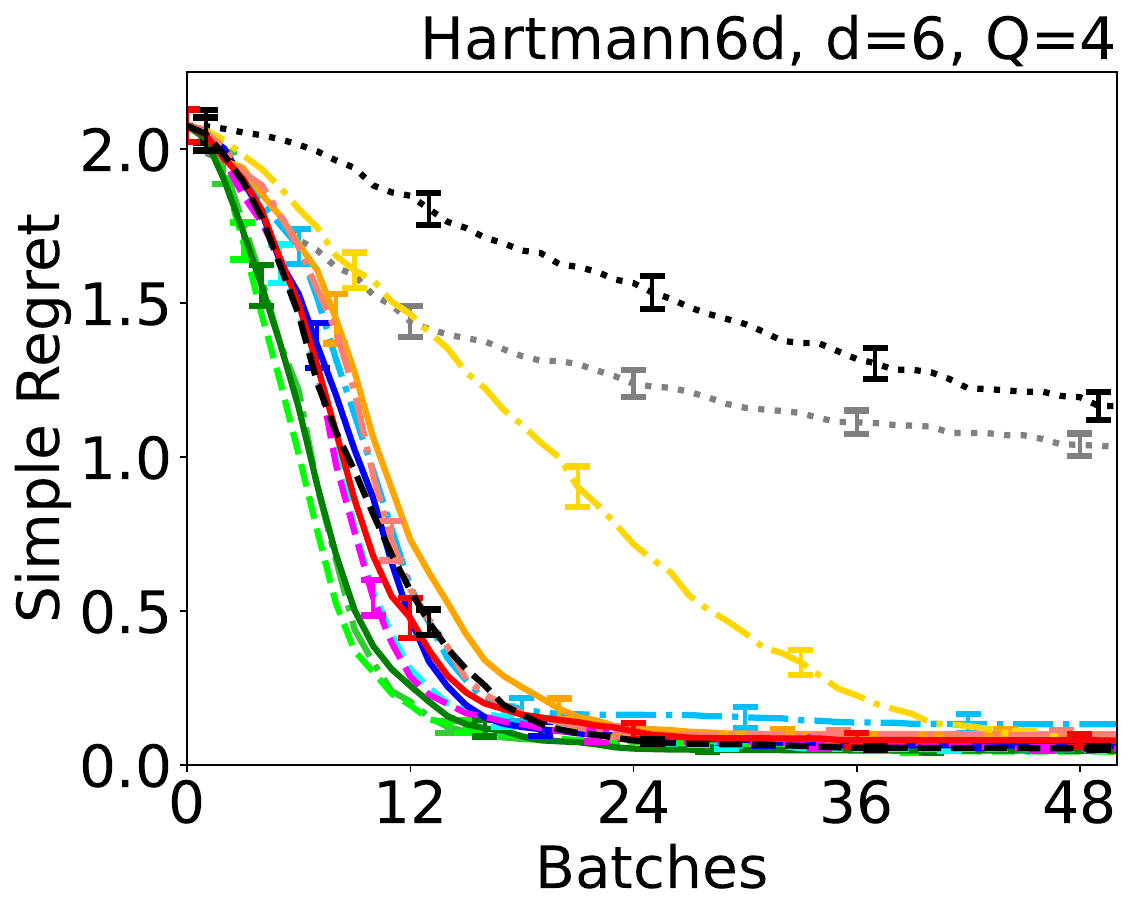}
\includegraphics[height=.25\linewidth]{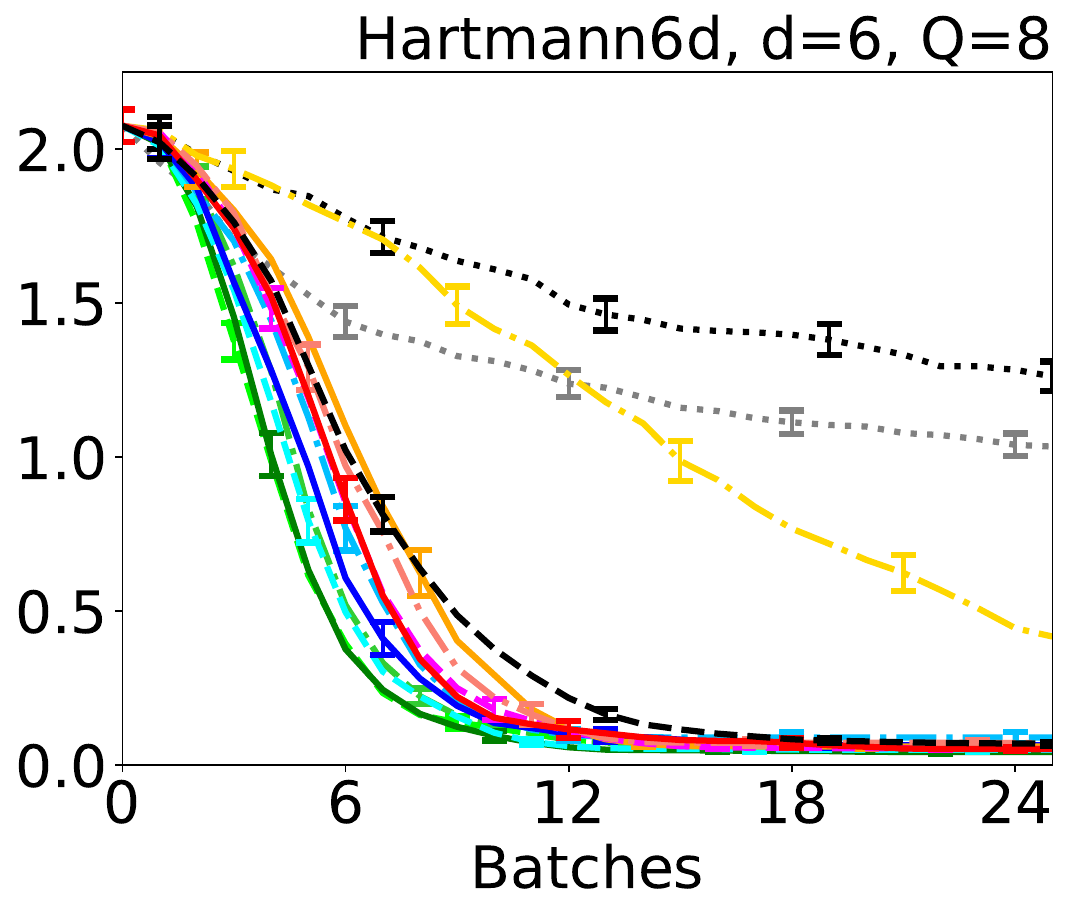}
\includegraphics[height=.25\linewidth]{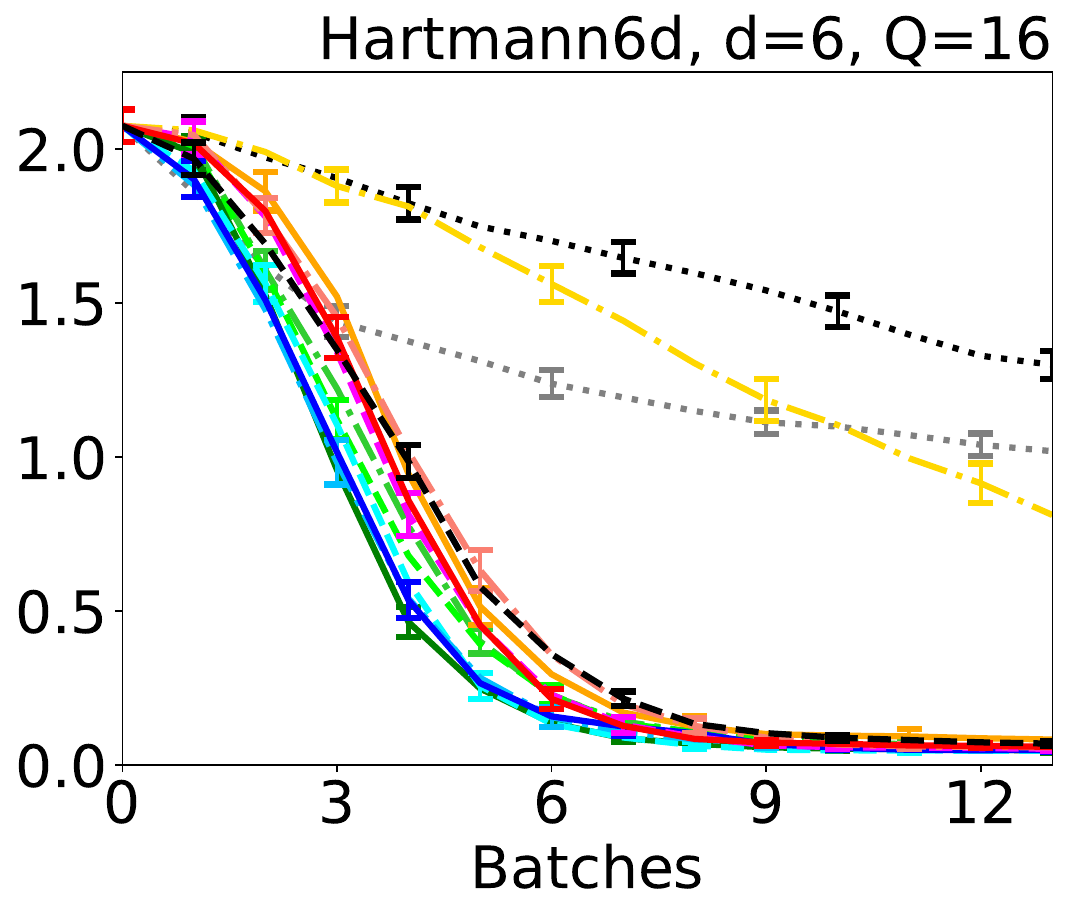}
\includegraphics[height=.25\linewidth]{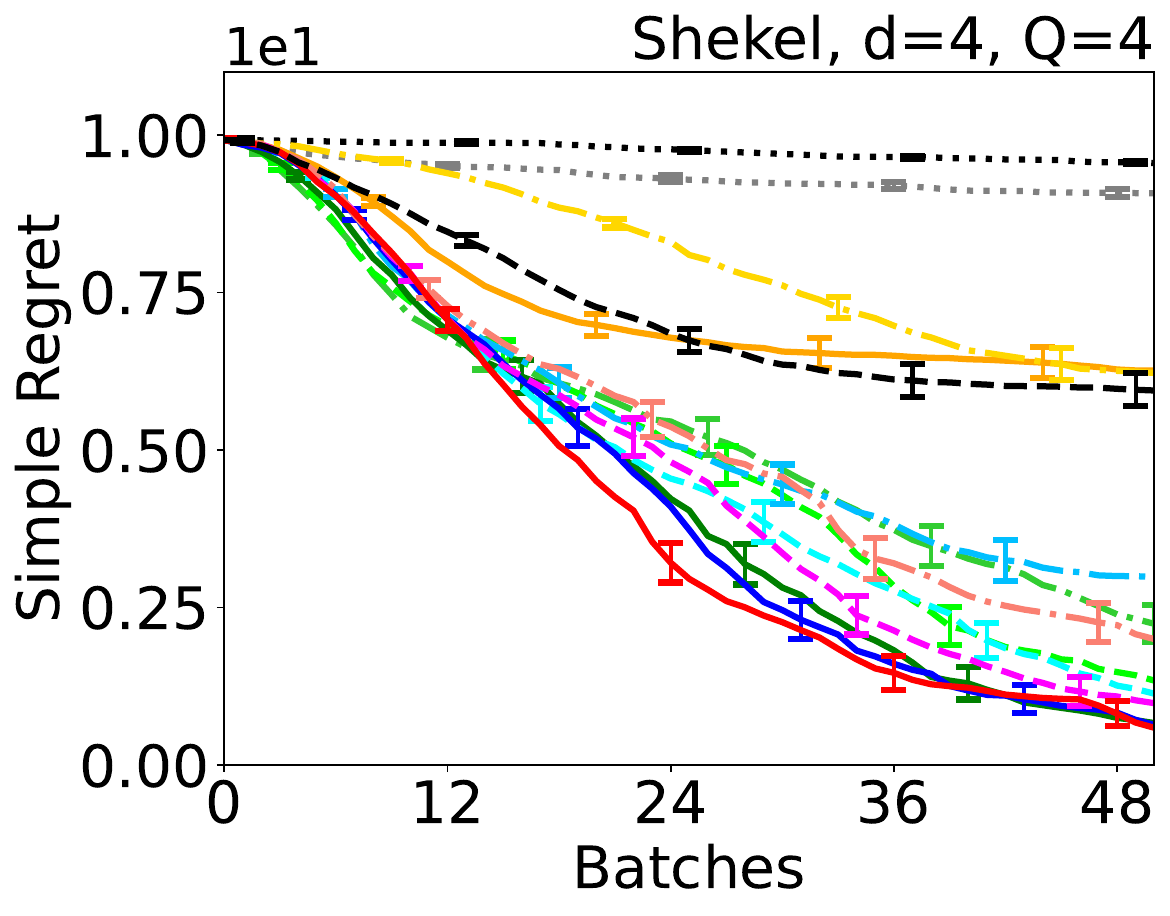}
\includegraphics[height=.25\linewidth]{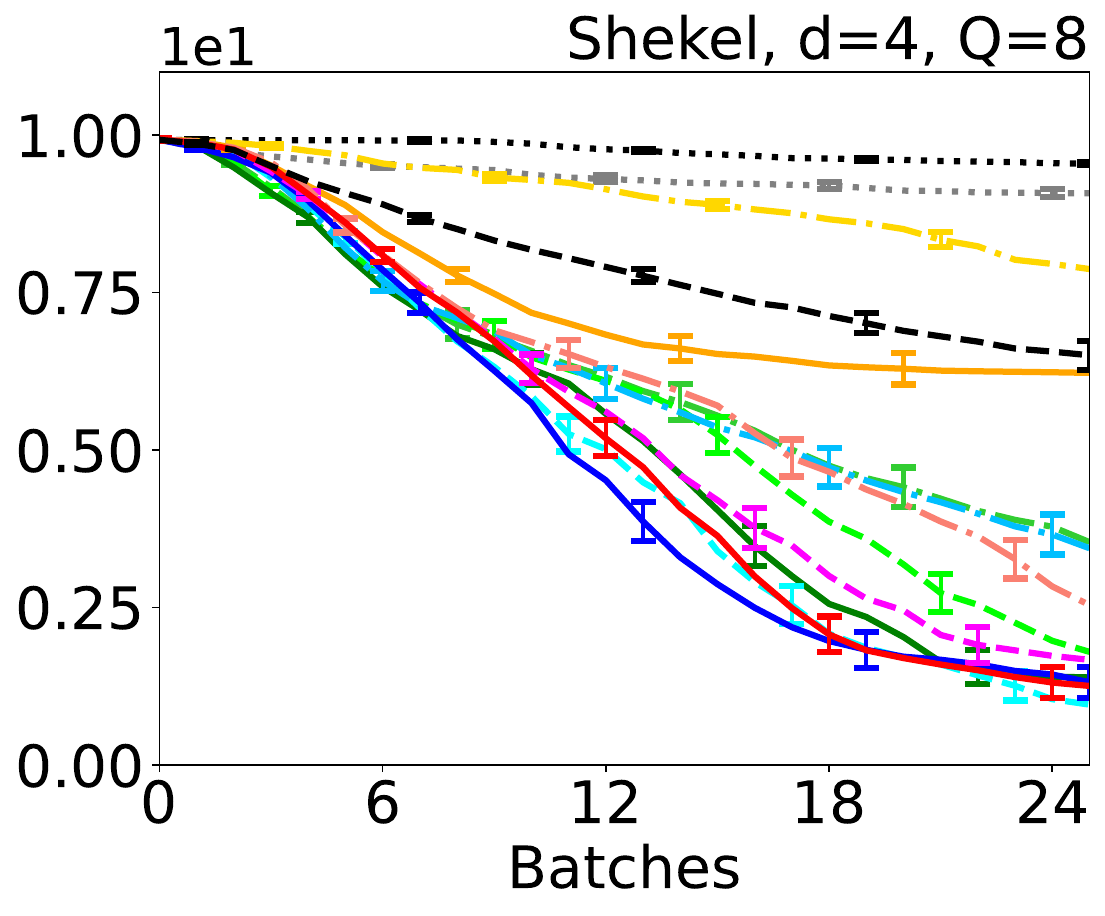}
\includegraphics[height=.25\linewidth]{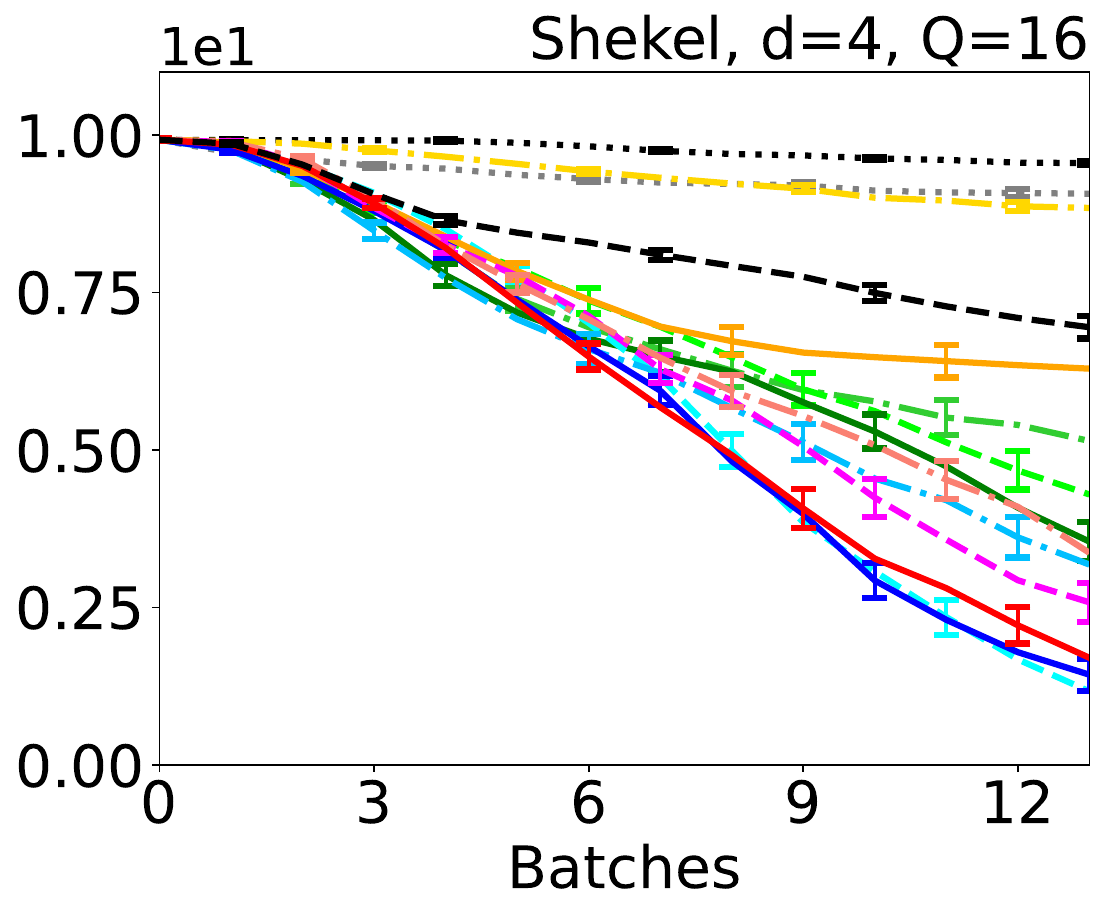}
\includegraphics[height=.25\linewidth]{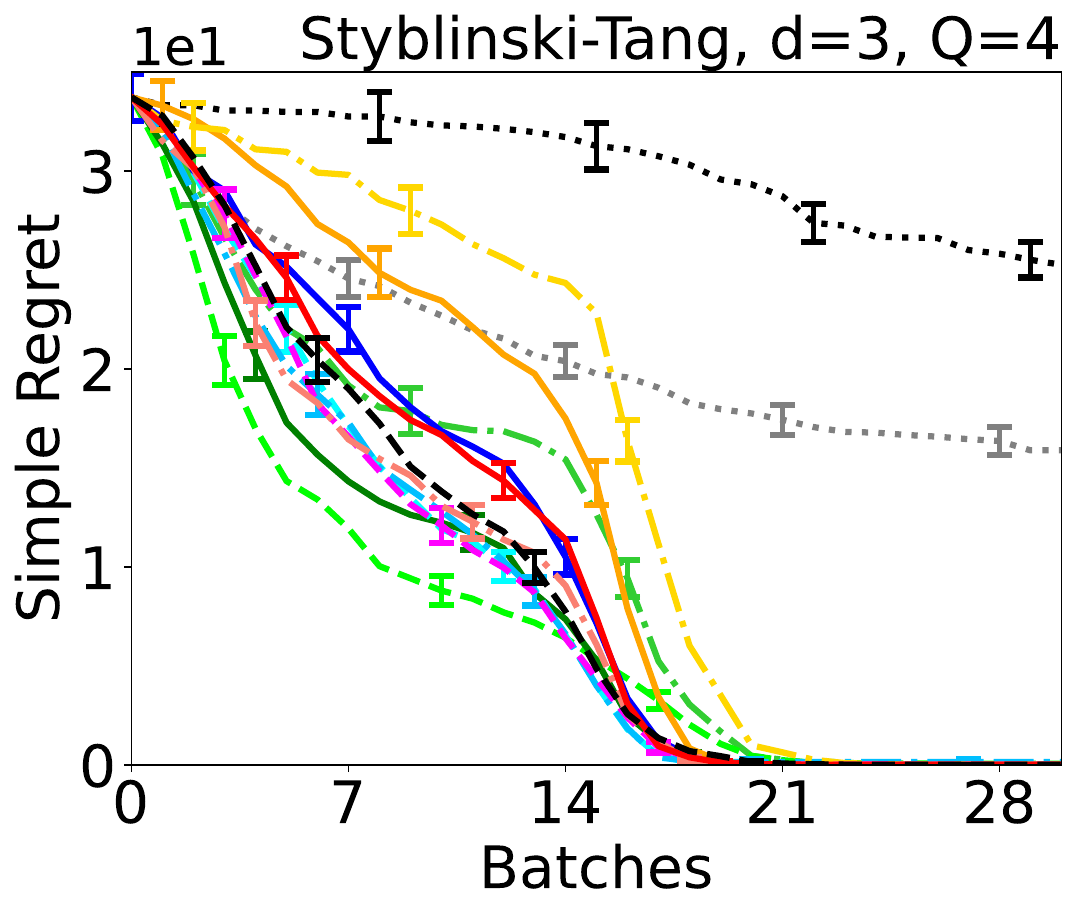}
\includegraphics[height=.25\linewidth]{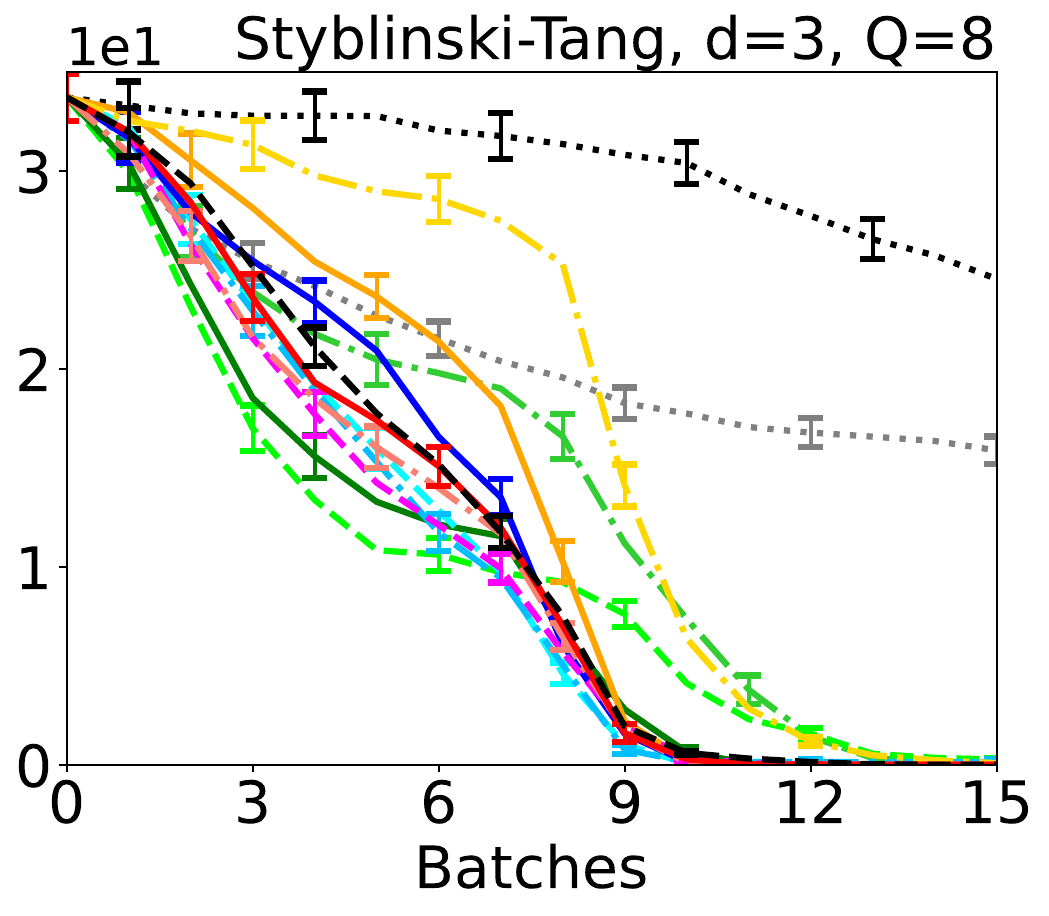}
\includegraphics[height=.25\linewidth]{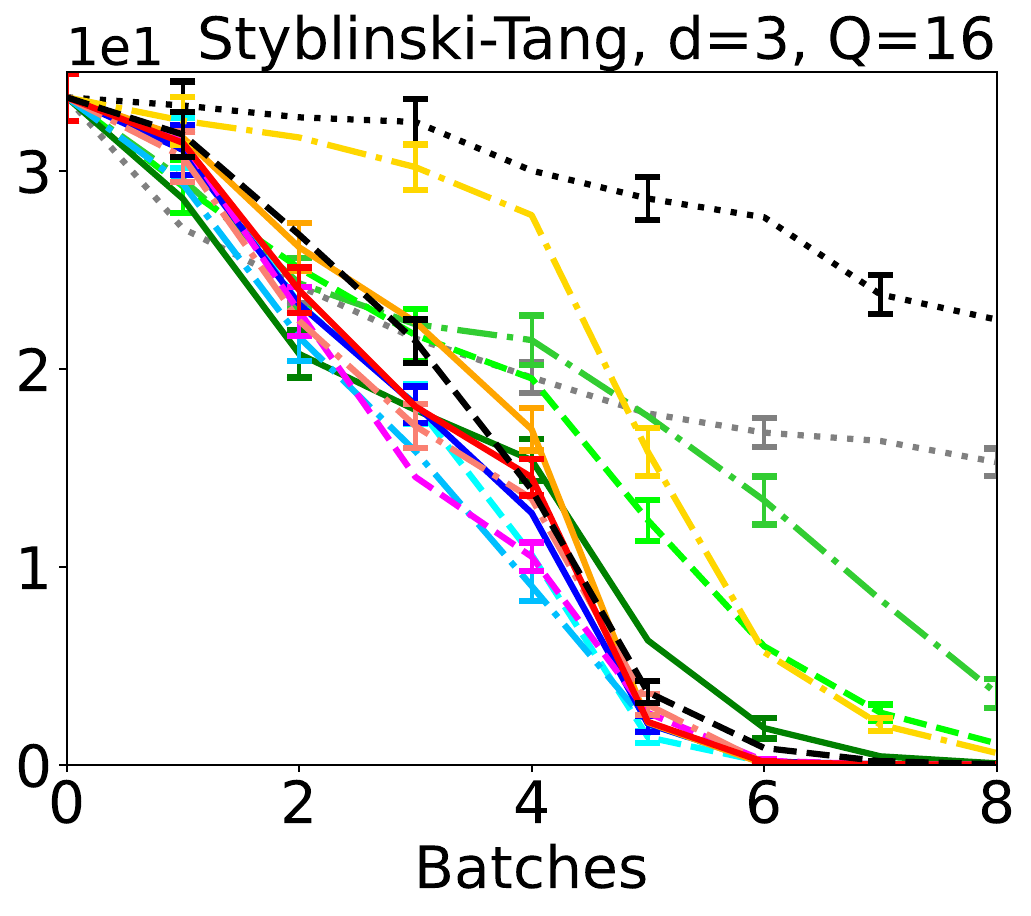}
\caption{
Result of experiments on benchmark functions with synchronous setting. The lines and error bars mean average and standard error of the simple regret $f^*-\max_{i\in\left[t\right]}f{\left(\bm x_i\right)}$ across the 100 experiments on each condition. One batch corresponds to $Q$ iterations.
}
\label{fig:result_ben_syn}
\end{figure}

\begin{figure}[t]
\centering
\includegraphics[width=.9\linewidth]{figures/legend_ben.pdf}\\
\includegraphics[height=.25\linewidth]{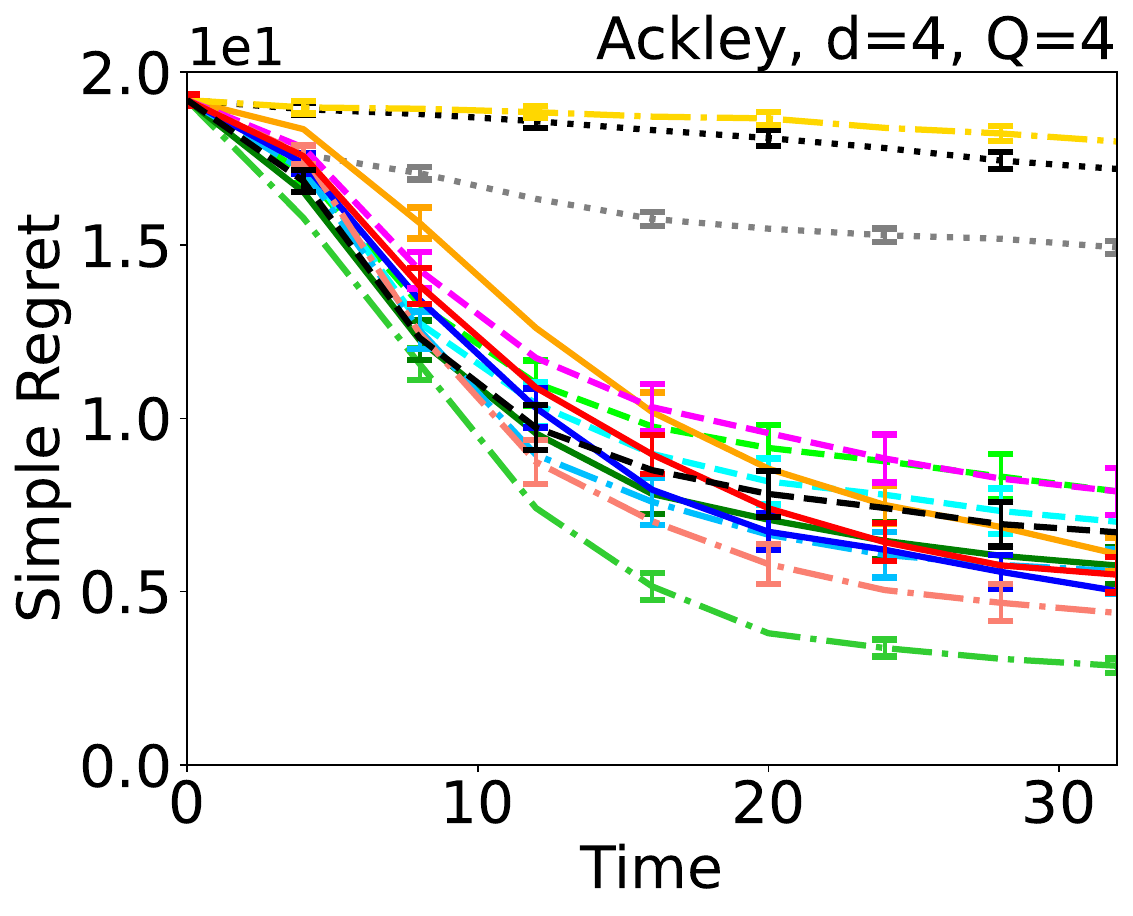}
\includegraphics[height=.25\linewidth]{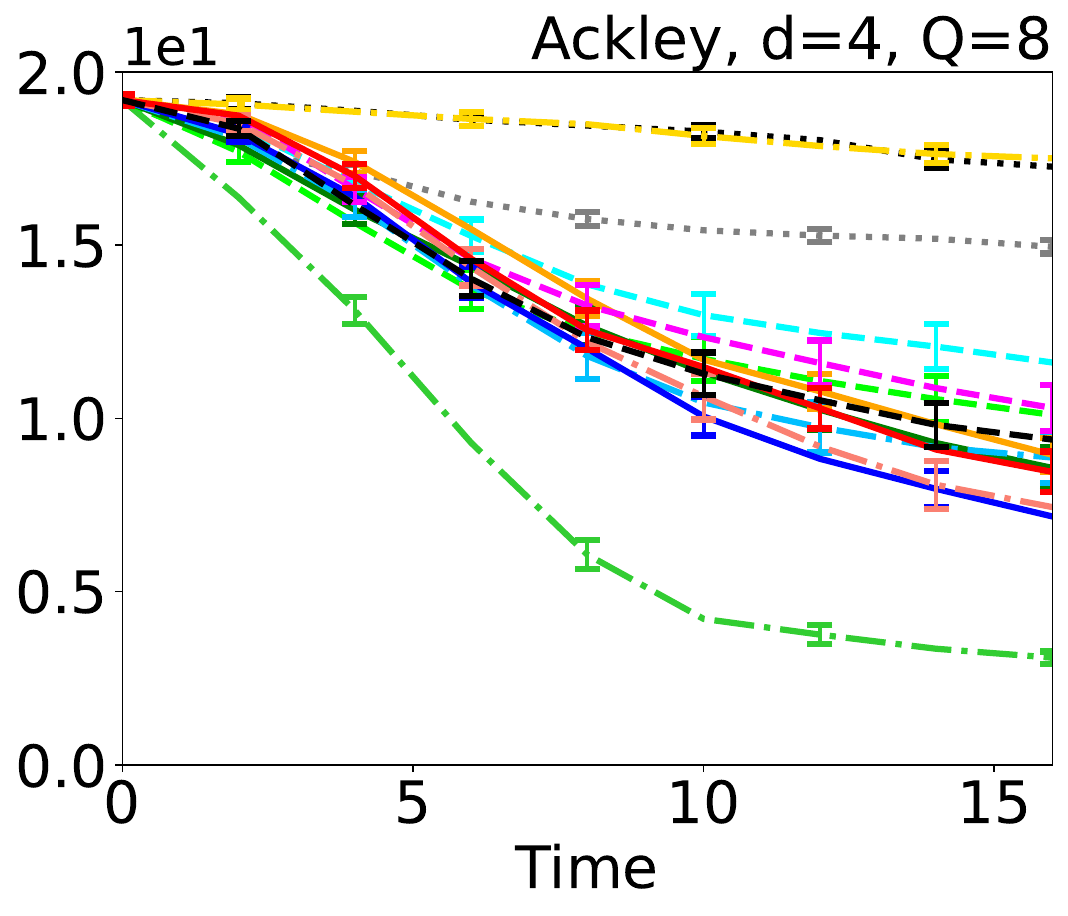}
\includegraphics[height=.25\linewidth]{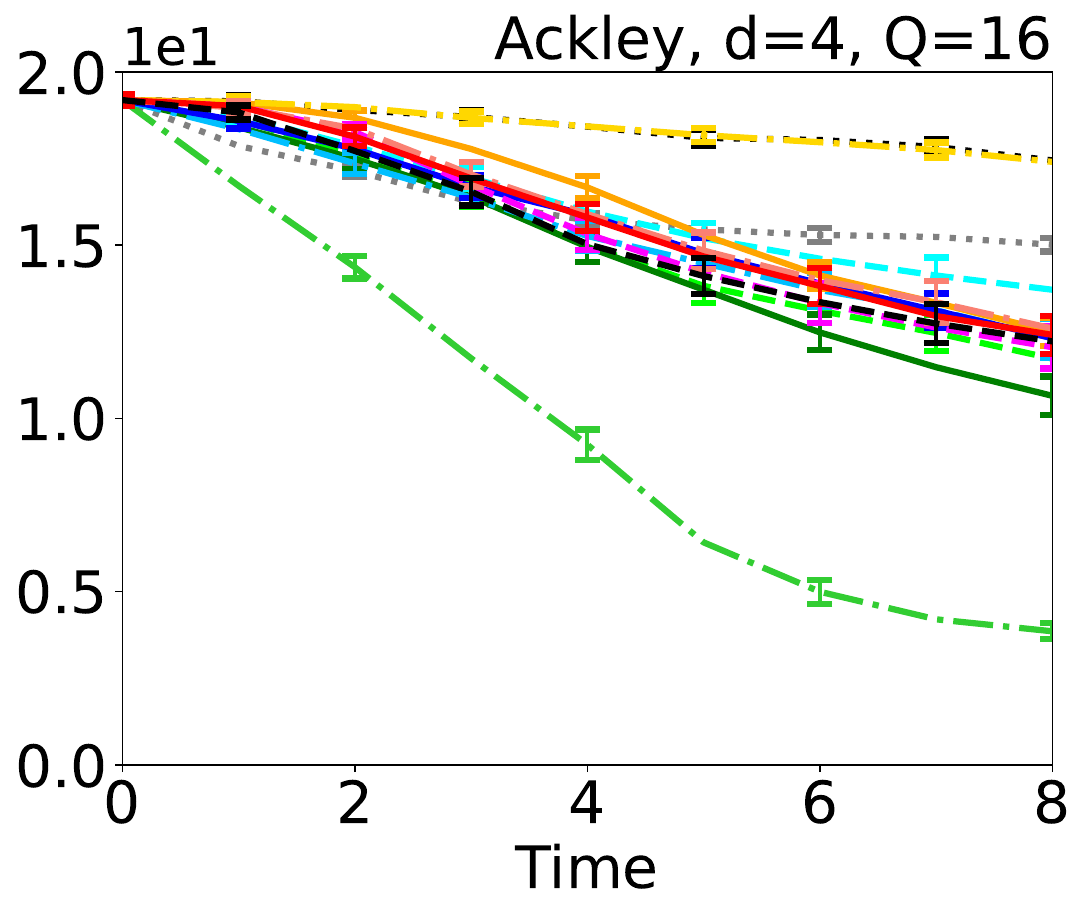}
\includegraphics[height=.25\linewidth]{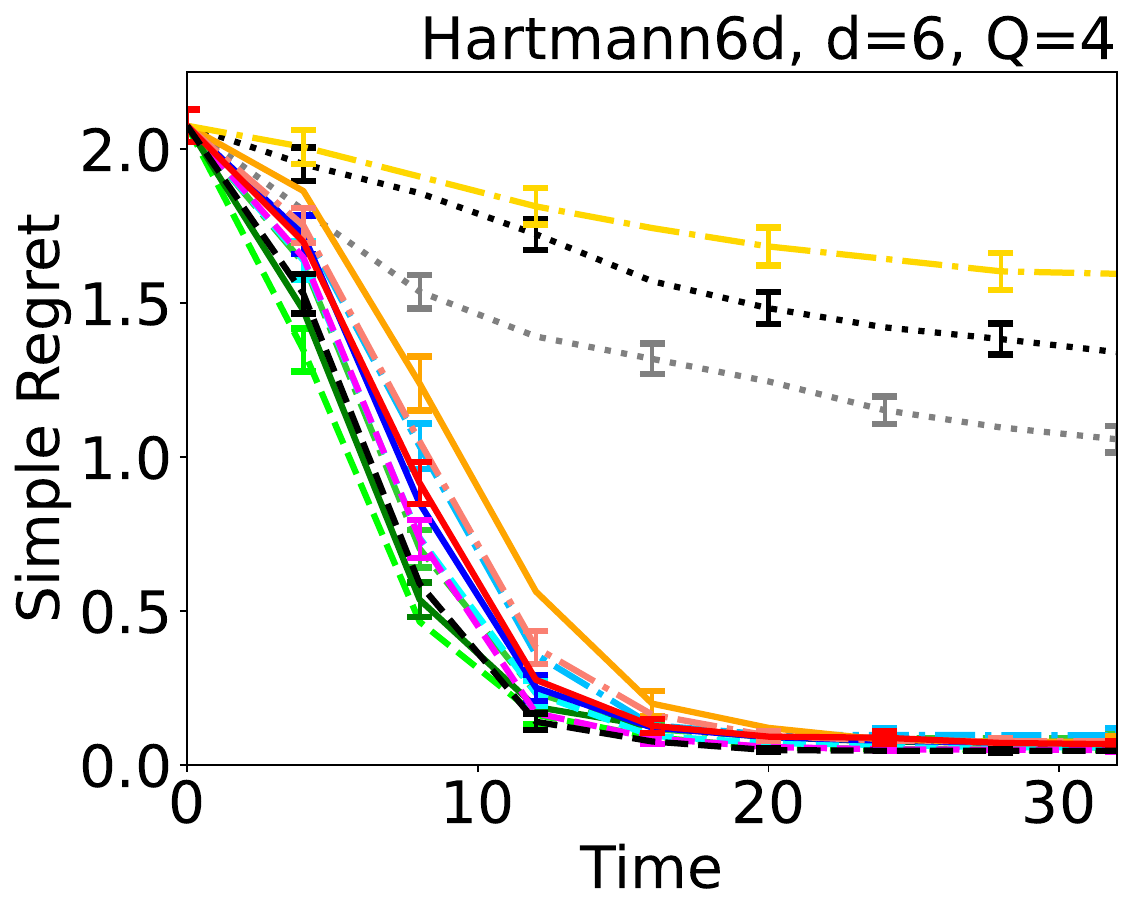}
\includegraphics[height=.25\linewidth]{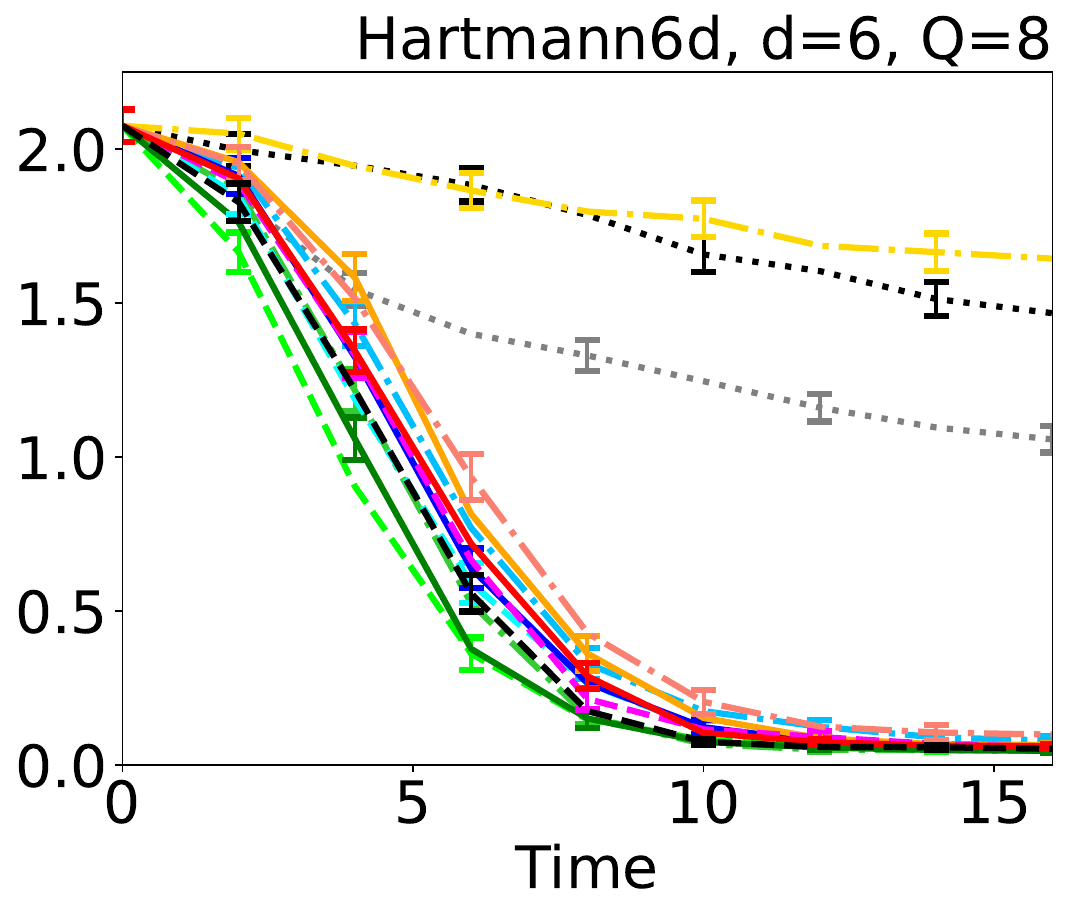}
\includegraphics[height=.25\linewidth]{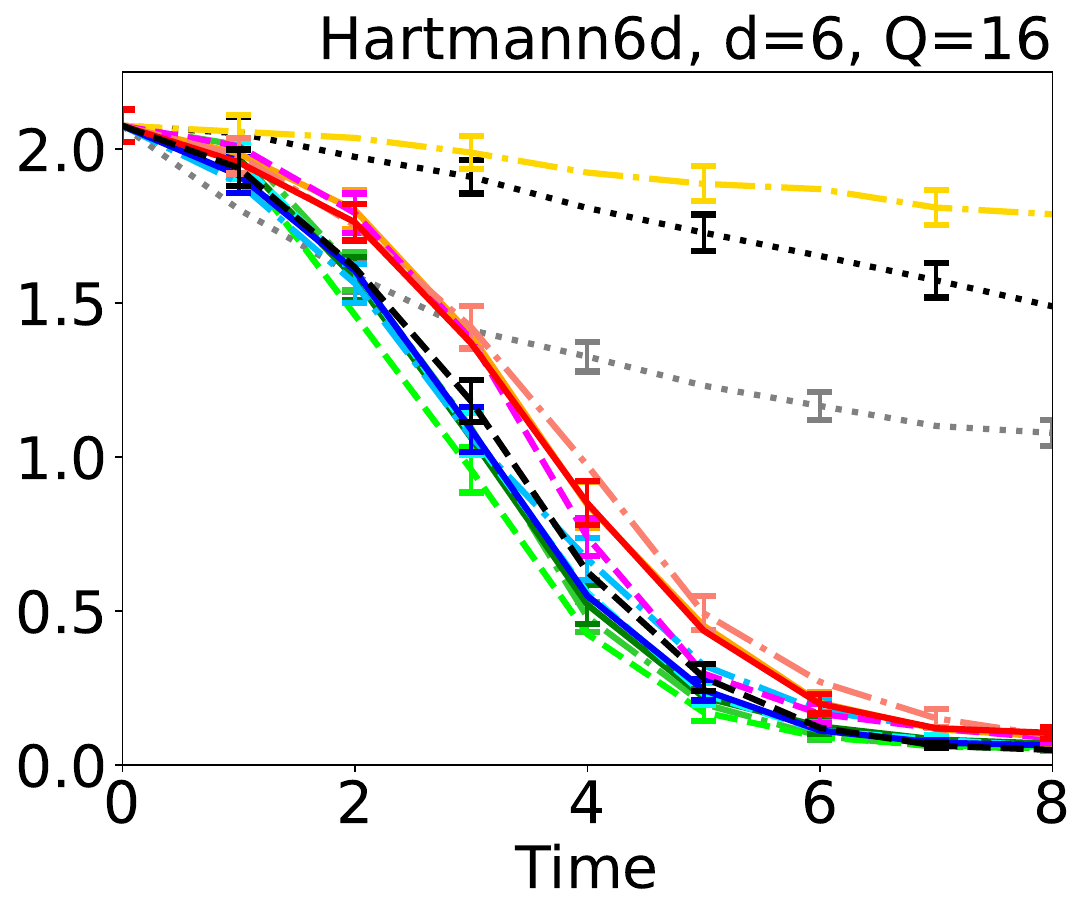}
\includegraphics[height=.25\linewidth]{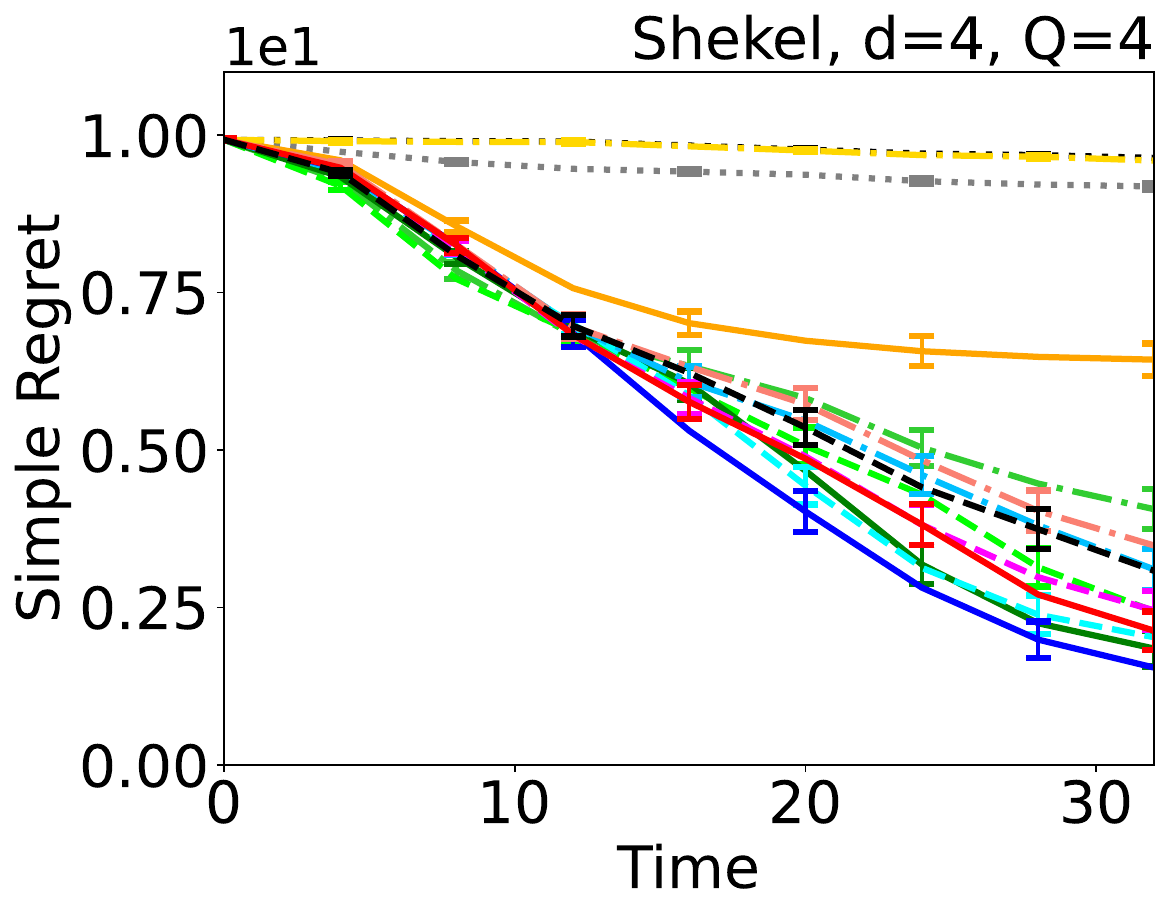}
\includegraphics[height=.25\linewidth]{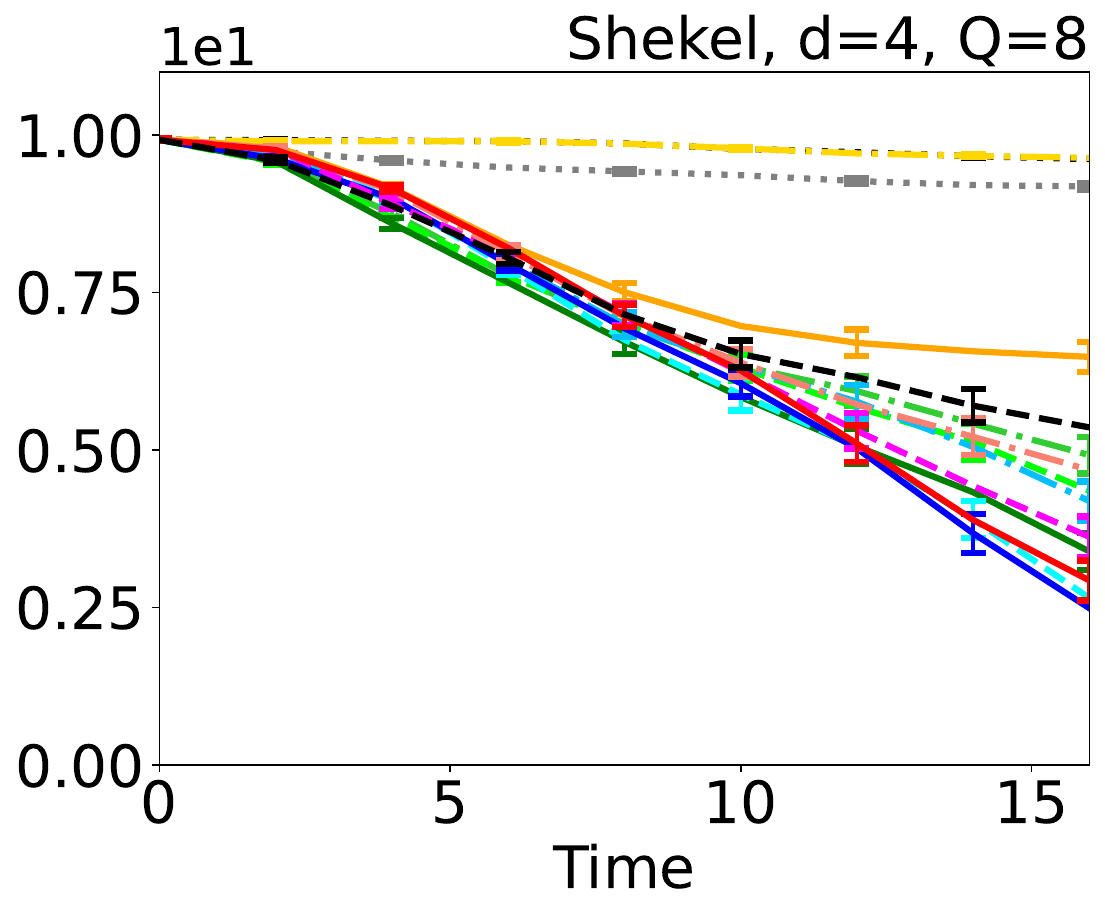}
\includegraphics[height=.25\linewidth]{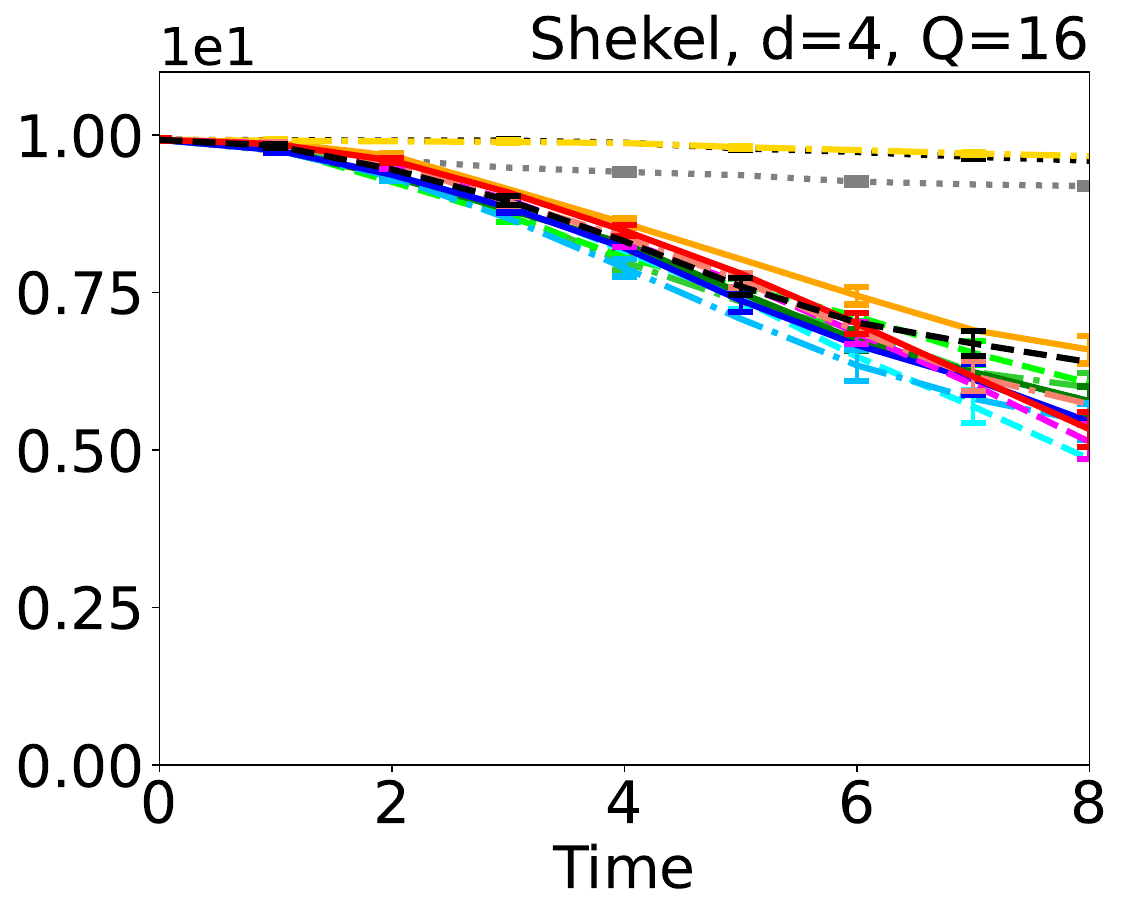}
\includegraphics[height=.25\linewidth]{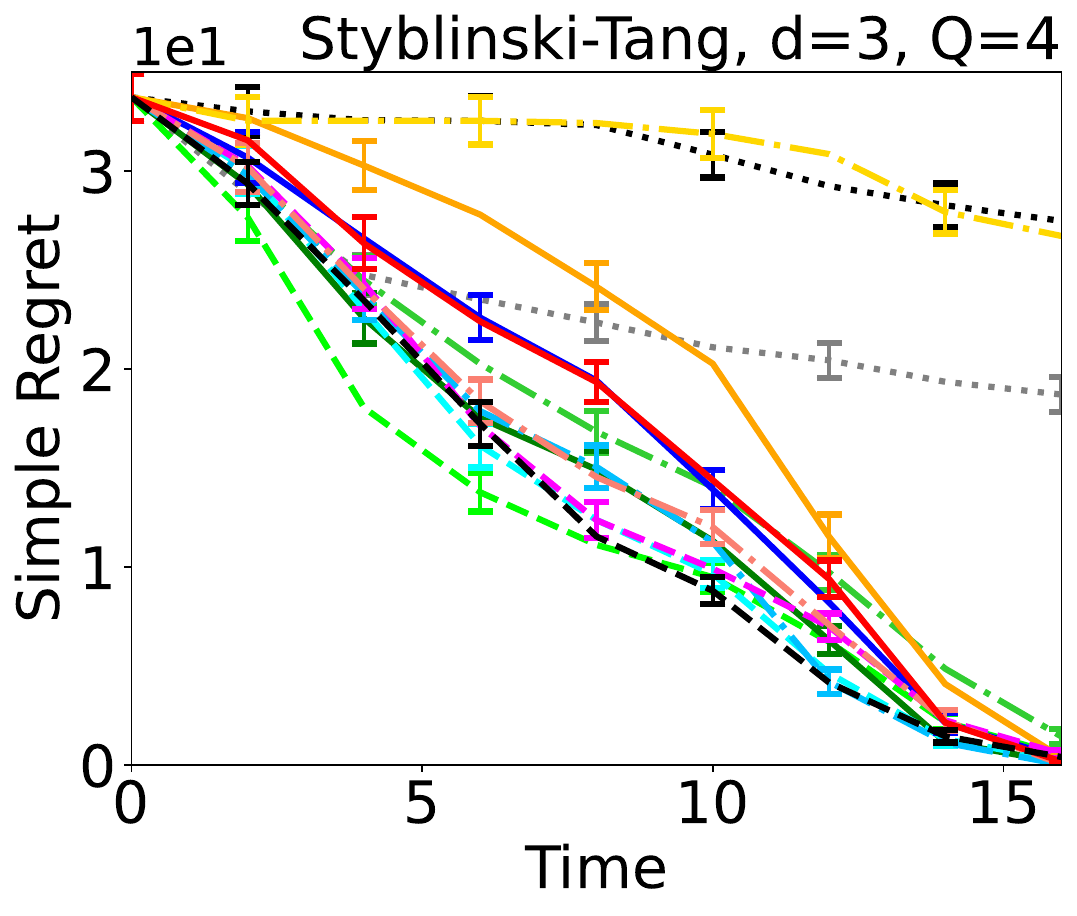}
\includegraphics[height=.25\linewidth]{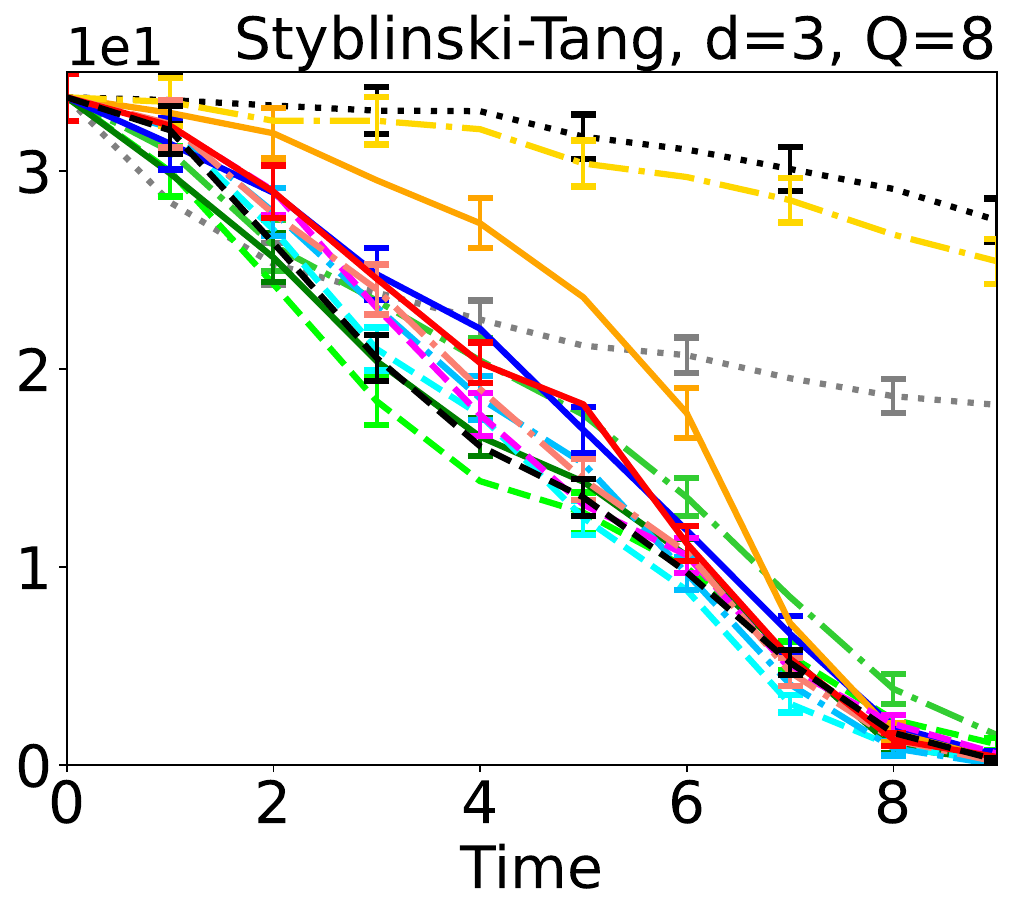}
\includegraphics[height=.25\linewidth]{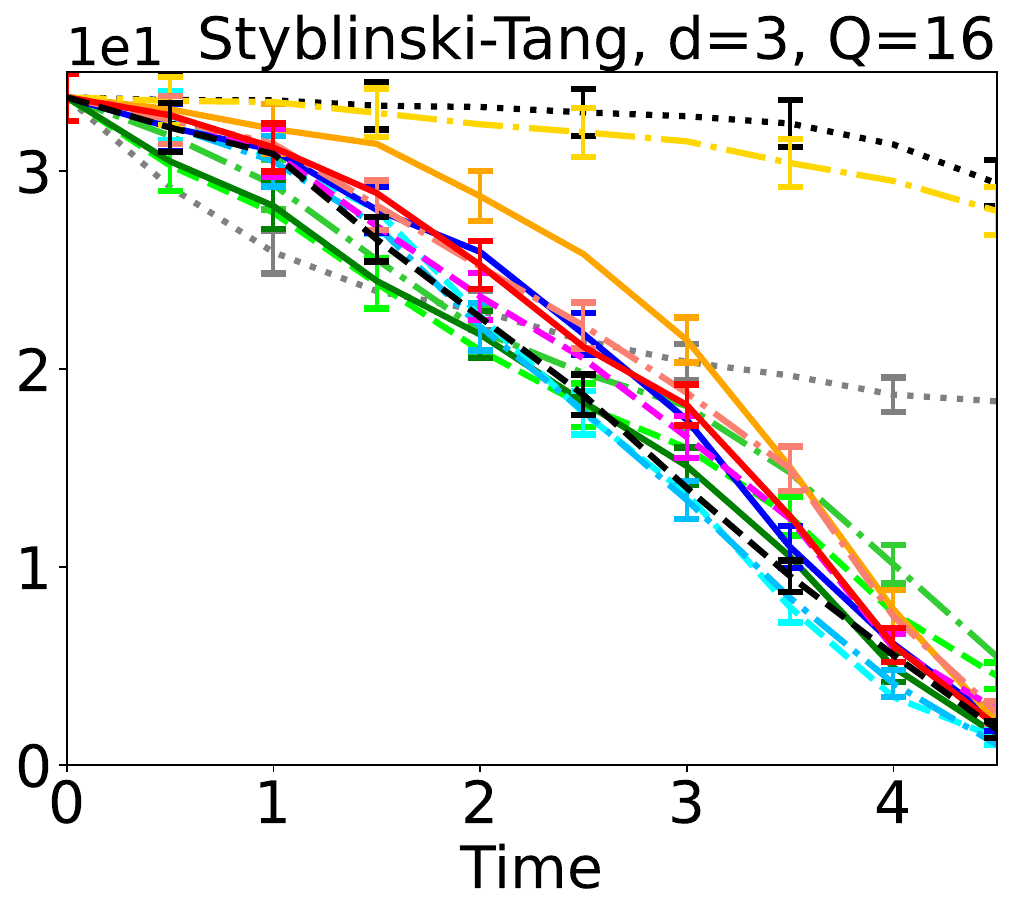}
\caption{
Result of experiments on benchmark functions with asynchronous setting. The lines and error bars mean average and standard error of the simple regret $f^*-\max_{i\in\left[t\right]}f{\left(\bm x_i\right)}$ across the 100 experiments on each condition.
}
\label{fig:result_ben_asyn}
\end{figure}
\begin{figure}[t]
\centering
\includegraphics[width=.9\linewidth]{figures/legend.pdf}\\
\includegraphics[height=.25\linewidth]{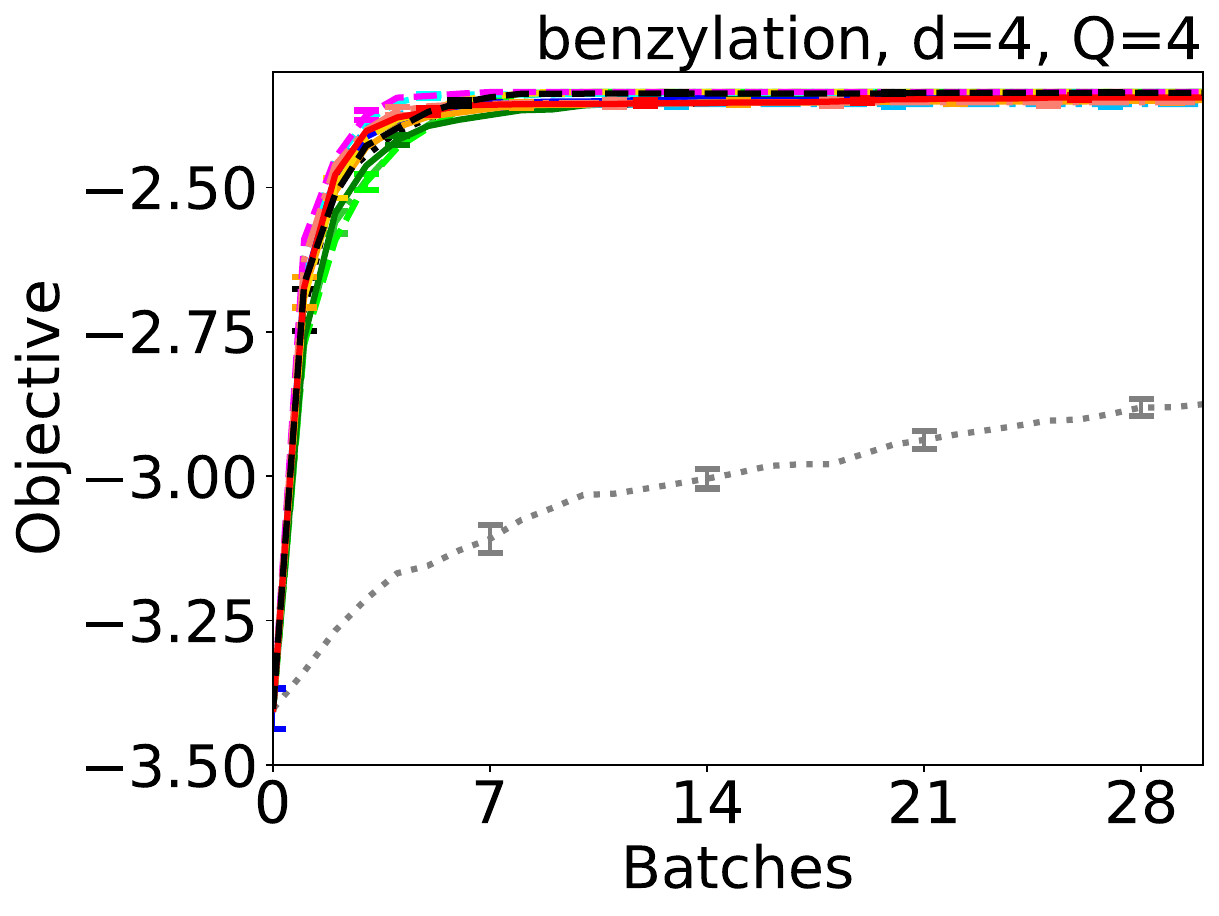}
\includegraphics[height=.25\linewidth]{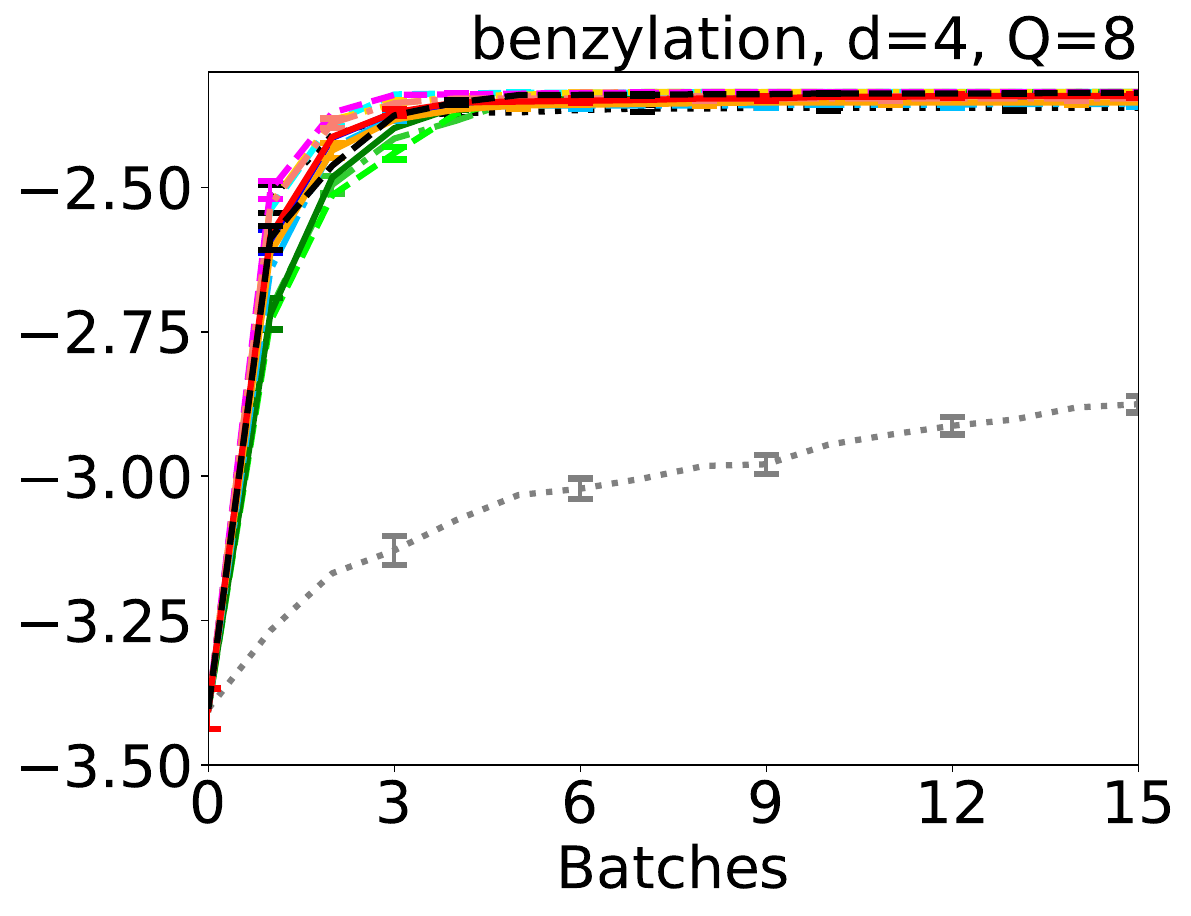}
\includegraphics[height=.25\linewidth]{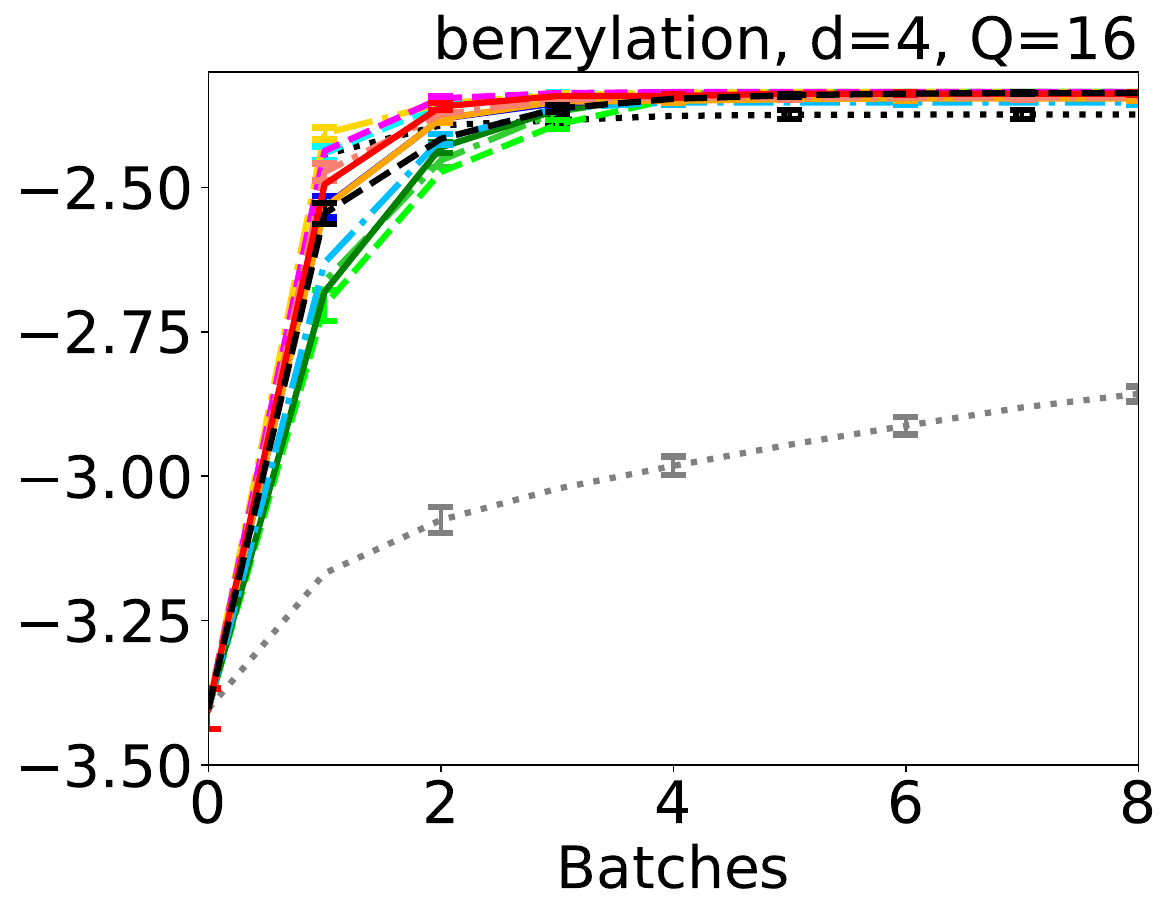}
\includegraphics[height=.25\linewidth]{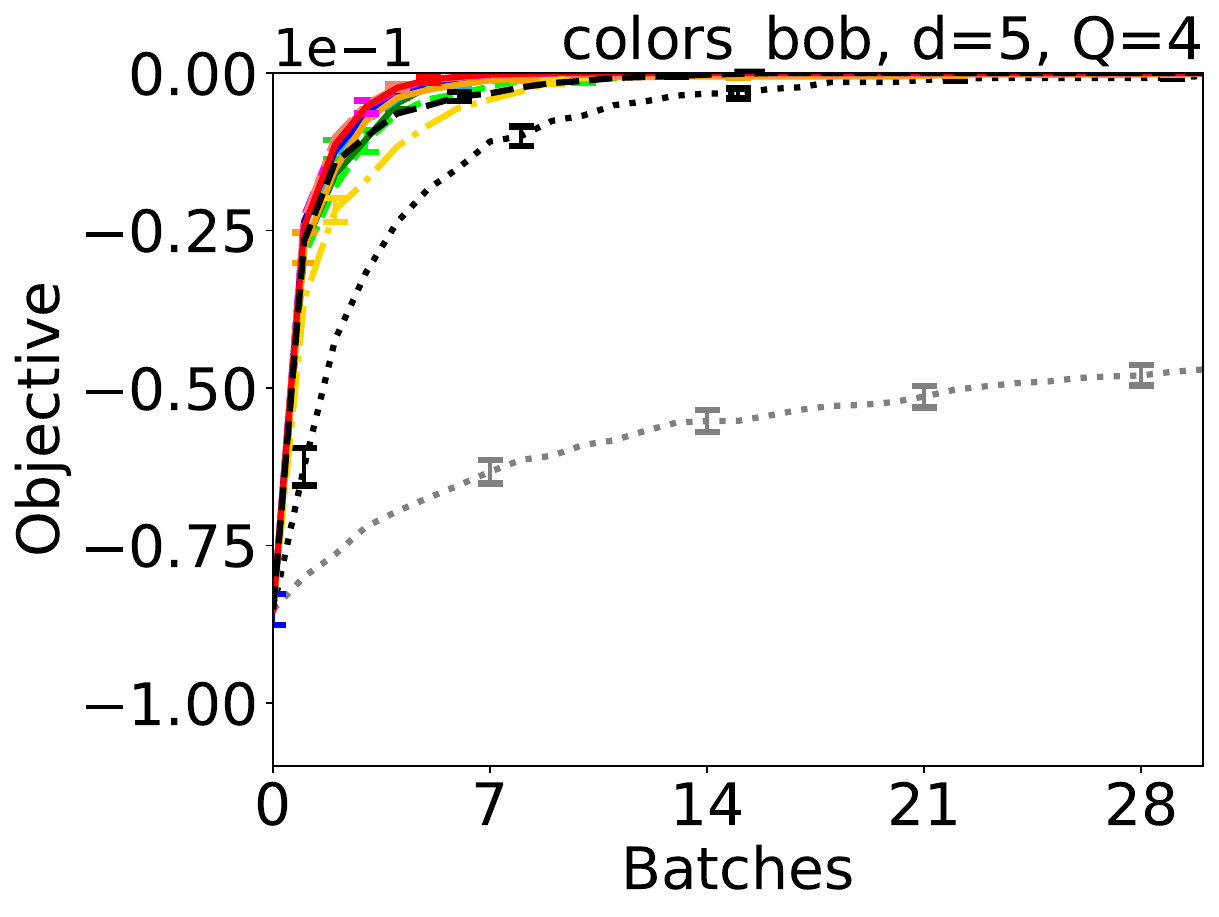}
\includegraphics[height=.25\linewidth]{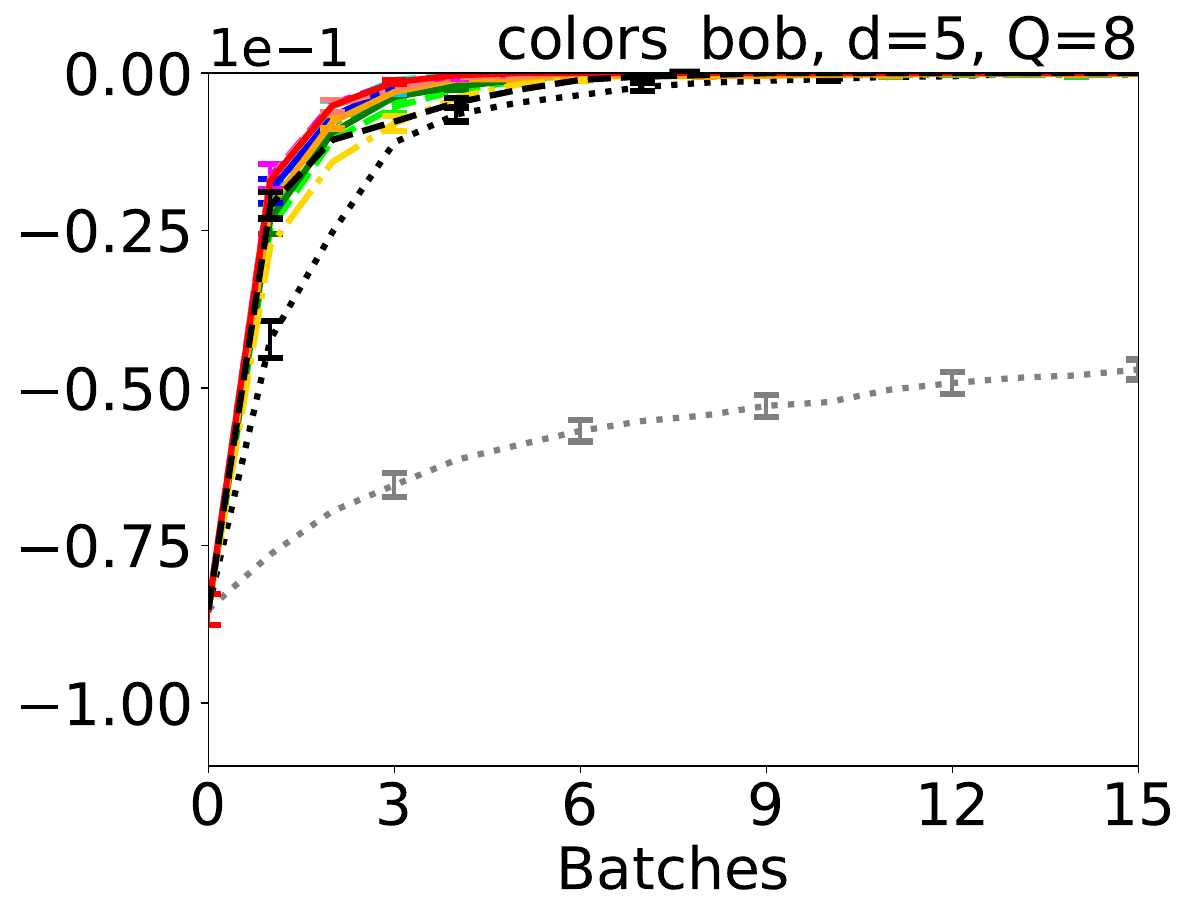}
\includegraphics[height=.25\linewidth]{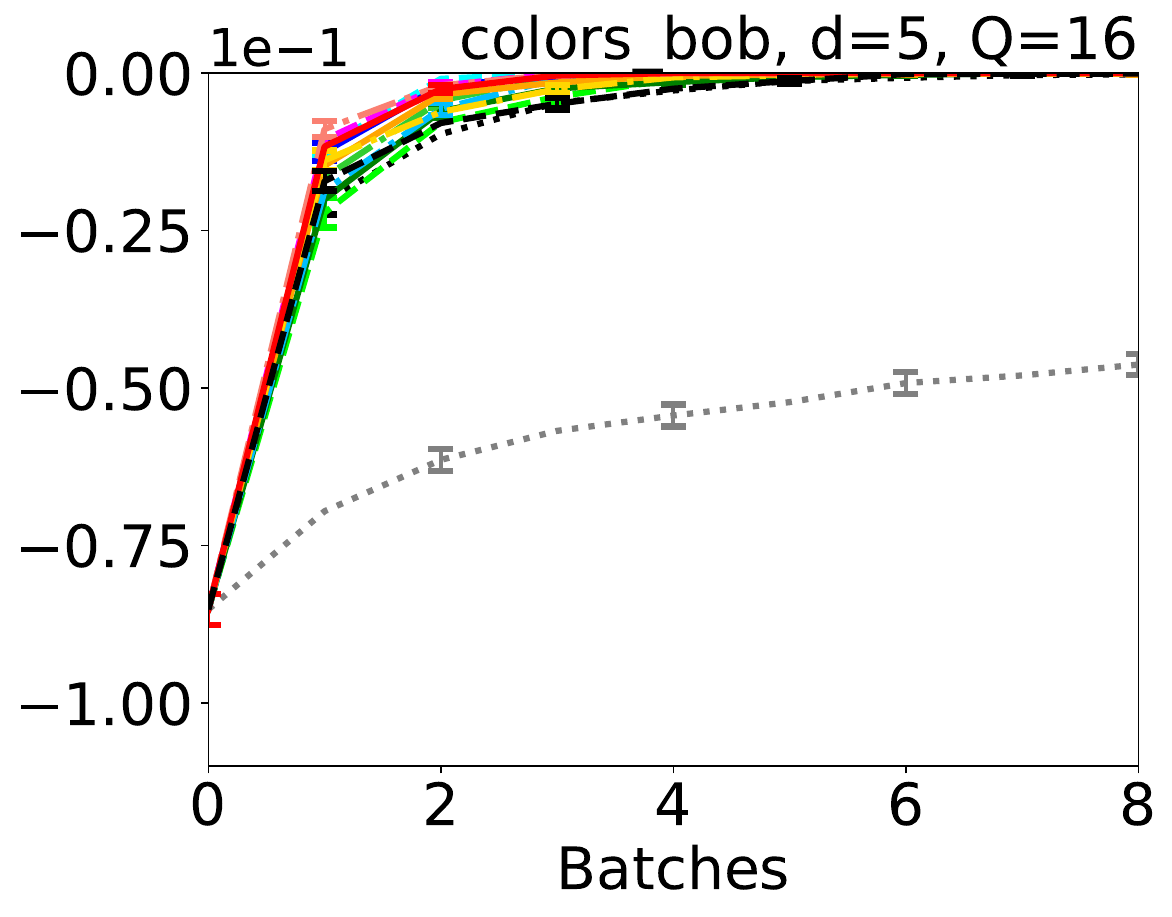}
\includegraphics[height=.25\linewidth]{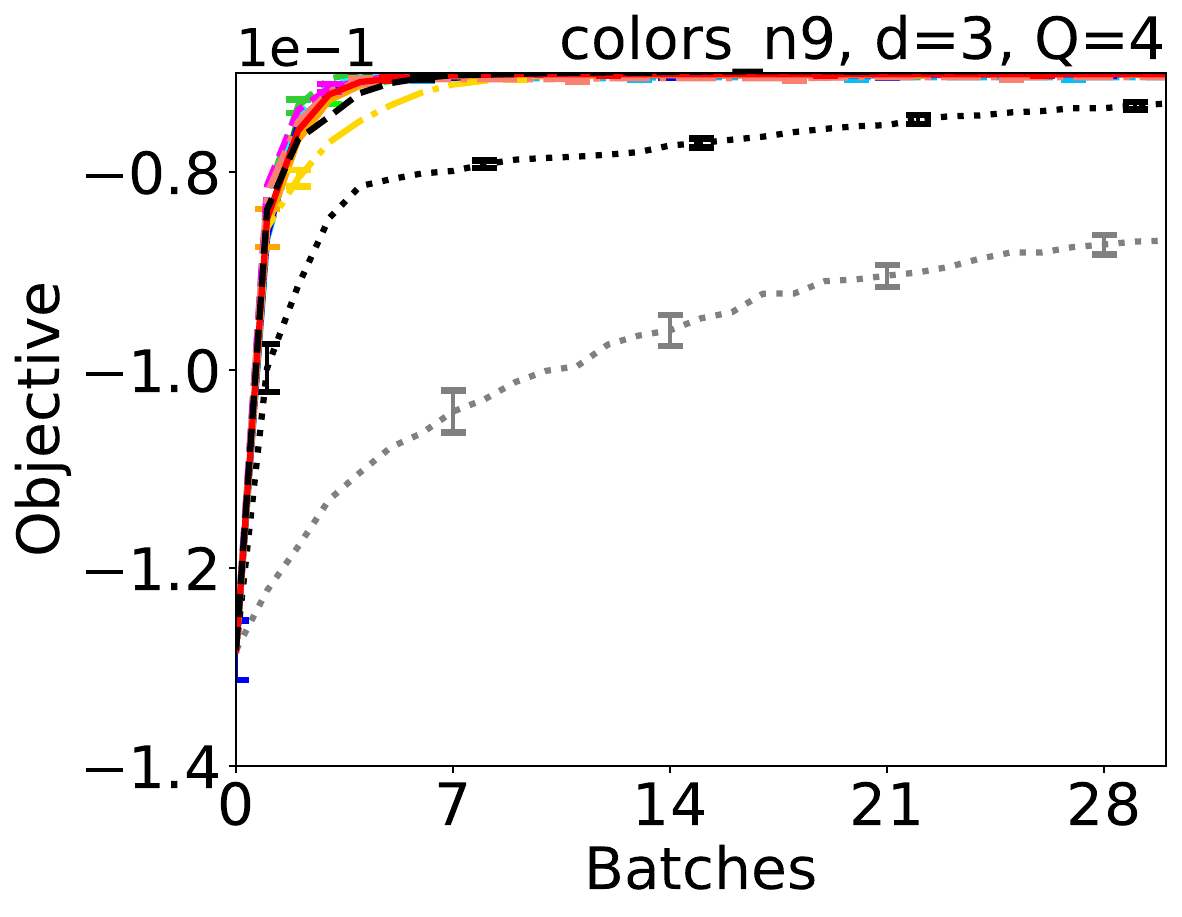}
\includegraphics[height=.25\linewidth]{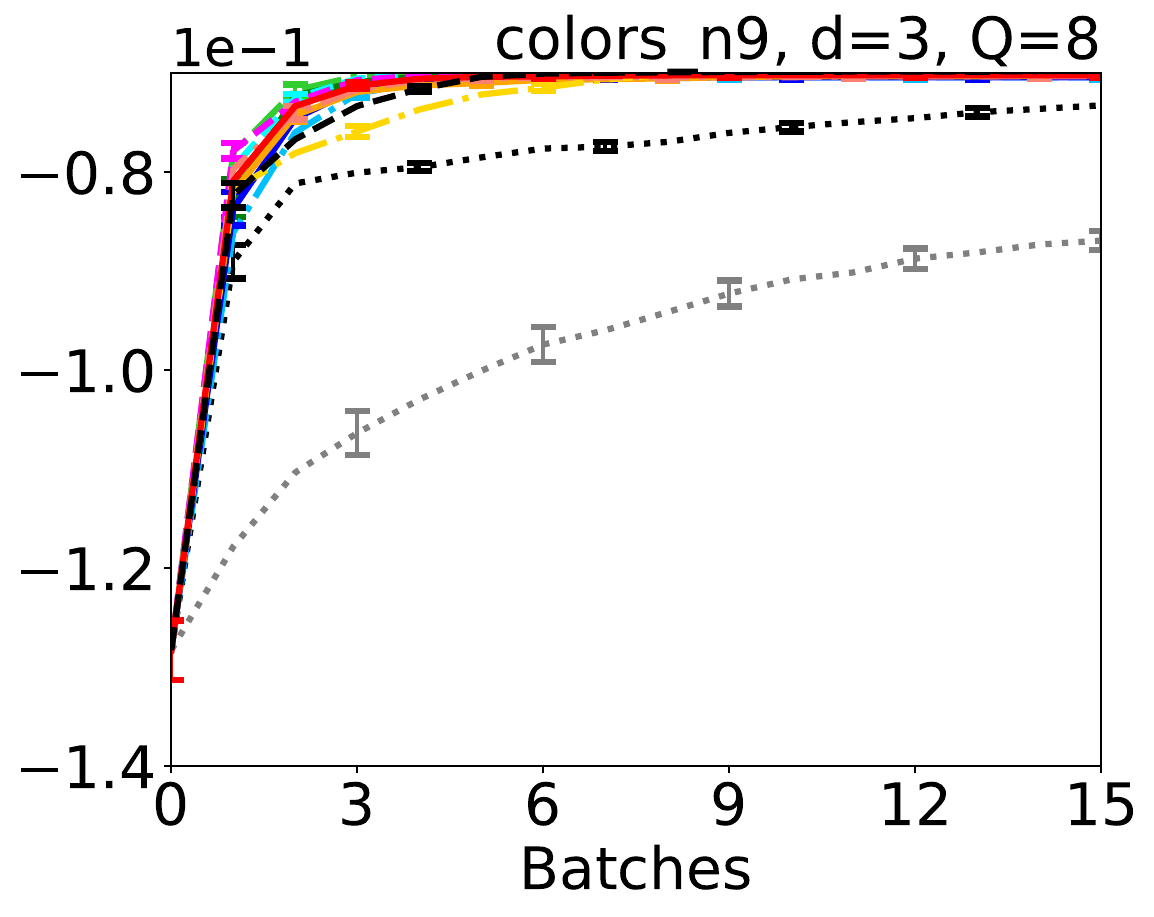}
\includegraphics[height=.25\linewidth]{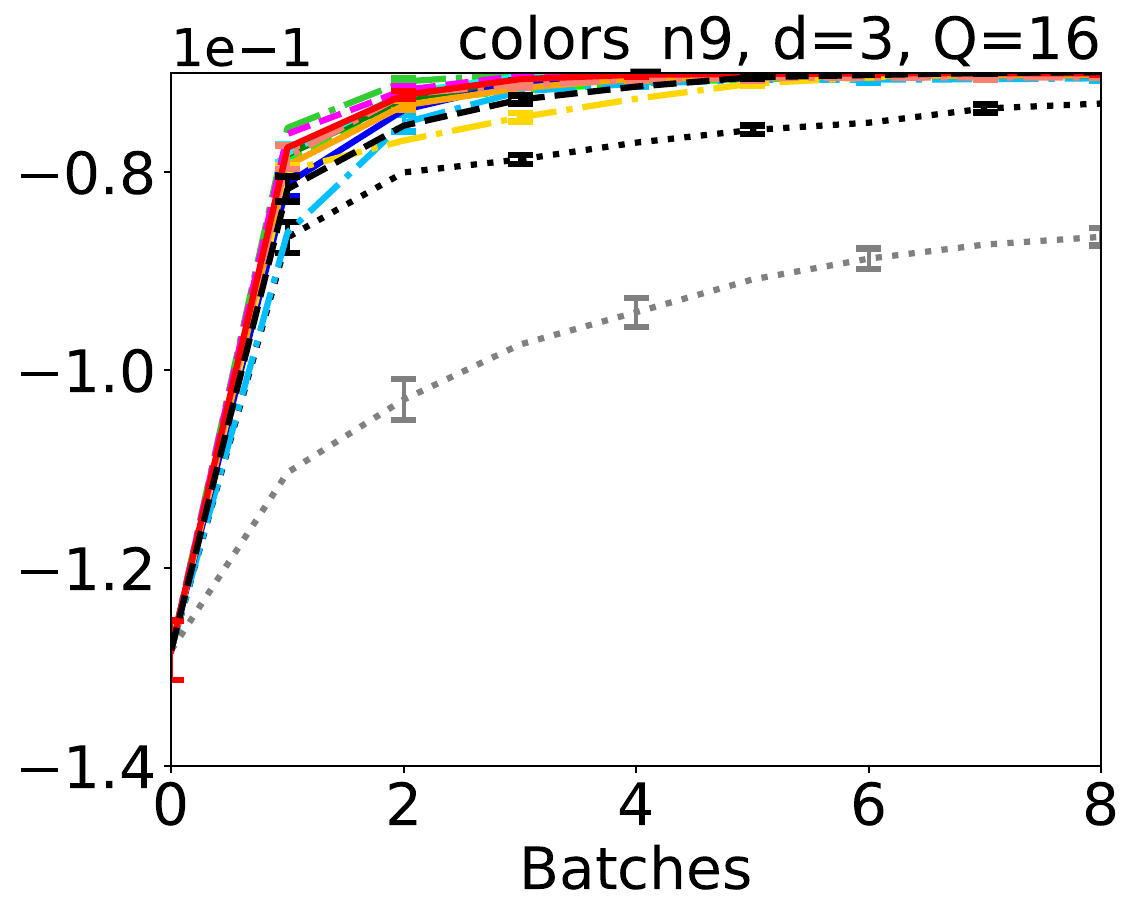}
\includegraphics[height=.25\linewidth]{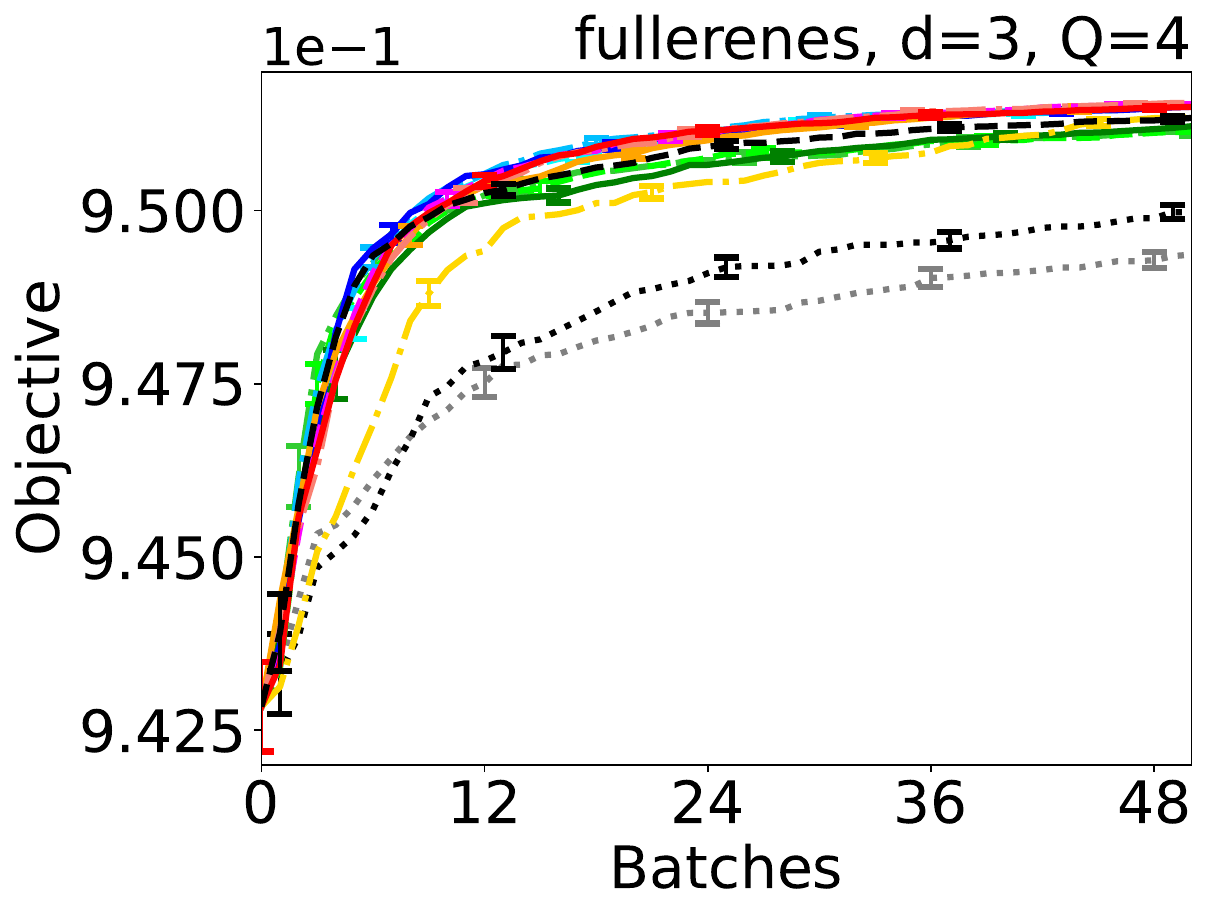}
\includegraphics[height=.25\linewidth]{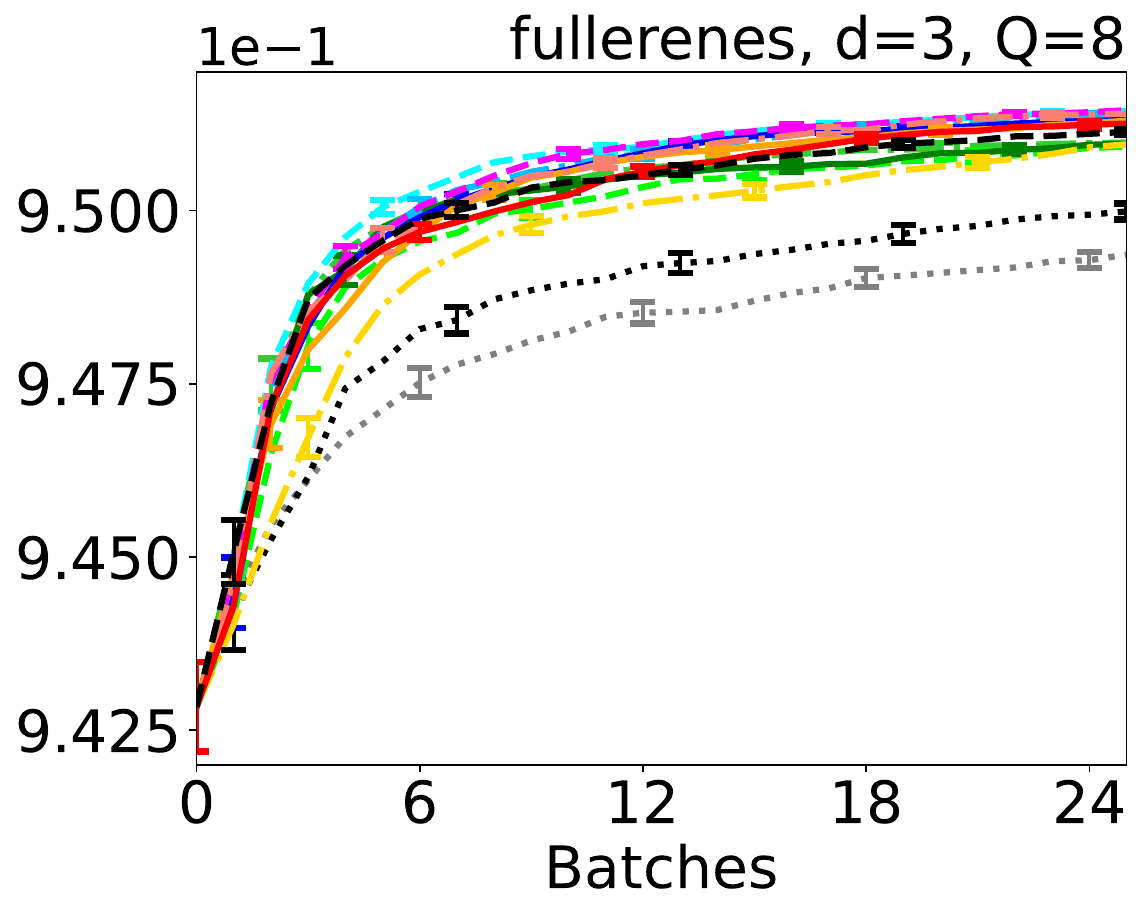}
\includegraphics[height=.25\linewidth]{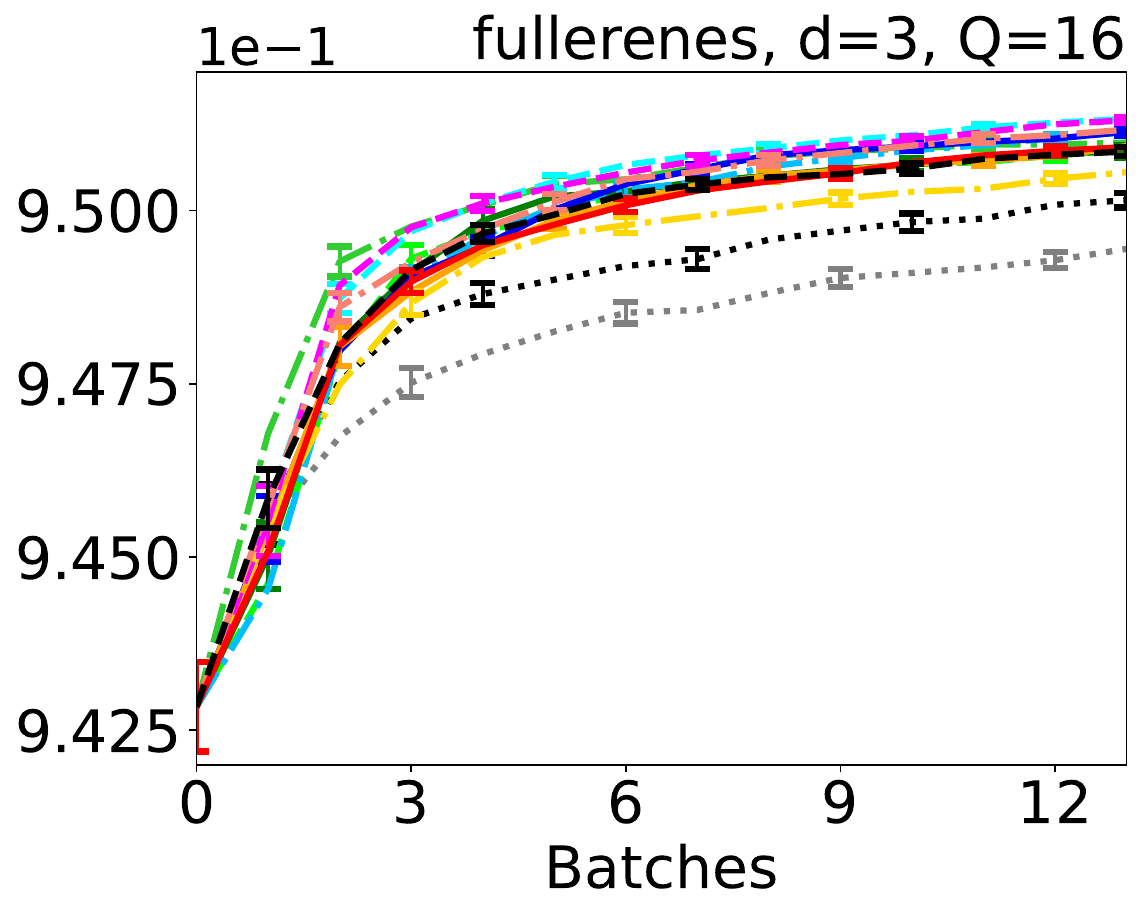}
\includegraphics[height=.25\linewidth]{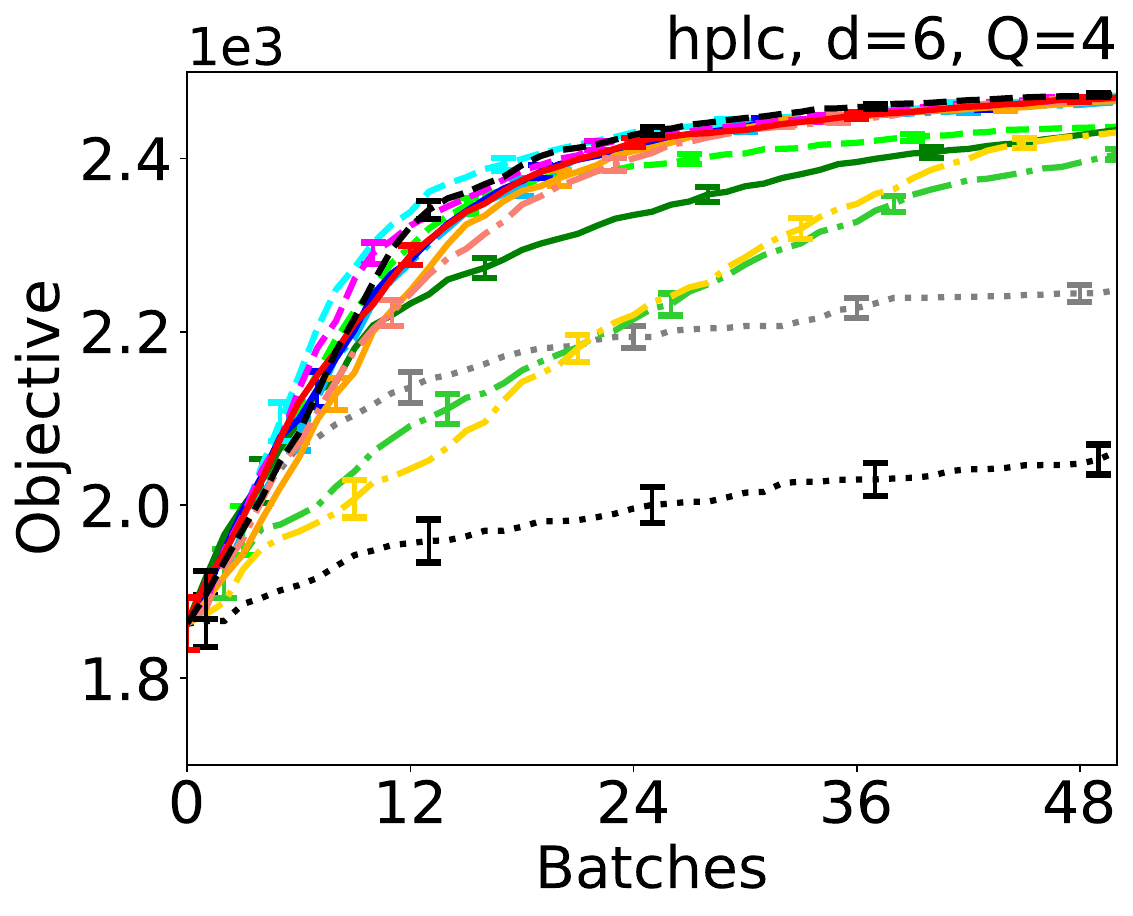}
\includegraphics[height=.25\linewidth]{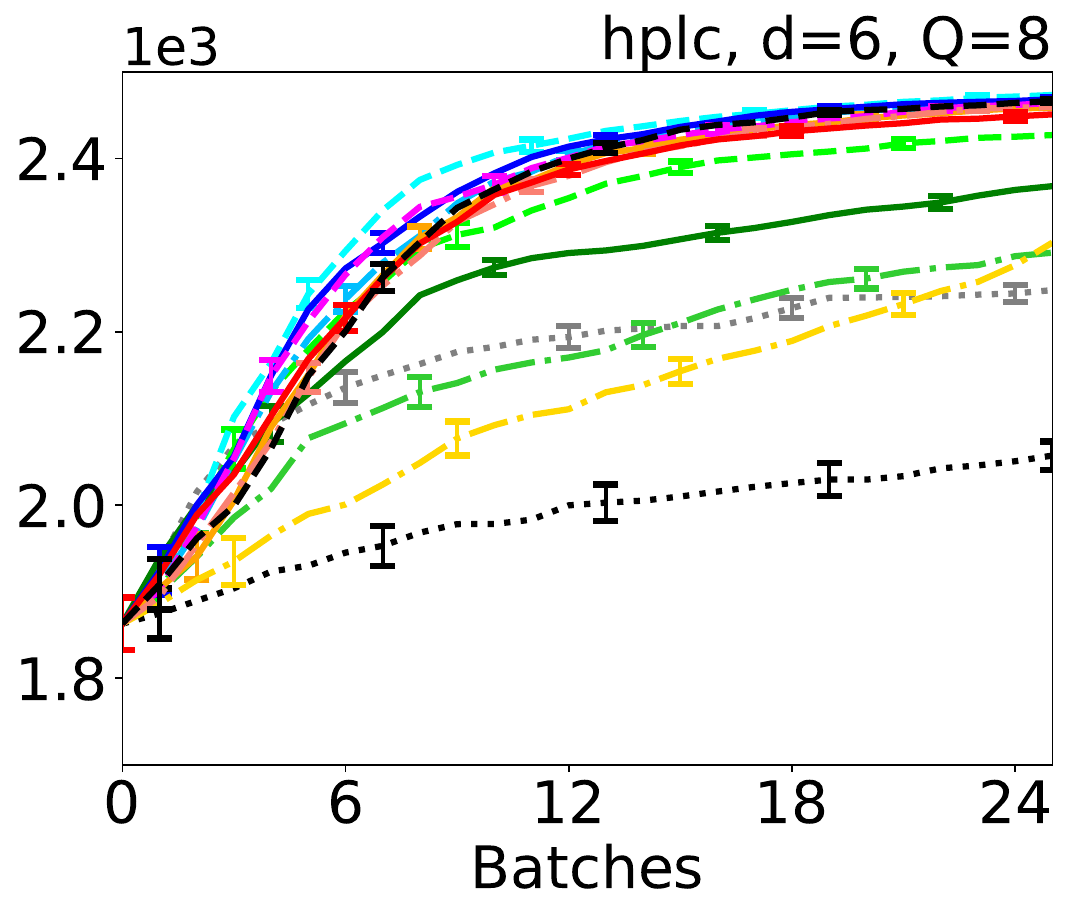}
\includegraphics[height=.25\linewidth]{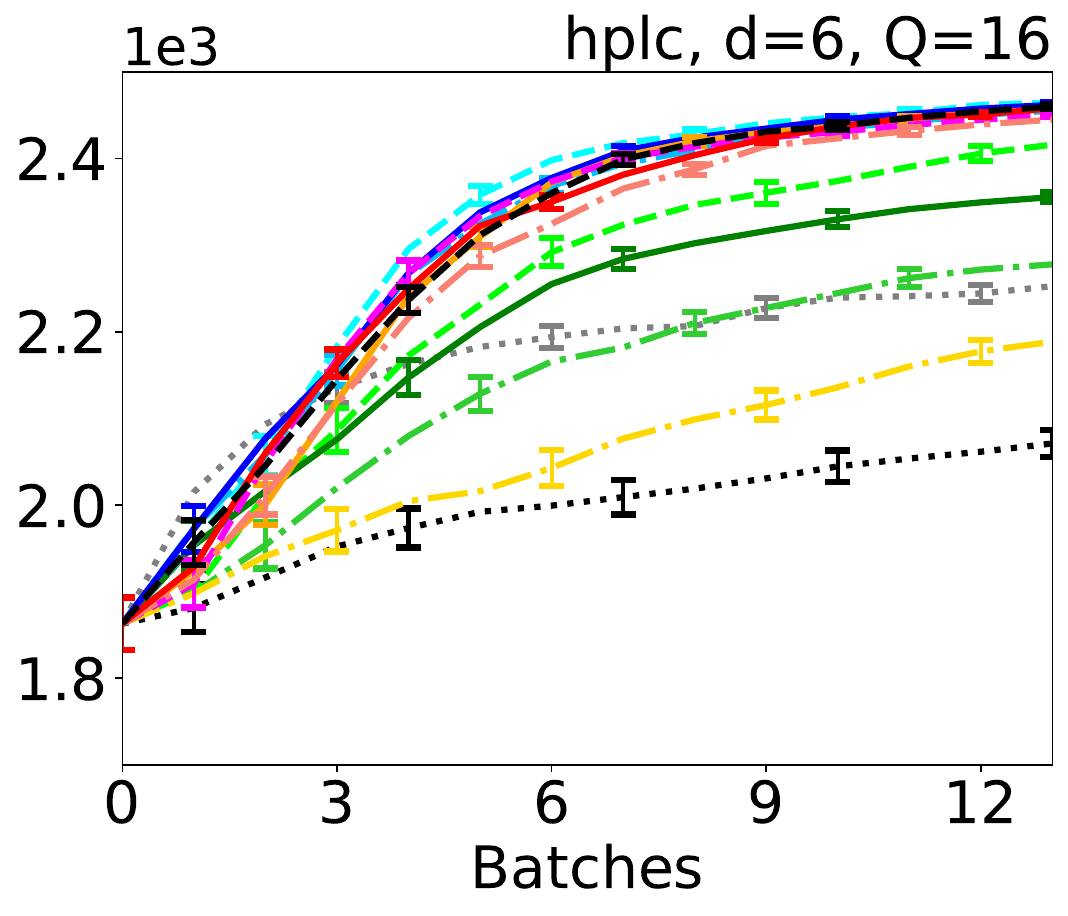}
\caption{
Result of experiments on emulators with synchronous setting. The lines and error bars mean average and standard error of the best objective value $\max_{i\in\left[t\right]}f{\left(\bm x_i\right)}$ across the 100 experiments on each condition. One batch corresponds to $Q$ iterations.
}
\label{fig:result_eml_syn}
\end{figure}

\begin{figure}[t]
\centering
\includegraphics[width=.9\linewidth]{figures/legend.pdf}\\
\includegraphics[height=.25\linewidth]{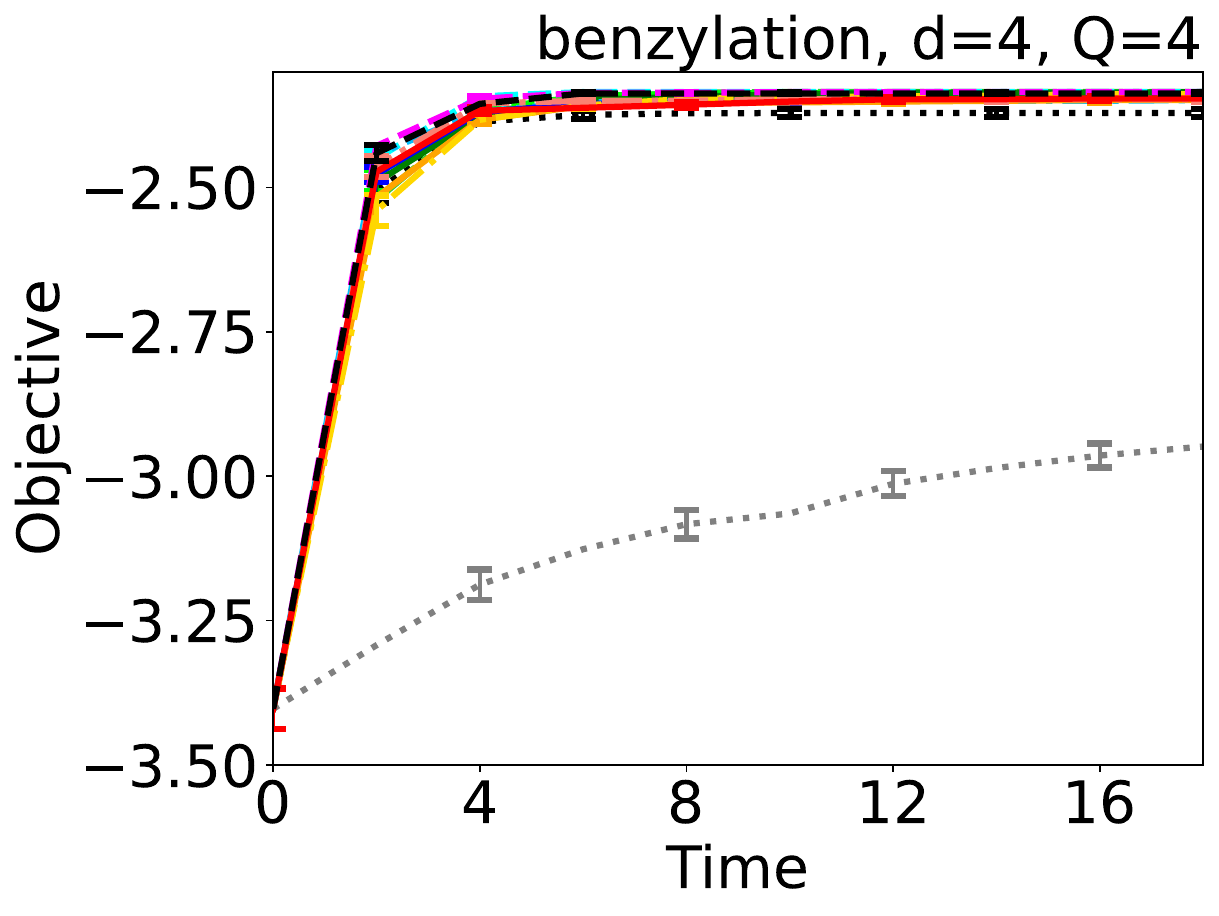}
\includegraphics[height=.25\linewidth]{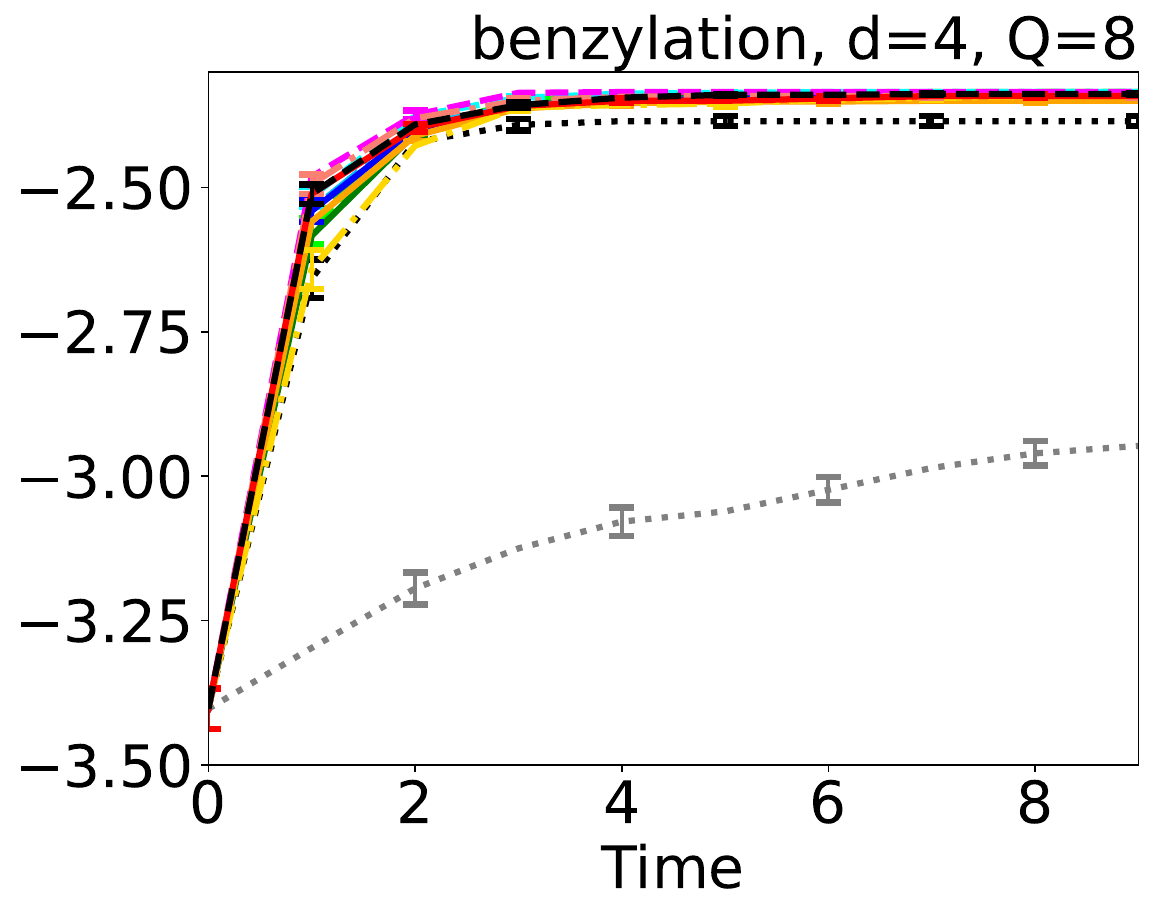}
\includegraphics[height=.25\linewidth]{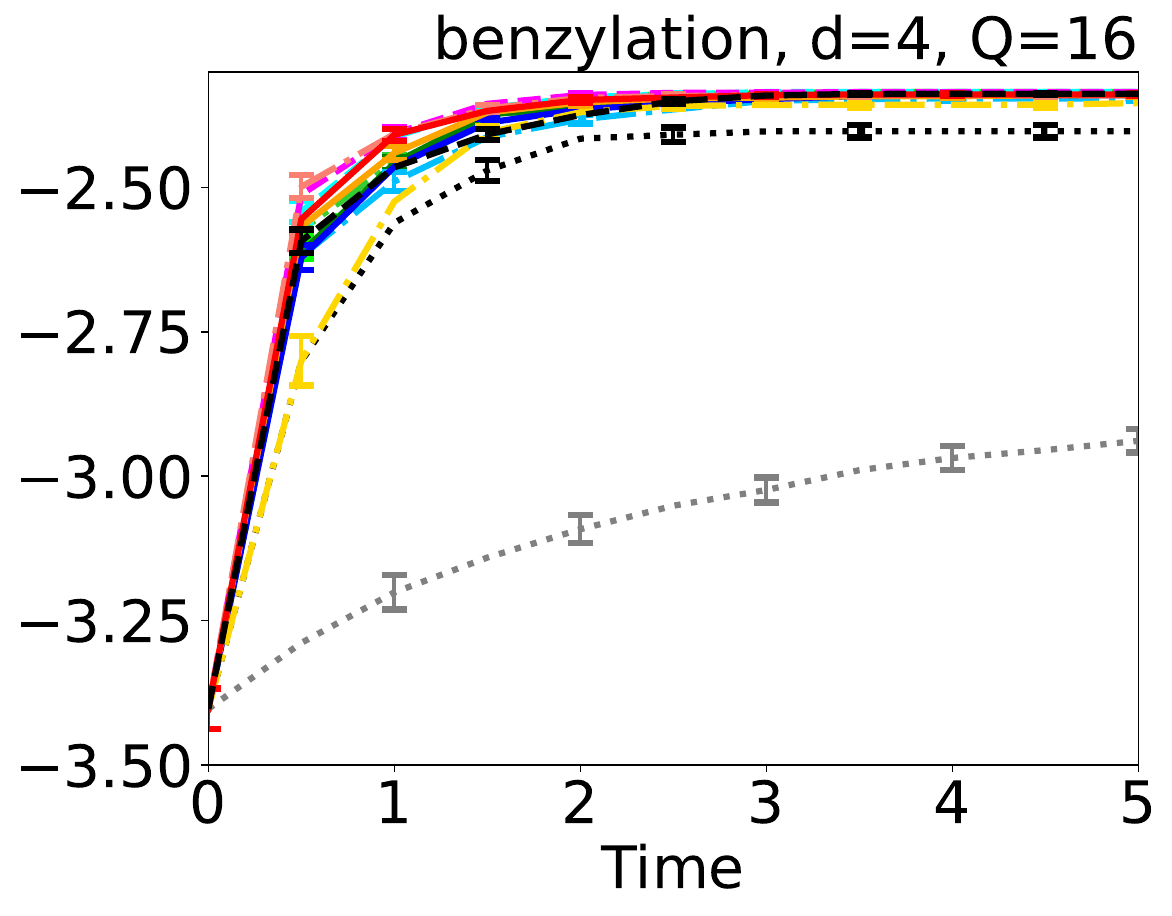}
\includegraphics[height=.25\linewidth]{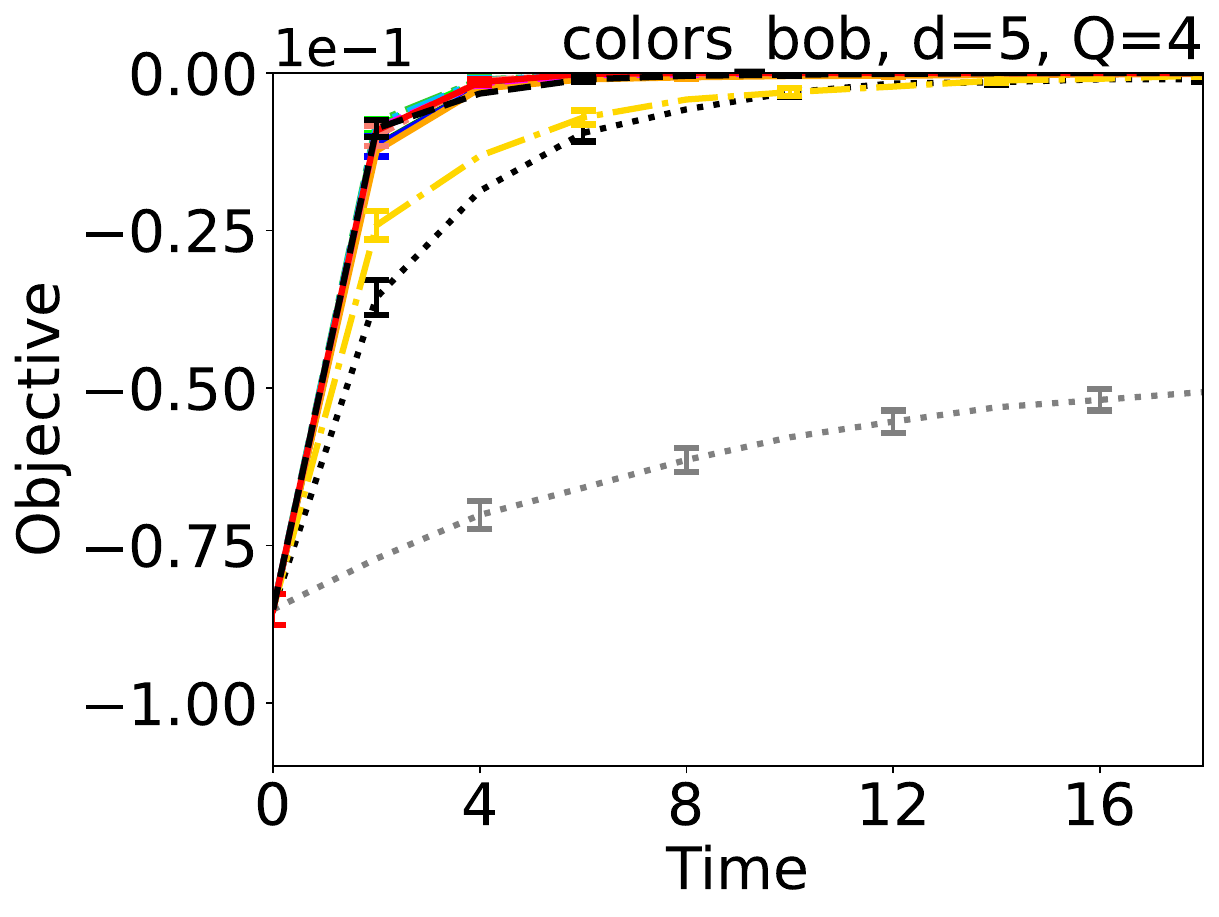}
\includegraphics[height=.25\linewidth]{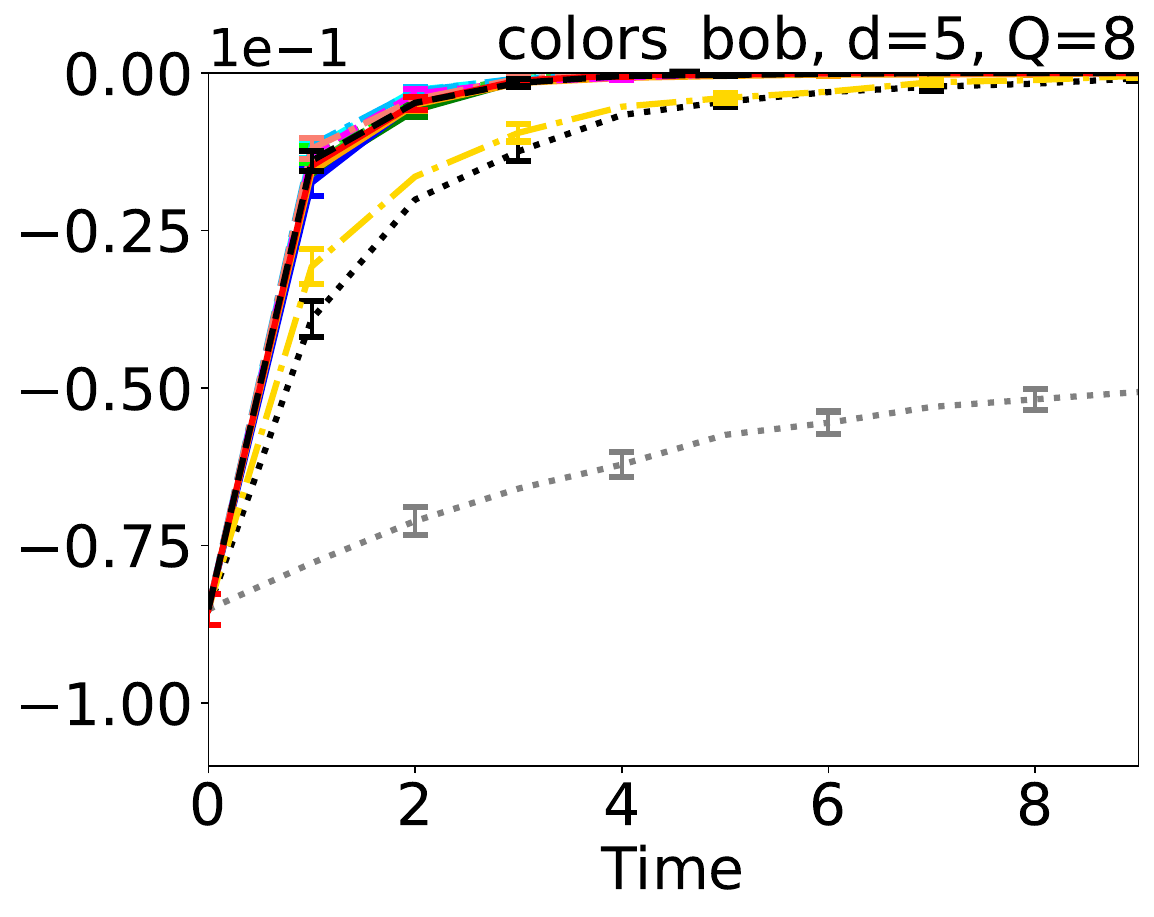}
\includegraphics[height=.25\linewidth]{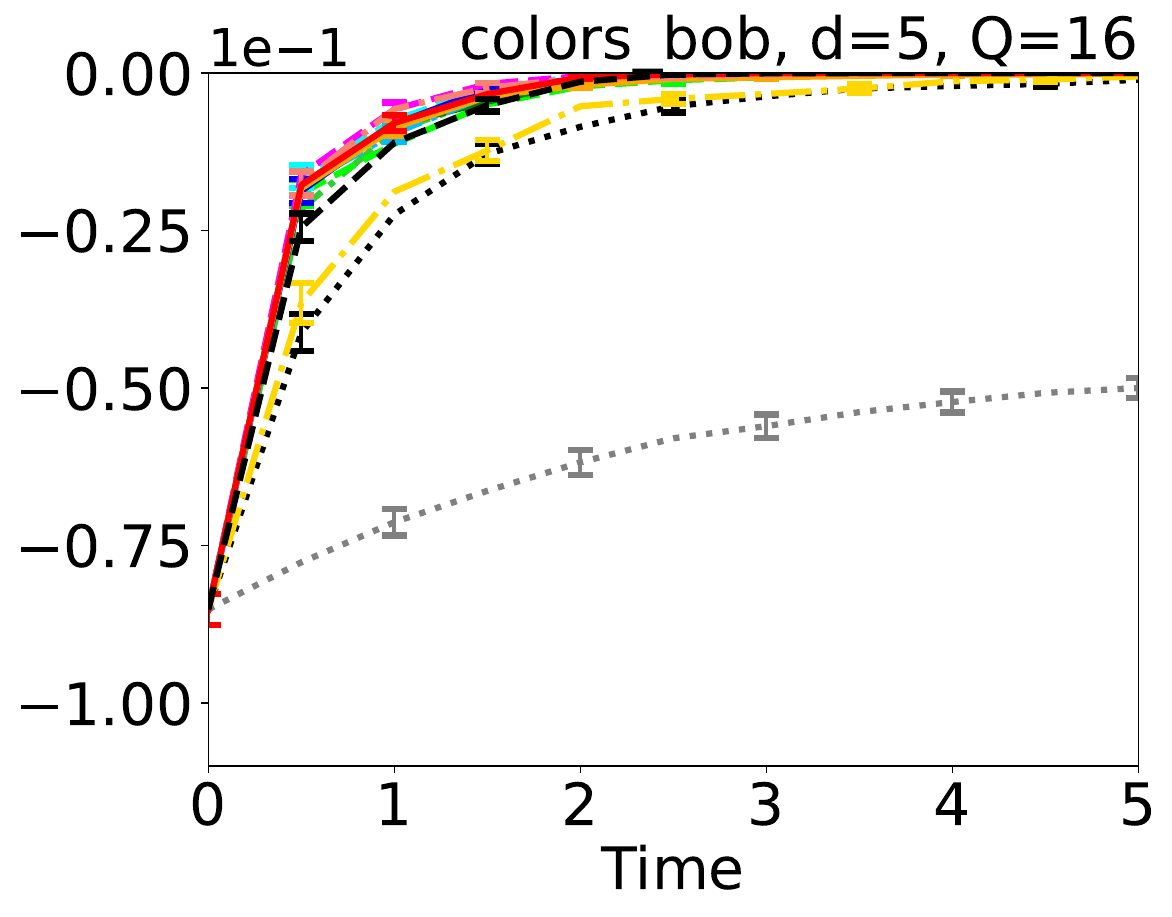}
\includegraphics[height=.25\linewidth]{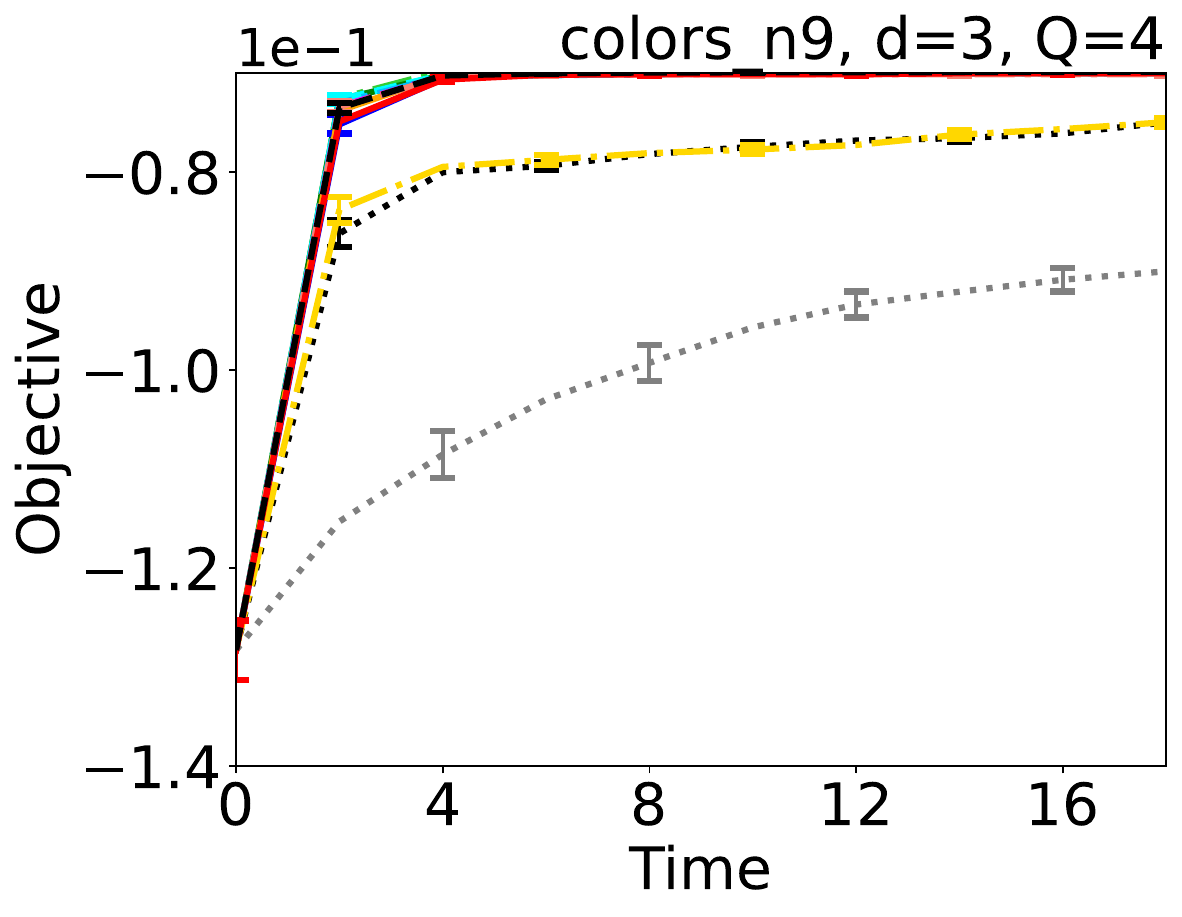}
\includegraphics[height=.25\linewidth]{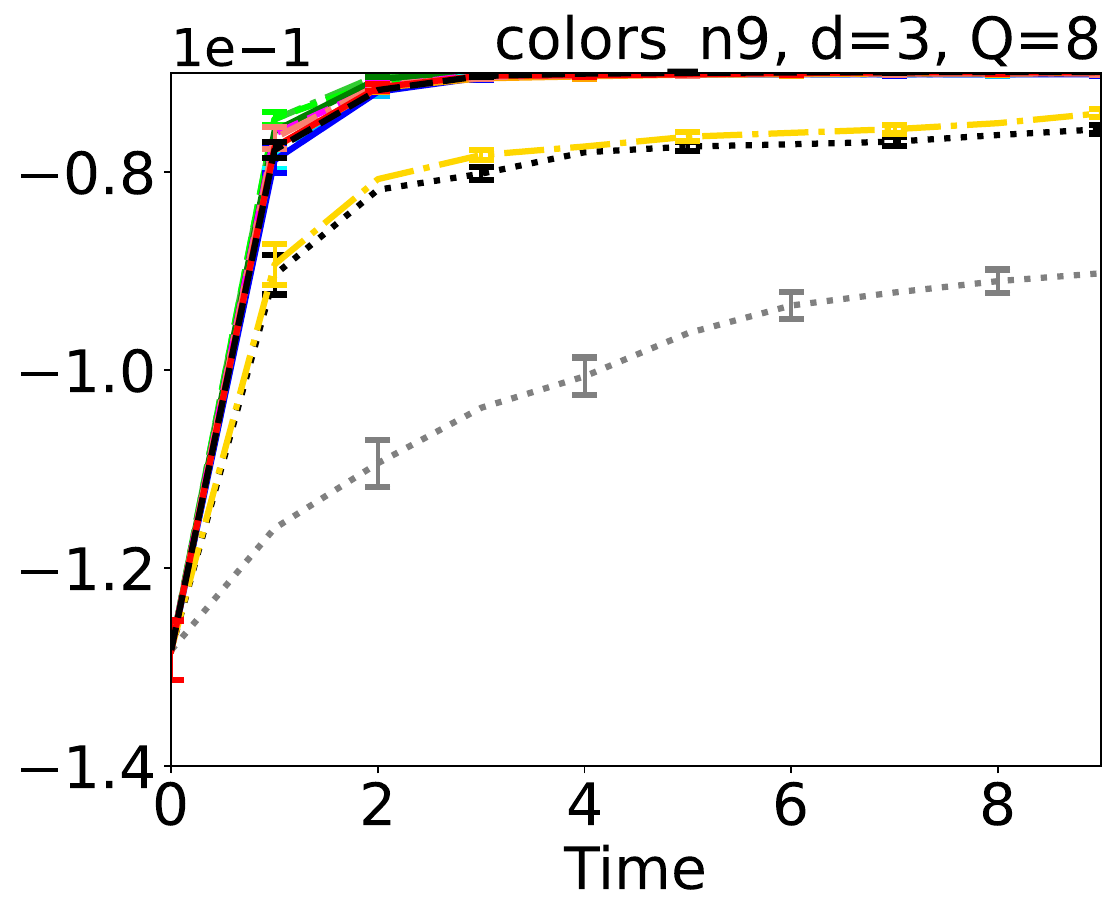}
\includegraphics[height=.25\linewidth]{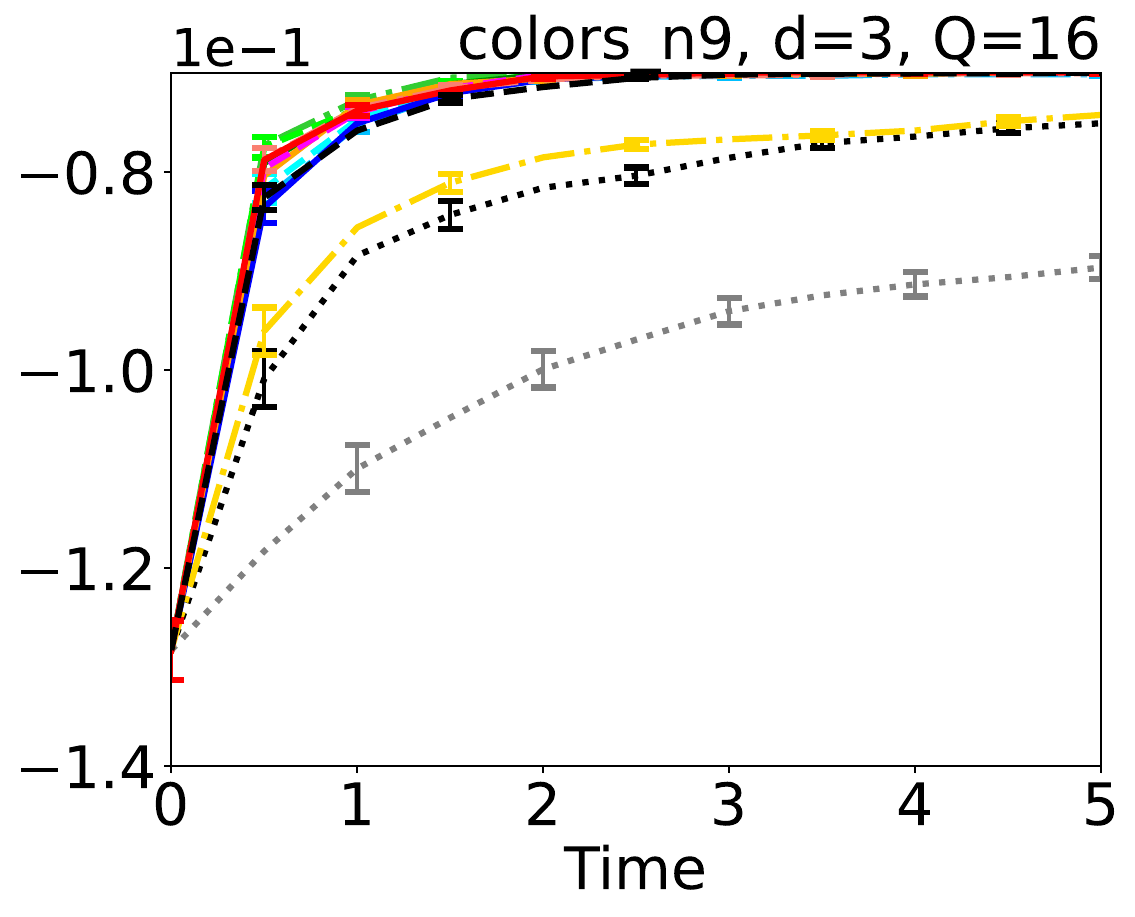}
\includegraphics[height=.25\linewidth]{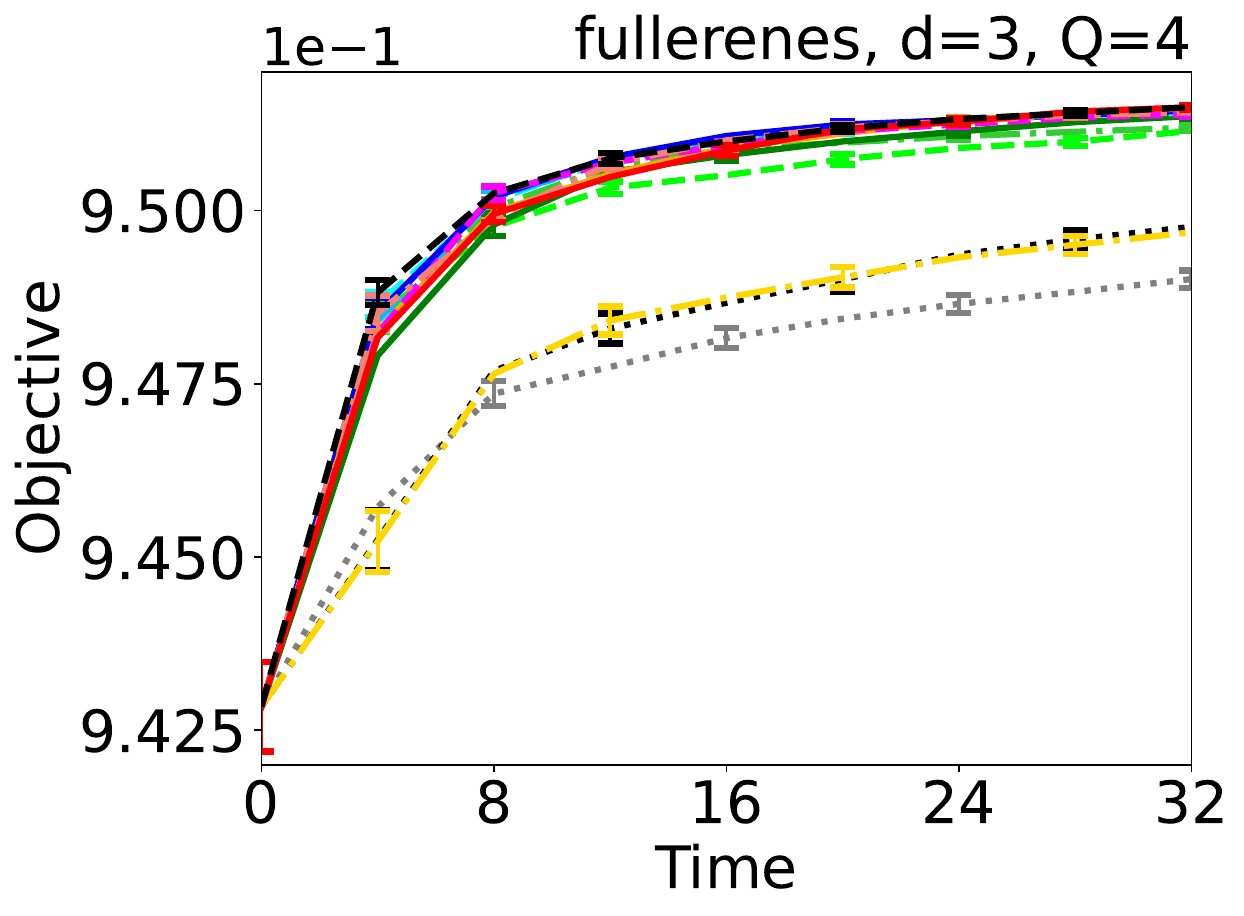}
\includegraphics[height=.25\linewidth]{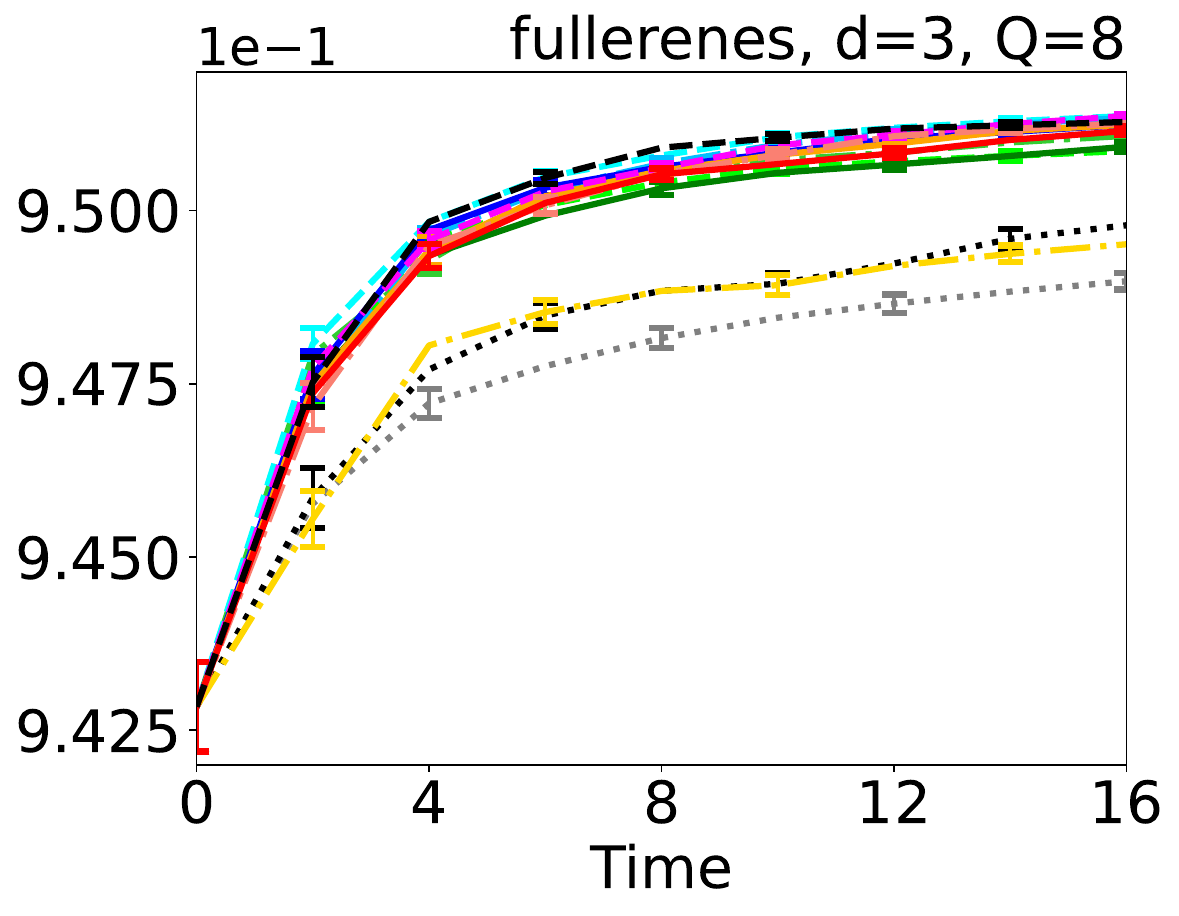}
\includegraphics[height=.25\linewidth]{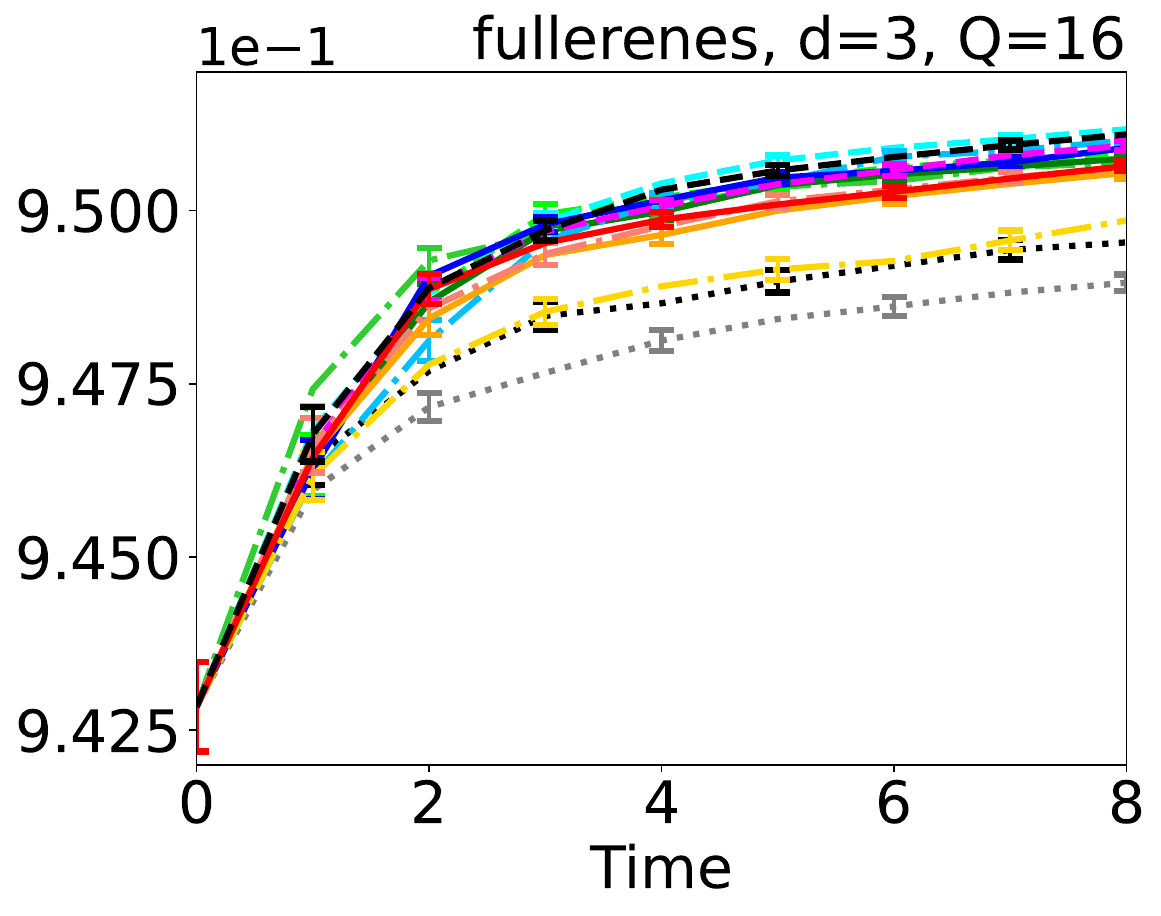}
\includegraphics[height=.25\linewidth]{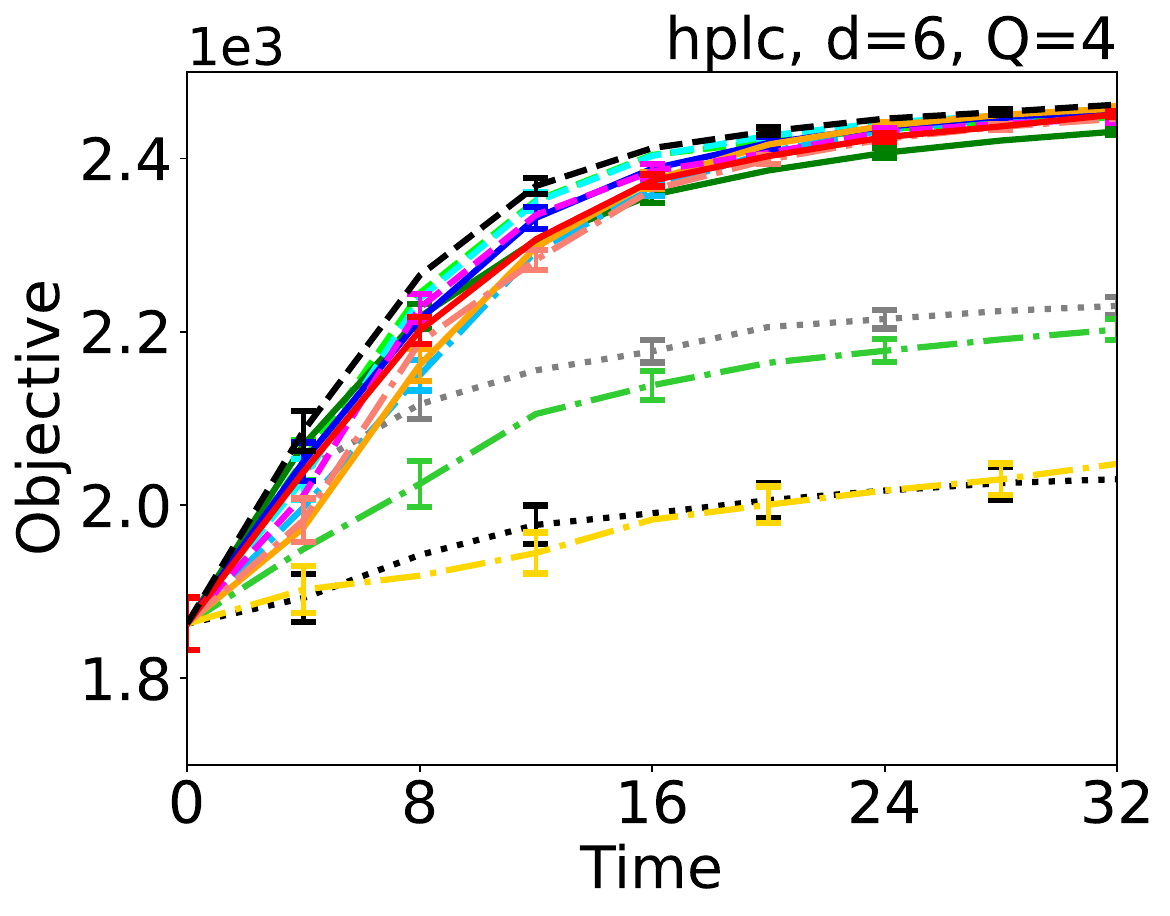}
\includegraphics[height=.25\linewidth]{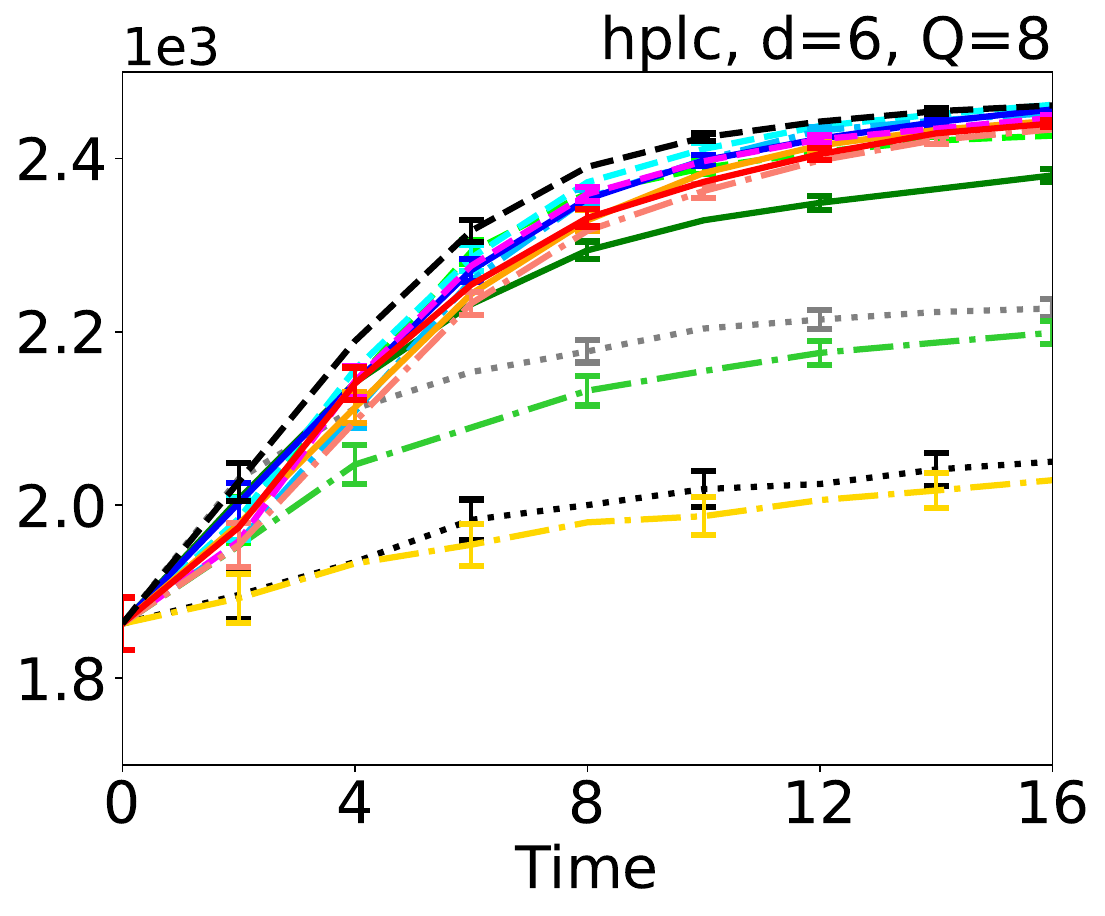}
\includegraphics[height=.25\linewidth]{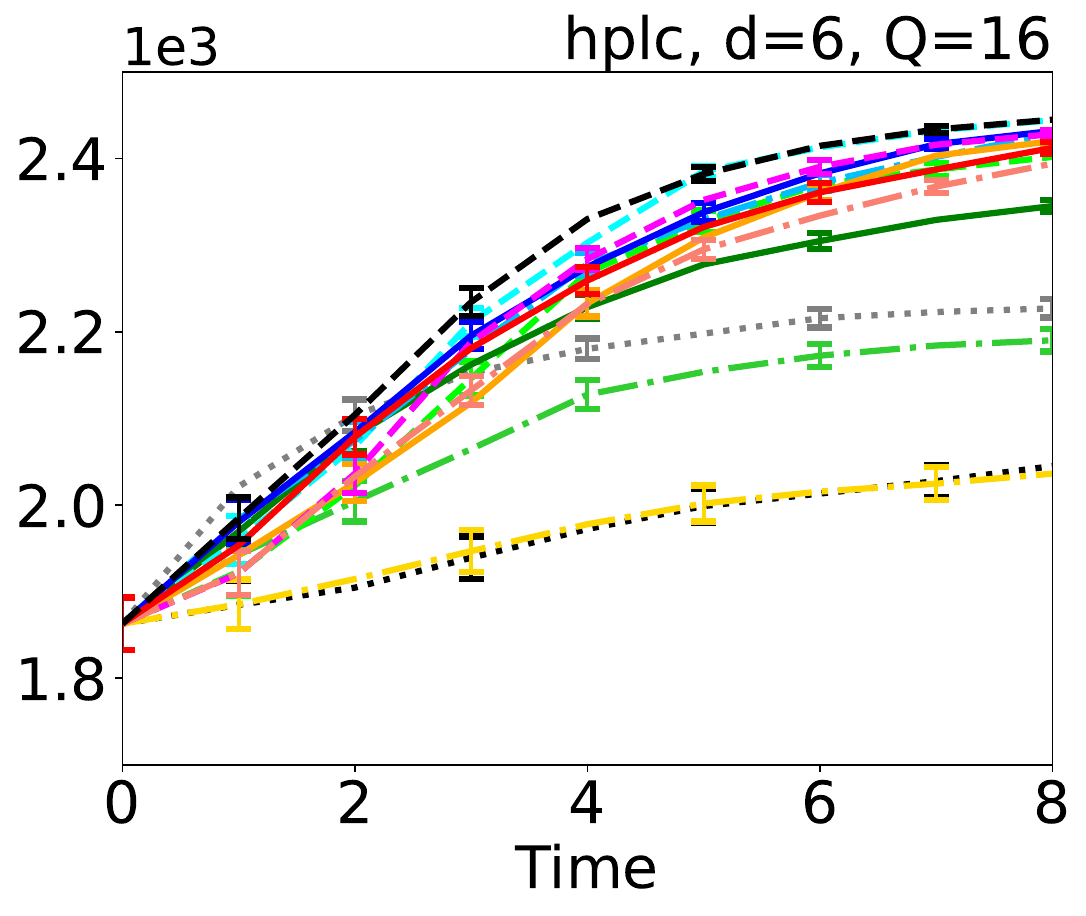}
\caption{
Result of experiments on emulators with asynchronous setting. The lines and error bars mean average and standard error of the best objective value $\max_{i\in\left[t\right]}f{\left(\bm x_i\right)}$ across the 100 experiments on each condition.
}
\label{fig:result_eml_asyn}
\end{figure}
\begin{figure}[t]
\centering
\includegraphics[width=.9\linewidth]{figures/legend.pdf}\\
\includegraphics[height=.25\linewidth]{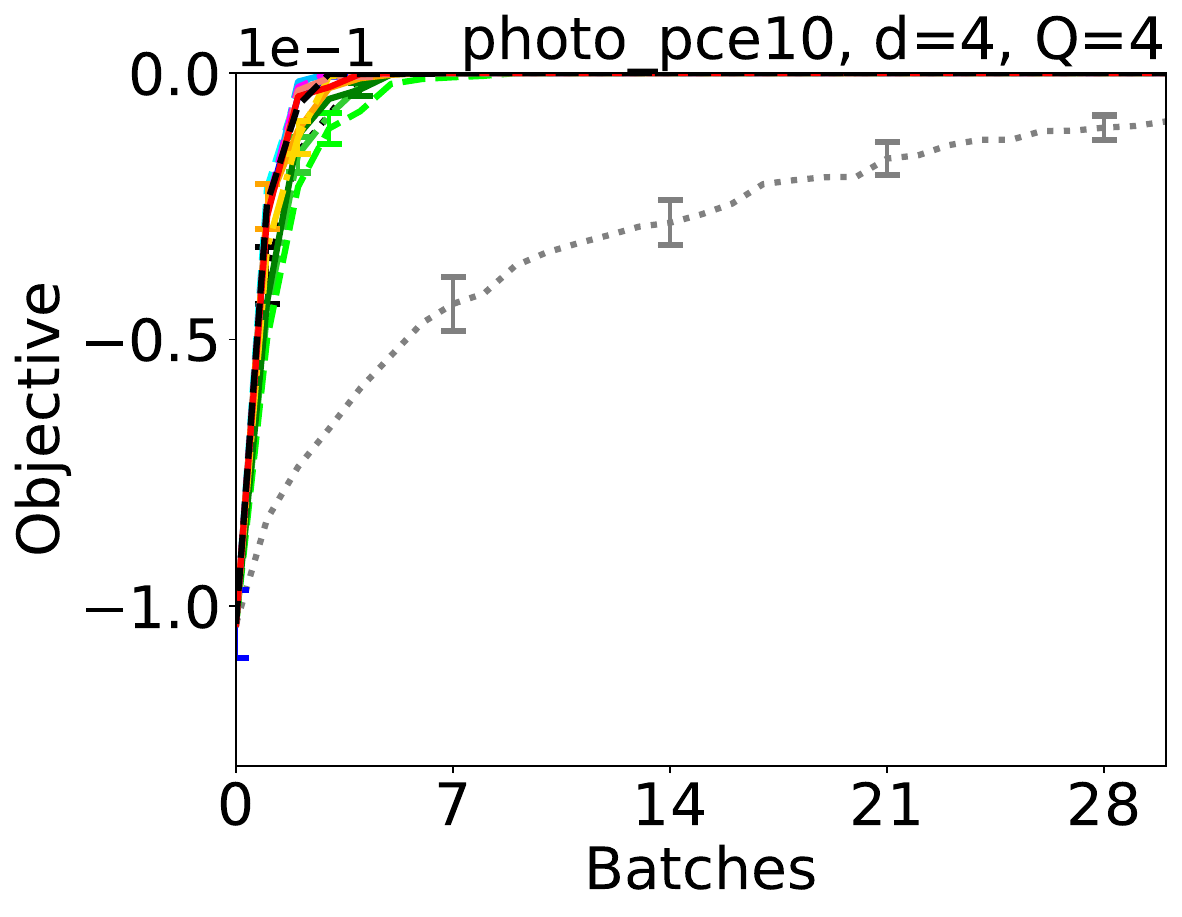}
\includegraphics[height=.25\linewidth]{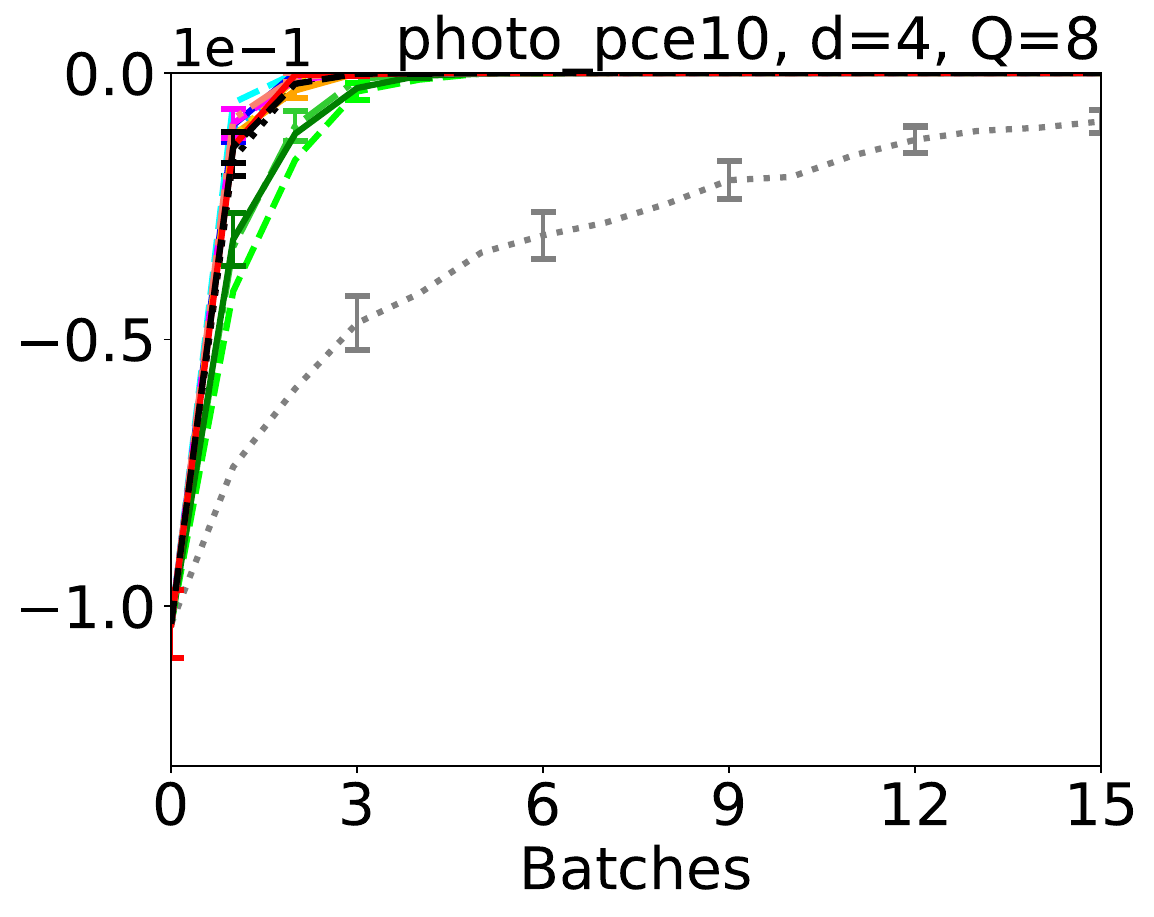}
\includegraphics[height=.25\linewidth]{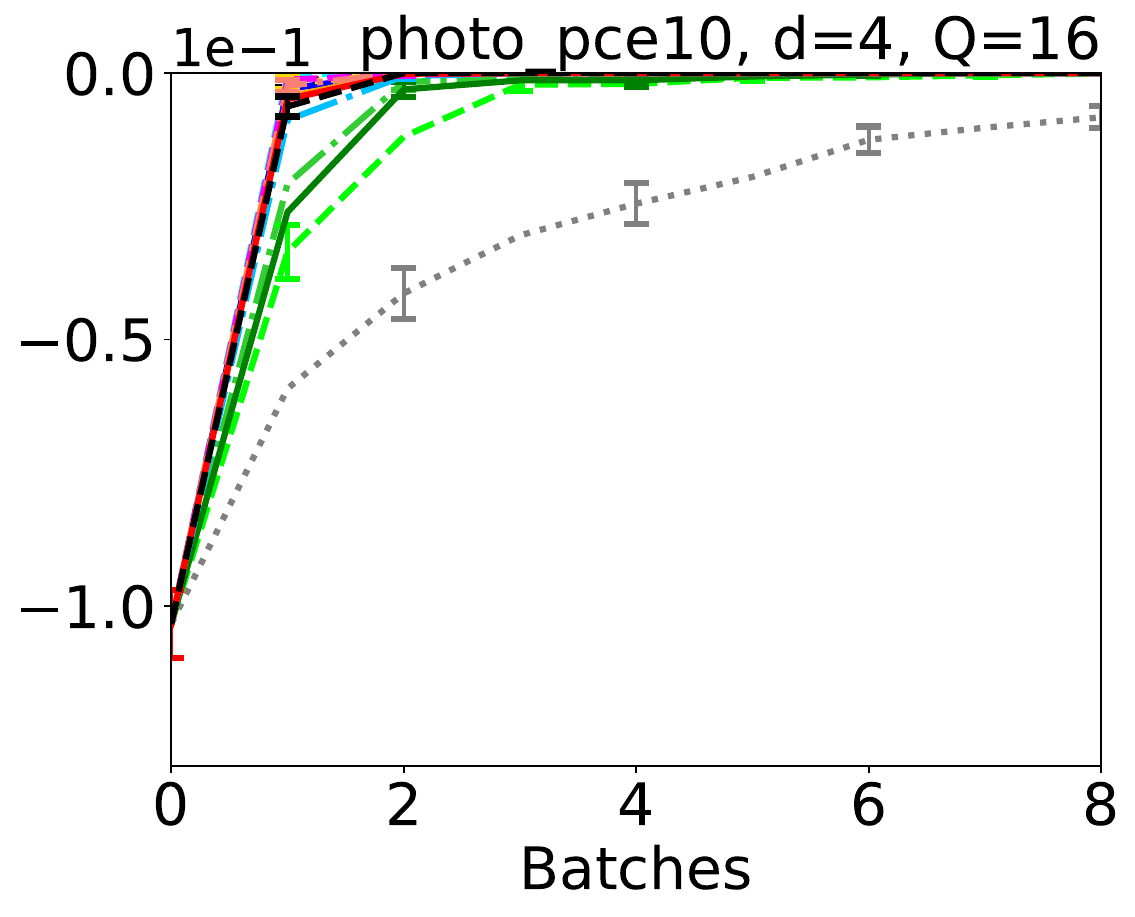}
\includegraphics[height=.25\linewidth]{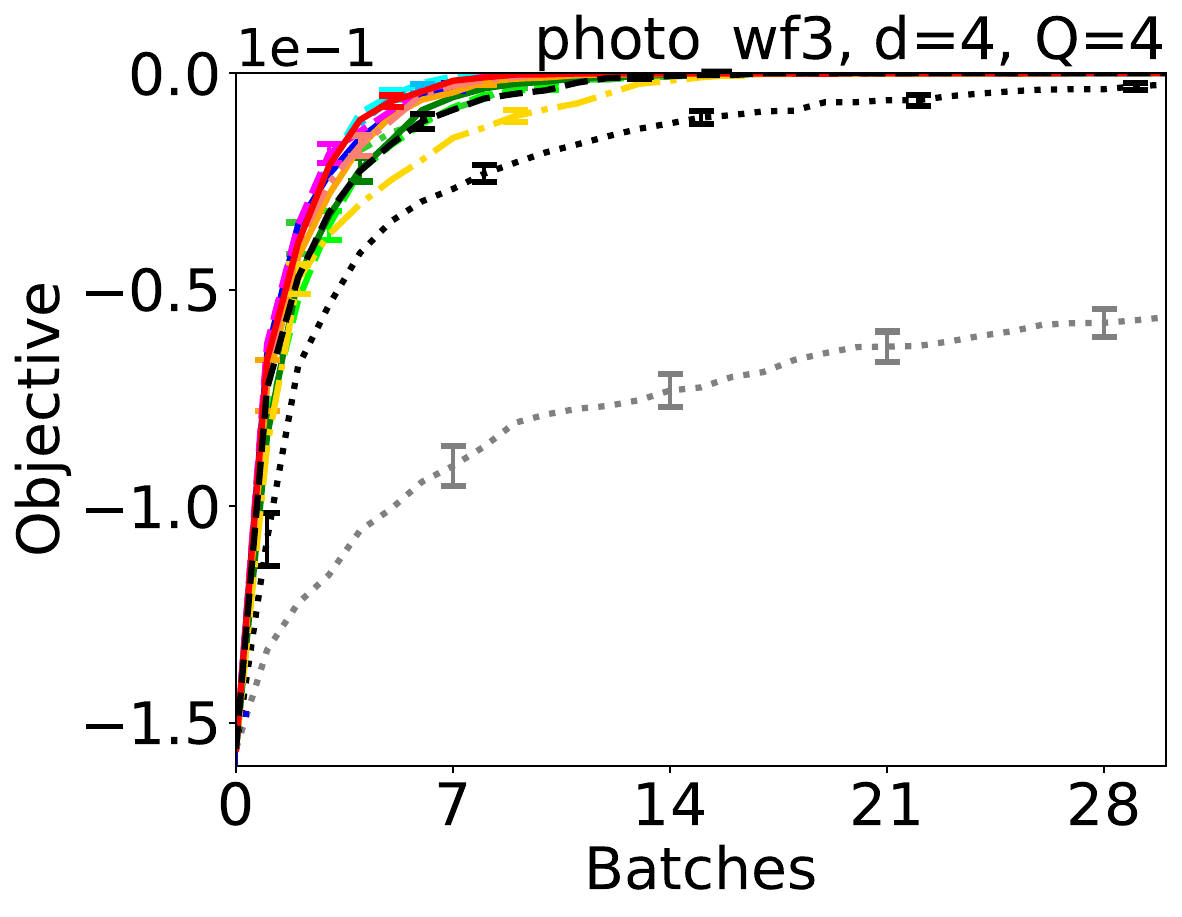}
\includegraphics[height=.25\linewidth]{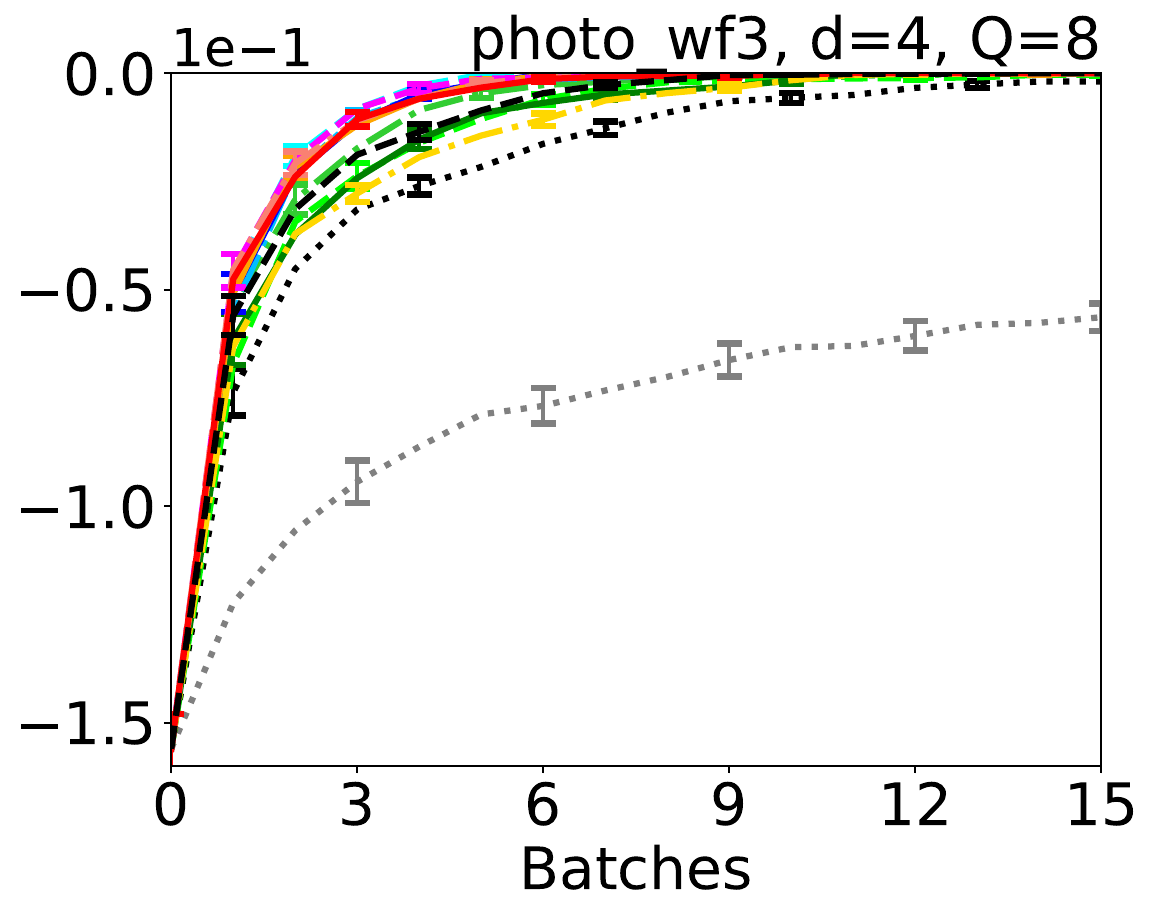}
\includegraphics[height=.25\linewidth]{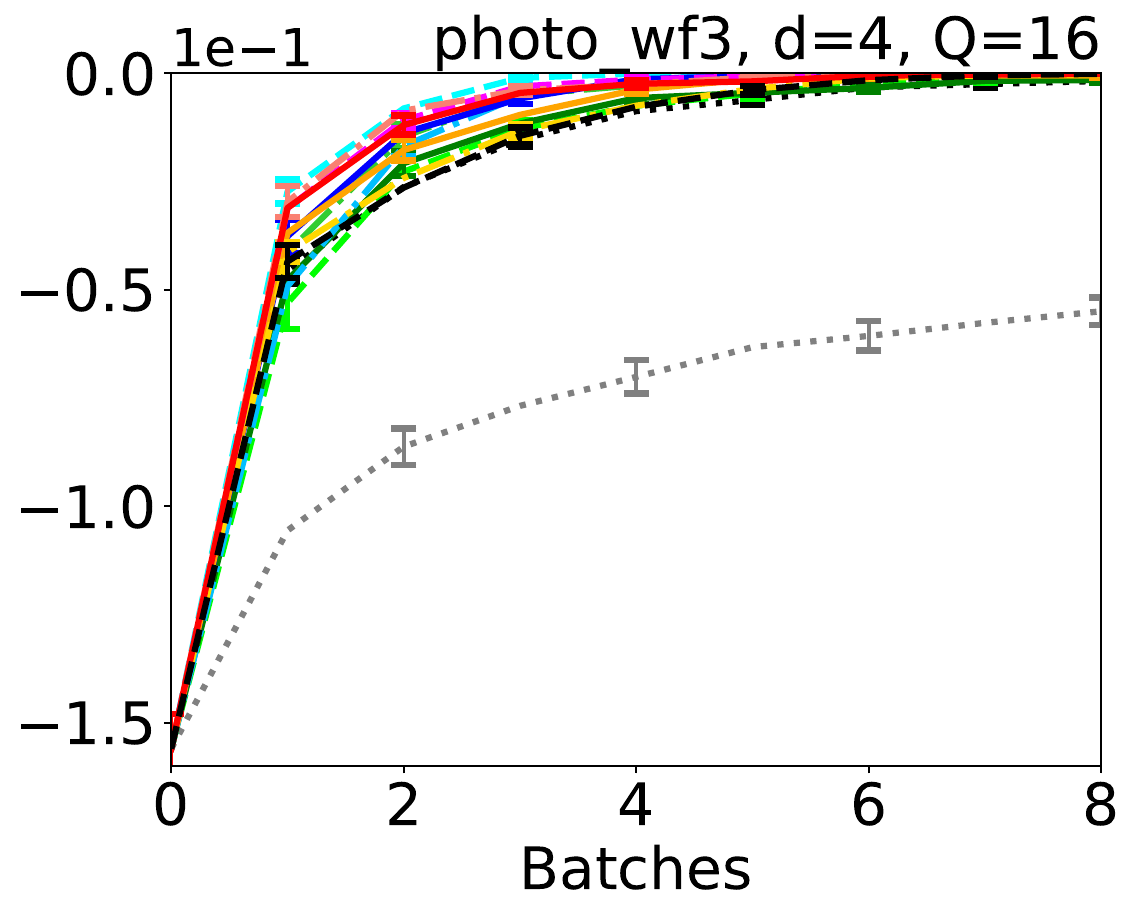}
\includegraphics[height=.25\linewidth]{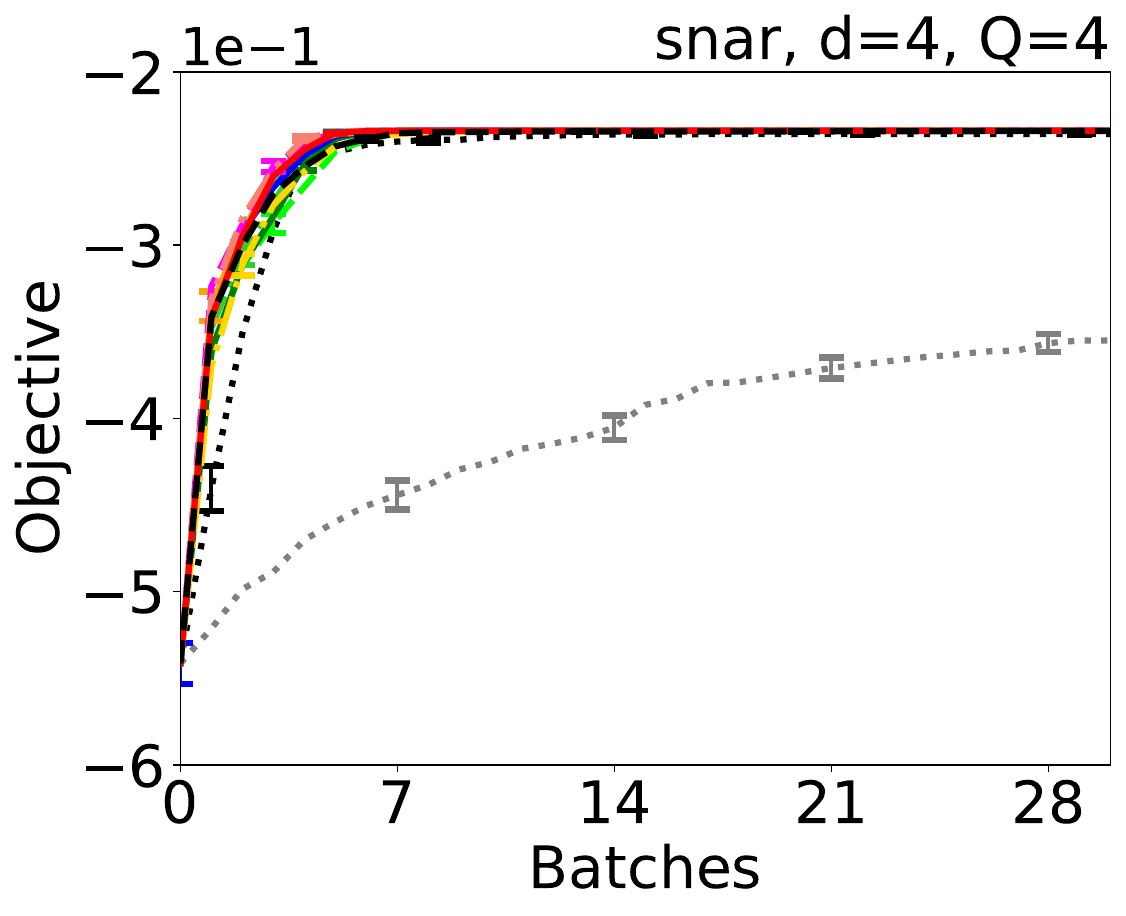}
\includegraphics[height=.25\linewidth]{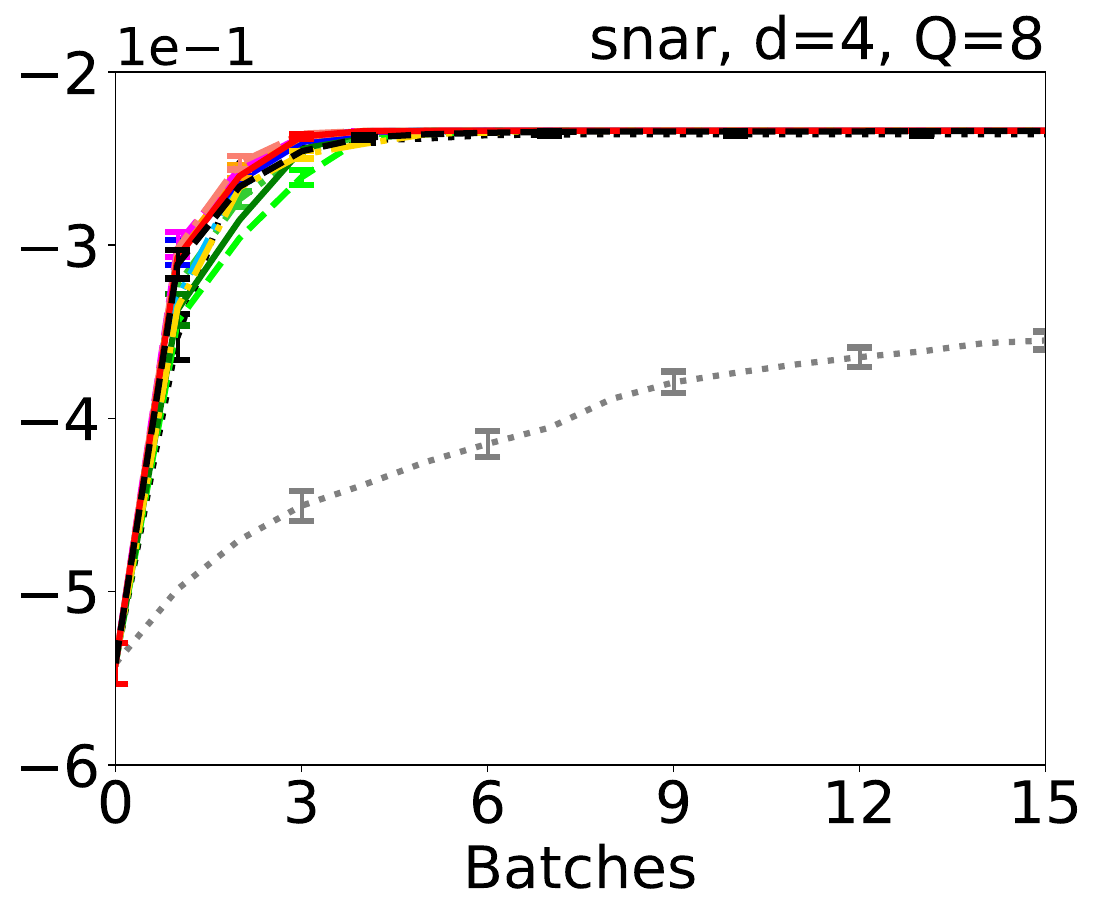}
\includegraphics[height=.25\linewidth]{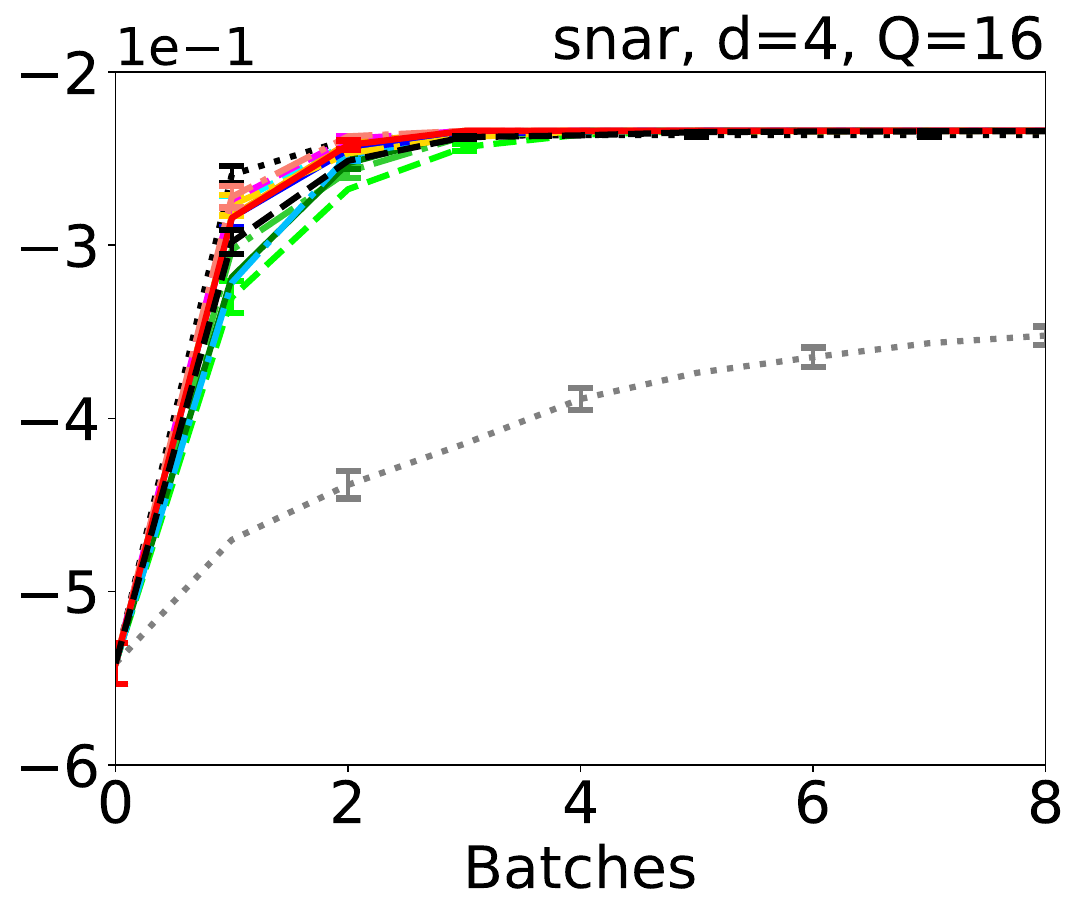}
\includegraphics[height=.25\linewidth]{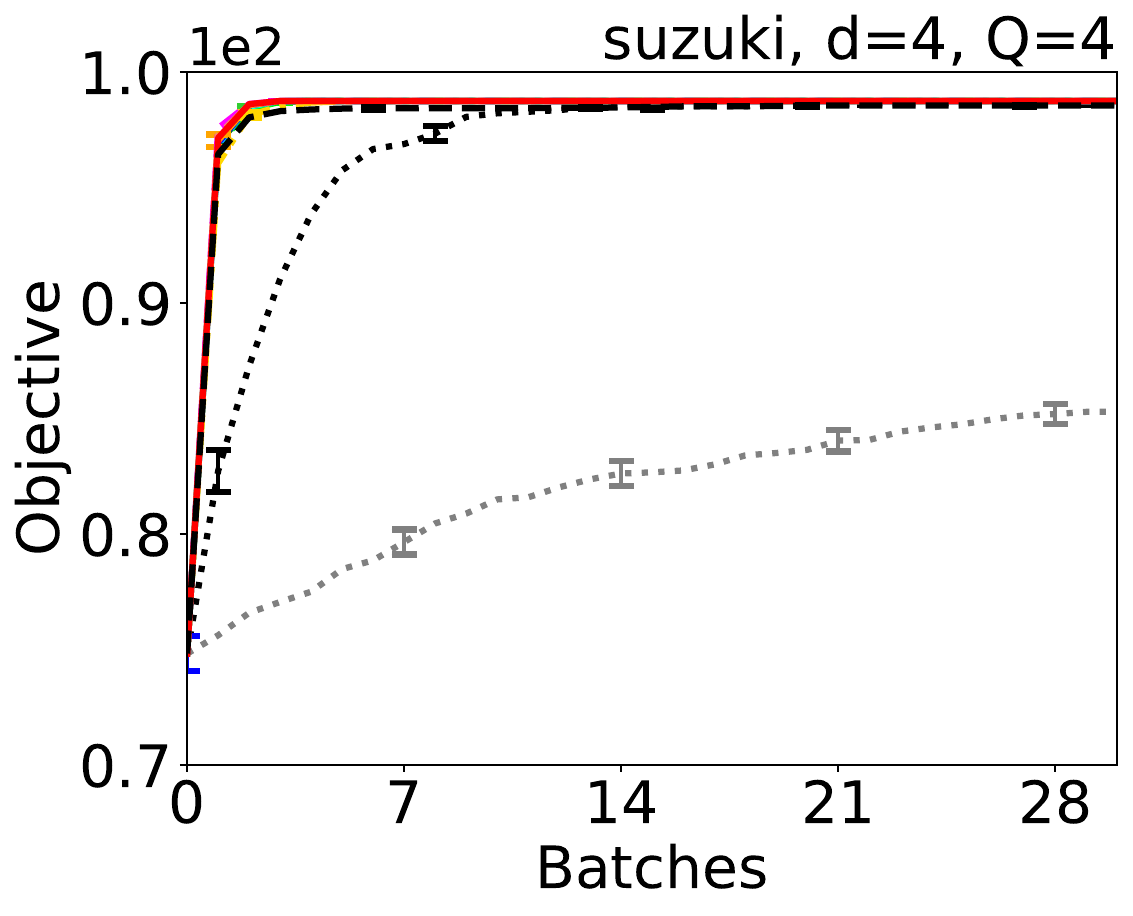}
\includegraphics[height=.25\linewidth]{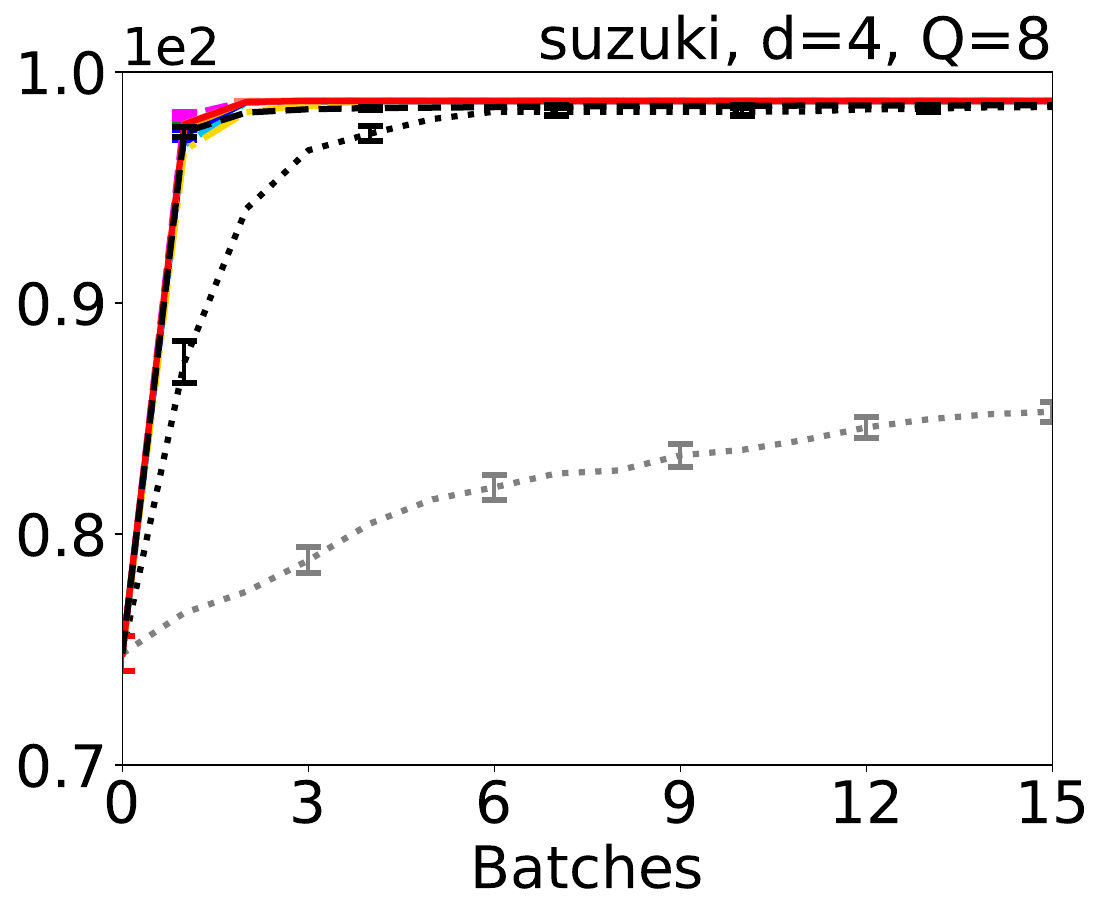}
\includegraphics[height=.25\linewidth]{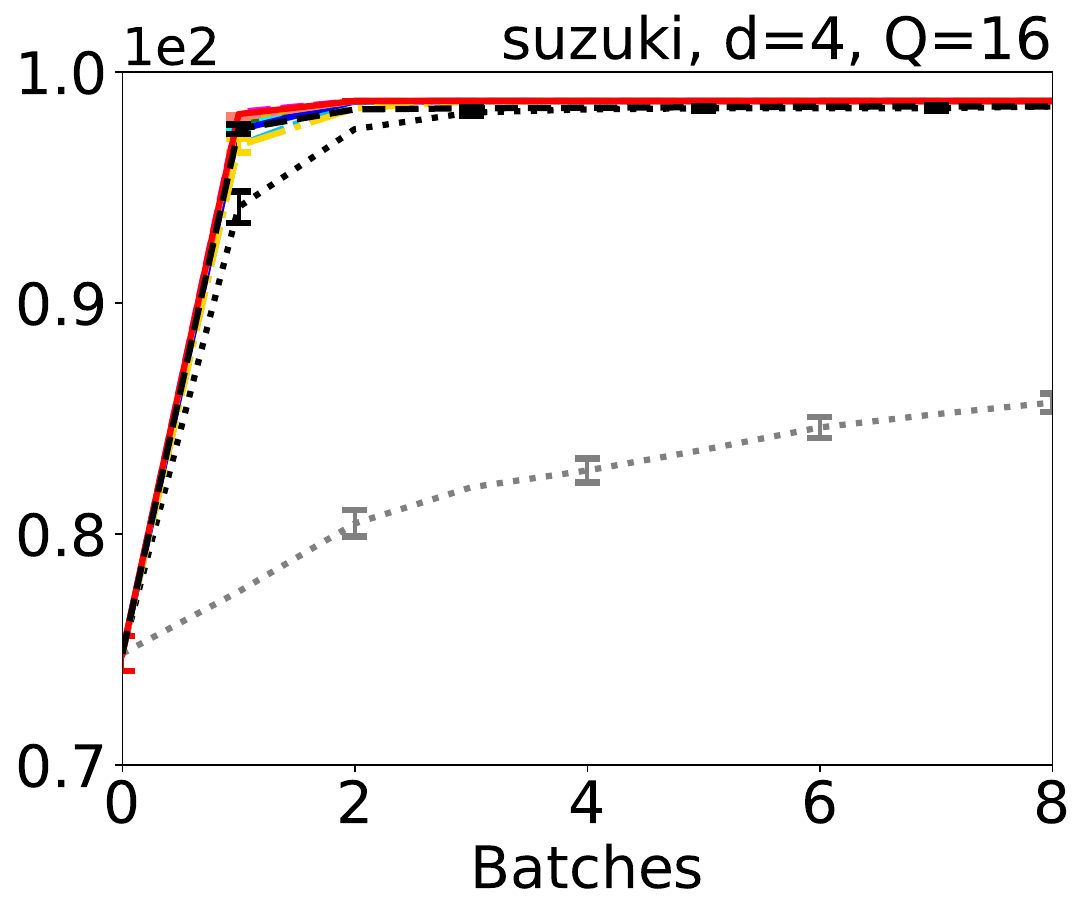}
\caption{
Result of experiments on emulators with synchronous setting. The lines and error bars mean average and standard error of the best objective value $\max_{i\in\left[t\right]}f{\left(\bm x_i\right)}$ across the 100 experiments on each condition. One batch corresponds to $Q$ iterations.
}
\label{fig:result_eml_syn2}
\end{figure}

\begin{figure}[t]
\centering
\includegraphics[width=.9\linewidth]{figures/legend.pdf}\\
\includegraphics[height=.25\linewidth]{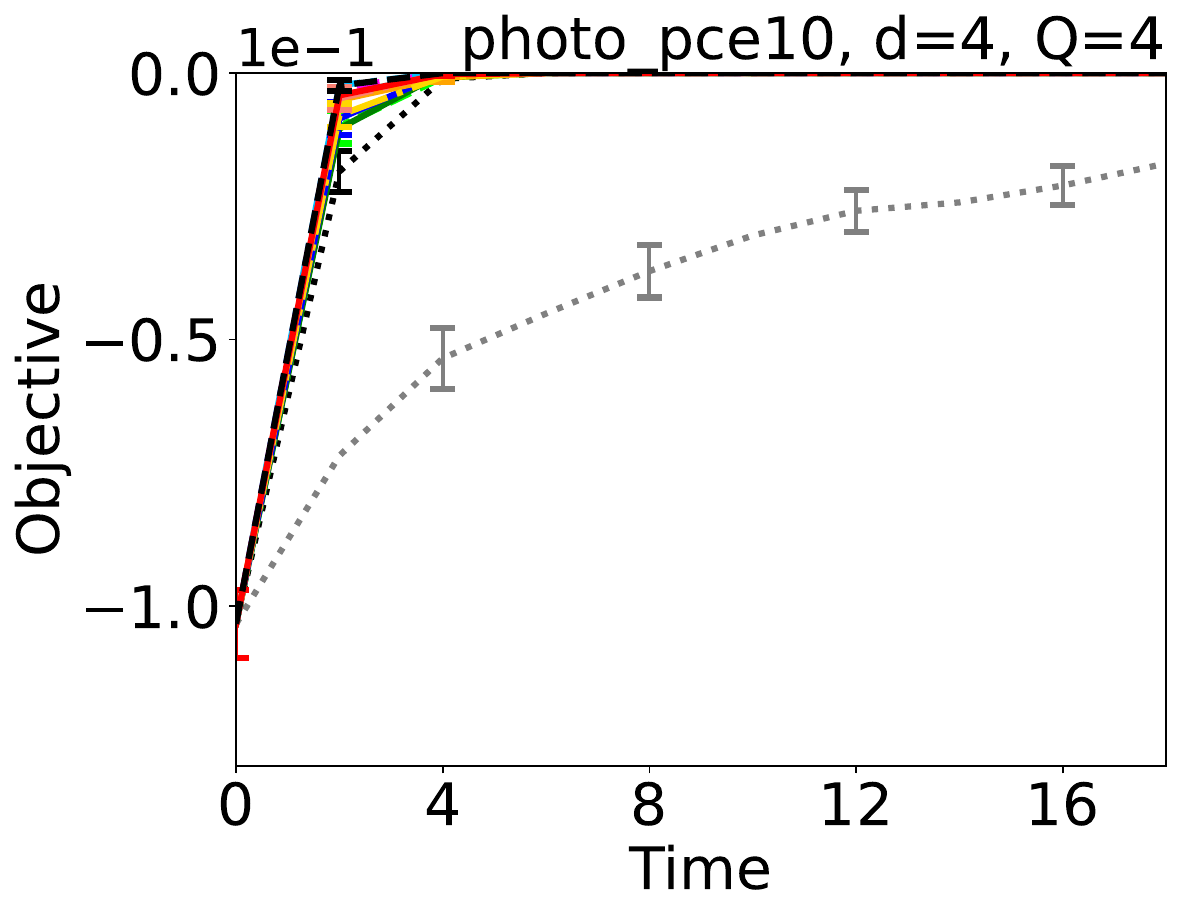}
\includegraphics[height=.25\linewidth]{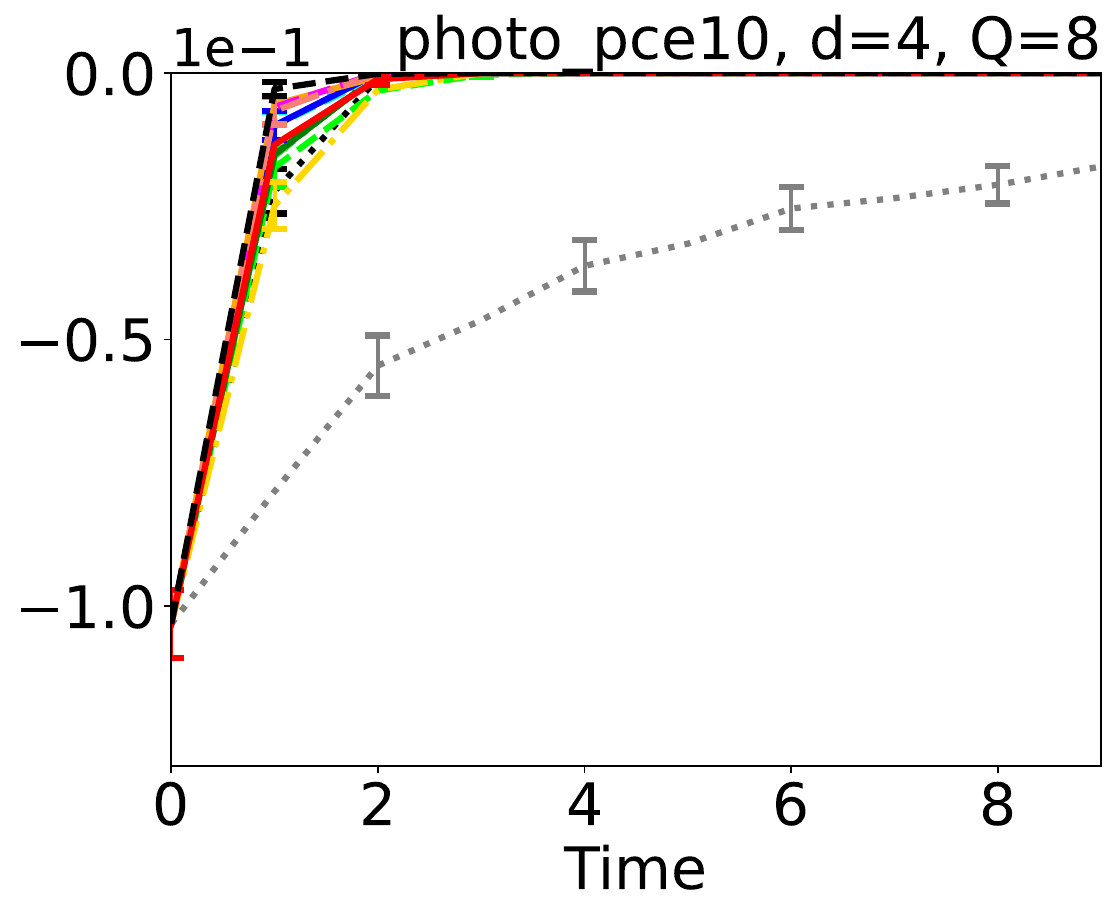}
\includegraphics[height=.25\linewidth]{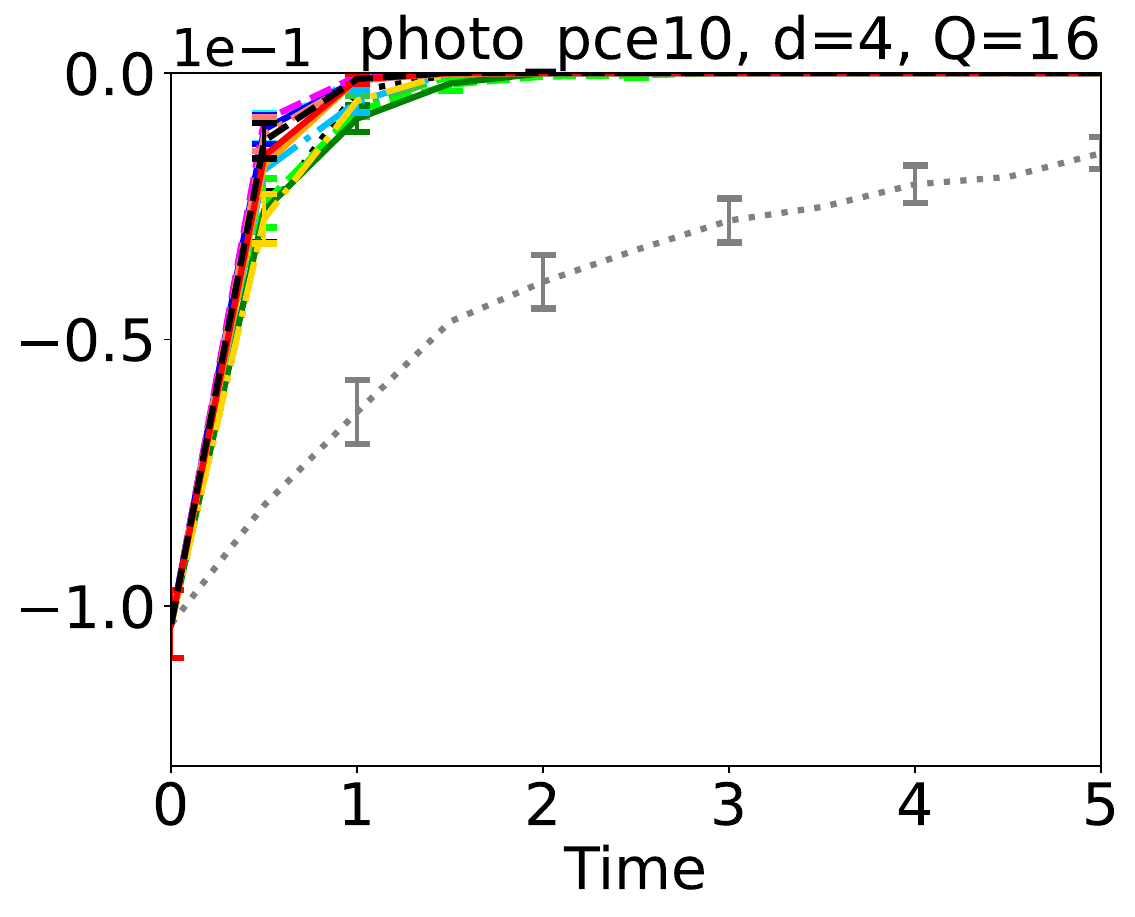}
\includegraphics[height=.25\linewidth]{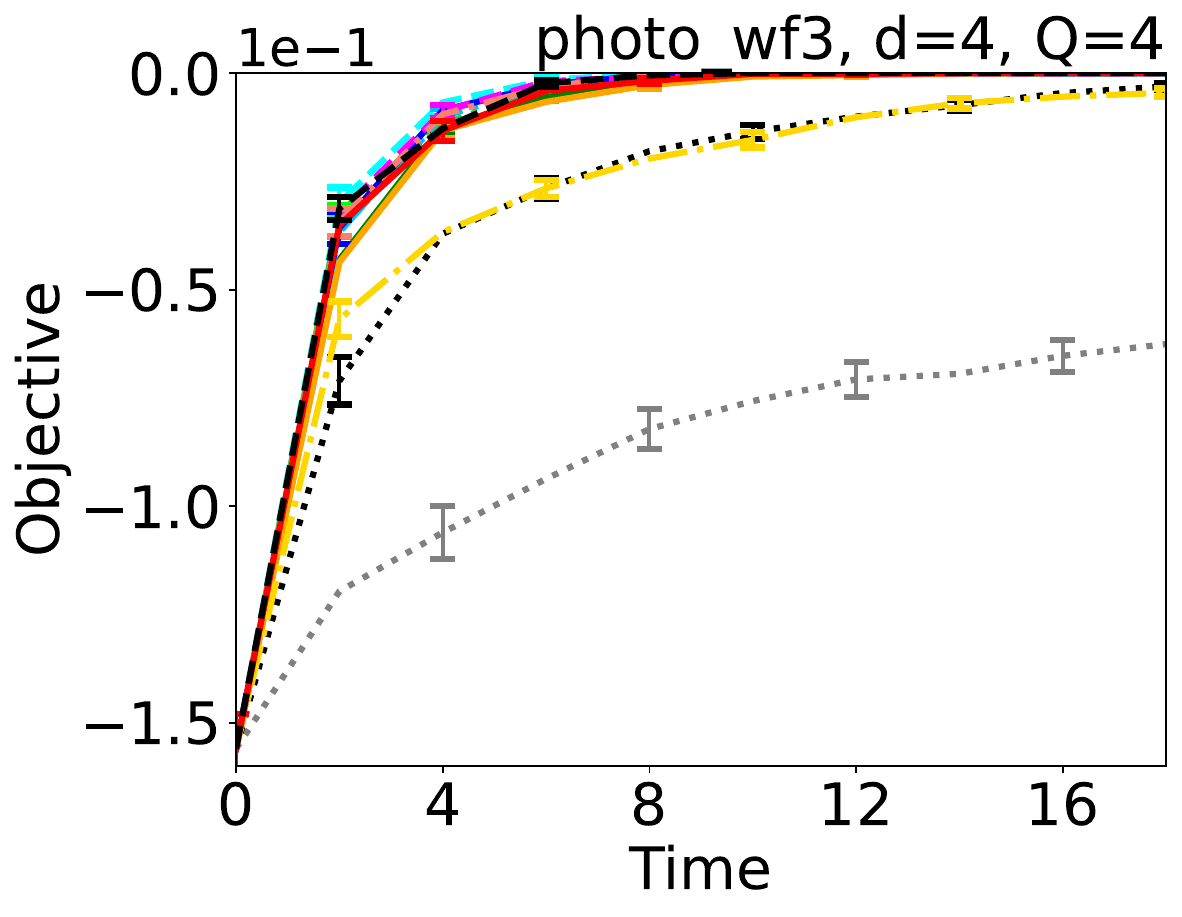}
\includegraphics[height=.25\linewidth]{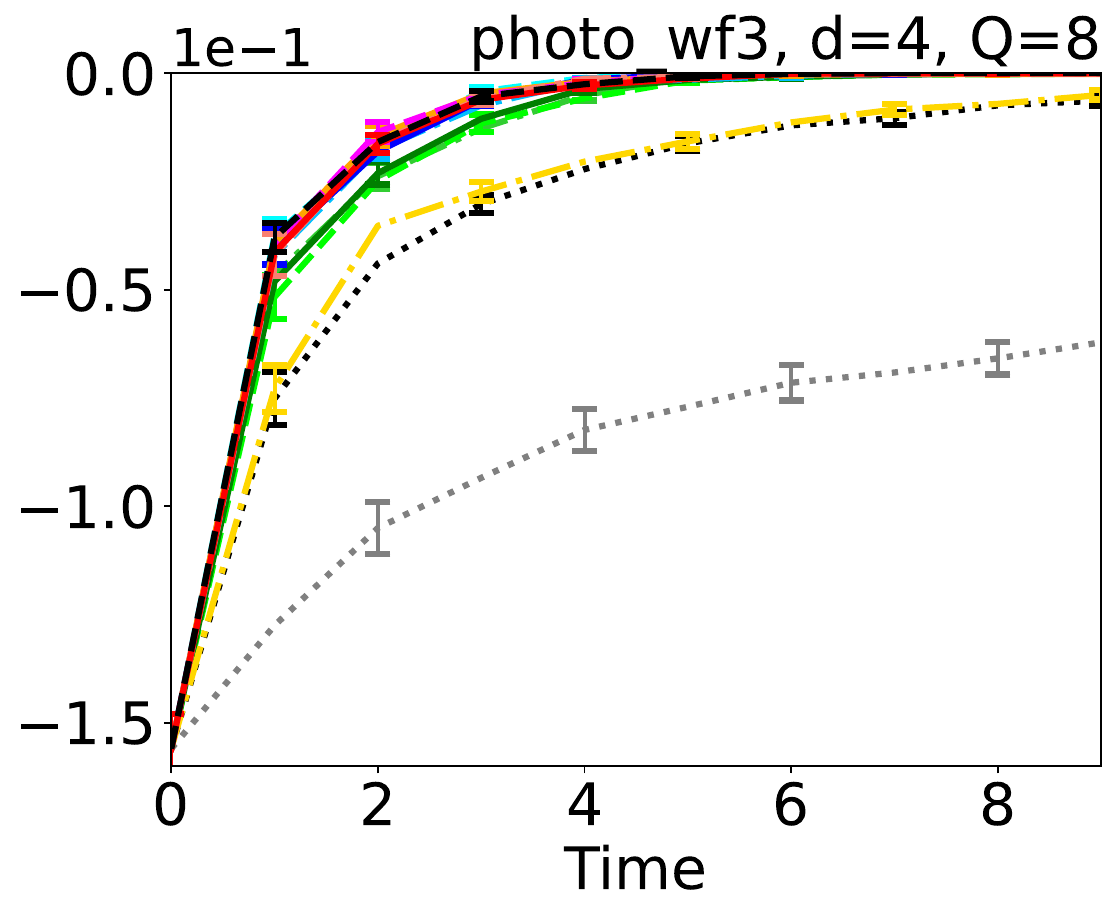}
\includegraphics[height=.25\linewidth]{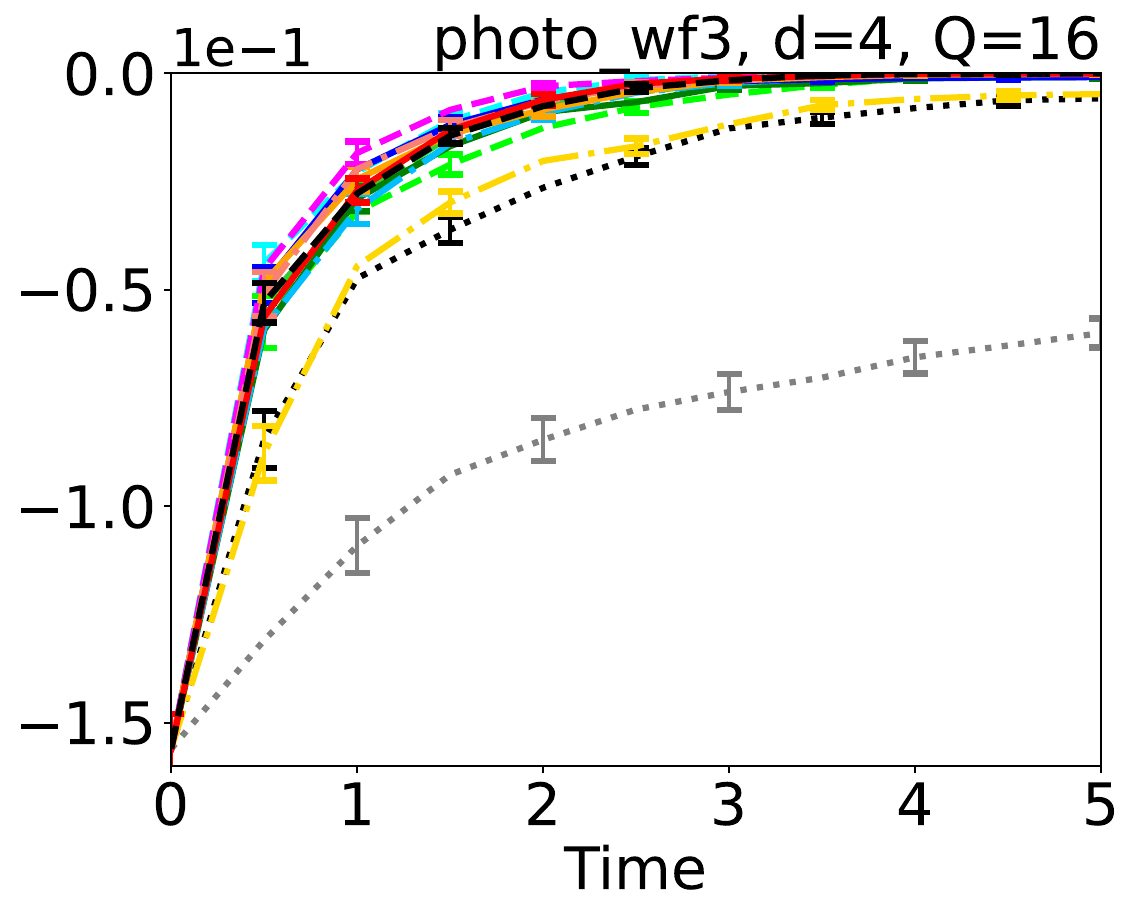}
\includegraphics[height=.25\linewidth]{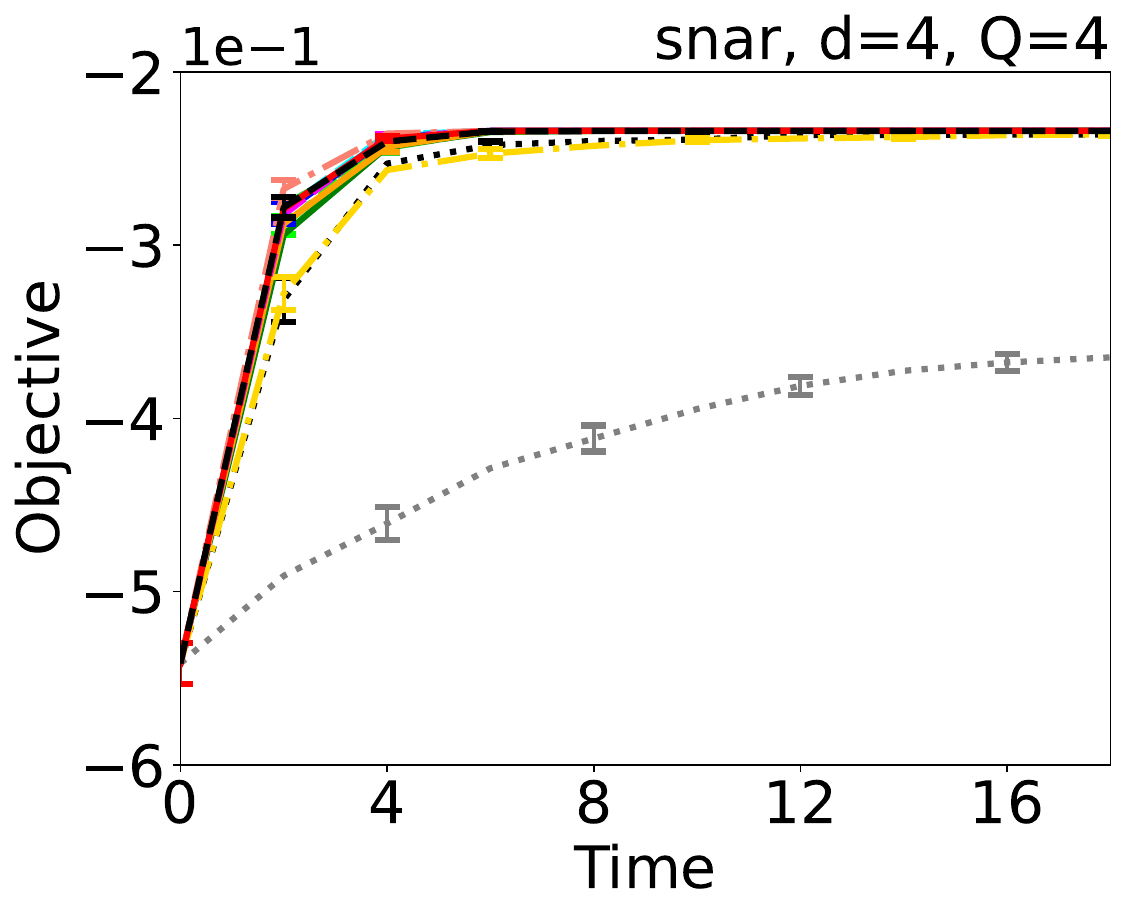}
\includegraphics[height=.25\linewidth]{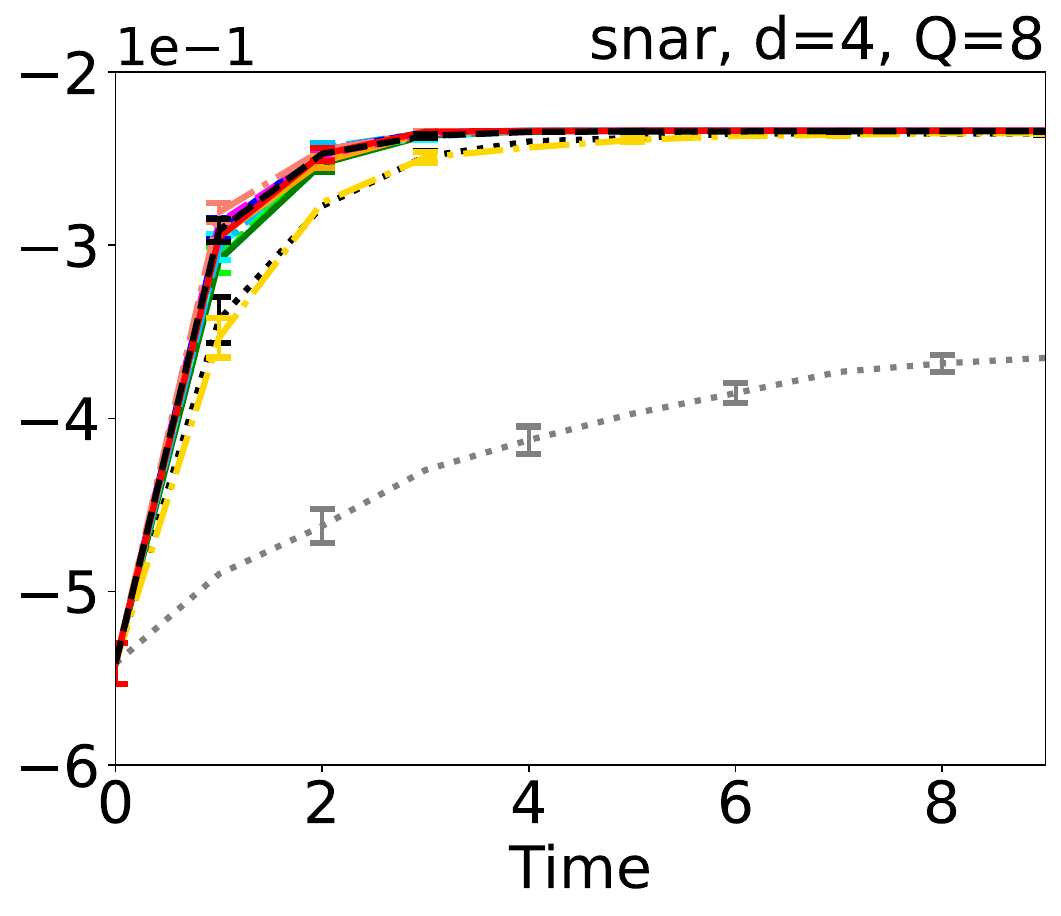}
\includegraphics[height=.25\linewidth]{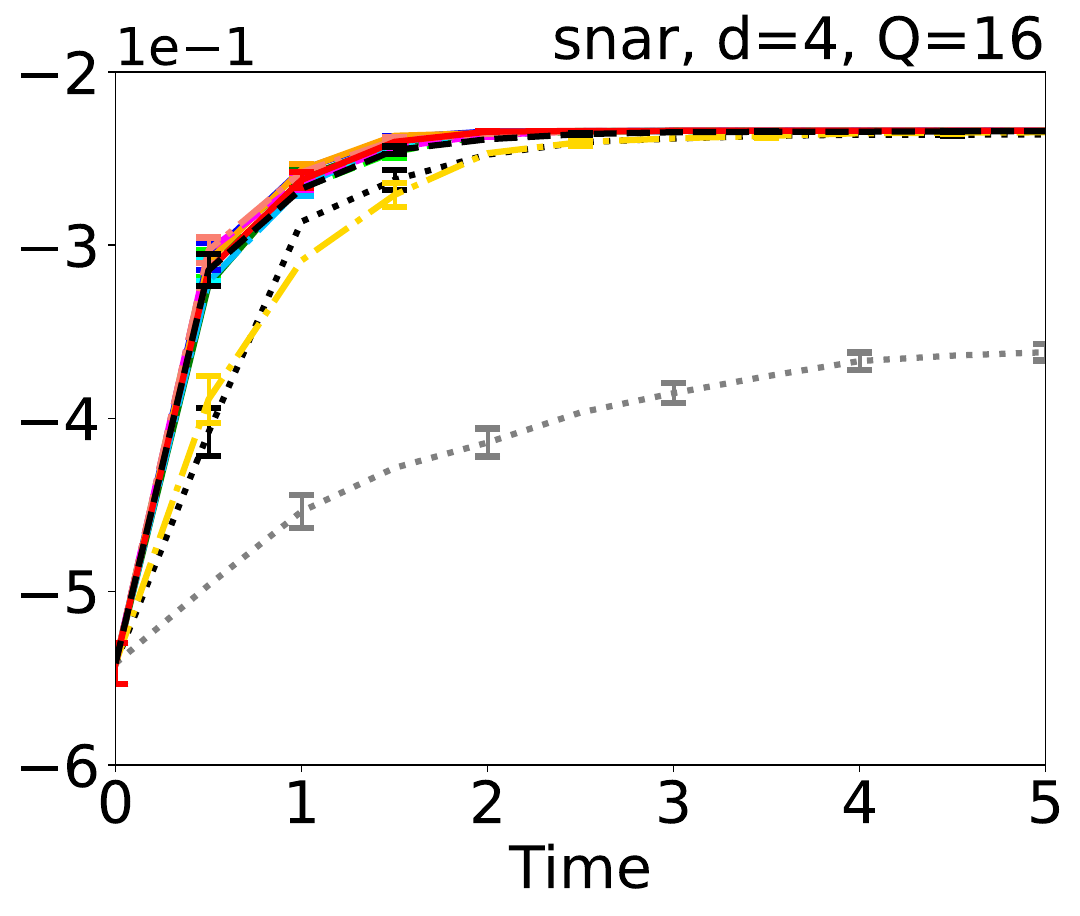}
\includegraphics[height=.25\linewidth]{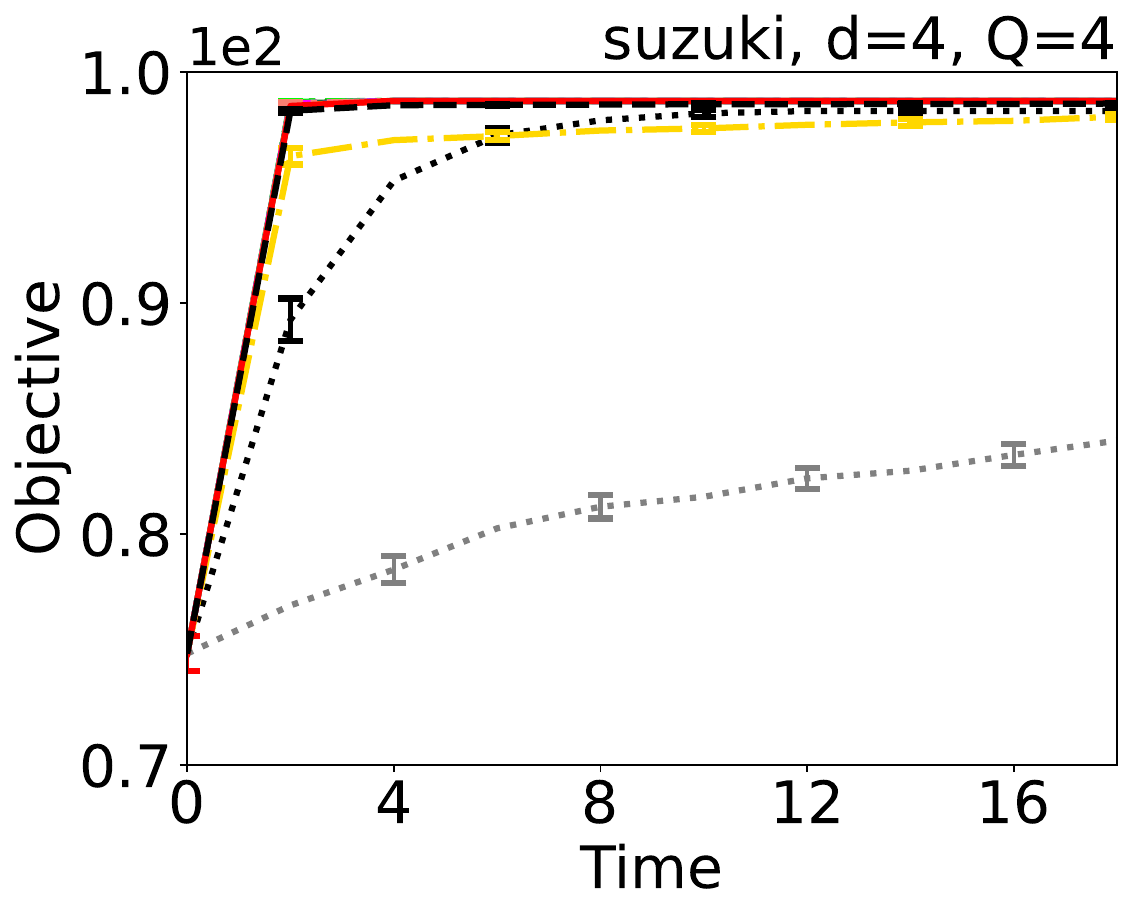}
\includegraphics[height=.25\linewidth]{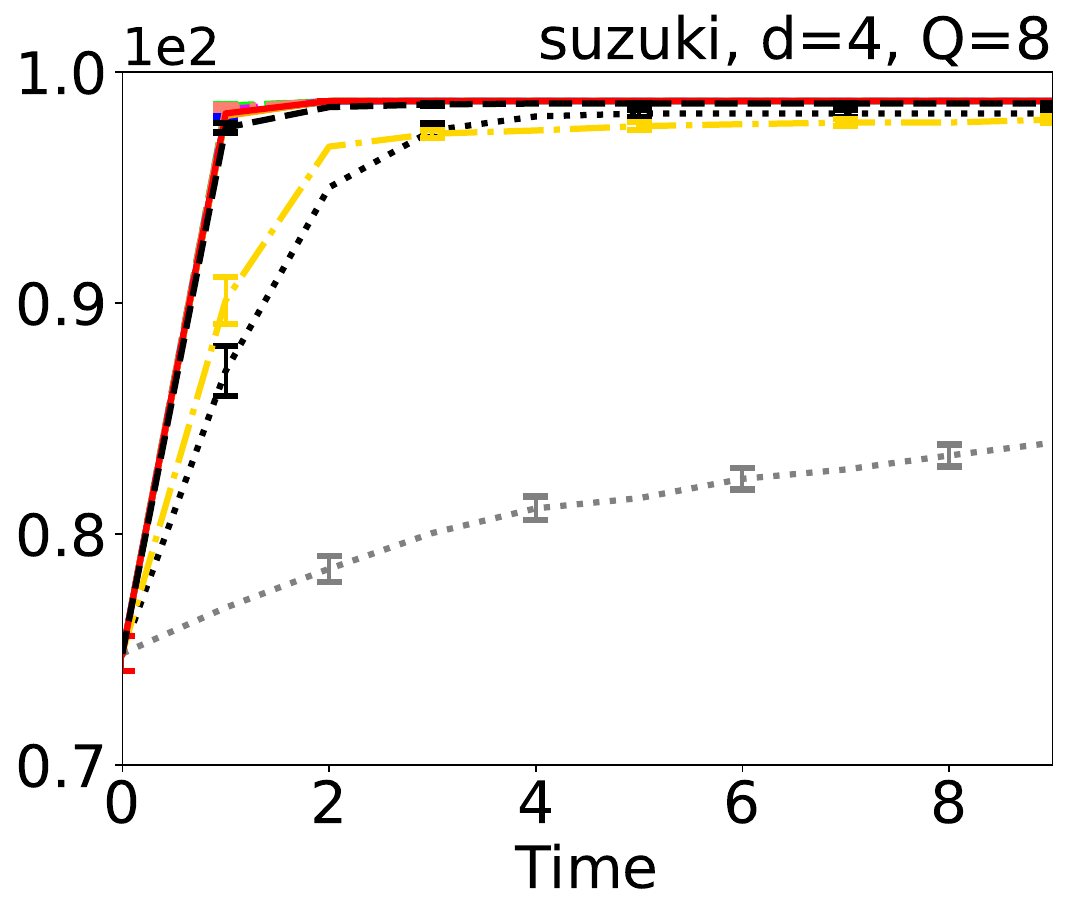}
\includegraphics[height=.25\linewidth]{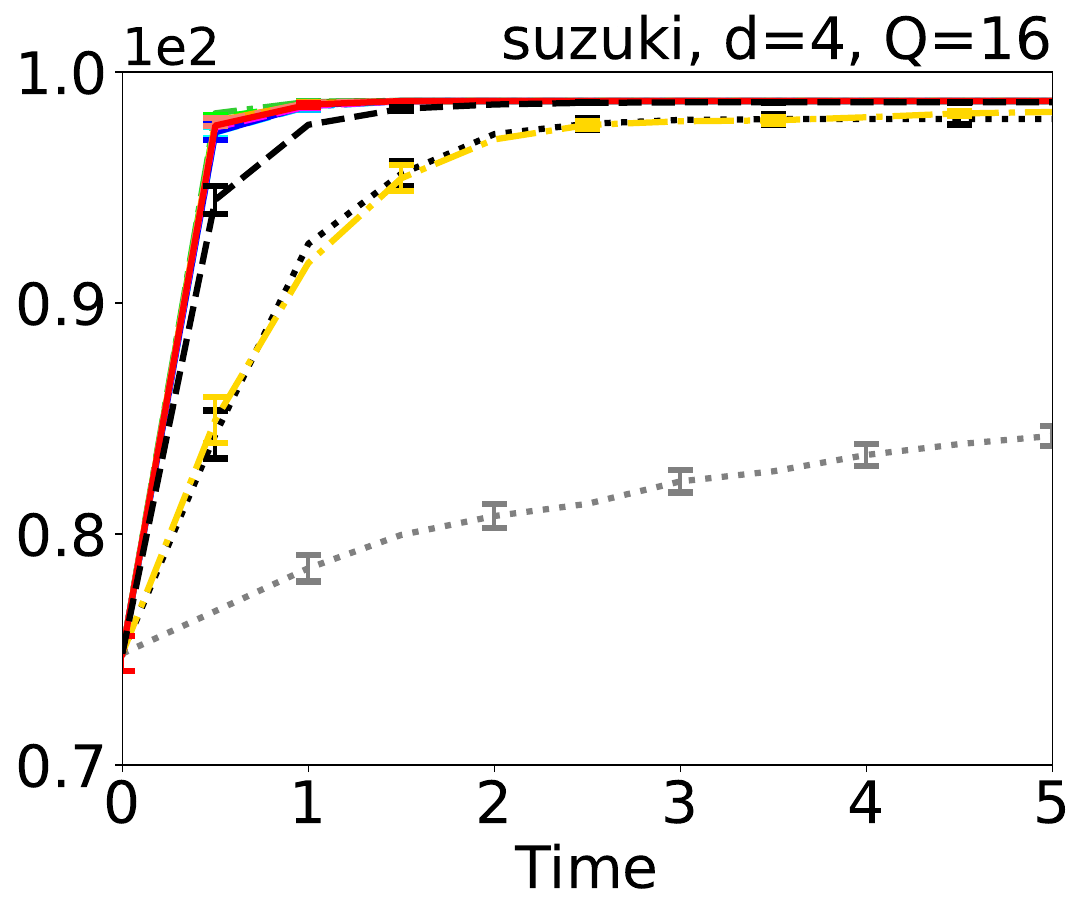}
\caption{
Result of experiments on emulators with asynchronous setting. The lines and error bars mean average and standard error of the best objective value $\max_{i\in\left[t\right]}f{\left(\bm x_i\right)}$ across the 100 experiments on each condition.
}
\label{fig:result_eml_asyn2}
\end{figure}

\begin{figure}[t]
\centering
\includegraphics[width=.7\linewidth]{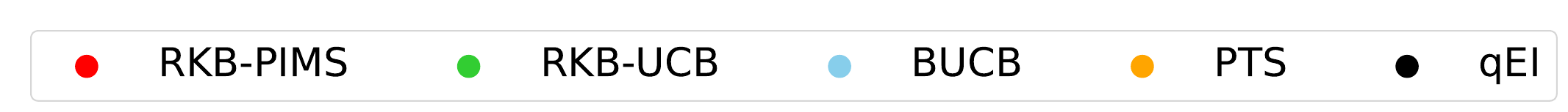}
\includegraphics[width=\linewidth]{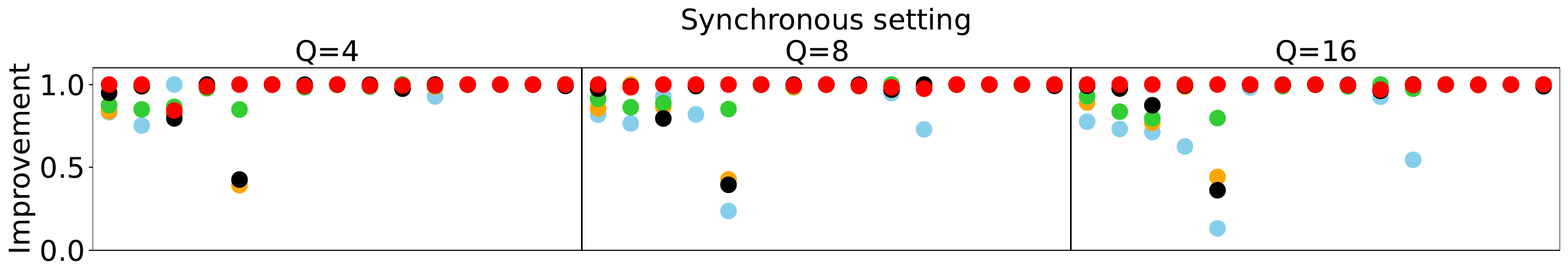}
\includegraphics[width=\linewidth]{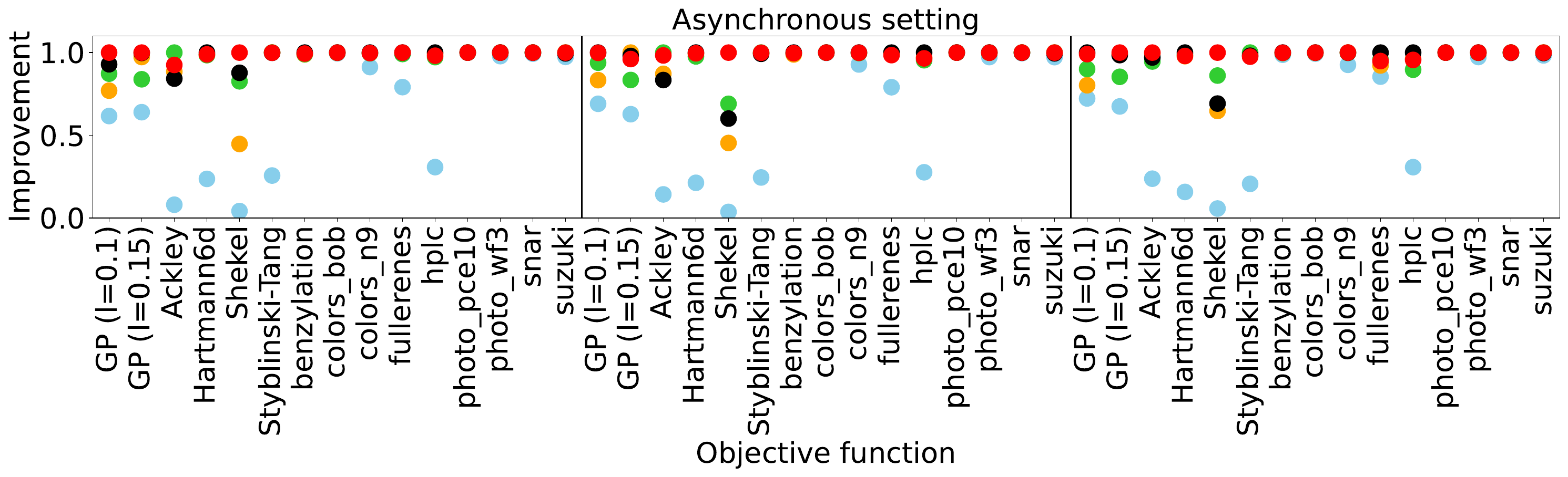}
\caption{
Mean improvement of the best objective value $\max_{i\in\left[t\right]}f{\left(\bm x_i\right)}$ across the 100 experiments on each condition.
The improvement amount has been normalized so that its maximum value among the five compared methods is one.
In asynchronous settings, the improvement is measured up to the earliest end time across the 100 experiments.
Methods qLEI and qLNEI are denoted as qEI in the legend.
}
\label{fig:result_comparison_apn}
\end{figure}